%% file: main.tex
\documentclass[pdflatex,sn-mathphys-num]{sn-jnl}

\usepackage[T1]{fontenc}
\usepackage{lmodern}
\usepackage{graphicx}
\usepackage{amsmath,amssymb,amsfonts}
\usepackage{amsthm}
\usepackage{algorithmic}
\usepackage{textcomp}
\usepackage{upquote}
\usepackage{xcolor}
\usepackage{wasysym}
\usepackage{booktabs}
\usepackage{enumitem}
\usepackage{listings}
\usepackage{placeins}

\hypersetup{bookmarksdepth=3}

\renewcommand{\orcidlogo}{\includegraphics[height=1em]{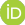}}

\newcommand{\resec}[1]{Section~\ref{sec:#1}}
\newcommand{\reappsec}[1]{Appendix~\ref{sec:#1}}
\newcommand{\refig}[1]{Fig.~\ref{fig:#1}}
\newcommand{\reqn}[1]{(\ref{eqn:#1})}
\newcommand{\retab}[1]{Table~\ref{tbl:#1}}
\newcommand{\relist}[1]{Listing~\ref{lst:#1}}

\newcommand{\seedetail}[1]{\reappsec{#1}}

\newif\ifIEEEtranloaded
\makeatletter
\@ifclassloaded{IEEEtran}{\IEEEtranloadedtrue}{\IEEEtranloadedfalse}
\makeatother

\makeatletter
\newcommand{\iftwocolumn}[2]{\if@twocolumn #1\else #2\fi}
\makeatother

\theoremstyle{thmstyleone}
\newtheorem{theorem}{Theorem}[section]
\newtheorem{lemma}[theorem]{Lemma}
\newtheorem{corollary}[theorem]{Corollary}

\theoremstyle{thmstyletwo}
\newtheorem{remark}{Remark}

\theoremstyle{thmstylethree}
\newtheorem{definition}{Definition}[section]

\newcommand{\True}{\texttt{True}}
\newcommand{\False}{\texttt{False}}

\DeclareUnicodeCharacter{02BC}{'}

\begin{document}

\newif\ifreview    
\reviewfalse   

\makeatletter
\def\email#1{\global\advance\emailcnt by 1\relax%
\if@corauemail%
   \g@addto@macro\corrauthemail{%
   \setcounter{footnote}{0}%
   \textcolor{blue}{#1}\ %
   }%
\else%
   \g@addto@macro\authemail{%
   \setcounter{footnote}{0}%
   \textcolor{blue}{#1}\ %
   }%
\fi}
\makeatother

\title[Encoding, Liveness, and Auditing of Synthesized Robot Supervisors]{Realizability Is Not Enough: Encoding, Liveness, and Auditing of Synthesized Robot Supervisors}

\ifreview

\author*{\fnm{Author information} \sur{redacted}\textsuperscript{*}\!}\email{redacted@example.com}

\affil{\orgname{Redacted for review}}

\else

\author*{\fnm{David C.} \sur{Conner}\textsuperscript{*}\!\,\orcid{https://orcid.org/0000-0002-2037-2560}}\email{david.conner@cnu.edu}
\author{\fnm{Joshua} \sur{Luzier}}
\author{\fnm{William J.} \sur{Doyle}}
\author{\fnm{Emma R.} \sur{Faith}}
\author{\fnm{Aubrie B.} \sur{Kooiker}}
\author{\fnm{Andrew J.} \sur{Farney}}
\author{\fnm{Sebastian} \sur{Fox}}
\author{\fnm{Evangelina} \sur{Grimes}}
\author{\fnm{Ian G.} \sur{Conner}}
\author{\fnm{Kyle} \sur{Bloom}}

\affil{\orgdiv{Capable Humanitarian Robotics and Intelligent Systems Lab (CHRISLab)\\}
  \orgname{School of Engineering and Computing, Christopher Newport University\\}
  \orgaddress{\city{Newport News}, \state{Virginia}, \postcode{23606}, \country{USA}}}

\fi

\ifreview
\hypersetup{
    pdftitle={Realizability Is Not Enough: Encoding, Liveness, and Auditing of Synthesized Robot Supervisors},
    pdfauthor={Author information redacted for review}
}
\else
\hypersetup{
    pdftitle={Realizability Is Not Enough: Encoding, Liveness, and Auditing of Synthesized Robot Supervisors},
    pdfauthor={David C. Conner, Joshua Luzier, William J. Doyle, Emma R. Faith, Aubrie B. Kooiker, Andrew J. Farney, Sebastian Fox, Evangelina Grimes, Ian G. Conner, and Kyle Bloom}
}
\fi

\abstract{\input{abstract}}


\maketitle

\input{Introduction}

\input{Background}

\input{PipelineOverview}

\input{Demonstrations}

\input{Pipeline}

\input{Auditor}

\input{Reduction}

\input{ControllerRealization}

\input{Discussion}

\input{Conclusion}

\input{Appendices}

\backmatter

\section*{Generative AI use and author responsibility}

\input{Acknowledgment}

\section*{Declarations}

\textbf{Funding}
This work was primarily supported by United States Navy contract
No. N00174-23-1-0018. Additional support was provided by the
Christopher Newport University School of Engineering and Computing
and the CNU Office of Research and Creative Activities. The views expressed
are those of the authors.

\textbf{Competing interests}
The authors declare no competing interests.

\textbf{Data and code availability}
The demonstration code at tag \texttt{0.0.2} pins the synthesis library at
its own tag \texttt{0.0.2}~\cite{flexbe-synthesis-demo,flexbe-synthesis}.
The quadcopter uses the Flexible Drones software~\cite{flexible-drones-git}.
The synthesis backend uses our
\href{https://github.com/CNURobotics/slugs/tree/flexbe-synthesis}{\texttt{flexbe-synthesis} branch of \texttt{slugs}}~\cite{CNU_slugs}.
The demonstration repository's \texttt{experiment-data} branch hosts the
\href{https://raw.githubusercontent.com/CNURobotics/flexbe_synthesis_demo/e7991e65/experiments/experiment_data_15Sep26.zip}{original experiment data (\texttt{e7991e65})}
and the
\href{https://raw.githubusercontent.com/CNURobotics/flexbe_synthesis_demo/1264e0fd/experiments/p0-supplemental-evidence.zip}{supplemental data (\texttt{1264e0fd})}
for the P0 carriage and Coffee/P0 retry-justice checks
(\reappsec{p0-carriage-check}, \reappsec{p0-retry-check}).

\bibliography{main}

\end{document}

%% file: abstract.tex
%

High-level robotic supervisors coordinate asynchronous capabilities through their
reported outcomes, such as success or failure. Building on our earlier capability-based synthesis
framework, we present the complete architecture, implementation, and formal
foundations of an open-source pipeline for generating executable Robot Operating
System (ROS) 2 Flexible Behavior Engine (FlexBE) supervisors. The pipeline integrates
capability-based Generalized Reactivity (1) (GR(1)) specification generation,
assumption analysis, reactive synthesis, strategy auditing, behavior-preserving
state reduction, and controller realization.

We evaluate enumerated and one-hot encodings and two liveness formulations in six
comparisons across four domains, with additional PiHawk quadcopter flights. Under
the tested backend, enumerated encoding usually synthesizes faster, although fewer
propositions do not reliably predict smaller controllers or binary decision
diagrams. Within the tested budget, only System-Goal without pending memory yielded
executable controllers under both encodings across all six comparisons.
Fair-Outcome with pending memory can yield realizable strategies that enter
cycles from which the configured goal is unreachable. These failures show why
realizability alone does not establish deployment readiness. The auditor checks
for protocol violations, deadlocks, and goal-unreachable traps.
For the tested backend and models, we recommend pairing
enumerated encoding with System-Goal without pending memory and auditing each
extracted strategy before realization.

%% file: Introduction.tex
\section{Introduction}\label{sec:intro}

High-level robot supervisors coordinate capabilities such as perception,
motion planning, and trajectory execution. These capabilities execute
asynchronously and take different amounts of time; the supervisor must
choose the next action when a capability reports an outcome.
This supervisory logic is commonly implemented as
hand-engineered \emph{hierarchical finite-state machines (HFSMs)}
or \emph{behavior trees (BTs)}
\cite{Schillinger2016,bt_robotics_18,Ghzouli_bt_sm_23}.
These controllers must account not only for nominal action sequences but also
for capabilities that may fail during execution. For example, motion planning can fail due to a
closed door, perception can become unavailable due to camera failure, and
trajectory execution can fail due to a low battery.
These failures arise from conditions outside the capability's influence.
In response, the supervisor may retry an action, select an alternative, or terminate the
behavior safely.
As the number of capabilities, outcomes, and recovery paths increases,
manually designing and validating this supervisory logic becomes increasingly
difficult.

\begin{figure}[t]
    \centering
    \includegraphics[width=0.9\linewidth]{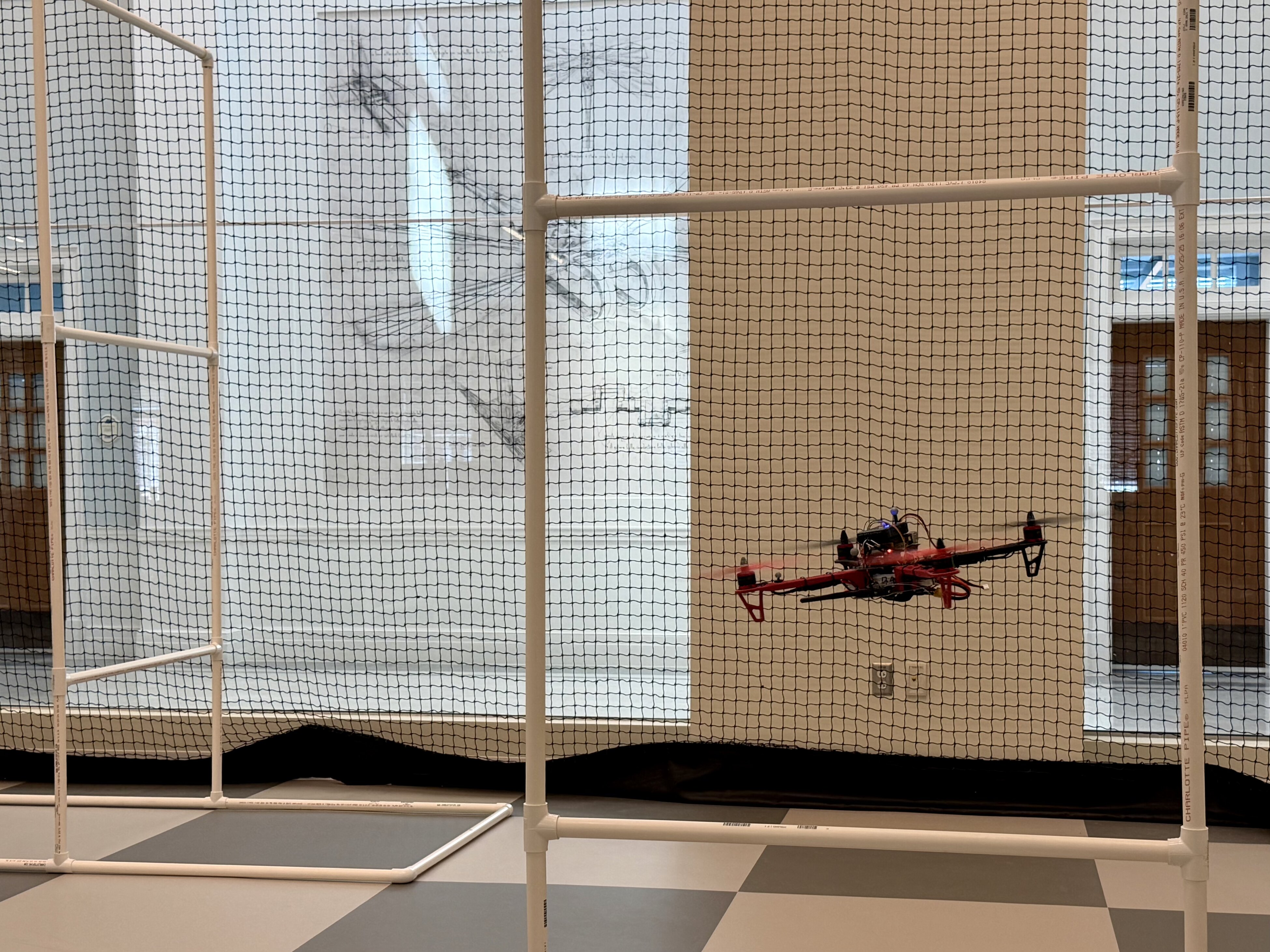}
    \caption{``PiHawk'' drone flying under supervisory control
    of synthesized state machine in our indoor flight space.}
    \label{fig:pihawk-flight}
\end{figure}

Formal methods offer an alternative by synthesizing reactive controllers from
logical specifications
\cite{Hadas09,Bloem2012,Maniatopoulos2016,cbc_bt17}.
Reactive synthesis generates supervisory decisions from explicit logical
specifications that encode environment assumptions and system guarantees.
The effort of writing specifications and running synthesis is borne
once, but the resulting supervisory controller can be executed many times.

A specification is realizable if the supervisor has a winning strategy that
satisfies the system guarantees for every environment behavior that meets the
encoded assumptions. Even with that formal assurance, the robot may fail to
complete the physical task or reach the designer's intended terminal condition.
This can happen when the specification permits repeated capability failures or
contains assumptions that can be violated along some execution traces.
This tension between formal realizability and designer
intent is a recognized hazard in reactive synthesis~\cite{KleinPnueli2010,MaozRingert2016,KindControllers24}.
\resec{synthesis} reviews these prior results.
Addressing this problem in robot software requires a workflow that connects
capability modeling and specification generation to strategy validation,
reduction, and executable controller realization.

This paper presents the complete architecture, implementation, and formal
foundations of such a workflow, building on and substantially extending our
open-source modular pipeline~\cite{flexbe-synthesis, flexbe-synthesis-demo}
first introduced in~\cite{eit26-synthesis}. The pipeline automatically generates
Generalized Reactivity (1) (GR(1)) specifications and synthesizes executable
HFSM supervisors for Robot Operating System (ROS)~2 systems.
Users describe capability interfaces and task requirements, then use the
pipeline to generate specifications, inspect synthesized strategies, and
realize executable supervisors. We explain how
its stages fit together, establish the guarantees of its auditing and
reduction stages, and evaluate specification choices that affect synthesis
cost and the resulting controllers' behavior.

The earlier pipeline supported parameterized capability interfaces, allowing the
supervisor to select both a capability and its configuration. It mechanically
converted synthesized strategies into executable Flexible Behavior Engine (FlexBE)
state machines whose transitions respond to reported capability outcomes.
The earlier work also demonstrated the pipeline's state-reduction module in
PyRoboSim and on Crazyflie hardware.

Here we make two specification-design axes configurable: capability encoding
and retry-liveness formulation. We apply the existing bounded-integer
support in \texttt{slugs}~\cite{Ehlers2016} to capability activations and
outcomes, extending its earlier use for parameter data in our pipeline.
The retry-liveness comparison pairs the System-Goal-style formulation
already used in the earlier pipeline~\cite{eit26-synthesis} with the
Fair-Outcome alternative and optional pending memory studied here.
We add pre-synthesis well-separation analysis and post-synthesis strategy
auditing, and give a formal treatment and correctness argument for
state-merging reduction. The evaluation expands the earlier demonstrations
into a comparison of specification choices across application domains.

The contributions are:
\begin{itemize}
    \item A complete architectural and formal account of the executable
    capability-based pipeline, integrating specification generation,
    assumption analysis, reactive synthesis, strategy auditing,
    behavior-preserving state reduction, and controller realization, with
    explicit assumptions and guarantee boundaries for these stages;

    \item Two configurable specification-design axes for this pipeline under
    the stated sequential, one-active-capability protocol: enumerated versus
    one-hot capability encoding and Fair-Outcome versus System-Goal liveness
    formulations for retry-capable behaviors;

    \item Two validation modules: an integration of established
    well-separation analysis through \texttt{slugs}, which classifies
    assumption-falsification cases and reports responsible-assumption cores;
    and a post-synthesis explicit-state strategy auditor, developed here for
    the pipeline's activation--outcome protocol, that checks protocol conformance,
    deadlock, and goal-unreachable traps. A correctness argument establishes
    soundness and completeness for detecting the defined defects under
    the stated exploration assumptions;

    \item A formal treatment of state-merging reduction for explicit Mealy
    strategies, with a correctness argument establishing that the reduced strategy preserves
    the core-output traces, and hence the capability-level guarantees, of
    the audited raw strategy synthesized from our specifications;

    \item An expanded comparative evaluation of capability encoding,
    retry-liveness formulation, and pending memory, measuring synthesis cost,
    controller size, and structural audit outcomes across four case-study
    domains, with hardware flights on a PiHawk quadcopter.

\end{itemize}

The pipeline uses the established GR(1) synthesis algorithm in
\texttt{slugs}~\cite{Ehlers2016}; our implementation extends its supporting
code for Python 3 and additional features~\cite{CNU_slugs}. The auditor and
reduction arguments apply established graph and behavioral-equivalence
reasoning to the pipeline's activation--outcome protocol and configured goal
semantics.

The four case-study domains presented in this paper are
an extended coffee-maker example, an
extended river-crossing puzzle, a simulated mobile-manipulation logistics
task, and supervisory control of quadcopter flight, validated on hardware. Treating the
Coffee base and extended capability sets and the PyRoboSim P0 and P1 models
separately yields six domain/subdomain comparisons.
The hardware case study includes the ``PiHawk'' quadcopter shown flying under
synthesized supervisory control in \refig{pihawk-flight}.
The first two compact examples are used as a
controlled testbed for isolating capability-encoding
and liveness-formulation effects.

The evaluation compares one-hot and enumerated capability encodings,
System-Goal and Fair-Outcome liveness, and configurations with and without
pending memory. When enabled, pending memory records that a completion is
still owed, so a retry-fairness assumption can span intervening steps.

This paper examines two research questions for failure-aware supervisory
synthesis. \textbf{RQ1 (encoding):} How does capability encoding affect
synthesis time, peak BDD size, and controller size across the evaluated
domains and declaration orderings? \textbf{RQ2 (retry semantics):} How do
retry-liveness formulation and pending memory affect realizability and the
structural defects detected by well-separation analysis and strategy
auditing?

Under the tested \texttt{slugs}/CUDD backend, enumerated encoding usually
synthesizes faster and is the practical empirical starting point.
Its time advantage is not uniform, however, and fewer propositions do not
reliably predict strategy size or peak BDD size.
Within the tested budget, only System-Goal without pending memory produced
realized FlexBE HFSMs under both encodings in all six comparisons.
A realizable Fair-Outcome strategy can still fail to make the progress
the designer intended.
These experiments show how the pipeline's analysis stages expose
specification and strategy defects that proposition counts or realizability
alone would leave undetected.

The pipeline blocks controller realization when the auditor confirms a
protocol or progress defect (\resec{auditor}); well-separation analysis
remains diagnostic (\resec{well-separation-module},
\reappsec{pipeline-integration}). Executions in the domain software
environments, including quadcopter kinematic simulation, and PiHawk flights
with both capability encodings provide empirical validation of controller realization
(\resec{crazyflie-failure-validation}).

The remainder of this paper is organized as follows. \resec{background}
reviews FlexBE and GR(1)-based synthesis, including well-separation.
\resec{pipeline-overview} presents the pipeline architecture that maps
parameterized capabilities to executable FlexBE state machines.
\resec{demonstrations} introduces the four case-study domains used
throughout the paper. \resec{pipeline} details the specification-generation
choices under study together with the well-separation module, the
post-synthesis auditor (\resec{auditor}), and the state-merging reduction
(\resec{reduction}). \resec{discussion} reports and discusses the experimental findings.
\resec{conclusion} concludes and outlines future work.
The appendices provide formal proofs for the auditor and reduction algorithms
and detail the experimental results summarized in \resec{discussion}.

%% file: Background.tex
\section{Background}\label{sec:background}

This work considers the synthesis of HFSM-based behavior controllers for
ROS~2 systems using Linear Temporal Logic (LTL). Other temporal
specification languages have also been used for robot behavior synthesis.
Signal Temporal Logic (STL), for instance, extends temporal reasoning to
continuous, real-valued signals rather than Boolean atomic propositions
\cite{MalerNickovic2004}. Kress-Gazit et al.~\cite{Hadas2018} survey this
broader landscape of specification languages, abstractions, and synthesis
algorithms for robotics. A different paradigm altogether generates
behavior through automated planning rather than reactive synthesis.
Classical planners built on the Planning Domain Definition Language
(PDDL) search offline for an action sequence that reaches a goal from a
fixed initial state, rather than synthesizing a controller that reacts
correctly to any environment behavior encountered at
runtime~\cite{McDermott1998PDDL}. This alternative is concretely
available for the same class of task studied in this paper's
mobile-manipulation logistics demonstration (\resec{pyrobosim}). The
simulator used for that demonstration ships with built-in
task-and-motion planning via PDDLStream, which integrates a symbolic PDDL
planner with blackbox samplers for continuous values such as grasp poses
and base placements~\cite{Garrett2020PDDLStream}. That planning stack is
well suited to finding sampled action sequences, while the problem studied
here is to synthesize a reusable supervisor that reacts at runtime to
reported capability outcomes, including failures and recovery branches.
Here, we limit the background discussion
to concepts directly relevant to the demonstrations presented
in this paper.

\input{Background-FlexBE.tex}

\input{Background-Synthesis.tex}

%% file: Background-FlexBE.tex
\subsection{Robot System Behaviors}

This work is demonstrated within the ROS~2 framework~\cite{ROS2022},
which provides modular communication and execution infrastructure for robotic systems.
Within ROS~2, high-level behaviors are commonly implemented as behavior trees (BTs)
or hierarchical finite state machines (HFSMs) \cite{Schillinger2016, bt_robotics_18, Ghzouli_bt_sm_23},
and correct-by-construction synthesis has been studied for BTs as well as
for the HFSM-based controllers this paper targets \cite{cbc_bt17}.

Supervisory-controller synthesis has also been applied to ROS-based systems
more broadly, using Ramadge-Wonham discrete-event supervisory control
theory rather than reactive game-based synthesis~\cite{Torta2023CEP, Wesselink2023RoboSC}. That line of work maps
ROS/ROS~2 components to plant models and requirements and synthesizes a
supervisor with the CIF supervisory-controller synthesis toolset, an approach distinct from the
formal reactive synthesis this paper builds on.
Analysis of ROS~2 itself, e.g.\ its communication and execution
semantics, has separately been addressed through verification rather than
synthesis~\cite{Dust2025FrontiersROS2}.

In this work, we use the open-source Flexible Behavior Engine
(FlexBE)~\cite{Schillinger2016, Kohlbrecher2016, Zutell_22},
an HFSM-based framework that supports collaborative autonomy and enables runtime interaction
between human operators and autonomous behaviors.
In FlexBE, states are implemented as Python components that interface with system capabilities
via ROS topics, services, and actions.

Prior work has demonstrated the use of formal synthesis of realizable controllers
deployable within FlexBE-based systems~\cite{Maniatopoulos2016, Hayhurst2018, wgcf24, eit26-synthesis}.
Building on that foundation, this paper's specification-generation,
capability-encoding, and liveness-formulation extensions are detailed in
\resec{overview-capability} and \resec{pipeline_config}.
Related robotics toolchains include LTLMoP, which translates Structured
English or temporal-logic task specifications into synthesized hybrid
controllers and supports their execution~\cite{Finucane2010LTLMoP}.
Our pipeline focuses on generating supervisors from parameterized ROS~2
capability interfaces and realizing them as executable FlexBE HFSMs.

%% file: Background-Synthesis.tex
\subsection{LTL-based State Machine Synthesis} \label{sec:synthesis}

Linear Temporal Logic (LTL) extends propositional logic with temporal operators
that describe system evolution over discrete time~\cite{Ehlers2011, Bloem2012, Ehlers2016, Hadas2018}.
In general, reactive synthesis from full LTL specifications is doubly exponential
in the size of the formula~\cite{PnueliRosner90, Vardi1994}.

To obtain tractable synthesis for robotic systems,
we adopt the Generalized Reactivity (1) (GR(1)) fragment,
which restricts specification structure and enables singly exponential synthesis~\cite{Bloem2012, Ehlers2016}.

We summarize the GR(1) assumption--guarantee structure by the implication
\begin{equation}
   \varphi \equiv (\varphi_i^a\land\varphi_s^a\land\varphi_l^a) \implies
    (\varphi_i^g\land\varphi_s^g\land\varphi_l^g)
\label{eqn:varphi}
\end{equation}

where environment \emph{assumptions} ($\varphi^a$) and system \emph{guarantees} ($\varphi^g$)
are partitioned into \emph{initialization} ($\varphi_i$),
\emph{safety} ($\varphi_s$), and \emph{liveness} ($\varphi_l$) constraints.
The \texttt{slugs} synthesis game uses strict GR(1) semantics
\cite{KleinPnueli2010,Ehlers2016}. The system must satisfy its safety
guarantees on prefixes that satisfy the environment's initialization and
safety assumptions, even if the environment later violates a liveness
assumption.
Synthesis is posed as a two-player game between the potentially adversarial environment
and the system.
A specification is \emph{realizable} if the system can choose controller outputs
that enforce the guarantees
for all environment behaviors satisfying the assumptions.
When realizable, the result is a finite-state Mealy machine that serves as a reactive controller.

An assumption can depend on choices the controller makes. For example,
the controller can prevent the environment from satisfying a liveness
assumption and thereby win without satisfying its own liveness guarantee,
while continuing to satisfy the system safety guarantees. Well-separation
analysis identifies specifications that permit this kind of assumption
falsification~\cite{KleinPnueli2010,MaozRingert2016}.
Such specifications are called non-well-separated (NWS).
Other studies
examine anomalous or assumption-falsifying controllers and how assumptions
should be handled during synthesis~\cite{DIppolito2013,
Bloem2014Assumptions}. Work on kind controllers goes further,
seeking controllers that avoid exploiting NWS assumptions when
such controllers exist~\cite{KindControllers24}.
Cooperative reactive mission planning~\cite{Ehlers2015Cooperative} and
environmentally-friendly (non-conflicting) GR(1)
synthesis~\cite{Majumdar2019EnvironmentallyFriendly}
address assumption exploitation during synthesis by preserving the possibility
of assumption-satisfying continuations. Our pipeline instead diagnoses
specifications and audits particular extracted strategies; its structural
checks do not establish that non-conflictingness property.
Complementary to this well-separation-focused literature,
Maoz et al.~\cite{Maoz2026} address GR(1) specification validity and quality
more broadly, independent of well-separation.
Their analysis includes systematic scenario generation that demonstrates
the meaning of each specification element on a synthesized controller, and
detection of unnecessary variables and specification clones as readability and
maintainability issues distinct from realizability. Their analysis reviews whether a
specification says what its author intended, entirely at the specification and
simulated-scenario level. This paper's well-separation module (\resec{well-separation-module}) and
post-synthesis auditor (\resec{auditor})
instead check, after the specification is fixed, whether its assumptions
can be forced or an already-synthesized strategy structurally fails to
progress toward the designer's goal.
This paper's pipeline carries
those checks through synthesis, reduction, and controller realization to
an executable FlexBE state machine that runs a real ROS~2 system, rather
than stopping at a specification-level artifact.

Using the notation of \reqn{varphi}, the well-separation check asks whether
the environment side $(\varphi_i^a,\varphi_s^a,\varphi_l^a)$ alone admits a
system strategy that can force an assumption violation. Following
\cite{MaozRingert2016}, this is reduced to an ordinary GR(1) winning-region
computation by replacing the system side with unconstrained initialization
and safety guarantees and an impossible liveness guarantee, i.e.,
$\varphi_i^g=\True$, $\varphi_s^g=\True$, and $\varphi_l^g=\False$. If the
system can win this reduced game from an environment-reachable state, then
the original specification is NWS. Some legal interaction lets the controller
make the environment unable to satisfy its assumptions. The diagnosis is
classified by whether this assumption-falsifying winning region is
unavoidable from every legal initial environment choice or only reachable
from some later game state, and by whether the forced violation concerns
initialization, safety, or liveness assumptions.
The implementation used in this paper exposes these cases and, when
requested, minimizes the responsible assumption set to a core.

GR(1) synthesis has a substantial history of application to robot behavior
generation. Early work used LTL mission specifications to synthesize
discrete supervisors for slow, high-level actions layered over continuous
low-level controllers~\cite{Hadas08, Raman12}, later extended to continuous
control for actions whose execution duration is not known in
advance~\cite{Raman13ICRA}. Our activation--outcome formulation
(\resec{syn_spec}) generalizes that durative-action abstraction to actions
with multiple possible outcomes rather than a single completion event.
Because not every mission specification is realizable, subsequent work
addressed diagnosing and minimally explaining unsynthesizable
specifications~\cite{Raman13}. This paper revisits that concern from the
opposite direction in \resec{auditor}. Our auditor instead diagnoses
strategies that \emph{are} realizable but still fail to make designer-intended
progress. More broadly, formal synthesis of control strategies from
finite-state abstractions has been applied well beyond mobile-robot mission
planning, including general dynamical systems~\cite{Belta16}.

Our pipeline uses the \texttt{slugs} GR(1) synthesizer~\cite{Ehlers2016}.
Finite-domain variables are established in reactive-synthesis environments:
\texttt{slugs} supports bounded integers, while Spectra supports enumerations
and bounded integers, translated into Boolean
representations~\cite{Ehlers2016, MaozRingert2021Spectra}.
Spectra also provides specification analyses and controller-execution
tools~\cite{MaozRingert2021Spectra}. Here we apply finite-domain encoding to
the capability activation--outcome interface and evaluate it within the
\texttt{slugs}/FlexBE pipeline.
\texttt{slugs} carries out its symbolic fixed-point computations using
binary decision diagrams (BDDs) through CUDD, the CU Decision Diagram Package~\cite{SomenziCUDD}.
Atomic propositions are divided into environment inputs $AP_I$ and system outputs $AP_O$,
with the synthesized controller assigning values to $AP_O$ in response to $AP_I$.
\texttt{Slugs} supports a restricted extension of GR(1) that permits the \emph{Next}
operator, provided each clause in \reqn{varphi} is a conjunction of Boolean
formulas satisfying the structural constraints summarized in
\retab{slugs_spec}. Every safety clause must hold \emph{Globally}, and every
liveness clause must satisfy a \emph{Globally Finally} condition. The
\emph{Next} operator is also restricted differently on the environment and
system sides. For example, environmental safety assumptions may apply
\emph{Next} only to input propositions $AP_I$, whereas system safety clauses
may apply \emph{Next} to both input and output propositions. The environment
cannot be constrained to react to the system's own next move, but the system
can react to either. This restricted use of \emph{Next} is what enables the
practical modeling of reactive robotic behaviors used throughout this paper.
Sections~\ref{sec:demonstrations} and~\ref{sec:pipeline} provide concrete examples drawn
from our implementation.

\begin{table*}[htbp]
\centering
\caption{\texttt{Slugs} Specification Structure}
\label{tbl:slugs_spec}
\renewcommand{\arraystretch}{1.4}
\begin{tabular}{|l|l|c|}
\hline
\textbf{Term} & \textbf{\texttt{slugs} section} & \textbf{Form} \\
\hline
\multicolumn{3}{|c|}{\textbf{Environment}} \\
\hline
$\varphi_i^a$ & \texttt{[ENV\_INIT]} &
$\psi : AP_I$ with $\land, \lor, \neg$ only \\
\hline
$\varphi_s^a$ & \texttt{[ENV\_TRANS]} &
$G\psi : AP_I \cup AP_O \cup \left\{Xy : y \in AP_I\right\}$ \\
\hline
$\varphi_l^a$ & \texttt{[ENV\_LIVENESS]} &
$GF\psi : AP_I \cup AP_O \cup \left\{Xy : y \in AP_I \cup AP_O\right\}$ \\
\hline
\multicolumn{3}{|c|}{\textbf{System}} \\
\hline
$\varphi_i^g$ & \texttt{[SYS\_INIT]} &
$\psi : AP_I \cup AP_O$ with $\land, \lor, \neg$ only\\
\hline
$\varphi_s^g$ & \texttt{[SYS\_TRANS]} &
$G\psi : AP_I \cup AP_O \cup \left\{Xy : y \in AP_I \cup AP_O\right\}$  \\
\hline
$\varphi_l^g$ & \texttt{[SYS\_LIVENESS]} &
$GF\psi : AP_I \cup AP_O \cup \left\{Xy : y \in AP_I \cup AP_O\right\}$ \\
\hline
\end{tabular}
\end{table*}

Studies of satisfiability and symbolic reasoning have shown
that computational difficulty depends strongly on constraint structure and
cannot be predicted from problem size alone
\cite{Cheeseman91,Selman199617,Ran3SatThicken01,RandSATBeyond04}.
For BDD-based methods, such as GR(1) synthesis, logically equivalent or similarly sized formulations
can produce substantially different symbolic representations because of
variable interactions and ordering~\cite{Bryant1986}. The activation
representations in the one-hot and enumerated capability encodings
correspond to the direct and logarithmic encodings of finite domains
studied in SAT and constraint problems, where encoding choices affect
propagation strength and search behavior~\cite{Walsh2000SATvCSP}.
For the BDD-based synthesis studied here, proposition count alone is
therefore not a reliable predictor of synthesis cost.

%% file: PipelineOverview.tex
\section{Synthesis Pipeline Overview}
\label{sec:pipeline-overview}

Before presenting the demonstration domains, we summarize how our synthesis
pipeline turns a description of a robot's software capability interfaces into an
executable FlexBE supervisor, and introduce the vocabulary the
demonstrations rely on. \resec{pipeline} formalizes this pipeline in full.
This overview focuses on the concepts needed to read a capability
configuration file and follow what the synthesized controller is doing.

\subsection{From Capability to Executable State}
\label{sec:overview-capability}

A \emph{capability} is a named, parameterized unit of robot functionality,
such as a motion planner, a perception module, or a controller.
Capabilities are typically implemented as ROS~2 components with topic,
service, or action interfaces. FlexBE states interact with these interfaces
to expose capability activation and reported outcomes to the supervisor.
The synthesizer decides \emph{when} to activate a capability via its FlexBE state interface
and \emph{with which parameter values}, while the underlying system capability does
the work and reports back its result to the FlexBE state instance.
The reported state outcomes, including timeouts and failures, correspond to
environmental inputs in the GR(1) formalism.
 \refig{pipeline-overview}
shows our pipeline with several such underlying robot system capabilities, including planning,
perception, control, navigation, monitoring, and sensing.
A configured FlexBE state implementation exposes each capability to the
supervisor. For example, a \texttt{move\_robot}
FlexBE state implementation might connect to an existing navigation module through
an action interface accepting a destination parameter and reporting \texttt{completed} or
\texttt{failure}. The synthesized supervisor chooses
which capability to activate next and with what
parameters, exactly as a human-authored FlexBE behavior would, so the
resulting controller can be inspected, modified, and executed with the same
tools used for any other FlexBE behavior.

Some FlexBE states, such as \texttt{LogState}, implement self-contained capabilities.

A \emph{capability configuration file} binds each named capability to its
FlexBE state implementation and records the details the synthesizer needs
that are not visible from the workspace alone. These include fixed or symbolic
parameter values, \emph{preconditions} that must hold before the capability may be
activated, \emph{postconditions} that update symbolic memory once an
outcome is reported, and \emph{transition relations} that force a specific
follow-on transition after a given outcome (for instance, always logging a
failure before a retry). Where a parameter's value is chosen by the
synthesizer rather than fixed in advance, the configuration file names it
with a leading \texttt{@},
e.g.\ \texttt{destination: \textquotesingle @move\_dest\textquotesingle}.
This tells the pipeline to read the value from
the corresponding enumerated
proposition in the synthesized strategy at runtime rather than treating it
as a constant. \resec{demonstrations} shows examples of concrete
configuration files using this notation.

This paper's notion of a \emph{capability}, formalized in \resec{syn_spec},
connects to a broader line of work on skill-based robot architectures,
which organize robot behavior around reusable, composable skill
abstractions rather than monolithic task programs~\cite{SkiROS23,
Skill23}. Recent work has also explored automatic encoding of robot
capabilities and repair of infeasible high-level reactive task
specifications synthesized from LTL~\cite{Asfour23}, and, most closely
related to the encoding choices compared in this paper, online
resynthesis of high-level collaborative tasks as a robot's available
capabilities change at runtime~\cite{FangYK25}. Our pipeline instead
synthesizes offline from a fixed capability catalog
(\resec{pipeline_config}), though runtime capability change is a natural
direction for future integration with that line of work.

\subsection{Pipeline Stages}
\label{sec:overview-stages}

\refig{pipeline-overview} sketches the resulting FlexBE synthesis pipeline
introduced in \cite{eit26-synthesis} and fully presented in this paper.
We label it \emph{FlexBE
Synthesis Pipeline} in the figure to distinguish the pipeline itself
from the FlexBE system and from the underlying
robot system capabilities that FlexBE's state
implementations interface. The pipeline is orchestrated by
a single \texttt{synthesis\_manager} ROS~2 node.
It separates a \emph{preprocessing} phase, run once at startup,
from \emph{request-time processing} run for
each synthesis request given new goals and initial conditions:

\begin{figure*}[t]
    \centering
    \includegraphics[width=\linewidth]{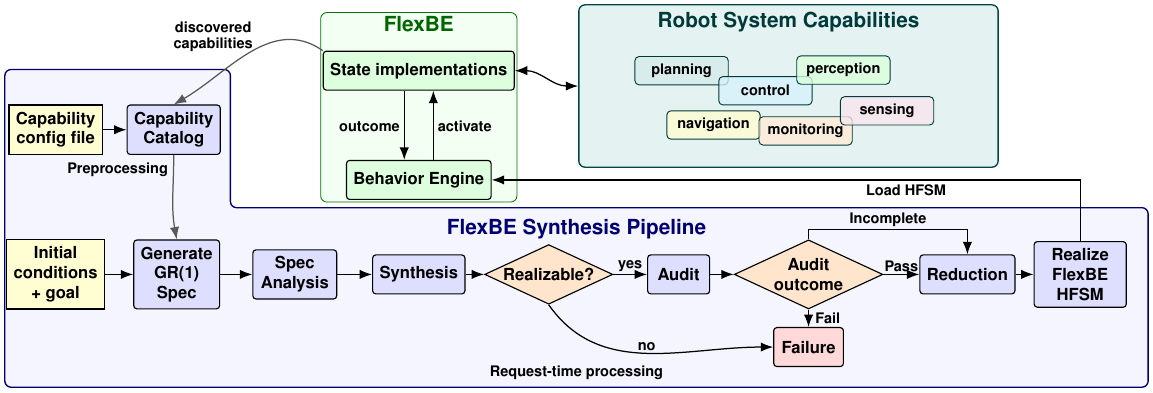}
    \caption{
    The FlexBE Synthesis Pipeline builds a reusable capability catalog
    once at startup (Preprocessing) by discovering FlexBE's existing
    state implementations, each of which interfaces with an underlying
    robot system capability such as planning, perception, or control.
    Request-time processing runs once per synthesis request and performs
    specification analysis, synthesis, auditing, and reduction before
    ending either in \emph{Failure} or in a realized FlexBE state machine.
    After an incomplete audit, generation can continue for inspection or
    testing, with the incomplete status retained. The FlexBE Behavior
    Engine loads the realized HFSM, which activates state implementations
    and uses their outcomes to select the next transition.
     }
    \label{fig:pipeline-overview}
\end{figure*}

\begin{itemize}
    \item \textbf{Preprocessing.} The manager inspects the ROS workspace
    for available FlexBE states and combines this with the capability
    configuration file to build a reusable capability catalog
    (\resec{pipeline_config}).
    \item \textbf{Capability-Based Specification Generation.}
    \refig{pipeline-overview} draws this as a single \emph{Generate GR(1) spec}
    box, but it bundles several sub-steps. It loads the request's relevant
    capability subset from the catalog, instantiates the propositions needed for the
    initial conditions and goal requirements, and generates the activation--outcome
    protocol and liveness constraints that make up the GR(1) specification
    (\resec{syn_spec}).
    \item \textbf{Specification analysis.} Before synthesis, the
    generated \texttt{slugs} specification can optionally be checked for
    assumption-falsification cases and responsible environment-assumption
    cores using a well-separation analysis (\resec{well-separation-module}).
    \item \textbf{Synthesis.} The specification is passed to the synthesizer.
    An unrealizable result ends the request, while a
    realizable result yields a raw Mealy strategy.
    The pipeline currently uses the \texttt{slugs} GR(1)-based synthesis tool.
    \item \textbf{Auditing.} The raw strategy is checked by the
    post-synthesis auditor (\resec{auditor}) for protocol violations and
    for regions from which the designer's goal has become structurally
    unreachable. A confirmed failure stops generation. An incomplete audit
    allows generation for inspection or testing and carries the incomplete
    status through to the final result (\reappsec{pipeline-integration}).
    \item \textbf{Reduction.} The optional reduction stage simplifies the strategy by
    merging behaviorally equivalent states. The implementation can repeat the
    local merge sweep to a fixed point, while retaining the same conservative
    merge test. \resec{reduction} proves this reduction preserves the
    strategy's core-output behavior under the same admissible input sequences.
    \item \textbf{Controller realization.} Drawn as a single
    \emph{Realize FlexBE HFSM} box, this stage bundles translating the
    strategy into an executable FlexBE hierarchical state machine
    and generating an initial GUI layout for that state machine so it can
    be opened directly in the FlexBE WebUI. For the sequential supervisors
    evaluated here, each nonterminal strategy state that selects an
    executable capability configuration becomes a corresponding FlexBE
    state. Realization instantiates the FlexBE state
    implementation named by preprocessing and mechanically wires each concrete FlexBE
    outcome through the same outcome mapping used during specification
    generation to the strategy successor for the corresponding
    abstract outcome.
\end{itemize}

The formal equivalence result in this paper concerns the raw strategy and the
reduced automaton used when state merging is enabled. At the capability
interface, the realized controller follows the strategy supplied to
realization. Each executable step
activates a capability (and, for symbolic parameters, selects a value), the
corresponding FlexBE state invokes the relevant system capability, and its
reported concrete outcome is deterministically remapped to the abstract
outcome that selects the next strategy state. This lets us execute a
synthesized supervisor directly in FlexBE. \resec{reduction} proves that the
state-merging reduction step preserves the strategy's core output traces; the
subsequent strategy-to-HFSM realization is a mechanical construction from the
strategy and the shared discrete abstraction. \resec{controller-realization}
describes the evidence for this translation and its execution on hardware.

\subsection{Slugs Example}
\label{sec:slugs-example}

Before \resec{demonstrations} introduces this paper's own
capability-generated specifications, it is worth working through the
Mealy-machine encoding, its mapping to a realized HFSM state, and the
\texttt{mealy2dot} visualization convention used throughout this paper on
a specification small enough to draw in full. We use the coffee-machine
example from Ehlers' introductory lecture on reactive
synthesis~\cite{Ehlers2022Video}, transcribed in \texttt{.structuredslugs} form in
\relist{slugs-example-spec}.

\begin{lstlisting}[
float={*ht},
basicstyle=\normalsize\ttfamily,
columns=fullflexible,
caption={A minimal coffee-machine {GR(1)} specification, transcribed
from~\cite{Ehlers2022Video}. The button press itself
(\texttt{bu}) is a raw environment input rather than the outcome of a
system-activated capability.},
label={lst:slugs-example-spec}]
[INPUT]
bu

[OUTPUT]
br
gr

[SYS_INIT]
gr <-> bu
! br

[SYS_TRANS]
br' <-> gr
gr' -> bu'

[ENV_TRANS]
bu' -> !gr & !br

[SYS_LIVENESS]
br

[ENV_LIVENESS]
bu
\end{lstlisting}

Every section of \relist{slugs-example-spec} is one term of \reqn{varphi},
using the same $\varphi_i, \varphi_s, \varphi_l$ notation defined in
\resec{synthesis} and \retab{slugs_spec}. The \texttt{[INPUT]} and
\texttt{[OUTPUT]} sections declare $AP_I=\{\texttt{bu}\}$ and
$AP_O=\{\texttt{br},\texttt{gr}\}$.

The remaining sections map directly to the GR(1) terms. The
\texttt{[SYS\_INIT]} section is $\varphi_i^g$,
\texttt{[SYS\_TRANS]} is $\varphi_s^g$, \texttt{[ENV\_TRANS]} is
$\varphi_s^a$, \texttt{[SYS\_LIVENESS]} is $\varphi_l^g$, and
\texttt{[ENV\_LIVENESS]} is $\varphi_l^a$. There is no
\texttt{[ENV\_INIT]} section, so $\varphi_i^a=\True$. The environment may
therefore start with \texttt{bu} either true or false.

This specification is realizable. Passed through this paper's own
\texttt{slugs}-based synthesis and extraction stages
(\resec{overview-stages}), it yields the three-state explicit strategy in JSON
form shown in \relist{slugs-example-json}.

\begin{lstlisting}[
float={*ht},
basicstyle=\normalsize\ttfamily,
columns={[c]fixed},
basewidth=0.5em,
caption={The explicit Mealy strategy \texttt{slugs} extracts for
\relist{slugs-example-spec} in the JSON format.
The \texttt{variables} array gives the proposition order for every
\texttt{state} array. For each node, \texttt{trans} lists
the node IDs reachable in one step. The \texttt{rank} field records which
system liveness obligation the node is currently working toward, and is~0
throughout here because \relist{slugs-example-spec} declares a single
\texttt{[SYS\_LIVENESS]} goal.},
label={lst:slugs-example-json}]
{"version": 0,
 "slugs": "0.0.1",
 "variables": ["bu", "br", "gr"],
 "nodes": {
  "0": {"rank": 0, "state": [0, 0, 0], "trans": [0, 1]},
  "1": {"rank": 0, "state": [1, 0, 1], "trans": [2]},
  "2": {"rank": 0, "state": [0, 1, 0], "trans": [0]}
}}
\end{lstlisting}

\begin{figure}[t]
    \centering
    \includegraphics[width=\linewidth]{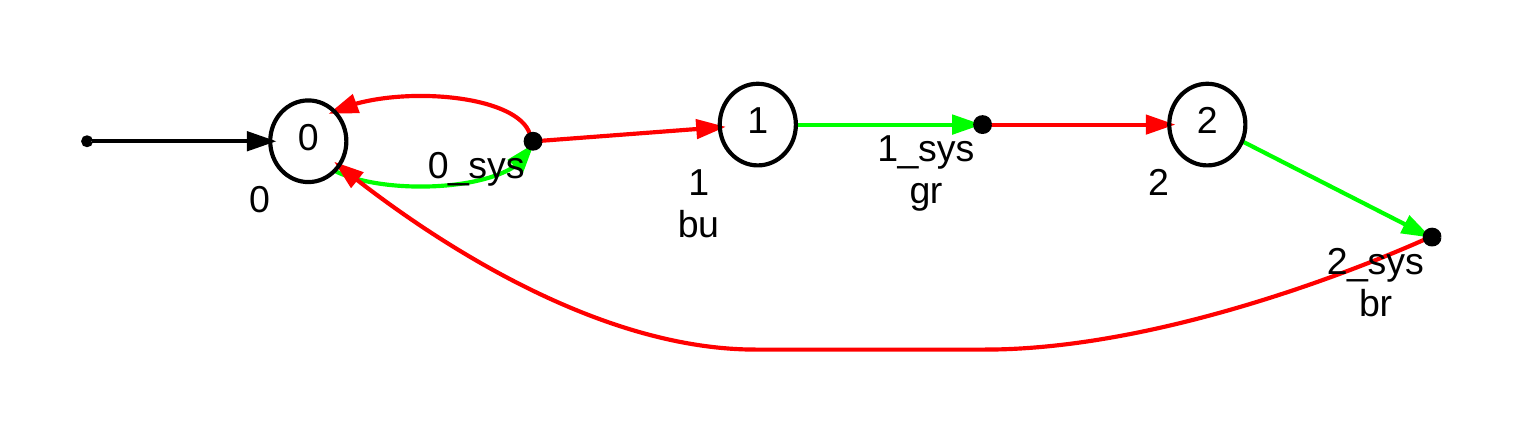}
    \caption{The Mealy strategy of \relist{slugs-example-json}, rendered
    by our \texttt{mealy2dot} tool in the visual convention used throughout this
    paper. Each raw \texttt{slugs} node is drawn as a linked pair: a
    large circle showing the node's ID and which input propositions are
    true (an environment-observation state), followed by a green
    env$\to$sys edge to a small filled dot showing the same node's ID
    suffixed \texttt{\_sys} and which output propositions it newly
    asserts (a system-activation state). A red sys$\to$env edge then
    leaves that dot for each of the node's \texttt{trans} successors,
    representing the environment's next choice. Node~0 is idle
    (\texttt{bu} false, no output asserted) and self-loops until the
    environment sets \texttt{bu}. Node~1 observes \texttt{bu}, and the
    system responds by asserting \texttt{gr}. Node~2 is entered with
    \texttt{bu} false again, and the system asserts \texttt{br} before
    returning to node~0.}
    \label{fig:slugs-example-mealy}
\end{figure}

Our \texttt{mealy2dot} tool used to generate \refig{slugs-example-mealy}
makes the env/sys alternation concrete:
\texttt{slugs} couples ``observe an input, then commit an output'' into a
single node, and \texttt{mealy2dot} draws the observed input and the system
response as separate nodes. This pairing carries over to the realized HFSM
(\resec{overview-stages}): each (circle, dot) pair becomes one FlexBE
state, the dot's asserted output is the capability that state activates,
and the state's outgoing transitions are keyed by which environment input
is reported back.
For \refig{slugs-example-mealy}, \texttt{bu} is a primitive, unstructured
environment input: the button press itself is given to the system for
free. \resec{coffee} extends this example
so that even noticing the button press is a system capability.

\FloatBarrier

%% file: Demonstrations.tex
\section{Demonstrations}\label{sec:demonstrations}

This paper presents four examples of synthesis using the modular pipeline.
This section introduces the demonstration domains used throughout the paper so
that the synthesis pipeline and specification choices can be described
concretely. We describe the experimental setup and show the synthesized state
machines executing here. \resec{discussion} compares specification-generation
variations.

All of the code required for these demonstrations is available open
source~\cite{flexbe-synthesis, flexbe-synthesis-demo}. Configuration files,
scripts, and raw trial data for the experiments reported below are published
with the demonstration artifacts; the exact tag and experiment-data commit are
listed in the Data and code availability declaration at the end of this paper.

\input{CoffeeDemo}

\input{TwoRiversDemo}

\input{PyroboSimDemo}

\input{CrazyFlieDemo}

%% file: CoffeeDemo.tex
\subsection{Extended Coffee Maker Demonstration}
\label{sec:coffee}

The basic demonstration starts with Ehlers' coffee-maker
example~\cite{bridginggap17, Ehlers2022Video}.
As described in \resec{slugs-example}, the original example
has an environmental input of \texttt{bu} for button, and system outputs of
\texttt{gr} and \texttt{br}, for grind and brew.
Our example extends this to include a button detection capability, \texttt{bd}, with
potential failure modes, and regular activation and completion variables as described in \resec{syn_spec}.
We use FlexBE \texttt{OperatorDecisionState}s to model the coffee-maker capabilities,
as shown in \refig{coffee_sm} using the capability configuration
excerpt in \relist{coffee-config}.

\begin{figure*}[tb]
    \centering
    \includegraphics[
    trim=3px 2px 0 0,
    clip, width=0.9\linewidth]{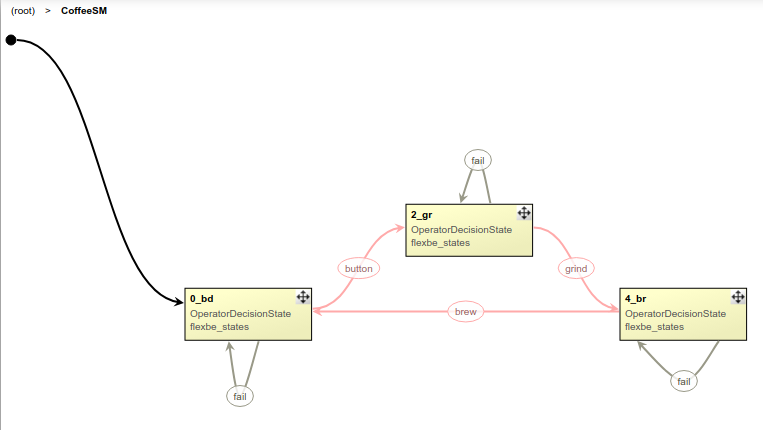}
    \caption{
    The realized FlexBE state machine for the extended coffee-maker example.}
    \label{fig:coffee_sm}
\end{figure*}

\begin{lstlisting}[
float={*ht},
basicstyle=\normalsize\ttfamily,
columns=fullflexible,
caption={Capability configuration excerpt for the extended Coffee demonstration,
showing an \texttt{OperatorDecisionState} capability with preconditions,
postconditions, parameters, and outcome autonomy levels.},
label={lst:coffee-config}]
capabilities:
  gr:
    interface: OperatorDecisionState
    parameters:
      outcomes:
        - grind
        - fail
      hint: "Grind"
      suggestion: grind
    preconditions:
      - "bd"
    postconditions:
      completed:
        - '!bd' # Clears button memory
        - 'gr'  # Sets memory flag
      failure:
        - "bd"  # Preserve button detected
    autonomy:
      # Autonomy level per outcome (1=low)
      grind: 2  # High autonomy required
      fail:  3  # Full autonomy required
\end{lstlisting}

In this example, the desired behavior does not have any finished or failed
outcomes, so the synthesized state machine can run forever as shown in
\refig{coffee_sm}. After brewing, it transitions back to the button detection
state.

In \resec{discussion}, we compare the basic capability version with the extended version that models failure modes.
Because the basic Coffee capability set has no failure-prone capabilities, its
liveness variants coincide; that case therefore informs only the encoding axis.
The impact of various encodings of the liveness conditions is discussed.

\FloatBarrier

%% file: TwoRiversDemo.tex
\subsection{Two-Rivers Demonstration}
\label{sec:two-rivers}

Like Coffee, this demonstration shown in \refig{two-rivers-flexbe}
uses a compact, hand-tractable domain
deliberately, not because the underlying problem is trivial to solve, but
because its small size makes it a controlled testbed. Encoding and
liveness effects that would be hard to isolate amid a larger domain's own
complexity show up here directly, and \resec{discussion} exploits exactly
that property for the encoding/liveness comparisons that follow
(\resec{pyrobosim} then tests whether the same conclusions hold as domain
complexity grows).

\begin{figure*}[!bt]
  \centering
  \includegraphics[width=0.8\linewidth]{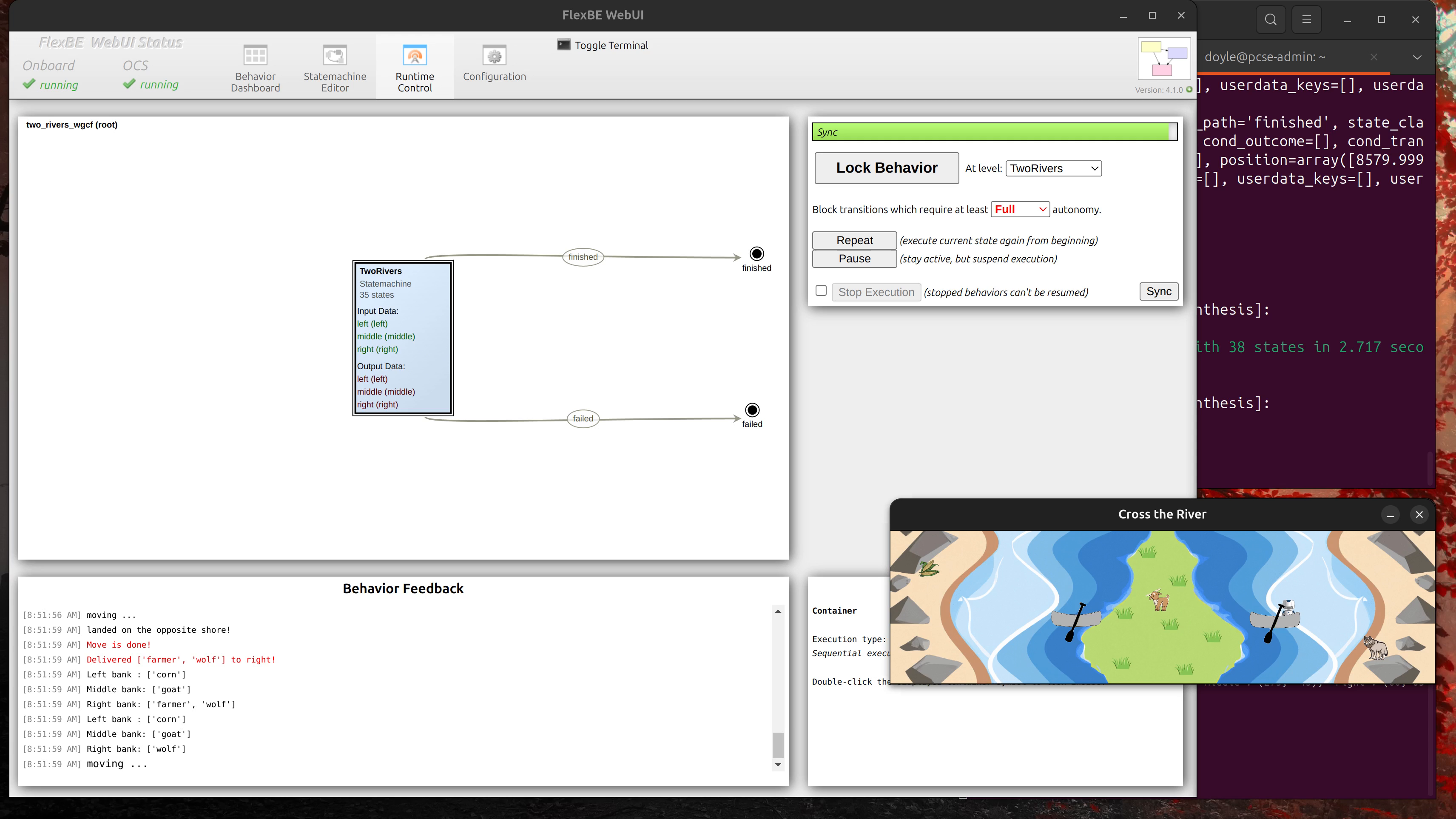}
  \caption{
     In this variation of the classic logic puzzle, the farmer must cross two rivers
  without leaving the
  goat unsupervised with either the corn or the wolf.
  This image shows a screenshot during execution of the synthesized state machine in FlexBE.
  (visualization adapted from~\cite{wgcf24})}
  \label{fig:two-rivers-flexbe}
\end{figure*}

This demonstration modifies an example from the literature.
Where~\cite{wgcf24} hard-coded the item moved by each state, our
approach uses enumerated values in \texttt{slugs} to enable more general parameterizations~\cite{flexbe-synthesis-demo}.
These enumerated values are then mapped via the capabilities file to
specific parameter values in the instantiated state machine.  Instead of simple \True/\False
\ to indicate item position and separate specifications for
\texttt{move\_goat}, \texttt{move\_wolf}, ... as used
in~\cite{wgcf24}, this approach uses a single \texttt{MoveState2R} capability and
numeric positions (e.g. 0=left bank, 1=middle, and 2=right bank) for the variables
(\texttt{wolf}, \texttt{goat}, \texttt{corn}, \texttt{farmer}, and \texttt{move\_dest}),
along with an enumerated \texttt{move\_item} to designate which item to move
 (e.g. 0=farmer alone, 1=wolf, 2=goat, 3=corn).
 \relist{two-rivers-config} shows an excerpt of the capability configuration
 file that maps these symbolic values to FlexBE state parameters.

\begin{lstlisting}[
float={*hbt},
basicstyle=\normalsize\ttfamily,
columns=fullflexible,
caption={Capability configuration excerpt for the Two-Rivers demonstration.
Only selected variable mappings and the \texttt{move\_state} capability are shown.
The \texttt{@} notation marks values supplied by the synthesized controller.},
label={lst:two-rivers-config}]
name: 'two_rivers_enumerated'
variable_mappings:
  # Use these to map enumerated binary
  # values to parameter values
  env_props:
    wolf:
      0: 'left'
      1: 'middle'
      2: 'right'
    goat:
      0: 'left'
      1: 'middle'
      2: 'right'

capabilities:
  move_state:
    interface: MoveState2R
    parameters:
      item: '@move_item'
      tracker: farmer
      destination: '@move_dest'
    transition_relation:
      # activate log on any failure
      failure: [log_failed]
    postconditions:
      ['@wolf', '@goat', '@corn', '@farmer']
\end{lstlisting}

The capability designates which FlexBE state implementation to use as an interface
and specifies parameter values.
The '@' symbol denotes retrieval of the mapped variable from the enumerated synthesis value returned by \texttt{slugs}
instead of a simple constant.
The transition relation in this example requires that any failures are logged.

Where~\cite{wgcf24} completely specified all GR(1) formulas for all FlexBE states,
our approach generates most of the specifications related
to FlexBE states and transitions, and only requires the user to specify the specific rules of the game.
In this case, the rules require that the goat
must never be left with either the wolf or the corn unless supervised by the farmer.

 \begin{figure*}[t]
   \centering
   \includegraphics[width=0.9\linewidth]{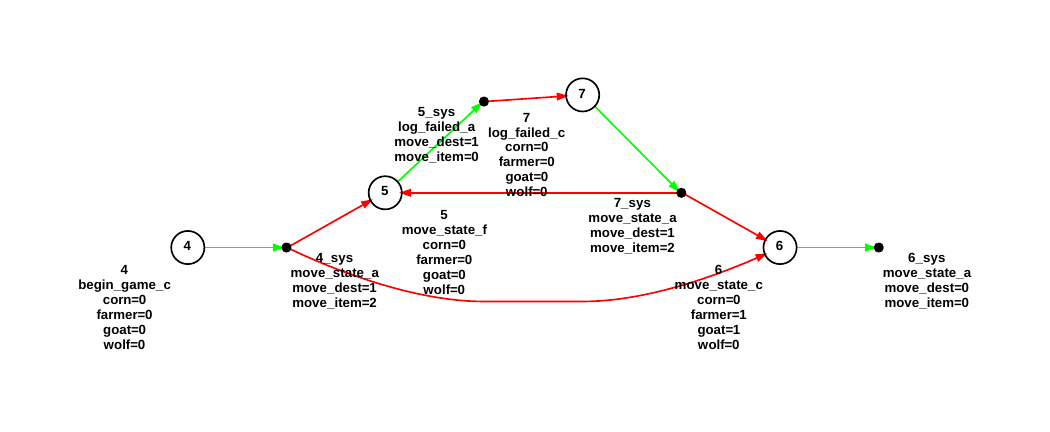}%
   \caption{Excerpt of the Two-Rivers symbolic state machine synthesized by the \texttt{slugs} tool,
 using the \texttt{mealy2dot} visual convention introduced in \refig{slugs-example-mealy}
 (\resec{slugs-example}). Large circles are environment-observation states, small dots are
 system-activation states, green edges are env$\to$sys transitions, and red edges are sys$\to$env
 transitions. Only \True\ and enumerated values are shown.
  The excerpt comes from the one-hot, System-Goal, no-pending configuration
  (\textbf{2S}) under the alphabetic declaration ordering.
  In this case, states \texttt{4\_sys} and \texttt{7\_sys} can be identified to reduce the state machine size.
  The entire file is available with our open source release.}
   \label{fig:two-rivers-mealy}
\end{figure*}

\refig{two-rivers-mealy} shows a portion of the representation
of the synthesized Mealy state machine for the one-hot System-Goal
configuration, which has 37 states.
As shown in \refig{two-rivers-mealy}, states \texttt{4\_sys} and \texttt{7\_sys} are equivalent
because they have the same output labeling and transitions. The state machine
can be reduced by identifying other equivalent states and
updating the transitions using our state reduction module in the pipeline.
Executing the FlexBE state machine, which had only 23 states after identifying the equivalent states
(as for every realizable capability-generated Two-Rivers configuration),
results in the safe crossing of both rivers.

\refig{two-rivers-flexbe} shows a screenshot of the synthesized
state machine executing in FlexBE along with the corresponding visualization.
A central feature of this approach is that the synthesized behavior
is inherently reactive rather than a fixed sequential plan.
In this demonstration, FlexBE is used to inject failures, for example by
emulating a crossing failure due to high wind as a forced failure outcome.
 The synthesized reactive controller successfully delivers all
 items across both rivers despite repeated, though not continual, failures.
 This experiment reproduces the core scenario presented in prior work~\cite{wgcf24},
 but within a simplified software framework and extended
 beyond simple binary propositions. The results further demonstrate
 the extraction of parameterized information from the synthesis process
 and enable direct execution of the synthesized state machine without
 manual modification once the capabilities and specifications are defined.
 \resec{discussion} discusses the comparison of different specification
 generation approaches enabled by the work presented in this paper.

\FloatBarrier

%% file: PyroboSimDemo.tex
\subsection{Logistics Demonstration}
\label{sec:pyrobosim}

The logistics demonstration provides a larger simulated domain built
on the PyRoboSim 2D simulator~\cite{Castro2022}.
The example is adapted from the ``Problem~1'' scenario from the ROSCon~2024
deliberation workshop and was used previously to validate the software pipeline
from an implementation perspective~\cite{eit26-synthesis}.
Here, we use it to illustrate how parameterized capabilities and custom
environment dynamics are represented before synthesis.
The world tracks six items (\texttt{bread0}, \texttt{snacks0}, \texttt{soda0},
\texttt{butter0}, \texttt{waste0}, \texttt{waste1}) across rooms including a
kitchen, dining room, closet, and office (Fig.~\ref{fig:pyrobosim-initial},
left). \texttt{butter0} begins inside the fridge and \texttt{bread0} inside the
pantry, and the hallway doors along the route the robot must travel are open.
The synthesis request asks the supervisor to retrieve \texttt{butter0} and
\texttt{bread0} from the kitchen and place them on the dining table, using the
unmodified FlexBE states provided in the workshop repository.

\begin{lstlisting}[
float={*hbt},
basicstyle=\normalsize\ttfamily,
columns=fullflexible,
caption={Capability configuration excerpt for the logistics demonstration
(\texttt{p1\_capabilities.yaml}). \texttt{grab\_item}'s \texttt{object\_name}
parameter is bound to \texttt{\textquotesingle @pick\textquotesingle}, an enumerated output the synthesized
controller sets to select which of the six tracked items to grab, rather
than a value fixed in the capability configuration.},
label={lst:pyrobosim-config}]
grab_item:
    interface: PickActionState
    parameters:
        robot_name: robot
        action_topic: /execute_action
        timeout: 2.5
        object_name: '@pick'
    transition_relation:
      failure:
        - log_grab_failed
place_item:
    interface: PlaceActionState
    parameters:
        robot_name: robot
        action_topic: /execute_action
        timeout: 2.5
    transition_relation:
      failure:
        - log_place_failed
\end{lstlisting}

We use two capability configurations of this domain, differing in whether
storage-door state is modeled. P0 abstracts doors away entirely. Its
capabilities are \texttt{move\_robot}, \texttt{grab\_item},
\texttt{place\_item}, and logging (\relist{pyrobosim-config}), and its
specification declares no door propositions, so the synthesized supervisor
never reasons about opening a storage location. P1 extends P0 with an
\texttt{open} capability and its own failure logging, and initializes the
fridge and pantry as closed, so the supervisor must open a container before it
can retrieve from it as shown in \refig{pyrobosim-initial}.
The P0 experiments reported below therefore do not
depend on the simulator's door configuration. P1 is used separately to
evaluate the added door-modeling capability.
The configuration maps each capability to a FlexBE state implementation and
binds symbolic parameters to controller outputs. \texttt{grab\_item}'s
\texttt{object\_name} parameter is bound to
\texttt{\textquotesingle @pick\textquotesingle} rather than a
fixed value, so which of the six tracked items to grab is chosen by the
synthesized controller through an enumerated output value.

Custom environment constraints encode the discrete task semantics that are not
contained in the generic capability template. These include the effects of
successful and failed motion and co-location requirements for grasping
objects. For example, \texttt{grab\_item} requires the robot and the target object to
share a location. In P1 they additionally include door-state conditions that
prevent retrieval until the relevant storage location has been opened. The
pipeline combines these domain-specific rules with the automatically generated
activation--outcome protocol and liveness constraints.

\begin{figure*}[tb]
  \centering
  \includegraphics[width=0.49\linewidth]{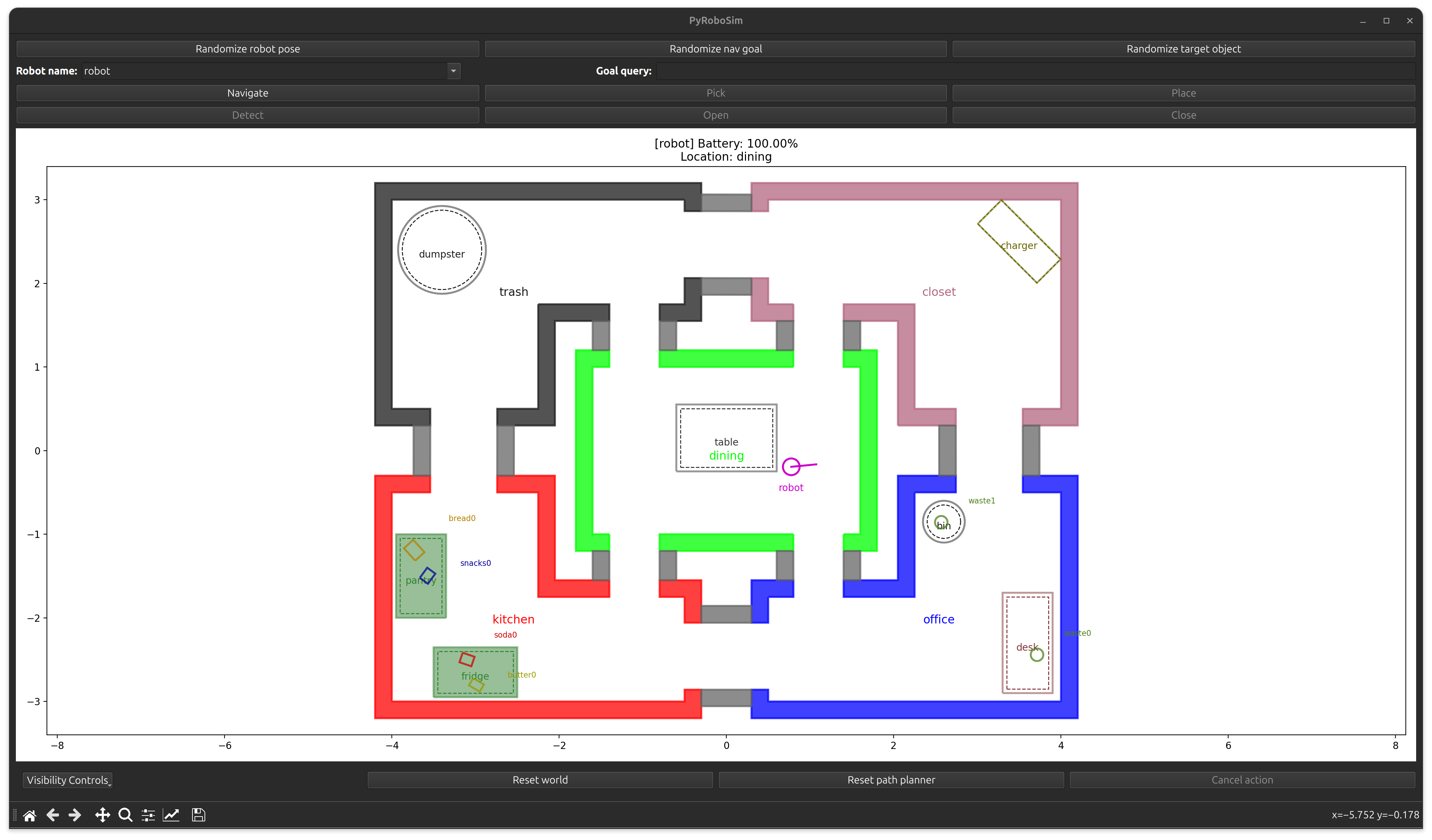}
  \includegraphics[width=0.49\linewidth]{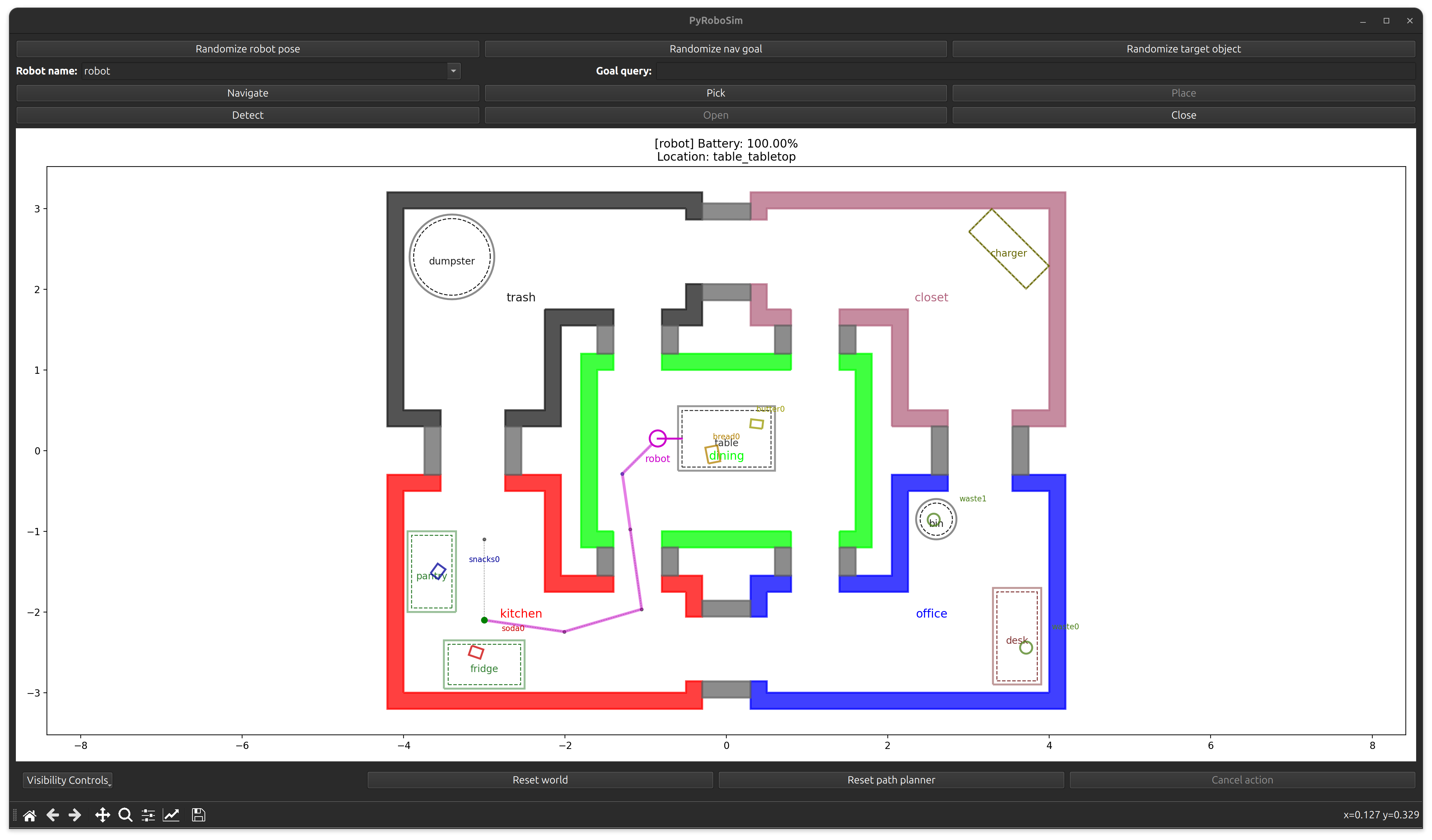}
  \caption{
  A simulated execution of the P1 configuration, which models
  storage-door state. In the initial state (left), \texttt{butter0} is in the
  fridge and \texttt{bread0} in the pantry, both closed, among the four other
  tracked items placed throughout the world. The synthesized controller opens
  both containers and delivers the two items to the dining table (right).
  The principal synthesis comparison of \resec{pyrobosim-p0-results} uses the
  P0 configuration, which abstracts doors away, and \resec{pyrobosim-p1-results}
  reports this door-modeling P1 configuration.
  }
  \label{fig:pyrobosim-initial}
\end{figure*}

\begin{figure*}[tb]
  \centering
  \includegraphics[width=0.8\linewidth, trim=7 0 5 0, clip]{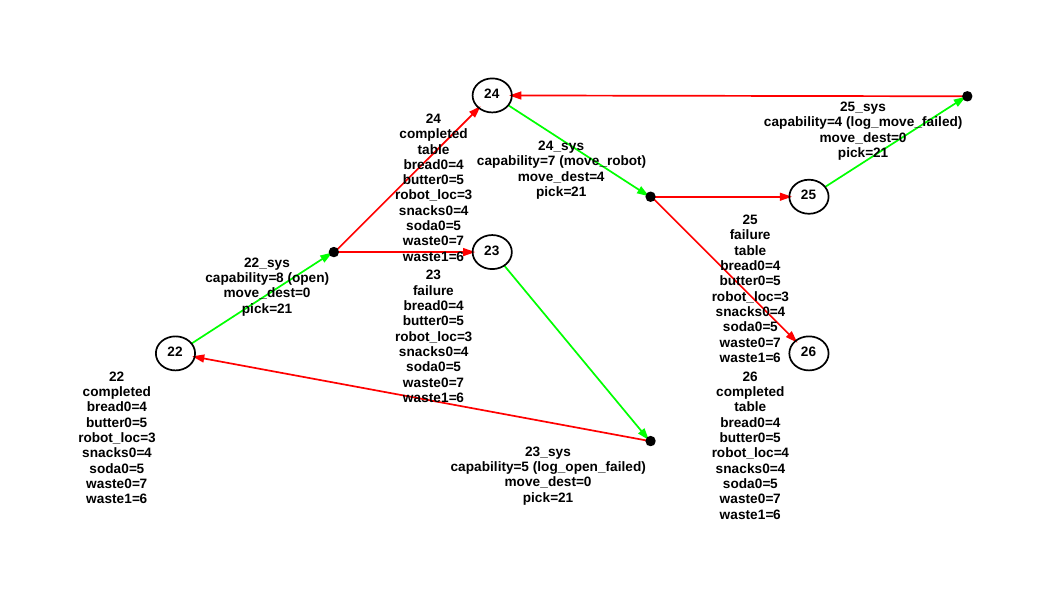}
  \caption{Excerpt of the PyRoboSim P1 symbolic state machine synthesized by
  the \texttt{slugs} tool under the enumerated capability encoding, prior to
  reduction, using the \texttt{mealy2dot} visual convention introduced in
  \refig{slugs-example-mealy} (\resec{slugs-example}). Capability activation
  is a single \texttt{capability} value under this encoding rather than a
  per-capability boolean, so each system state is annotated with the
  activated index and its name. Hallway propositions are omitted from the
  labels. States 22--26 show the \texttt{open}
  capability (\texttt{capability=8}) activated at state 22. On failure
  (23), \texttt{log\_open\_failed} (\texttt{capability=5}) retries by
  re-activating \texttt{open}. On success (24), the controller proceeds to
  \texttt{move\_robot} (\texttt{capability=7}), which itself retries through
  \texttt{log\_move\_failed} (\texttt{capability=4}) on failure before
  reaching state 26.}
  \label{fig:pyrobosim-mealy}
\end{figure*}

The resulting P1 GR(1) specification produces a 46-state Mealy machine.
A post-processing pipeline module performs state-merging reduction by
identifying equivalent states and updating transitions,
reducing the controller to 27 states; \refig{pyrobosim-mealy} shows an
excerpt of the raw strategy before this reduction.
The reduced strategy is realized as an HFSM and executed directly in FlexBE.
\refig{pyrobosim-initial} shows an execution of the corresponding P1
controller, which additionally opens the fridge and pantry before retrieving
the two items and delivering them to the table.

This demonstration provides a larger simulated task-level domain for exercising
the same capability abstraction and realization pipeline used in the other
examples. \resec{pyrobosim-results} reports quantitative synthesis results
for the P0 configuration across the same encoding/liveness sweep used in
the other case studies. \resec{pyrobosim-p1-results} extends this to the
door-modeling P1 configuration shown in \refig{pyrobosim-initial}.

\FloatBarrier

%% file: CrazyFlieDemo.tex
\subsection{Quadcopter Hardware-Integration Demonstration}
\label{sec:crazyflie}

The hardware case study incorporates the earlier pipeline demonstration on
a Bitcraze Crazyflie 2.1+ nano-quadcopter~\cite{eit26-synthesis}
and extends it to the larger ``PiHawk'' drone, using a PixHawk flight
controller~\cite{MeierPixhawk2011} and ArduCopter firmware~\cite{ArduPilot}.
The Crazyflie flights used the earlier capability package;
the final PiHawk flights used strategies checked by the auditor.
The purpose of this case study is to demonstrate integration of the
discrete-event supervisor with asynchronous ROS~2 capabilities and reported
execution outcomes.
Flight primitives are exposed as synthesizable capabilities by the
open-source ``Flexible Drones'' ROS~2 system, which provides trajectory
planning and real-time control through ROS~2 action interfaces and runtime
feedback~\cite{flexible-drones26, flexible-drones-git}.

\begin{figure*}[t]
    \centering
    \begin{minipage}[t]{0.370\textwidth}
        \centering
        \includegraphics[width=\linewidth,height=2.7in,keepaspectratio]{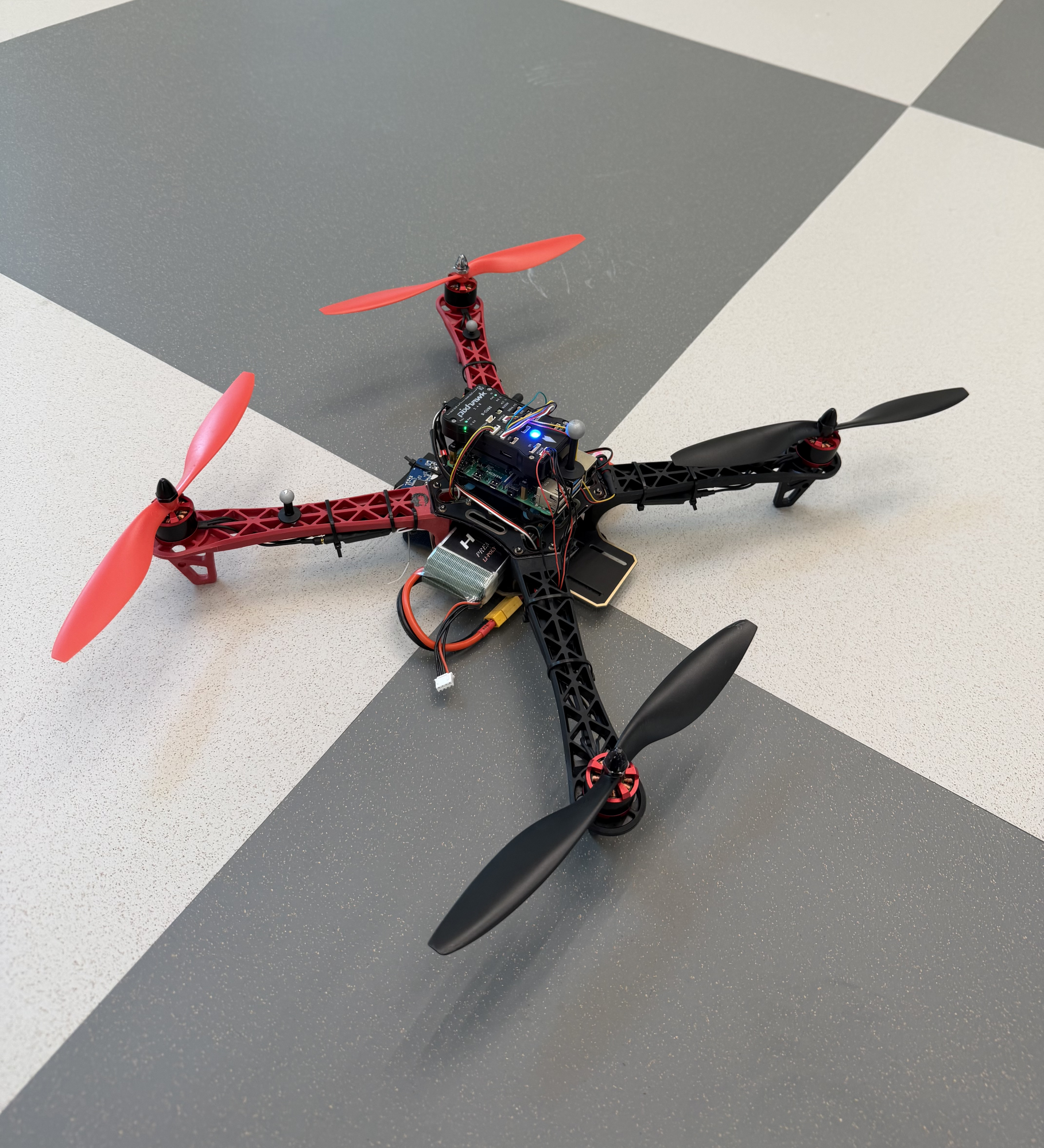}
        \caption{``PiHawk'' drone
        platform used in this case study.}
        \label{fig:pihawk-static}
    \end{minipage}\hfill
    \begin{minipage}[t]{0.580\textwidth}
        \centering
        \includegraphics[width=\linewidth,height=2.7in,keepaspectratio]{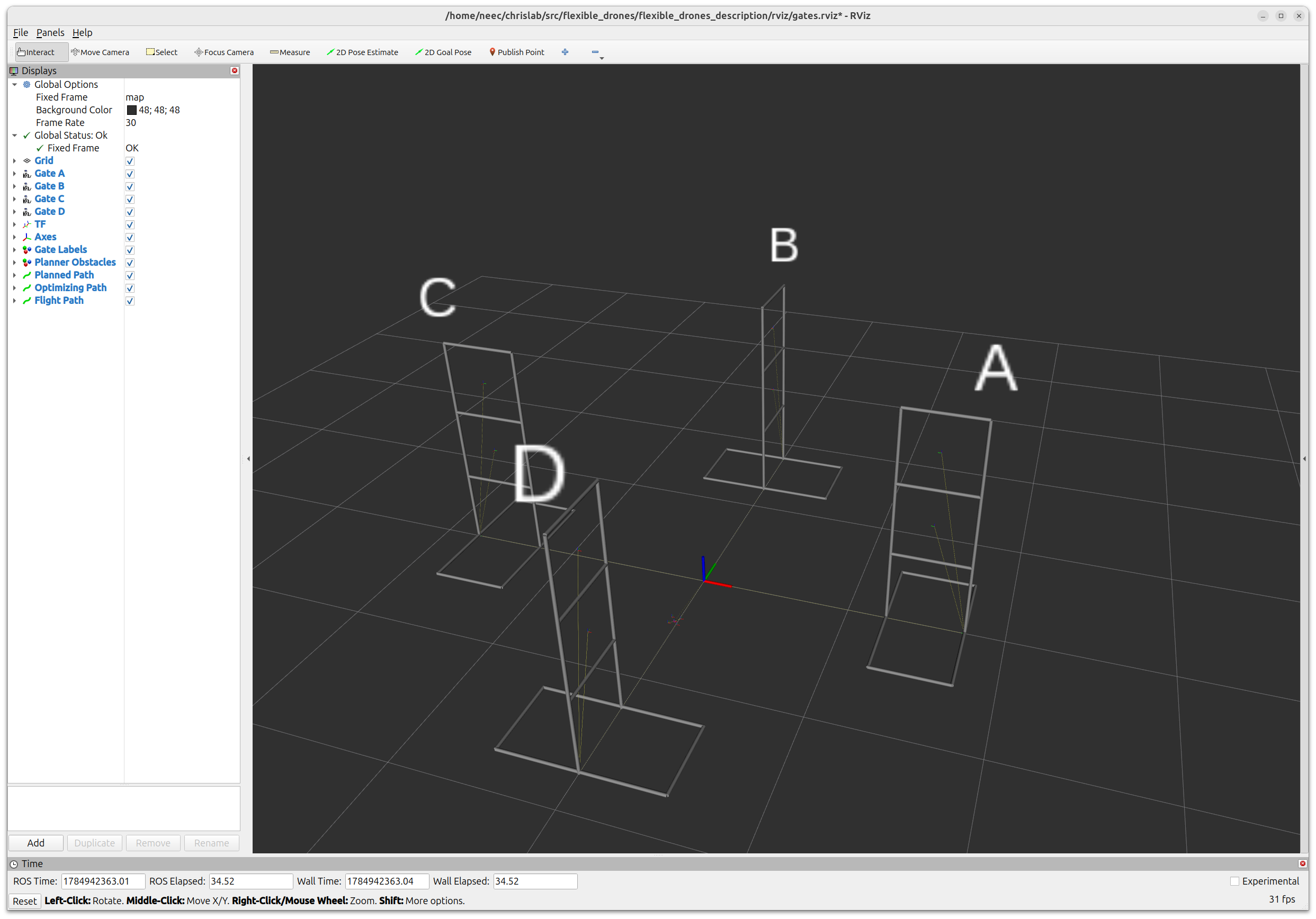}
        \caption{RViz screenshot showing the layout of labeled gates in our flight space.}
        \label{fig:gates-layout}
    \end{minipage}
\end{figure*}

The post-synthesis auditor (\resec{auditor}) checks the abstract strategy
for the defined structural defects. A clean safety verdict covers every
reachable product state and transition. We also forced capability-failure
outcomes in simulation and confirmed that the supervisor's response matched
the audited strategy in each case.

In nominal PiHawk flights, the supervisor completed the mission via the
all-success branch. During observed pose-tracking dropouts, the supervisor
followed the audited failure-handling branch
(\resec{crazyflie-failure-validation}). These flights provide
qualitative integration evidence; graph-level auditing does not establish
the correctness of middleware, timing, or physical execution.

Battery depletion, localization dropout, and safety-triggered forced landings
are environment-chosen outcomes in the sense of \resec{syn_spec}. The state
implementation reports them, and lower-level mechanisms bound how long a
physical failure can go undetected, since each FlexBE state enforces its own
execution timeout and the flight stack runs independent safety watchdogs, for
example on loss of external pose tracking or excessive trajectory deviation.
Synthesis guarantees that the supervisor responds correctly to whichever
outcome is reported. The reliability of the underlying capability, and whether
its fault-detection layer reports the correct outcome under degraded rather
than fully lost sensing, are properties of the capability implementation and
the flight stack rather than of the supervisor considered in this paper.

Four gates, labeled `A'-`D', each with two directional ``doors,''
 labeled `1' (bottom) and `2' (top),
 are defined within the
tracked workspace (see Figs.~\ref{fig:pihawk-flight}
and~\ref{fig:gates-layout}).
We consider defined
forward and reverse (`R') directions through each door for
sixteen traversal checkpoints.
The objective is to synthesize a supervisory controller that plans and
executes five sequential trajectory segments,
Home$\rightarrow$G1$\rightarrow$G2$\rightarrow$G3$\rightarrow$G4$\rightarrow$Home,
through four designated gates.
For example, a flight `\mbox{A1-B2R-D2-C2R}' would travel first forward through `A1',
then in reverse through `B2', and so forth, as shown in \refig{crazyflie-flexbe}.

The trajectory generation capability is implemented as a ROS~2 action server that computes
seventh-order polynomial segments with snap continuity~\cite{MellingerKumar11, flexible-drones26, flexible-drones-git}.
The optimization considers obstacle avoidance around the gate structure.
The synthesized high-level controller must request each segment generation,
handle potential planning failures,
compose segments into a consistent flight plan, and coordinate execution
under reactive assumptions about action outcomes.

Synchronization between the discrete supervisor and continuous flight
execution follows FlexBE's standard action-client pattern. On entering a
state, the FlexBE implementation sends the corresponding ROS~2 action goal
(e.g., segment generation or trajectory execution) and polls the action
client's status on each subsequent execution cycle rather than blocking. This
keeps the supervisor responsive to operator input (\refig{crazyflie-flexbe})
while the action runs. Feedback messages update internal state but do not
themselves trigger a transition. Only the action's terminal result, mapped
to the corresponding activation-outcome proposition (\resec{syn_spec}),
causes the state implementation to return an outcome and the synthesized
strategy to select the next state.
The full message-level sequencing between the supervisor and the underlying
ROS~2 action interfaces, including goal, feedback, and result exchanges
across the five-segment flight, is documented in the companion Flexible
Drones system description~\cite{flexible-drones26}. This paper's
contribution is the synthesis pipeline and the correctness properties
governing the discrete decisions made at each step, not the underlying
ROS~2 action architecture.

\begin{figure*}
    \centering
    \includegraphics[width=0.75\linewidth]{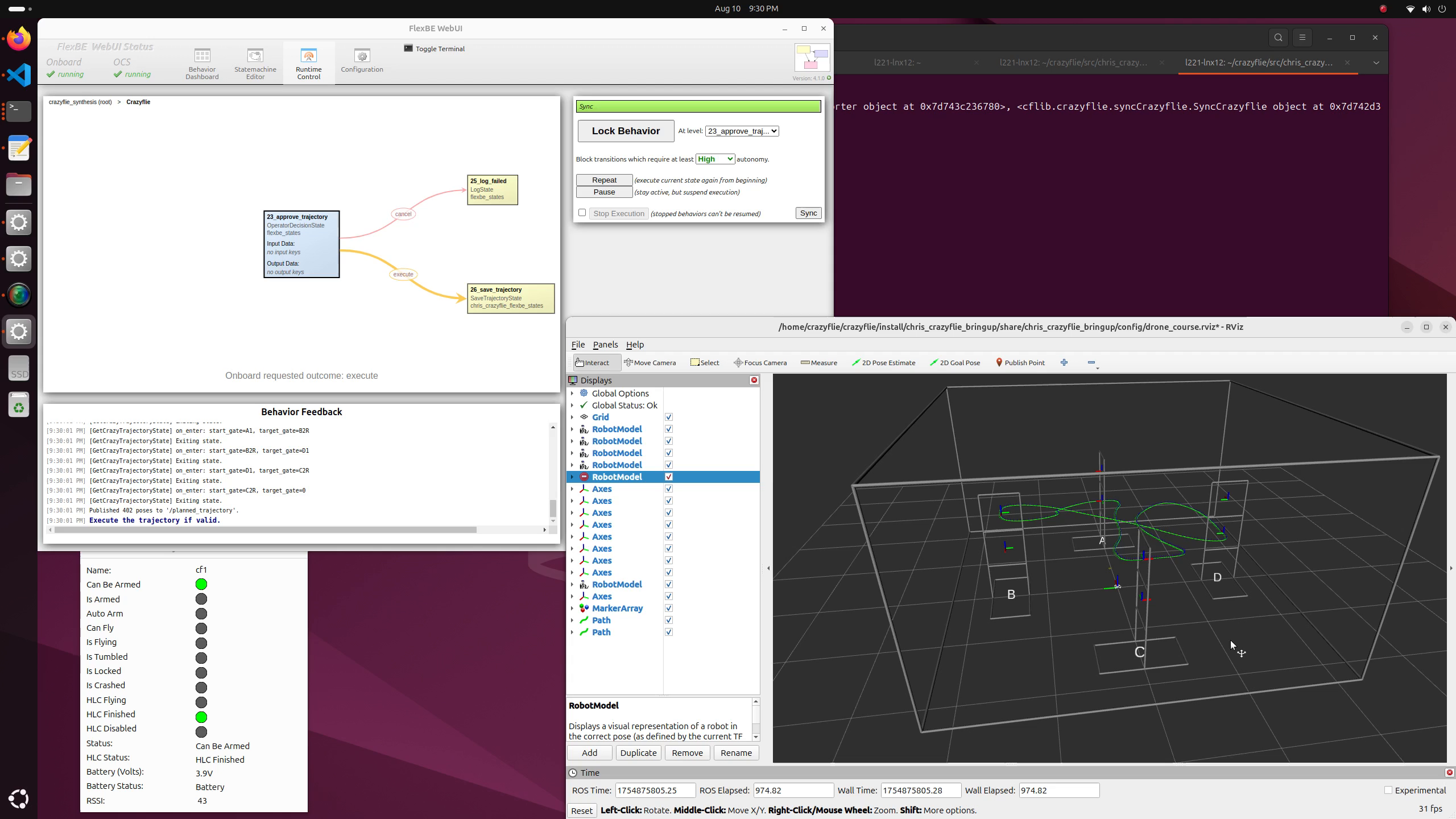}
    \includegraphics[width=0.75\linewidth]{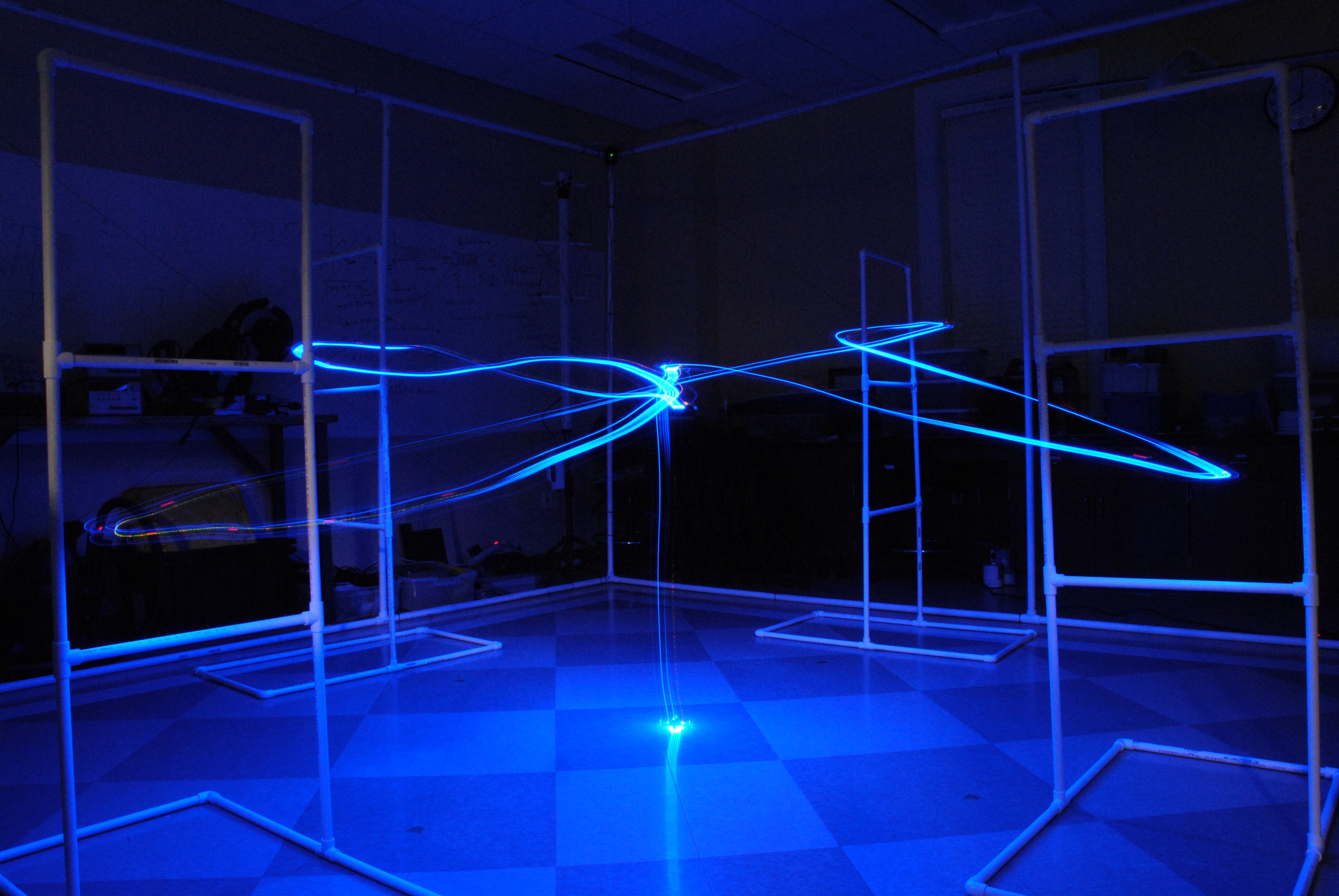}
    \caption{ The screenshot (top) shows FlexBE execution using collaborative autonomy
    to solicit operator approval before
    executing the planned trajectory on a Crazyflie nanocopter (bottom).
    Earlier pipeline demonstration reported in~\cite{eit26-synthesis}.}
    \label{fig:crazyflie-flexbe}
\end{figure*}

Prior to execution, the complete planned trajectory is published for operator approval in RViz,
providing a simple demonstration of FlexBE’s collaborative autonomy (\refig{crazyflie-flexbe}).
Upon approval, under nominal operation, the controller uploads the trajectory, arms the vehicle, performs takeoff,
executes the plan, returns home, lands, and disarms.
Capability-level failures are logged, and how the flight proceeds afterwards
depends on which capability failed. A pre-flight failure, such as segment
generation, trajectory upload, or arming, is logged and terminates the
behavior. An in-flight trajectory-execution failure sequences a
recovery landing instead, as shown in \relist{crazyflie-config}. A takeoff or
return-home failure defers to the operator, who may either confirm that the
vehicle is nonetheless airborne, or at the route start, and continue, or
abort.

\begin{lstlisting}[
float={*ht},
basicstyle=\normalsize\ttfamily,
columns=fullflexible,
caption={Capability configuration excerpt for the hardware demonstration
(\texttt{drones\_gates.yaml}). An in-flight \texttt{execute\_trajectory}
failure forces a transition to \texttt{log\_execution\_failed}, whose
postcondition \texttt{execution\_failure\_logged} is in turn a precondition of
\texttt{recover\_land}. The recovery landing is therefore sequenced by the
specification itself, rather than by handler logic inside any one state
implementation.},
label={lst:crazyflie-config}]
execute_trajectory:
    interface: ExecuteTrajectoryState
    parameters:
        timeout: 30.0
    preconditions:
        - flying
        - at_home
        - uploaded
        - '!execution_complete'
    postconditions:
        completed:
            - execution_complete
    transition_relation:
        failure:
            - log_execution_failed
log_execution_failed:
    interface: LogState
    parameters:
        text: Trajectory execution failed
    preconditions:
        - execute_trajectory_f
    postconditions:
        - execution_failure_logged
recover_land:
    interface: LandState
    preconditions:
        - flying
        - execution_failure_logged
    postconditions:
        - '!flying'
        - recovery_landed
\end{lstlisting}

This experiment, and the one pictured in \refig{pihawk-flight},
 highlights the nontrivial aspects of synthesis in hardware:
environment-controlled action outcomes (e.g., planner success/failure),
liveness assumptions that affect realizability, and
tight coupling between symbolic decision-making and real-time flight control.
Both Figs.~\ref{fig:pihawk-flight} and~\ref{fig:crazyflie-flexbe} show
interfacing with system capabilities and
 autonomous execution under supervision of the synthesized state machine.
Quantitative comparisons across encoding and liveness
variants are presented in \resec{discussion}.

\FloatBarrier

%% file: Pipeline.tex
\section{Synthesis Pipeline}
\label{sec:pipeline}

\resec{pipeline-overview} introduced this pipeline's stages and the
capability vocabulary needed to read the demonstrations
(\resec{demonstrations}). This section follows the stages shown in
\refig{pipeline-overview}: preprocessing, capability-based specification
generation, specification analysis, the \texttt{slugs} synthesis wrapper
that produces the raw strategy audited in \resec{auditor}, and controller
realization after auditing and reduction.

The \texttt{synthesis\_manager} is a ROS~2 node that interfaces with the FlexBE WebUI
via a ROS~2 action interface
and orchestrates the pipeline stages introduced in \resec{overview-stages},
operating over the capability and capability-configuration-file vocabulary
defined in \resec{overview-capability}.
\label{sec:architecture}

Named pipeline modules are declared in YAML with explicit input-output
interfaces, enabling validation and extensibility. Users can implement
custom preprocessing or synthesis modules without modifying the core
manager. The implementation is available in the repository~\cite{flexbe-synthesis}.
\relist{preprocess-module} and \relist{process-module} show
representative examples of each. Both are excerpted from the Coffee
demonstration's pipeline configuration and use the same named, typed
input/output convention.

\begin{lstlisting}[
float={*t},
basicstyle=\normalsize\ttfamily,
columns=fullflexible,
caption={A preprocessing module declaration (\resec{pipeline_config}): the
\texttt{generate\_discrete\_abstraction} stage of the preprocessing chain, with its
typed input and output interface.},
label={lst:preprocess-module}]
- entry_point:
    name: generate_discrete_abstraction
    inputs:
        system_capabilities: SystemCapabilities
        state_implementations_used: StateImplementationsUsed
        behaviors_used: BehaviorsUsed
    outputs:
        discrete_abstraction: DiscreteAbstraction
\end{lstlisting}

\begin{lstlisting}[
float={*t},
basicstyle=\normalsize\ttfamily,
columns=fullflexible,
caption={A request-time processor module declaration from the same
pipeline: the core \texttt{slugs} synthesis stage
(\resec{synthesis-module}), showing the same typed interface
convention applied to a later pipeline stage.},
label={lst:process-module}]
- entry_point:
    name: slugs_synthesizer
    inputs:
        specs_output_dir_path: FilePathStr
        spec_name: NameStr
        slugs_timeout_s: float
    outputs:
        realizability_result: dict
        raw_strategy: MealyStrategy
        error_code: SynthesisErrorCode
\end{lstlisting}

\subsection{Preprocessing}
\label{sec:pipeline_config}

At startup, the manager inspects the ROS workspace for known FlexBE state
implementations and combines those discovered interfaces with the
demonstration-specific capability configuration files. Capability configuration
(\resec{overview-capability}) is domain-specific. It binds a named capability
to a FlexBE state implementation and may define its preconditions,
postconditions, transition relations, fixed parameter values, and symbolic
parameter names, as illustrated by the Coffee and Two-Rivers excerpts in
\relist{coffee-config} and \relist{two-rivers-config}. Pipeline configuration, by
contrast, specifies which preprocessing and request-time modules run and the
input/output data each stage declares.

This division is important for the examples in \resec{demonstrations}.
The Coffee and Two-Rivers examples use compact abstractions to isolate encoding
and liveness effects, the PyRoboSim example adds custom symbolic environment
dynamics, and the quadcopter hardware demonstrations bind the same synthesis interface to
physical robot execution.
In each case, the pipeline mechanics remain the same while the capability
catalog and domain-specific specifications determine the concrete supervisor
that is synthesized.

One such preprocessing module, \texttt{generate\_discrete\_abstraction},
compiles each capability configuration into the per-capability form that
specification generation (\resec{syn_spec}) consumes directly. This form
contains a FlexBE state class declaration and an explicit mapping from FlexBE outcomes to
capability outcome propositions. \relist{discrete-abstraction} shows this
compiled form for the Coffee demonstration's \texttt{gr} (grind) capability,
compiled from the configuration in \relist{coffee-config}. Its
\texttt{state\_outcome\_mapping} records exactly the \texttt{gr\_c}/\texttt{gr\_f}
outcome propositions used by the one-hot encoding of \resec{one-hot-encoding}.
The same compiled block is consumed again during controller realization;
\texttt{class\_decl} determines which FlexBE state implementation is
instantiated, while \texttt{state\_outcome\_mapping} deterministically
normalizes each concrete state outcome before the generated HFSM transition is
wired to the corresponding strategy successor. In
\relist{discrete-abstraction}, for example, \texttt{grind} maps to
\texttt{gr\_c} and \texttt{fail} maps to \texttt{gr\_f}. Other domains use
their own concrete outcome names, such as \texttt{done}, \texttt{true},
\texttt{failed}, or \texttt{timeout}, but the realization step follows the
same shared mapping rather than relying on a hand-authored transition table.

\begin{lstlisting}[
float={*t},
basicstyle=\normalsize\ttfamily,
columns=fullflexible,
caption={The compiled discrete-abstraction excerpt for the Coffee
demonstration's \texttt{gr} capability
(\texttt{coffee\_maker\_discrete\_abstraction.yaml}), generated by
preprocessing from the configuration in \relist{coffee-config}.
\texttt{class\_decl} names the bound FlexBE state and its parameters.
\texttt{state\_outcome\_mapping} records the per-capability outcome
propositions consumed by specification generation.},
label={lst:discrete-abstraction}]
gr_a:
  class_decl:
    name: OperatorDecisionState
    parameters:
      outcomes:
      - grind
      - fail
      hint: Grind
      suggestion: grind
  state_outcome_mapping:
    gr_c:
    - grind
    gr_f:
    - fail
  autonomy:
    gr_c: 2
    gr_f: 3
\end{lstlisting}

\subsection{Capability-Based Specification Generation}
\label{sec:syn_spec}
\label{sec:specification-encoding}

At request time, the processing pipeline modules load the relevant
capabilities from the preprocessed catalog, instantiate the propositions
needed for the request's initial conditions, goal, and selected encoding, and
emit the \texttt{structuredslugs} input file.
Because capability configurations can expose both
fixed parameter values and symbolic parameter names, the generated specification
can let the controller choose a capability and selected parameter values within
the same synthesis problem, rather than relying only on fixed state
instantiations.

As discussed in \resec{overview-capability}, each system capability is
accessed through a FlexBE state implementation. This mapping is not necessarily
one-to-one: multiple FlexBE state implementations could interface with a given system
capability, or a single state implementation can interface with multiple
capabilities providing the same ROS~2 interface.
We use the term ``capability'' for either the underlying system capability
or the FlexBE state implementation/system capability pair.

\subsubsection{Specifications}
\label{sec:specifications}

\noindent\textbf{Executable Capabilities and Outcomes.}
Let $\mathcal C$ denote the set of executable capability pairings.
For each $\pi\in\mathcal C$, let $O(\pi)\subseteq AP_I$ denote its
nonempty set of declared outcome propositions, where $AP_I$ is the set
of environment input propositions and $AP_O$ is the set of system output
propositions.
For each capability $\pi \in \mathcal{C}$:
\begin{itemize}
\item The system controls its activation $\pi_a \in AP_O$
\item The environment player determines the reported outcome from $O(\pi)$
\end{itemize}
In other words, the synthesized high-level controller will choose to activate
an instance of a FlexBE state implementation that will engage
the system capability, and the result of that engagement will determine
the outcome of that FlexBE state as an environmental input ($AP_I$).

\phantomsection\label{sec:terminal-outcomes}
\noindent\textbf{Terminal Behavior Labels.}
Let $\Theta$ denote the set of terminal behavior labels, such as
\texttt{finished} and \texttt{failed}, disjoint from $\mathcal C$.
For each $\theta\in\Theta$, $\theta_a\in AP_O$ denotes its
system-controlled terminal flag (written simply \texttt{finished} or
\texttt{failed} in the generated specification). Terminal labels have
no environment outcome tokens; they denote termination of the supervisor
rather than execution of a capability.
Terminal flags are persistent: once asserted, they remain true at every
subsequent strategy step. For each $\theta\in\Theta$, the persistence
constraint $G(\theta_a\rightarrow X\theta_a)$ is included in the system
safety guarantees $\varphi_s^g$ of \reqn{varphi} and emitted in the
\texttt{SYS\_TRANS} block of the generated \texttt{structuredslugs} input
file; for example,
$G(\mathit{finished}\rightarrow X\mathit{finished})$ is emitted as
\texttt{finished -> finished\textquotesingle}.
Persistence fixes that flag, not necessarily every other strategy-state label.
In one-hot, \reqn{sys-guarantees} excludes further capability activations
after a terminal flag is asserted; enumerated strategies retain an
action-selector value as symbolic terminal context (\resec{enumerated}).
For the quadcopter, $\Theta=\{\mathit{finished},\mathit{failed}\}$;
Two-Rivers and PyRoboSim use $\Theta=\{\mathit{finished}\}$.
Coffee (\resec{coffee}) runs cyclically without a terminal output,
so $\Theta=\emptyset$.
The auditor (\resec{auditor}) also supports an environment-reported
completion input in place of a terminal system output.

We adopt an activation-outcome abstraction similar to~\cite{Maniatopoulos2016},
but modify the formulation to support more general specifications.
Each activation is followed in the next discrete transition by exactly
one of the capability's declared outcomes (e.g., \texttt{completed} or
\texttt{failure}), so a discrete step corresponds to an
HFSM transition rather than a fixed interval of elapsed time, abstracting away
execution time within a state.
Execution-time limits are therefore enforced by the underlying system
capability implementation or timeouts within the FlexBE state implementation
itself. The supervisor sees only the reported capability outcome; the model requires
an eventual outcome for each activated FlexBE state.

\paragraph{Outcome Vocabulary}
\label{sec:outcome-vocabulary}
Nothing in $O(\pi)$ restricts a capability to two outcomes. As
in the activation-outcome formulation of~\cite{Maniatopoulos2016}
that we follow here, a capability may in general report any number of
environment-determined outcomes. Every demonstration in this paper uses a
two-outcome global vocabulary, $\{\texttt{completed},
\texttt{failure}\}$, for exposition. This is a modeling choice, not a
limitation of the safety assumptions or of the encodings in \resec{encodings}.
A capability configuration's raw FlexBE outcomes are first normalized onto
this global vocabulary. \relist{discrete-abstraction} shows the corresponding
capability-specific tokens for \texttt{gr}: \texttt{grind} maps to the
completion token \texttt{gr\_c}, and \texttt{fail} maps to the failure
token \texttt{gr\_f}. Representative examples of this normalization
elsewhere in our configurations include:
\begin{itemize}
\item \texttt{done}, \texttt{finished}, \texttt{received}, \texttt{true} $\to$ \texttt{completed}
\item \texttt{aborted}, \texttt{canceled}, \texttt{failed}, \texttt{timeout}, \texttt{unavailable} $\to$ \texttt{failure}
\end{itemize}
The pipeline's only requirement on the global vocabulary itself is that its
members be distinguishable by first letter. The one-hot and enumerated
encodings (\resec{encodings}) tag each outcome proposition with that
character (e.g., \texttt{gr\_c}, \texttt{gr\_f}), and preprocessing rejects a
vocabulary in which two outcomes share one, such as \texttt{completed} and
\texttt{canceled} both tagging \texttt{c}. A third outcome, e.g.
\texttt{offline} $\to$ \texttt{o}, can therefore be added to the global
vocabulary without changing underlying specification equations described below.
We simply did not need more than $\{\texttt{c},
\texttt{f}\}$ for the demonstrations in this paper.

\paragraph{System Safety Guarantees}
For the one-hot encoding, the system safety guarantees, $\varphi_s^g$,
select exactly one executable capability or terminal label at each
strategy step:
\begin{subequations}
\label{eqn:sys-guarantees}
\begin{align}
G\bigvee_{\pi \in \mathcal C\cup\Theta} \pi_a\,,\label{eqn:sys-selection}\\
\bigwedge_{\{\pi,\pi'\}\subset \mathcal C\cup\Theta}
G\neg\left(\pi_a \land \pi'_a\right)\,.\label{eqn:sys-mutex}
\end{align}
\end{subequations}
The first disjunction, which is written as a single safety specification in
the \texttt{SYS\_TRANS} block generated for \texttt{slugs},
rules out idle supervisor steps where no FlexBE state is active while the behavior is running.
The second conjunction rules out
concurrent activations, including activation alongside a terminal flag,
and is written as multiple pairwise statements. The enumerated encoding's
selection rules are distinguished in \resec{enumerated}.
Because this paper targets sequential FlexBE state execution, these
formulas encode the one-active-state abstraction used by the synthesized
supervisor. Future plans include synthesis with concurrent state activations
using FlexBE's \texttt{ConcurrencyContainer}, and more general hierarchical constructs.

\paragraph{Environment Safety Assumptions}
As mentioned above, the corresponding
environment safety assumptions, $\varphi_s^a$, apply only to executable
capabilities in $\mathcal C$, excluding $\Theta$. They require an activated
capability to produce exactly one declared outcome in the next strategy
step, while ruling out spurious outcomes for inactive capabilities:
\begin{subequations}
\label{eqn:env-assumptions}
\begin{align}
\bigwedge_{\pi \in \mathcal{C}}
G\!\left(\pi_a \rightarrow X\!\bigvee_{o \in O(\pi)} o\right)\,,\\
\bigwedge_{\pi\in\mathcal C}
\left(\bigwedge_{\{o,o'\}\subseteq O(\pi)}
G\!\left(\neg \big(Xo\land Xo'\big)\right)\right)\,,\\
\bigwedge_{\pi \in \mathcal{C}}
\left(\bigwedge_{o \in O(\pi)}
G\!\left(\neg \pi_a \rightarrow \neg Xo\right)\right).
\end{align}
\end{subequations}
Each \emph{Globally}, $G\left(\cdot\right)$, term is written as a separate specification in
the \texttt{ENV\_TRANS} block generated for \texttt{slugs}.
These generated constraints restrict only what can happen within a single
transition step. Unlike the activation--outcome encoding of~\cite{Maniatopoulos2016},
which requires an activation proposition to be false in the next step after an outcome is returned, our
activation--outcome constraints permit the supervisor to reactivate the same capability on consecutive steps,
which supports retry behaviors and parameter-dependent restrictions.

The activation-outcome constraints above can be compiled using different
Boolean encodings before they are passed to \texttt{slugs}.
In \resec{encodings}, this paper presents the
one-hot capability encoding used in prior capability-based synthesis
pipelines~\cite{wgcf24, eit26-synthesis},
and adds an enumerated encoding by applying the existing
bounded-integer support in \texttt{slugs}~\cite{Ehlers2016} to
capability activations and outcomes.
Both encodings represent the
same high-level interface from \resec{overview-capability}. The supervisor
chooses a capability, supplies any symbolic parameters, and then observes one
declared outcome from the underlying FlexBE state.

\paragraph{On-Demand Propositions}
\label{sec:on-demand-spec}

To limit the number of atomic propositions, two kinds of auxiliary
propositions are introduced only when a capability configuration or the
selected liveness formulation actually needs them, rather than declaring the
full set up front.

\emph{Memory propositions} ($\pi^m$) are free auxiliary variables.
Typically these are declared
directly in a capability's configuration as pre- and post-conditions for
a particular capability. In this case, the variables are controlled by the system such that $\pi^m \in AP_O$
each with an explicit initial
condition.  In our examples, the capability's own precondition and postcondition formulas set
and clear them, so a memory proposition lets a demonstration author encode
arbitrary persistent context, e.g., whether a precondition-relevant
capability has previously completed. See \relist{coffee-config} for an example.
The system supports defining auxiliary environment variables,
$x^m \in AP_I$, or system variables $y^m \in AP_O$.
The pipeline does not generate the transition dynamics
automatically in this case, so they must be loaded as additional specifications.
In this way, the pipeline does not commit to one
fixed semantics for what ``memory'' means.

\emph{Pending propositions} ($\pi_a^p \in AP_O$), by contrast, are optionally generated
automatically with fixed dynamics whenever both a completed and a
non-completed outcome are declared for $\pi$. The proposition $\pi_a^p$ is set when
$\pi$ is activated, held while $\pi$ continues to report a non-completed
outcome, and cleared on completion unless $\pi$ is reactivated in the same
step, with $\pi_a^p$ initially \texttt{false}:
\begin{equation}
\begin{aligned}
&G\big(X\pi_a \rightarrow X\pi_a^p\big) \;\land\\
&G\Big(\pi_a \land X\!\left(\bigvee_{o \in O(\pi)\setminus\{c\}} o\right) \rightarrow X\pi_a^p\Big) \;\land\\
&G\big(\neg X\pi_a \land \pi_a \land X\pi_a^c \rightarrow \neg X\pi_a^p\big) \;\land\\
&G\big(\neg(X\pi_a \lor \pi_a) \rightarrow (X\pi_a^p \leftrightarrow \pi_a^p)\big)
\end{aligned}
\label{eqn:pending}
\end{equation}
where $c$ denotes $\pi$'s \texttt{completed} outcome.
\retab{pending-trace} works \reqn{pending} through a concrete six-step
trace for the Coffee \texttt{gr} capability under the one-hot encoding, in
which each row exercises a different one of the equation's four
conjuncts. Activation ($\tau{=}1$) sets \texttt{gr\_p} via the first
conjunct. A failure outcome with no reactivation ($\tau{=}2$) holds it via
the second. A retry ($\tau{=}3$) holds it again via the first. A completion
outcome with no reactivation ($\tau{=}4$) clears it via the third. The idle
step ($\tau{=}5$) leaves it cleared via the fourth (inertia) conjunct.

\begin{table}[ht]
\centering
\caption{A six-step \texttt{gr\_a}/\texttt{gr\_c}/\texttt{gr\_f}/\texttt{gr\_p}
trace of \reqn{pending} for the Coffee \texttt{gr} capability. Each
\texttt{gr\_p} value is annotated with the conjunct of \reqn{pending} that
determines it.}
\label{tbl:pending-trace}
\begin{tabular}{cllll}
\toprule
$\tau$ & \texttt{gr\_a} & outcome observed & \texttt{gr\_p} & rule \\
\midrule
0 & F & --- & F & (initial) \\
1 & T (activate) & --- & T & conjunct 1 \\
2 & F & \texttt{gr\_f} (failure) & T & conjunct 2 \\
3 & T (retry) & --- & T & conjunct 1 \\
4 & F & \texttt{gr\_c} (completed) & F & conjunct 3 \\
5 & F & --- & F & conjunct 4 \\
\bottomrule
\end{tabular}
\end{table}

Currently the pipeline optionally enables the pending variables,
but they are created for all capabilities with failure outcomes if enabled.

Both kinds of auxiliary propositions reduce $|AP|$ relative to declaring the
full set up front, since each only appears for the capabilities whose
configuration or liveness formulation actually requires it.

\subsubsection{Encodings}
\label{sec:encodings}

This paper investigates two potential encodings for the specifications.

\paragraph{One-Hot Capability Encoding}
\label{sec:one-hot-encoding}

In the one-hot encoding, each capability activation and each capability-specific
outcome is represented by a separate Boolean
proposition~\cite{Maniatopoulos2016, Hayhurst2018, wgcf24, eit26-synthesis}.
Terminal labels have separate Boolean flags. The activation and terminal
outputs scale with $|\mathcal C|+|\Theta|$ and use the explicit mutual
exclusion constraints in \reqn{sys-mutex}.

For the extended
Coffee demonstration's \texttt{gr} capability, the system output contains an
activation bit \texttt{gr\_a}. The environment inputs include distinct outcome
bits \texttt{gr\_c} and \texttt{gr\_f}, corresponding to the \texttt{completed}
and \texttt{failure} outcomes in \resec{overview-capability}. If a capability's
outcome instead depends on a symbolic parameter, an additional
system-controlled proposition records the requested value, as with the
Two-Rivers \texttt{move\_state} capability's \texttt{move\_dest} parameter
(\resec{two-rivers}). With this encoding, \reqn{env-assumptions} requires
every \texttt{gr\_a} step to be followed by exactly one of the two outcome
propositions, and prevents either outcome from appearing after a step that did
not activate \texttt{gr}.

This representation is direct and easy to inspect in the generated Mealy
machine. A state labeled with \texttt{gr\_a} activates the grind capability,
and a successor labeled with \texttt{gr\_c} or \texttt{gr\_f} records the
environment's response. The cost is that
the encoding introduces one activation proposition per capability, one outcome
proposition per capability-outcome pair, and explicit mutual-exclusion
constraints to ensure that exactly one capability is active and at most one
outcome is reported for the activated capability.

When a capability also carries a pending proposition (\resec{on-demand-spec}),
the one-hot outcome propositions above are exactly the $O(\pi)\setminus\{c\}$
and $c$ terms \reqn{pending} refers to. \relist{pending-trans-example} shows
the generated \texttt{SYS\_TRANS} form of \reqn{pending} for the extended
Coffee demonstration's \texttt{gr} capability under this encoding.

\begin{lstlisting}[
float={*t},
basicstyle=\normalsize\ttfamily,
columns=fullflexible,
caption={Pending-proposition dynamics from the extended Coffee
demonstration's generated \texttt{CoffeeSM.structuredslugs} (\resec{coffee}):
the \texttt{gr\_p} initial condition and its four auto-generated
\texttt{SYS\_TRANS} rules, the one-hot instantiation of \reqn{pending}.},
label={lst:pending-trans-example}]
[SYS_INIT]
!gr_p                              # pending IC

[SYS_TRANS]
# set when selected
(gr_a') -> gr_p'

 # keep/set on non-completed outcome
((gr_a) & (gr_f')) -> gr_p'

# clear on completion unless re-selected
((!(gr_a')) & ((gr_a) & (gr_c'))) -> !gr_p'

 # pending inertia otherwise
!((gr_a) | (gr_a')) -> (gr_p' <-> gr_p)
\end{lstlisting}

\paragraph{Enumerated Capability Encoding}
\label{sec:enumerated}

We apply the existing bounded-integer support in
\texttt{slugs}~\cite{Ehlers2016} to form an
\emph{enumerated capability encoding}\footnote{Configuration names use
\texttt{parsed} (short for \texttt{parsed-binary}), as in \texttt{p0\_parsed}.},
with three effects:
\begin{itemize}
    \item Replaces executable activation bits with $\lceil\log_2 N\rceil$
    selector bits: $N=|\mathcal C|$ when an explicit startup capability
    such as \texttt{begin\_game} occupies slot 0, and
    $N=|\mathcal C|+1$ when a synthetic null slot is added.
    Terminal flags, parameters, and memory remain separate outputs.
    \item {Uses shared outcome propositions (e.g., \texttt{completed}, \texttt{failure})
    that replace per-capability outcome variables in $AP_I$.}
    \item {Makes executable capability selections mutually exclusive by
    the numeric selector, without pairwise activation constraints.}
\end{itemize}

For example, the one-hot activation proposition \texttt{gr\_a} is encoded
as the shared numeric \texttt{capability} variable taking a specific value for
\texttt{gr}. Symbolic parameters can be carried the same way as in the
one-hot encoding, as with the Two-Rivers \texttt{move\_dest} parameter
(\resec{two-rivers}). The encoding also replaces capability-specific outcome
inputs such as \texttt{gr\_c} and \texttt{gr\_f} with shared outcome values
such as \texttt{completed} and \texttt{failure}. The previous
\texttt{capability} value identifies which active state those shared outcomes
belong to.

When \texttt{gr} also carries a pending proposition, the enumerated encoding
generates the same four \reqn{pending} rules, but over the shared
\texttt{capability} selection variable and shared outcome propositions in
place of the one-hot \texttt{gr\_a}/\texttt{gr\_c}/\texttt{gr\_f}.
\relist{pending-trans-enumerated-example} shows the generated form for the
same extended Coffee \texttt{gr} capability under this encoding, with
\texttt{gr} bound to \texttt{capability=3}.

\begin{lstlisting}[
float={*t},
basicstyle=\footnotesize\ttfamily,
columns=fullflexible,
caption={Pending-proposition dynamics from the extended Coffee
demonstration's generated \texttt{CoffeeSM.structuredslugs}
(\resec{coffee}) under the enumerated encoding: the \texttt{[OUTPUT]}
capability-number mapping, the \texttt{gr\_p} initial condition, and its
four auto-generated \texttt{SYS\_TRANS} rules, the enumerated instantiation
of \reqn{pending}. Per the mapping, \texttt{gr} is bound to
\texttt{capability=3}, and \texttt{completed}/\texttt{failure} are the
shared outcome propositions gated by that binding, in place of the one-hot
\texttt{gr\_c}/\texttt{gr\_f} in \relist{pending-trans-example}.},
label={lst:pending-trans-enumerated-example}]
[OUTPUT]
capability:0...3
# 0: null (startup slot)
# 1: bd
# 2: br
# 3: gr

[SYS_INIT]
!gr_p                              # pending IC

[SYS_TRANS]
((capability'=3)) -> gr_p'                                  # set when selected
(((capability=3)) & (((capability=3) & failure'))) -> gr_p' # keep/set on non-completed outcome
# clear on completion
((!((capability'=3))) & (((capability=3)) & (((capability=3) & completed')))) -> !gr_p'
!(((capability=3)) | ((capability'=3))) -> (gr_p' <-> gr_p) # pending inertia otherwise
\end{lstlisting}

Slot 0 is reserved for initialization: an explicit startup capability
when declared, otherwise the synthetic \texttt{null} value used in Coffee.
\texttt{SYS\_INIT} fixes \texttt{capability=0}, and the generated
\texttt{SYS\_TRANS} rule \texttt{capability'!=0} prohibits returning to
that slot. The conversion replaces the one-hot selection block with this
rule; terminal flags remain separate Booleans. Consequently, a terminal
flag can coexist with a nonzero selector, and the selector enforces mutual
exclusion among executable capabilities only, not between them and terminal
flags. For example, the alphabetic P0 enumerated System-Goal/no-pending
strategy has a terminal self-loop with \texttt{finished} and
\texttt{capability=2} (\texttt{log\_finished}). The raw-strategy auditor
continues to monitor the selected capability's outcomes there; controller
realization maps the terminal context to a FlexBE state-machine outcome rather than
executing another action (\resec{controller-realization}).

Quantitative comparisons with the one-hot baseline from prior work
are provided in \resec{discussion}, where we discuss the impact of this
approach. The synthesis pipeline presented in this paper works with either
encoding, specified by the appropriate named pipeline module.

\subsubsection{Liveness Modeling}
\label{sec:liveness}

Realizability depends critically on liveness (justice) assumptions \cite{Bloem2012, Ehlers2016}:
the environment liveness assumption $\varphi_l^a$ and the system liveness
condition $\varphi_l^g$ from \reqn{varphi}.
Let $g$ denote this pipeline's \emph{configured goal condition}. For a
state machine with terminal outputs, $g$ is the disjunction of their
activation propositions from \resec{terminal-outcomes},
$g=\bigvee_{\theta \in \Theta}\theta_a$, or \emph{Globally
Finally} the state machine reaches one of its programmed outcomes.
For a state machine designed to run forever ($\Theta = \emptyset$),
the pipeline may instead configure $g$ as an environment-reported
recurring completion event,
matching the more general goal-label treatment of the auditor's
goal-unreachable-trap check (\resec{auditor}); the extended Coffee
demonstration (\resec{coffee}) uses its final capability's completion,
\texttt{br\_c}, this way. A state machine with neither a terminal output
nor any configured recurrence goal uses $g=\False$ and omits the
corresponding disjunct entirely. The base system liveness condition is
$\varphi_l^g = \mathrm{GF}\,g$; when $g=\False$, no base goal clause is emitted, though
System-Goal Liveness \reqn{system-goal-liveness} may still contribute
the non-completion disjuncts below, so a \texttt{SYS\_LIVENESS} clause
can still result.
Unfortunately, this base liveness condition is not enough for failure-prone
capabilities. Any capability with both
\texttt{completed} and \texttt{failure} as possible outcomes
permits the adversarial environment to continually
return \texttt{failure} and make the specifications unrealizable.
Let $\mathcal{C}_f \subseteq \mathcal{C}$ denote this set of \emph{failure-prone}
capabilities, that is, those $\pi$ for which \texttt{failure} is a possible
outcome.

We implement two alternative liveness specification-generation modules and compare their impact.
Both liveness formulations below are stated over $\mathcal{C}_f$.
Earlier FlexBE synthesis approaches~\cite{Maniatopoulos2016, Hayhurst2018}
assumed that an activated capability eventually returns \emph{some} declared
outcome, success or failure, and aborted execution on failure.
The earlier pipeline~\cite{eit26-synthesis} supported retries using
System-Goal-style liveness (\reqn{system-goal-liveness}), with memory
variables recording past completion rather than an outstanding attempt.
Fair-Outcome and pending memory are alternatives studied here, rather than
formulations inherited from those pipelines.

\paragraph{Fair-Outcome Liveness.}
This formulation is intended to express the fairness intuition that repeated
capability attempts should eventually yield \texttt{completed}.
It places that fairness intuition on the environment and supports retry-based recovery behaviors.
With the plain antecedent, $\mathrm{GF}(\pi_a \rightarrow X\pi_a^c)$ holds
whenever $\pi_a$ is false infinitely often, as in any retry loop that
interleaves other steps (e.g., logging a failure before retrying), so the
environment can fail every attempt without violating it. A pending
proposition $\pi_a^p$ (\resec{on-demand-spec}) keeps the antecedent asserted
across those intervening steps; it is the natural repair, and the one whose
side effect is analyzed below.
The antecedent differs depending on whether the pipeline introduced a pending
proposition $\pi_a^p$ for $\pi$ (\resec{on-demand-spec}). Let
$\mathcal{C}_f^p \subseteq \mathcal{C}_f$ denote the failure-prone capabilities for
which it did, so that $\mathcal{C}_f \setminus \mathcal{C}_f^p$ are those still using the
plain activation proposition $\pi_a$. The completion side of the implication is
always the same next-step proposition, $X\pi_a^c$, regardless of which
antecedent is used, matching the one-step activation-to-outcome delay of
\reqn{env-assumptions}.
The environment liveness assumption $\varphi_l^a$ is then
\begin{equation}
\begin{aligned}
    \varphi_l^a \;=\;
    &\bigwedge_{\pi \,\in\, \mathcal{C}_f \setminus \mathcal{C}_f^p} \mathrm{GF}\big(\pi_a \rightarrow X\pi_a^c\big) \\
    &\;\land\;
    \bigwedge_{\pi \,\in\, \mathcal{C}_f^p} \mathrm{GF}\big(\pi_a^p \rightarrow X\pi_a^c\big)
\end{aligned}
\label{eqn:fair-outcome-liveness}
\end{equation}
that is, one recurrence clause per failure-prone capability, using $\pi_a^p$ as the antecedent exactly where a
pending proposition tracks that capability, and $\pi_a$ otherwise.
$\mathrm{GF}(\pi_a^p \rightarrow X\pi_a^c)$ does not simply mean ``once
pending, completion must eventually occur''. Its exact reading is that
infinitely many steps satisfy $\neg\pi_a^p \lor X\pi_a^c$, which, by the
standard $\mathrm{GF}(A\lor B)\Leftrightarrow\mathrm{GF}A\lor\mathrm{GF}B$
identity, holds whenever \emph{either} $\pi_a^p$ is false infinitely often
\emph{or} $\pi_a^c$ holds infinitely often. Under \reqn{pending}'s
inertia rule, however, abandoning a capability after a failure does
\emph{not} make $\pi_a^p$ false; once set, it remains asserted until a
completion is observed, and a capability that is never reactivated again
never reports one. If activation permanently ceases after even one
failure, $\pi_a^p$ is instead permanently \emph{true}, so neither
disjunct of $\neg\pi_a^p \lor X\pi_a^c$ can hold infinitely often; thus, the
environment can no longer satisfy this liveness assumption at all.
Because activation is system-controlled, the system can withhold this
outcome itself by declining to reactivate the capability.
Under the strict GR(1) semantics of \resec{synthesis}, preventing the
environment from satisfying this liveness assumption releases the controller
from its system liveness obligation. System safety guarantees still apply
along prefixes that respect the environment's initialization and safety
assumptions. Well-separation analysis identifies this assumption-falsification
risk, and the traps in \resec{auditor} show how it manifests in extracted
strategies. This formulation augments only the environment liveness assumption
$\varphi_l^a$, leaving the system liveness condition at its base form,
$\varphi_l^g = \mathrm{GF}\,g$.

Concretely, in the extended Coffee demonstration (\resec{coffee}), the
synthesized controller cycles through button-detection (\texttt{bd}),
grinding (\texttt{gr}), and brewing (\texttt{br}) in sequence
(\refig{coffee_sm}), and \texttt{gr}'s \texttt{failure} postcondition
re-asserts \texttt{bd} rather than clearing it (\relist{coffee-config}), so
a failed grind is retried rather than abandoned. When \texttt{gr}
additionally uses a pending proposition $gr_a^p$ (\resec{on-demand-spec})
to record that an activation is still owed a completion, Fair-Outcome's
environment assumption becomes $\mathrm{GF}(gr_a^p \rightarrow Xgr_a^c)$.
Because \texttt{gr\_p} is cleared only by a completed \texttt{gr} outcome,
keeping it asserted makes retrying \texttt{gr} the only way for the supervisor
to discharge the pending obligation.
Intuitively, a pending grind is expected eventually to report completion.
\relist{pending-liveness-example} shows the resulting
\texttt{ENV\_LIVENESS} clause as generated for this configuration. Because
\texttt{gr} carries a pending proposition, its antecedent is $gr_a^p$ rather
than the plain-activation form $\mathrm{GF}(\neg gr_a \lor gr_c')$ that
\reqn{fair-outcome-liveness}'s other branch would produce for a
non-pending capability. In the extended demonstration, \texttt{bd} and
\texttt{br} also carry pending propositions, so all three
\texttt{ENV\_LIVENESS} conjuncts take this $\pi_a^p$-antecedent form. None
is ever generated in the raw-activation form. But $gr_a^p$ is itself a
system-controlled output, set and cleared by the same strategy this
assumption constrains. The resulting environment promise therefore depends
on the supervisor's retry bookkeeping, rather than only on behavior
controlled by the environment.

\begin{lstlisting}[
float={*t},
basicstyle=\normalsize\ttfamily,
columns=fullflexible,
caption={Fair-Outcome \texttt{ENV\_LIVENESS} excerpt from the extended
Coffee demonstration's generated \texttt{CoffeeSM.structuredslugs}
(\resec{coffee}). The \texttt{gr\_c} clause uses the pending proposition
\texttt{gr\_p} as its antecedent, per \reqn{fair-outcome-liveness}'s
$\mathcal{C}_f^p$ branch.},
label={lst:pending-liveness-example}]
[ENV_LIVENESS]
(!bd_p) | bd_c'
(!br_p) | br_c'
(!gr_p) | gr_c'
\end{lstlisting}

\paragraph{System-Goal Liveness.}
The second formulation augments the system liveness condition with terms
satisfied on an activation followed by a non-completion outcome.
Specifically, the system liveness condition becomes
\begin{equation}
    \varphi_l^g \;=\; \mathrm{GF}\!\left(g \;\lor\;
    \bigvee_{\pi \,\in\, \mathcal{C}_f} \neg(\pi_a \rightarrow X\pi_a^c)\right)
\label{eqn:system-goal-liveness}
\end{equation}
that is, a single recurrence clause augmenting the configured goal
condition $g$ with one additional disjunct per failure-prone capability.
The generated \texttt{slugs} liveness sections omit the explicit
$\mathrm{GF}$ operator, which the solver applies to each listed clause.
Consequently, if a failure-prone capability has infinitely many activations
followed by noncompletion, the corresponding disjunct
$\neg(\pi_a \rightarrow X\pi_a^c)$ is satisfied infinitely often,
discharging $\varphi_l^g$ without requiring $g$ to hold. System-Goal
Liveness therefore does not guarantee that the state machine reaches a
terminal outcome against an adversarial capability. Nor do the added
disjuncts make the specification realizable on their own: conflicting safety
constraints or infeasible preconditions can still prevent synthesis.
The guarantee is conditional. On a winning execution that satisfies the
environment assumptions, if only finitely many failure events
$\neg(\pi_a \rightarrow X\pi_a^c)$ occur in total across all
$\pi \in \mathcal{C}_f$, then $\mathrm{GF}\,g$ holds. Occasional completions
do not meet this premise: failures that recur forever, even interleaved with
completions, can still discharge $\varphi_l^g$ without reaching $g$.
\resec{auditor} checks the synthesized strategy for obstructions it
introduces itself (deadlock, protocol violation, or a goal-unreachable
region) rather than for the capability's own reliability, which is outside
this paper's scope (\resec{crazyflie}). Passing the audit alone does not
guarantee eventual completion; the conditional guarantee above applies when
the finite-failure premise holds.
Unlike the Fair-Outcome formulation, this condition does not introduce a
separate environment fairness assumption for each capability. It leaves
the environment liveness assumption at its default, $\varphi_l^a = \True$.

For the same extended Coffee \texttt{gr} capability, System-Goal's system
liveness condition \reqn{system-goal-liveness} is instantiated as the
single disjunctive \texttt{SYS\_LIVENESS} clause in
\relist{system-goal-liveness-example}. A negated $gr_a \rightarrow Xgr_a^c$
disjunct sits alongside \texttt{bd}'s and \texttt{br}'s own terms, so
\texttt{gr}'s retry needs no dedicated environment-side promise to
entangle with, unlike the Fair-Outcome case above.

\begin{lstlisting}[
float={*t},
basicstyle=\normalsize\ttfamily,
columns=fullflexible,
caption={System-Goal \texttt{SYS\_LIVENESS} excerpt from the extended
Coffee demonstration's generated \texttt{CoffeeSM.structuredslugs}
(\resec{coffee}), one-hot no-pending configuration (\textbf{3S} in
\retab{coffee-1hot-results-summary}). The single disjunctive clause instantiating
\reqn{system-goal-liveness} over \texttt{bd}, \texttt{gr}, and \texttt{br}.},
label={lst:system-goal-liveness-example}]
[SYS_LIVENESS]
# Reach the configured goal or observe non-completion
((br_c) | !(bd_a -> bd_c') | !(br_a -> br_c') | !(gr_a -> gr_c'))
\end{lstlisting}

This kind of assumption entanglement is exactly what well-separation
(\resec{well-separation-module}) is designed to catch before synthesis
runs, and what the post-synthesis auditor (\resec{auditor}) checks for
directly in the extracted strategy since well-separation alone is not
sufficient. These alternative liveness definitions significantly affect
realizability, the time required to determine it, and state machine size.
\resec{discussion} reports the concrete cases and compares the resulting
realizability, synthesis cost, and controller structure across both
formulations.
As shown in \resec{discussion}, these results motivate using System-Goal
liveness for the demonstrations presented in this paper.

\subsection{Specification Analysis Module}
\label{sec:well-separation-module}

After GR(1) specification generation and before ordinary synthesis, the pipeline can invoke
an analysis module on the compiled \texttt{structuredslugs} input.
Following~\cite{MaozRingert2016}, we implement the well-separation algorithm
of Maoz and Ringert as an extension to \texttt{slugs}.
The module temporarily replaces the system side of \reqn{varphi} with
$\varphi_i^g=\True$, $\varphi_s^g=\True$, and $\varphi_l^g=\False$, then
reuses the existing winning-region fixed point analysis in \texttt{slugs} to ask whether the system
can force violation of the environment assumptions
$(\varphi_i^a,\varphi_s^a,\varphi_l^a)$ from an environment-reachable state.
The ordinary synthesis path is unchanged.

The module emits a machine-readable diagnostic rather than a controller. For
well-separated specifications it records a passing verdict. For
NWS specifications it reports the Maoz--Ringert case class
(\texttt{P-all} or \texttt{P-reach}, paired with \texttt{E-ini},
\texttt{E-safe}, or \texttt{E-just}), a representative reachable witness
state when available, and the responsible environment assumptions~\cite{MaozRingert2016}.
For the experiments in \resec{discussion}, we additionally enable core minimization
for selected configurations, which repeatedly reruns the diagnosis over
subsets of named environment assumptions while keeping the original reachable
state set fixed, yielding a small assumption core that explains the
NWS verdict. These pre-synthesis diagnostics are optional and are complementary to the
post-synthesis auditor discussed next. The module can optionally be
configured to treat an NWS or incomplete verdict as fatal
(\texttt{fail\_on\_non\_well\_separated} and \texttt{fail\_on\_incomplete}
in the analyzer's module configuration), but every configuration evaluated in this paper
runs with both disabled.
Well-separation is reported, but never blocks
synthesis, auditing, or reduction. \reappsec{pipeline-integration} states the
resulting three-status policy (unrealizable, NWS, strategy audit failure)
this pipeline follows in full.

\subsection{Synthesis}
\label{sec:synthesis-module}

Our pipeline is generic, but is currently built using \texttt{slugs}~\cite{Ehlers2016}.
The pipeline module is a wrapper that invokes our customized version of
the \texttt{slugs} executable~\cite{CNU_slugs}.
It takes in the specifications from \resec{syn_spec} written in \texttt{structuredslugs} form,
and compiles them into the required \texttt{slugsin} format using a Python~3
port of the original compiler for \texttt{slugs}.  The wrapper invokes \texttt{slugs} in its
standard GR(1) synthesis mode, captures the tool's realizability verdict and
internal timing output, and parses the extracted Mealy strategy when the
specification is realizable.

An unrealizable result terminates the FlexBE synthesis request with a structured failure record,
including the generated specification and \texttt{slugs} diagnostics. A realizable result
continues through the pipeline as a raw Mealy strategy (\relist{slugs-example-json}) whose states carry the
input and output labels.  The raw output is converted to a standard YAML form used
by the remaining modules in our pipeline.
Keeping the \texttt{slugs} wrapper behind the same YAML-declared module
interface as the other stages allows the pipeline to swap synthesis configurations in the future
without changing the downstream modules.

%% file: Auditor.tex
\subsection{Post-Synthesis Strategy Auditor}
\label{sec:auditor}

Realizability alone does not certify that a synthesized strategy makes the
progress a designer intends~\cite{KleinPnueli2010,MaozRingert2016,KindControllers24}.
A GR(1) strategy need only satisfy its guarantees
for environment behaviors that satisfy the stated assumptions. A strategy
can therefore win by forcing an environment-assumption violation, leaving
the intended goal unfulfilled. Separately, even on assumption-satisfying executions,
a winning strategy can revisit non-terminal states forever if its system
liveness condition permits noncompletion (\resec{liveness}).

In Fair-Outcome liveness, adding pending memory seems like a reasonable way
to express the expectation that repeated attempts should eventually succeed.
The pending flag keeps the fairness obligation active during intervening
actions, such as logging a failure (\resec{liveness}). The subtle defect is
that the pending flag records an outstanding completion obligation while the
system remains free to stop retrying. Once it does, the environment cannot report
completion because outcomes require an activation (\resec{syn_spec}).
The system can therefore make the fairness assumption impossible to satisfy.

\refig{pyrobosim-trap} shows how this plays out in the raw strategy for the
PyRoboSim P0 mobile-manipulation delivery domain (\resec{pyrobosim}), in the
one-hot Fair-Outcome-with-pending trial (\textbf{2FP}, alphabetic ordering).
States 72, 73, and 83 form a closed strongly connected component in which
the strategy repeatedly activates \texttt{place\_item}, with completion and
failure outcomes. It never reactivates \texttt{move\_robot}, so the move
pending flag stays true and move completion stays false. The extracted edges
satisfy environment safety~\cite{CNU_slugs}. Every infinite execution in the
component violates the environment justice clause
$GF(\neg\texttt{move\_robot\_p}\lor X\texttt{move\_robot\_c})$.

The component contains no state labeled \texttt{finished}, and no edge
leaves it. The strategy can still win the GR(1) game because the task-goal
guarantee applies only when the environment satisfies its assumptions.
A plausible specification can therefore pass realizability checking and
yield a strategy that abandons the task.
The post-synthesis auditor flags that loss of goal reachability in the
extracted strategy. \resec{pyrobosim-p0-results} reports its occurrence
across trials.

\begin{figure*}[t]
    \centering
    \includegraphics[width=0.85\linewidth]{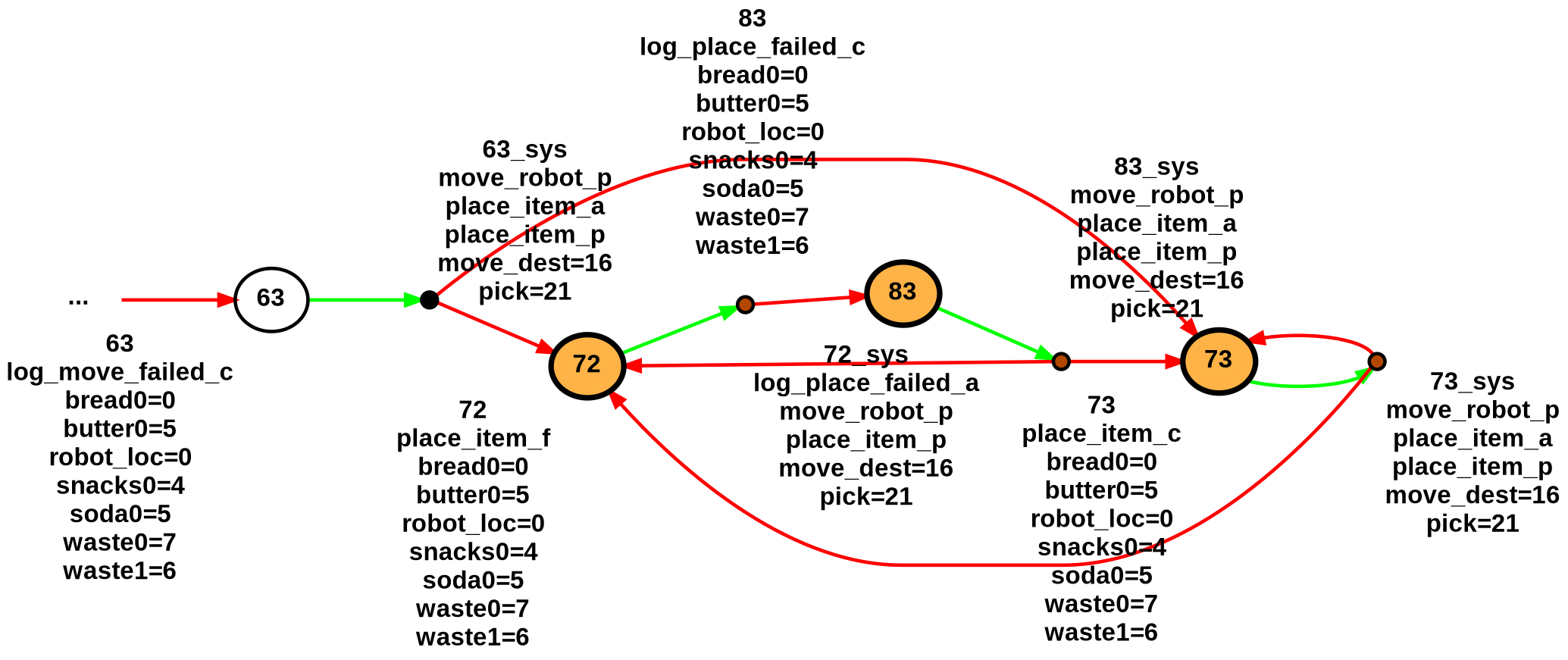}
    \caption{
    A goal-unreachable trap in the synthesized strategy for the PyRoboSim P0
    one-hot Fair-Outcome-with-pending configuration
    (\textbf{2FP}, alphabetic ordering).
    States 72, 73, and 83 (orange) form a closed strongly-connected
    component reachable from state 63 (white). No edge leaves the highlighted
    cycle. State 63 is itself entered from elsewhere in the strategy,
    marked with a ``\textbf{\ldots}'' stub rather than expanded, since this
    excerpt is a small cut-out of the full strategy graph rather than the
    whole reachable state machine. Uses the \texttt{mealy2dot} visual
    convention introduced in \refig{slugs-example-mealy}
    (\resec{slugs-example}). Rendered with \texttt{mealy2dot}
    on the raw (unreduced) strategy.}
    \label{fig:pyrobosim-trap}
\end{figure*}

The post-synthesis, explicit-state \emph{strategy auditor} checks the Mealy
strategy \texttt{slugs} extracts for two classes of semantic defects that
realizability checking does not rule out. Its role complements the
pre-synthesis analysis of \resec{well-separation-module}, as summarized in
\retab{auditor-well-separation}. The first class
comprises strict capability activation--outcome protocol violations, deadlocks,
and violations of immediate consecutive-retry bounds (\emph{bounded-failure
violations}). The second class is reachable regions from which a designer's goal
becomes structurally unreachable. A strategy can pass the trap check even
when an execution can stay indefinitely in a cycle with an unforced exit
toward the goal (Corollary~\ref{cor:unforced-exit}).
Determining whether every fair execution eventually reaches the goal requires
analysis of the strategy under the environment assumptions.

Because the auditor is explicit-state, it walks the finite Mealy automaton
\texttt{slugs} produces as its strategy rather than the BDD-encoded GR(1) formula, so the
check is independent of encoding (\resec{encodings}) and synthesizer, and of
whether state-merging reduction has run. The primary audit runs immediately
after synthesis, where it enforces the concrete activation--outcome protocol on
the raw strategy.

After reduction, the same graph analysis can run again
as a confirmatory check for deadlocks and goal-unreachable traps.
Merged states may inherit input-label metadata from several unreduced
contexts, so they cannot be held to a strict per-state protocol. Both applications
use the same core technique: strongly connected component (SCC) analysis
for reachable goal-unreachable traps, namely cyclic components with no path
to a cyclic goal-containing component. This uses classical SCC-based
techniques for cycle detection in model
checking~\cite{ClarkeGrumbergPeled1999, Courcoubetis1992} and the
safety/liveness distinction it rests on~\cite{AlpernSchneider1985}.

Unlike well-separation, which applies established prior
work~\cite{MaozRingert2016} within our pipeline, the strategy auditor itself
is a new contribution of this work, albeit based on classical results.
The remainder of this subsection states
its core soundness and completeness results and their complexity. The
underlying semantic model, full proofs, and pipeline-integration details are
given in Appendix~\ref{app:auditor-proofs}.

\begin{table*}[tp]
\centering
\small
\setlength{\tabcolsep}{4pt}
\renewcommand{\arraystretch}{1.25}
\caption{Complementary roles of specification-level well-separation and the
post-synthesis strategy auditor for the retry-liveness pathology studied in
this paper.}
\label{tbl:auditor-well-separation}
\begin{tabular}{p{0.23\textwidth}p{0.34\textwidth}p{0.34\textwidth}}
\toprule
\textbf{Question} & \textbf{Well-separation / assumption-falsification analysis} & \textbf{Post-synthesis strategy auditor} \\
\midrule
Artifact checked &
The GR(1) specification and game before synthesis. &
The extracted Mealy strategy after synthesis. \\
Main failure exposed &
Whether the specification lets the system win by falsifying or depending on
environment assumptions, including broader game-level fairness hazards. &
Whether the concrete strategy violates the activation--outcome protocol or
enters a reachable cyclic region from which the configured goal is
structurally unreachable. \\
What it catches in our examples &
The specification-level source of the Fair-Outcome retry hazard; that is, progress can
depend on assumptions about future outcomes that the controller does not force. &
The realized symptom in the generated controller; that is, pending-without-activation
cycles in which no completion outcome can occur and the task goal is cut off. \\
What it does not certify &
That a particular extracted or reduced strategy obeys the FlexBE capability
protocol, respects bounded-failure counters, or remains free of concrete
strategy-level traps. &
That every execution reaches the goal (Corollary~\ref{cor:unforced-exit}). \\
Benefit &
Diagnoses the modeling weakness before committing to a controller. &
Validates the actual artifact that the pipeline will execute. \\
\bottomrule
\end{tabular}
\end{table*}

The auditor's two checks use standard automata-theoretic reachability
analysis over the explicit synthesized strategy rather than a new
verification technique. The safety check (Lemma~\ref{lem:safety}) explores
the reachable product of the strategy and an
activation--outcome monitor (Appendix~\ref{app:auditor-proofs},
Definition~\ref{def:product}). The goal-unreachable-trap check
(Theorem~\ref{thm:goal-unreachable-trap}) adapts classical SCC-based cycle
detection from model checking~\cite{ClarkeGrumbergPeled1999, Courcoubetis1992}
to a goal-reachability check for the concrete goal condition generated by
our pipeline; it is not a general B\"uchi accepting-cycle search. The contribution lies in this
post-synthesis application. The auditor operates on the extracted Mealy
strategy, independently of the specification encoding and the synthesizer's
internals.

With strict protocol monitoring enabled, a completed safety check determines
whether the strategy contains a reachable protocol violation, bounded-failure
violation, or deadlock (Lemma~\ref{lem:safety}). A confirmed safety defect
ends the audit before the trap check. A clean safety verdict requires complete
exploration of the reachable product graph, which the auditor then uses for
the trap check. If the time budget expires before an audit verdict, the
auditor reports an incomplete result (\reappsec{pipeline-integration}).

A \emph{deadlock} (Definition~\ref{def:protocol}, Appendix~\ref{app:auditor-proofs})
is a reachable state with no outgoing transition, already caught by the
safety check above. The goal-unreachable-trap check targets a structural
subclass of a different failure mode, \emph{livelock}: execution can
continue forever after reaching a region from which the configured goal
is no longer reachable.
Here and below, ``goal'' means the explicitly configured terminal or
recurrence label supplied to the auditor -- the base goal condition $g$
of \resec{liveness}, before System-Goal liveness's failure-prone-capability
augmentation \reqn{system-goal-liveness}. Depending on the domain, this
label denotes successful task completion or acceptable programmed
termination. \retab{auditor-goals} lists these configured goal labels.
Under System-Goal liveness specifically, the liveness condition can also
be satisfied by the failure-prone-capability non-completion disjuncts in
\reqn{system-goal-liveness}; those disjuncts are not semantic task goals.
This distinction does not affect the goal-unreachable-trap results
reported in \resec{discussion}; in this paper, the traps occur only under Fair-Outcome
liveness, where $\varphi_l^g = \mathrm{GF}\,g$ exactly and the table's entries are the
complete goal condition.

\begin{table*}[tp]
\centering
\small
\setlength{\tabcolsep}{4pt}
\renewcommand{\arraystretch}{1.15}
\caption{Configured goal condition $g$ (\resec{liveness}) supplied to the
post-synthesis strategy auditor for each domain.}
\label{tbl:auditor-goals}
\begin{tabular}{p{0.18\textwidth}p{0.23\textwidth}p{0.47\textwidth}}
\toprule
\textbf{Domain} & \textbf{Goal condition $g$} & \textbf{Interpretation} \\
\midrule
Coffee & \texttt{br\_c} & Recurring successful brew completion in a controller designed to run forever. \\
Two-Rivers & \texttt{finished} & Successful terminal output after all entities reach the far bank. \\
PyRoboSim P0/P1 & \texttt{finished} & Successful terminal output after the delivery task is complete. \\
Quadcopter & \texttt{finished} $\lor$ \texttt{failed} & Programmed terminal output; the audit certifies reachability of termination, not mission success alone. \\
\bottomrule
\end{tabular}
\end{table*}

Formally, for a deadlock-free reachable product graph, a strategy has a
reachable node from which no cyclic goal-containing component is reachable
if and only if the reachable product graph contains a
\emph{goal-unreachable trap}, a cyclic strongly connected component (SCC)
with no path, in the condensed component graph, to any cyclic
goal-containing component (Theorem~\ref{thm:goal-unreachable-trap}). Every
goal-unreachable trap admits an infinite non-goal execution once the reachable
product graph is deadlock-free, so a flagged component is a genuine
structural livelock witness. Whether such a path is fair under additional
environment assumptions is outside this check.

Because completion outcomes are activation-gated (\resec{syn_spec}), a
reachable cyclic region in which activation has permanently ceased has no
path to a completion-requiring goal. This is the causal chain behind the trap in
\refig{pyrobosim-trap}; the auditor detects this closed, goal-unreachable
cycle as part of the pipeline rather than through manual strategy tracing.

The SCC and reachability checks are linear in the reachable product graph
(\reqn{auditor-complexity}, Appendix~\ref{app:auditor-proofs}). Under the
immediate-retry reset rules, the pending capability is fixed by the strategy
state and only its counter can be nonzero. With no finite failure bounds,
as in every evaluated configuration, the product has at most $n$ nodes and
$m$ edges, giving $O(n+m)$ graph operations.
The auditor operates on the extracted explicit strategy,
not on the BDD representation used by synthesis;
audit timings also include harness overhead (\retab{timing-rows}).

The auditor runs as a pipeline module. Its coverage is determined by the
reachable strategy graph, independently of which outcomes were exercised
during testing. Controller size across our case studies ranges from single digits
(Coffee) to 677 states (the largest quadcopter configuration,
\textbf{2FP}'s raw strategy, \seedetail{crazyflie-reliability-app}).
Appendix~\ref{app:auditor-proofs} gives the full
pipeline-integration details, including the wall-clock budget policy,
error-code semantics, and the optional post-reduction run.

%% file: Reduction.tex
\subsection{Strategy State-Merging Reducer}
\label{sec:reduction}

The optional state-merging step replaces the raw Mealy strategy
\texttt{slugs} extracts with a smaller automaton before it is realized as an
executable FlexBE HFSM. Controller realization translates the reduced automaton into the FlexBE
HFSM executed on the robot. Thus, collapsing Mealy strategy states is not
merely a cosmetic simplification. It substitutes the synthesized transition
system for another.
This section
gives the semantic argument for why that substitution is safe. In particular,
it explains why the reduced automaton preserves the executable
capability-level behavior of the strategy \texttt{slugs} proved winning and
identifies what the reduction does not preserve.

\subsubsection{Core and Pending Outputs}

Recall from \resec{on-demand-spec} that pending propositions
$\pi_a^p \in AP_O$ are introduced only to let the Fair-Outcome liveness
formulation (\resec{liveness}) distinguish a freshly issued activation from
one already in flight. They are bookkeeping internal to the GR(1)
specification, not part of what a capability actually does. Let
$AP_O^p \subseteq AP_O$ denote the set of all such pending propositions, and
for $o \subseteq AP_O$ write $o {\downarrow} = o \setminus AP_O^p$ for its
\emph{core} projection, retaining every nonpending system-output
proposition, including capability activations, parameter selections, and
any auxiliary system-memory outputs.
The reducer's merge test compares Mealy
strategy states by $\mathit{out}(s){\downarrow}$ rather than $\mathit{out}(s)$,
so two states whose outputs agree on every proposition except which pending
flags are set are eligible to merge; Definition~\ref{def:merge-test} gives
the second, successor-matching condition the states must also satisfy
before a merge actually happens. Although $AP_O^p$ affects the environment-fairness
assumption used at synthesis time (\resec{liveness}), no downstream
consumer -- the HFSM realization in particular -- branches on pending
propositions once synthesis has completed (\resec{pipeline_config}).
Appendix~\ref{app:reduction-proofs} gives the formal output-trace
equivalence relation used by the proof.

\subsubsection{The Reduction Procedure}

The pipeline's reduction stage does two
things, in order. It performs an iterated indegree-zero pruning pass
(Appendix~\ref{app:reduction-proofs} notes this is sound but not always
complete), then it merges Mealy strategy states satisfying a
concrete, checkable approximation of Definition~\ref{def:trace-equiv}.

For readability, we use the appendix notation here: $A=(S,\delta,s_0)$ is
the raw strategy automaton, $\mathit{in}(s)$ and $\mathit{out}(s)$ are the
arrival-input and output labels on state $s$, and $\tau_0(s)$ maps each
outgoing arrival label to its successor after pruning. At merge stage $G_k$,
$\tau_k$ retains those exact input keys while redirecting their targets to
current representatives. The relation $u \approx w$
denotes output-trace equivalence, defined formally in
Appendix~\ref{app:reduction-proofs}. As in the JSON strategy representation
of \relist{slugs-example-json}, each serialized node corresponds to one Mealy
strategy state in $S$. We use \emph{Mealy strategy state} here to distinguish
these automaton states from realized FlexBE HFSM states and from the split
visual states in \refig{slugs-example-mealy}.

The formal statements and proofs are given in
Appendix~\ref{app:reduction-proofs}. Pruning is behavior preserving as no state
removed by Definition~\ref{def:pruning} lies on any $\delta$-path from
$s_0$, so removing it changes neither the set of executions reachable from
$s_0$ nor any label along them (Lemma~\ref{lem:pruning-sound}).

Definition~\ref{def:merge-test} is exactly the equality check implemented by
the reducer. It compares masked output valuations and compares the two
states' current outcome-indexed target maps $\tau_k(\cdot)$
(Definition~\ref{def:outcome-targets}). Their exact input keys are fixed once
immediately after pruning, before any merge can make a target's arrival label
ambiguous; only their target values are redirected. Initializing the maps is
one additional linear pass over the raw edge set, so it does not change the
sweep's asymptotic cost.

This construction means the reducer consults $\mathit{in}(\cdot)$ only on
successor states, and only before merging begins. The procedure applies
Definition~\ref{def:merge-test} left to right over the states remaining
after pruning, updating the current automaton in place after each merge.
The implementation may repeat this same merge sweep until a sweep makes no
change, but each sweep still uses the same sufficient local merge test rather
than a partition-refinement decision procedure such as Hopcroft's
algorithm~\cite{Hopcroft2001}.
Definition~\ref{def:merge-test} is a sufficient, checkable condition for
$\approx$, not a decision procedure for it.

Appendix~\ref{app:reduction-proofs} formalizes this argument and later returns
to its cost. Every pair $(u,w)$ the reducer merges satisfies $u \approx w$ in
the original raw automaton $A$ (Lemma~\ref{lem:merge-sound}).

Together, pruning and merge soundness give the headline result
(Theorem~\ref{thm:reduction-behavior}).
Definition~\ref{def:reduced-automaton} constructs $A'$ explicitly: it first
prunes $A$ to $G_0$, then obtains each $G_{k+1}$ from $G_k$ by one merge and
global transition redirection, and finally sets $A'=G_m$. The resulting
automaton is
guaranteed to have the same core-output traces as $A$ from their respective
initial states. Its input-indexed transition function retains the exact raw
input labels as fixed keys while redirecting their targets to surviving
representatives; a merged representative's unioned arrival metadata is not
used to execute the reduced automaton. Concretely, $A$ and $A'$ admit exactly
the same environment input sequences from $s_0$ and $s_0'$, respectively;
executing any such sequence against both automata produces identical sequences
of core output valuations
$\mathit{out}(\cdot){\downarrow}$, and hence identical capability
activations and parameter choices at every step. Corollary~\ref{cor:guarantee-transfer}
is therefore a statement about the
formal object passed to controller realization; if the raw strategy $A$
witnesses realizability of \reqn{varphi}, then the reduced automaton $A'$
preserves the capability-level behavior guaranteed by the raw strategy,
within the core-output projection formalized in Appendix~\ref{app:reduction-proofs}.

Corollary~\ref{cor:guarantee-transfer} is the precise sense in which reduction
``preserves correctness''. It is a representation change of an
already-winning strategy, not a new synthesis step whose winning status must
be independently re-established. This is also why \resec{auditor} treats the
post-reduction audit as a confirmatory structural check rather than as the
source of the correctness argument. That source is
Theorem~\ref{thm:reduction-behavior}, established here independently of the
auditor.

\subsubsection{What Reduction Does Not Preserve}

Theorem~\ref{thm:reduction-behavior} is deliberately a claim about
\emph{forward} behavior, not about how a merged state was reached. This
leaves two caveats.

The first caveat concerns the pending propositions. A merge preserves the
core output projection, but not the discarded pending bits, so $A'$ is not,
in general, a valid assignment to the original GR(1) propositions. Only its
core projection is meaningful, which is all
Corollary~\ref{cor:guarantee-transfer} claims.

The second caveat concerns arrival metadata. A merged state's incoming label
is the union of the metadata from the raw states folded into it, so
\resec{auditor}'s strict per-edge protocol checks cannot be applied
unmodified after reduction. The remaining structural checks do not require
exact per-edge arrival metadata for the goal-label kinds used by this
pipeline: deadlock is graph-structural, nonpending output goals are preserved
by the merge test's core-output equality, and outcome-token input goals are
tracked in the unioned incoming-label metadata. For those labels, reduction
can make a state appear more able to reach the goal, never less. It can mask
an existing trap but cannot manufacture a new one, which is why we
audit before reduction in the pipeline.
Appendix~\ref{app:reduction-proofs}
spells out this goal-label restriction; the argument does not extend to
arbitrary $AP_I$ goal propositions outside the outcome vocabulary.

The theory above already rules out reduction introducing a genuinely new
violation for the pipeline's configured goal labels.
In every evaluated configuration, each raw strategy admitted to reduction
passed a completed raw audit, so the post-reduction audit confirms the
reducer implementation rather than supplying the correctness argument.
If a non-fatal incomplete raw audit lets generation continue
(\reappsec{pipeline-integration}), that premise is not established.
Appendix~\ref{app:reduction-proofs} gives the full argument and explains why
repeating the local merge sweep to a fixed point still need not produce the
smallest possible reduced automaton.

%% file: ControllerRealization.tex
\subsection{Controller Realization}
\label{sec:controller-realization}

After synthesis and auditing (\resec{auditor}), controller realization uses the
preprocessed discrete abstraction to mechanically translate the strategy into an
executable FlexBE HFSM. When state-merging reduction is enabled
(\resec{reduction}), the input strategy is the reduced automaton $A'$.
Executable strategy states whose outputs select a concrete capability
become ordinary FlexBE states; bootstrap/null and
terminal contexts become container wiring or state-machine outcomes.
For each executable state, realization instantiates
the FlexBE state implementation named by the capability configuration's
compiled \texttt{class\_decl} (\resec{pipeline_config}), supplies any fixed or
symbolic parameters, and
wires each concrete FlexBE outcome through the same
\texttt{state\_outcome\_mapping} used during specification generation. This
shared mapping is what connects a concrete outcome such as \texttt{done},
\texttt{true}, \texttt{failed}, or \texttt{timeout} back to the abstract
outcome proposition that selects the next strategy state.

Interpreting the abstract guarantees during physical execution assumes that
capability implementations faithfully realize the declared activation--outcome
interface, including parameter bindings, preconditions, postconditions, and
the reporting of one normalized outcome per activation. Asynchronous execution
and middleware must preserve that interface.
The formal results for synthesis, auditing, and reduction apply to the abstract
strategy and do not verify this runtime contract. Targeted generator unit tests
and forced-outcome simulation exercise the mapping from the strategy to the
FlexBE HFSM. PiHawk flights show how the deployed HFSM responds to reported
capability outcomes on hardware (\resec{crazyflie-failure-validation}). This
evidence for controller realization and physical execution is empirical, and
a machine-checked compiler proof for the generator remains outside the scope
of this paper.

Not every strategy state becomes a separate executable FlexBE state.
Some strategy states exist only to initialize the generated behavior or to
hold the enumerated encoding's startup value. We refer to these as
bootstrap/null states. The initial strategy state becomes the root
container's start transition. The enumerated encoding's synthetic
\texttt{null} value, introduced as the startup slot in
\resec{enumerated}, is never emitted as a robot action.
Likewise, multiple strategy states that differ only in terminal symbolic
bookkeeping can map to the same FlexBE terminal outcome, such as
\texttt{finished} or \texttt{failed}. This projection changes the count of
realized FlexBE state instantiations, but it does not change the
capability-level behavior represented by the executable action states.

The \textbar{}Slugs SM\textbar{}
(where SM abbreviates state machine) and
\textbar{}Reduced\textbar{} rows reported later in this paper count states
in the explicit Mealy strategy returned by \texttt{slugs} and in the
reduced strategy automaton, respectively. They are therefore
strategy-level quantities, not necessarily the number of FlexBE state
instantiations in the realized controller. When comparing
\textbar{}Reduced\textbar{} with the realized FlexBE state count, small
differences often reflect root-container bookkeeping, while
larger differences can arise when many reduced terminal contexts map to a
shared FlexBE outcome. The domain discussions below call out the cases
where this distinction affects how the strategy-size rows should be read.

%% file: Discussion.tex
\section{Results and Discussion}
\label{sec:discussion}
\FloatBarrier

This section evaluates the pipeline choices introduced above across the
released case-study suite, using the four domains of \resec{demonstrations}
to address the two research questions raised in \resec{intro}. The
implementation, demonstration artifacts, directions for running the
demonstrations, and raw per-seed experimental data are released open
source~\cite{flexbe-synthesis, flexbe-synthesis-demo}. The Data and code
availability declaration identifies the code tags and original experimental
data revision. These data support the ordering-level robustness check
reported below. \reappsec{p0-carriage-check} and \reappsec{p0-retry-check} describe
the separate supplemental checks.
The subsections below summarize the main per-domain findings in compact
tables focused on the key outcome and timing rows; full per-configuration
tables and supporting analyses appear in
\seedetail{coffee-results-app}, \seedetail{two-rivers-results-app},
\seedetail{pyrobosim-results-app}, and \seedetail{crazyflie-results-app}.

\paragraph{Strategy Non-Uniqueness.} The quantitative comparisons describe
the tested \texttt{slugs}/CUDD runs and the strategies selected for those
declaration orderings. Extracted and reduced state counts can vary among
winning strategies for the same specification. BDD statistics and timings
also depend on the encoding, declaration order, and CUDD settings.

\resec{summary-results} answers RQ1 (encoding, \resec{enumerated}) and RQ2
(retry semantics, \resec{liveness}) across all four domains; the Coffee,
Two-Rivers, PyRoboSim, and quadcopter subsections below give the supporting
per-domain evidence.

\resec{demonstrations} introduced four case-study domains. For quantitative
comparison, Coffee contributes separate base and extended capability sets and
PyRoboSim contributes separate P0 and P1 models, yielding six comparisons:
Coffee base, Coffee extended, Two-Rivers, PyRoboSim P0, PyRoboSim P1, and the
quadcopter domain. Coffee base has no failure-prone capabilities, so its
liveness variants coincide and it informs only the encoding axis.
\retab{experimental-design} summarizes the design factors crossed across these
comparisons: capability encoding, liveness formulation, and pending memory under
a shared declaration-ordering and CUDD-configuration protocol.

\begin{table*}[b]
\centering
\small
\caption{Experimental design factors compared across all six
domain/subdomain comparisons (\resec{demonstrations}). Crossing encoding
$\times$ liveness $\times$ pending gives up to eight named configurations
per comparison, labeled by configuration digit (\textbf{1}/\textbf{2}/\textbf{3}/\textbf{4}),
liveness letter (\textbf{S}/\textbf{F}), and an optional \textbf{P} suffix
for pending memory, following this section's naming convention. Not every
combination is realizable or attempted in every domain, and each domain
additionally reports one or two hand-written baselines outside this grid.}
\label{tbl:experimental-design}
\setlength{\tabcolsep}{4pt}
\begin{tabular}{p{0.20\textwidth}p{0.67\textwidth}}
\toprule
\textbf{Factor} & \textbf{Levels tested} \\
\midrule
Capability encoding & One-hot (\resec{one-hot-encoding}) and enumerated (\resec{enumerated}) \\
Liveness formulation & System-Goal (\resec{liveness}) and Fair-Outcome (\resec{liveness}) \\
Pending memory & Without or with a persistent proposition tracking whether a completion is still owed (\resec{on-demand-spec}) \\
Comparisons & Coffee base and extended capability sets (\resec{coffee-results}), Two-Rivers (\resec{two-rivers-results}), PyRoboSim P0 without door modeling (\resec{pyrobosim-p0-results}), PyRoboSim P1 with door modeling (\resec{pyrobosim-p1-results}), and the quadcopter domain (\resec{crazyflie-results}) \\
Declaration ordering & Fixed alphabetic, fixed capability-grouped (``domain-grouped''), and $N$ random-seed orderings ($N=40$ for Coffee/Two-Rivers and $N=20$ for PyRoboSim P0/the quadcopter). Each random seed shuffles the input declarations and the output declarations, except that the enumerated encoding's shared \texttt{capability} declaration is always written first in the output block. PyRoboSim P1's main table reports single fixed-ordering trials; \reappsec{pyrobosim-p1-ordering-app} adds a 20-random-seed sweep for \textbf{4S}/\textbf{4F}. \\
CUDD reordering policy & Threshold 250 for the principal experiment grid unless otherwise noted. PyRoboSim P1's original single-ordering tables (\retab{pyrobosim-p1-results}, \retab{pyrobosim-p1-results-app}) use the P1 configuration's uncapped (``auto'') policy; \reappsec{pyrobosim-p1-ordering-app} separately reports threshold-250 P1 ordering experiments for \textbf{4S}/\textbf{4F}. \\
\bottomrule
\end{tabular}
\end{table*}

\paragraph{Notation.} Per-domain configurations are named
\texttt{<configuration digit><liveness letter>[P]}, e.g.\ \textbf{2S} or
\textbf{3FP}. They are read as follows:
\begin{itemize}[nosep]
    \item \textbf{Configuration digit}: Interpreted according to the
    domain's table legend. For Coffee, \textbf{2}/\textbf{3} denotes the
    base/extended capability set, and separate tables identify one-hot
    versus enumerated encoding. For the other domains, the digit
    distinguishes capability encodings, with the lower digit denoting
    one-hot and the higher digit enumerated.
    \item \textbf{S}/\textbf{F}: System-Goal/Fair-Outcome liveness (\resec{liveness}).
    \item \textbf{P} suffix: Pending memory added to the activation specification (\resec{on-demand-spec}).
\end{itemize}

The summary tables in this section report specification sizes, strategy and
reduced state machine sizes, selected BDD statistics tracked internally by
\texttt{slugs}, and seven timing rows in pipeline order:
\textbf{Generate Specs},
\textbf{Well-Separation}, \textbf{Realizability}, \textbf{Extraction}, \textbf{Auditing},
\textbf{Realize HFSM}, and \textbf{Overall}. \retab{timing-rows} summarizes what each row
measures and how it is timed. In short, \textbf{Well-Separation}, \textbf{Realizability}, and
\textbf{Extraction} are internally reported times from \texttt{slugs} and so avoid the Python pipeline
wrapper's process and subprocess-invocation overhead, while the remaining four rows are module
wall-clock times measured by the Python pipeline harness and therefore include interpreter and
subprocess start-up overhead the three \texttt{slugs}-internal rows do not
(\reappsec{auditor-complexity} discusses this same effect for \textbf{Auditing} specifically). The
appendices give the full per-configuration rows behind these summaries. All
reported specification counts are taken from the compiled \texttt{slugsin}
format after enumerated types are expanded into binary encoded bits.
The peak-node statistic is the node storage CUDD has allocated, counted in
blocks of 1{,}022 nodes and including nodes on its free list. It is therefore
an exact multiple of 1{,}022 that slightly overestimates the true peak table
size, and it is coarse when peaks are only a few thousand nodes, as for
Coffee. The live-node statistic is the number of live nodes when the
realizability check ends. It is not a peak, and it is not quantized.
The CUDD reordering threshold is the live-node count at which the first
dynamic reordering fires. CUDD's default is 4{,}004 live nodes, and after each
reordering CUDD schedules the next one at about twice the live-node count, so
a threshold of 250 changes when the first reordering happens and the schedule
that follows, not a fixed cadence.

\begin{table*}[t]
\centering
\small
\caption{The seven timing rows reported for every synthesis configuration in
\resec{discussion}, in pipeline order. \textbf{Well-Separation}, \textbf{Realizability}, and
\textbf{Extraction} are internally reported times from \texttt{slugs}. The remaining four are Python
pipeline-harness wall-clock measurements and so include interpreter and subprocess start-up
overhead the three \texttt{slugs}-internal rows do not.}
\label{tbl:timing-rows}
\renewcommand{\arraystretch}{1.3}
\begin{tabular}{p{0.13\textwidth}p{0.55\textwidth}p{0.15\textwidth}}
\toprule
\textbf{Row} & \textbf{What It Measures} & \textbf{Timing Source} \\
\midrule
Generate Specs & Building the GR(1) specification from the request's capability subset (\resec{overview-stages}). & Python harness \\
Well-Separation & Assumption-falsification check on the compiled specification (\resec{well-separation-module}), a separate \texttt{slugs} invocation. & \texttt{slugs}-internal \\
Realizability & The \emph{Synthesis time} reported by \texttt{slugs}, the GR(1) fixed-point computation that decides realizability and, when realizable, yields the symbolic winning strategy as a byproduct rather than a separate pass. & \texttt{slugs}-internal \\
Extraction & The subsequent walk converting the symbolic strategy into the explicit Mealy machine \texttt{slugs} returns. Realizable rows only. & \texttt{slugs}-internal \\
Auditing & The post-synthesis auditor's pass over the extracted machine (\resec{auditor}). & Python harness \\
Realize HFSM & Translating the reduced strategy into an executable FlexBE HFSM and generating its initial GUI layout (\resec{overview-stages}). & Python harness \\
Overall & Full pipeline wall-clock time, including preprocessing and start-up overhead not attributed to any single stage above. It is \emph{not} simply the sum of the other six, and this fixed overhead can exceed their sum combined at millisecond scale. & Python harness \\
\bottomrule
\end{tabular}
\end{table*}

This data was gathered using Ubuntu 24.04 on a Dell Inc. Precision Tower 5810 with
Linux 6.17.0-35-generic, an Intel Xeon E5-1603 v3 CPU at 2.80GHz, 4 cores,
and 15GiB memory. Unless explicitly described as a reproduction on the
faster machine below, the tables and quantitative discussion report these
Tower 5810 results. As a further sanity check, a separate rerun on
another Tower 5810-class workstation reproduced the qualitative outcomes
and strategy sizes; differences were limited to wall-clock timing and
sub-1\% shifts in a few aggregate BDD-node means. This rerun is not
separately tabulated below.

To check whether the findings
below are an artifact of the specific hardware rather than the
specification structure they are attributed to,
we additionally reproduced the full experiment matrix, across all four
domains, on a third, substantially faster machine: a Dell Pro Max Tower
T2 with an Intel Core Ultra 9 285K CPU (24 cores) and 64\,GiB memory.
This is the
same machine used for the quadcopter hardware-synthesis validation of
\resec{crazyflie-failure-validation}. Realizability,
well-separation, strategy size, and audit outcome matched across every
domain and configuration, as with the same-class rerun above. Wall-clock
timing differed by a roughly consistent factor noted at the end of each
domain's discussion below, and BDD-node means again shifted only
slightly, without changing any comparison or conclusion. The
one exception is enumerated \textbf{2FP} for the quadcopter, which did not
resolve within 24\,h on the original machine but reached a verdict in 16.6\,h
on this faster one (\resec{crazyflie-reliability}). These were not matched,
budget-normalized runs; the follow-up establishes status rather than a
controlled machine-speed comparison.

Reported means and standard deviations are compact summaries, not a claim of
normality. Random-ordering tables summarize all trials: Coffee and
Two-Rivers use 40 orderings with five repeats each; P0, the P1
\textbf{4S}/\textbf{4F} sweep, and the quadcopter use 20 with one trial
each. Repeats measure timing noise, not independent orderings; state counts
are identical across repeats. For a robustness check, we averaged repeated
realizability timings within each ordering before computing medians.
Across the five random-ordering System-Goal/no-pending encoding comparisons,
medians preserve every mean-based ranking; P1's single-ordering encoding
comparison remains separate. Several distributions
below are skewed, with outliers noted per domain where they matter most
(\resec{crazyflie-reliability} discusses the most extreme case).
The Coffee ordering-robustness sweep gives the largest direct test of this
issue: across 7{,}140 synthesis trials, an explicit CUDD reordering threshold
removed most initial declaration-order sensitivity where such sensitivity
appeared, making the threshold a more reliable practical lever than
hand-tuning declaration order at Coffee scale (\resec{coffee-results},
\seedetail{coffee-ordering-app}). That advice does not carry over unchanged
to larger models, where the threshold changed PyRoboSim P1's \textbf{4S} cost
by more than an order of magnitude relative to the uncapped policy
(\seedetail{pyrobosim-p1-ordering-app}).

One asymmetry limits the random-ordering comparison across encodings. The
enumerated encoding's shared \texttt{capability} declaration is always
written first, whereas every one-hot proposition is shuffled, so the random
sweeps sample a narrower ordering space for enumerated encoding.
The comparisons therefore characterize the encodings under this generator's
ordering policy, without isolating encoding from declaration placement
(\seedetail{two-rivers-ordering-app}).

\input{CoffeeDiscussion}

\input{TwoRiversDiscussion}

\input{PyRoboSimDiscussion}

\input{CrazyFlieDiscussion}

\subsection{Summary}
\label{sec:summary-results}

Across the evaluated domains, the results demonstrate that reducing
proposition count via the enumerated encoding does
not automatically translate into a single, predictable effect on
strategy size or symbolic synthesis cost. The direction of the effect is
domain- and configuration-dependent rather than a uniform trade-off.
In the Two-Rivers System-Goal comparison, enumerated encoding used fewer
propositions. Under randomized declaration orderings, it had a larger mean
peak BDD size and took longer on average to check realizability. Under both
fixed orderings, it had a smaller peak BDD size and shorter realizability
time (\resec{two-rivers-results}, \seedetail{two-rivers-ordering-app}).
Coffee and the quadcopter case study instead show enumerated encoding
synthesizing faster, with mixed but comparable symbolic footprints, despite
producing a \emph{larger} unreduced strategy in most configurations
(\resec{coffee-results}, \resec{crazyflie-results}). After state reduction,
Coffee's reduced-strategy sizes stay within one state across encodings
(3.5 vs.\ 4 in the base set and 3 vs.\ 4 in the extended set), but the
quadcopter case study's do not. Enumerated encoding's reduced strategy
remains substantially larger than one-hot's even after reduction
(106 vs.\ 76.7 states for System-Goal, \resec{crazyflie-results}).
Controller realization shrinks the enumerated strategy far more than the
one-hot one, reversing the gap. In the alphabetic row of the appendix
realization table, \textbf{1S} realizes to 69 FlexBE states against 54 for
\textbf{2S}; since \textbf{1S}'s reduced size varies with ordering
($76.7 \pm 5.1$), this comparison should not be read as fixed by the
encoding alone (\resec{controller-realization},
\seedetail{crazyflie-realization-app}).
Reduced-strategy size therefore should not be read as a proxy for the size of
the executable state machine.
The reducer used for these runs can repeat its local state-merge sweep to a
fixed point, but this did not change the reported reduced-strategy sizes. Every
experiment reached that fixed point after at most one productive merge sweep;
any additional sweep only confirmed that no further local merges were enabled.
PyRoboSim P0 shows enumerated encoding as both smaller and faster than
one-hot across 20 random-seed orderings; the P1 single-trial comparison,
run under the uncapped reordering policy, shows the same direction, with a
narrower threshold-250 appendix sweep confirming the reported
\textbf{4S}/\textbf{4F} realizability verdicts.
No single direction, smaller-but-slower or otherwise, holds across all four
domains studied here, although PyRoboSim and the quadcopter case study show
substantial synthesis-cost savings from the enumerated encoding.

Within each domain, the liveness formulation had a clear practical
effect on synthesis cost and structural audit outcomes. System-Goal Liveness was the
preferred formulation on those measures.
Fair-Outcome Liveness increased synthesis difficulty in the main
realizable pending comparisons and most dramatically in the quadcopter case
study, while small flat cases and unrealizable no-pending checks do not follow
a uniform larger-and-slower pattern.

Inspection of the synthesized strategies also showed that formal realizability
does not by itself establish designer-intended progress.
In particular, the interaction among persistent pending memory,
activation-gated outcomes, and environment fairness assumptions admitted
cyclic retry structures that did not ensure the intended terminal condition.
First identified in PyRoboSim (\resec{pyrobosim-p0-results},
\resec{pyrobosim-p1-results}), the same \texttt{GOAL\_UNREACHABLE\_TRAP}
pattern was also confirmed in the quadcopter synthesis model
(\resec{crazyflie-reliability}), recurring in both the one-hot and enumerated
encodings of the Fair-Outcome-with-pending configuration.
This is an instance of a known GR(1) hazard, a strategy winning by driving
the environment out of its own assumptions, reviewed in \resec{synthesis}.
Our pipeline checks this risk before synthesis with the well-separation
module (\resec{well-separation-module}) and after synthesis with the auditor
(\resec{auditor}) that detects concrete progress failures on extracted
strategies rather than requiring manual inspection.

Every realized System-Goal strategy passed the audit. The trap check
establishes that a path to the goal remains available from each reachable
state. An execution may still stay indefinitely in a cycle with an unforced
exit toward the goal, so the verdict leaves open whether every fair execution
reaches the goal (Corollary~\ref{cor:unforced-exit}).

\begin{table*}[bt]
\centering
\caption{System-Goal liveness without pending memory, both encodings,
across every domain/subdomain studied. Within the tested budget, this is the only formulation confirmed
as a realized FlexBE HFSM across all six domain/subdomain comparisons under
both encodings. \resec{coffee-results}, \resec{two-rivers-results},
\resec{pyrobosim-p0-results}, \resec{pyrobosim-p1-results}, and
\resec{crazyflie-results} report the corresponding Fair-Outcome and
pending-memory comparisons in full. Bold marks the smaller value in each
one-hot/enumerated pair per metric (ties left unbolded). Numerical values
are means over random declaration orderings except for the PyRoboSim P1
rows, which are single alphabetic trials under the uncapped (``auto'') reordering policy,
while the other rows use threshold 250. Coffee's base
capability set has no failure-prone capabilities, so its liveness variants
coincide and that row informs only the encoding axis.}
\label{tbl:summary-sg-nopending}
\input{tables/summary_sg_nopending}
\end{table*}

The six rows treat Coffee's base and extended capability sets, Two-Rivers,
PyRoboSim P0, PyRoboSim P1, and the quadcopter domain as separate
domain/subdomain comparisons; the quadcopter domain contributes one comparison.
As in the experimental-design summary above, Coffee base is included only as an
encoding comparison because its liveness variants coincide.
The winning encoding is not consistent across these comparisons, and
reduced-strategy size converges to exactly the same value regardless of
encoding only in Two-Rivers and the two PyRoboSim models.
System-Goal with pending memory is not established the same way, because its
one-hot PyRoboSim P1 configuration did not finish within the 10\,h budget,
which is not evidence of unrealizability.

With System-Goal and no pending memory, enumerated encoding yielded shorter
realizability times in five of the six reported comparisons. The Two-Rivers
result in this table uses randomized declaration orderings.

System-Goal is therefore the practical default for the capability and
outcome model evaluated here, not a general GR(1) dominance claim.
Fair-Outcome without pending fails to realize at all in Two-Rivers and both
PyRoboSim configurations; where Fair-Outcome does realize, System-Goal is
usually faster, with Coffee's sub-millisecond base cases too small to
interpret strongly. More importantly, System-Goal does not introduce the
supervisor-controlled environment promise that produced the Fair-Outcome
retry pathology.

Overall, the experiments show that a generated specification should be
evaluated not only by realizability and proposition count, but also by its
symbolic synthesis cost, the operational validity of its environment
assumptions, and the progress properties of the extracted strategy. In our
experiments every \texttt{GOAL\_UNREACHABLE\_TRAP} audit failure occurred in
a specification already flagged as non-well-separated, though the converse
did not hold: several non-well-separated specifications still produced
auditor-valid strategies that were realized as FlexBE HFSMs.
A clean audit rules out the defined structural defects in the extracted
strategy and leaves open whether the strategy preserves the environment's
ability to satisfy its assumptions. Whether well-separation is necessary to
rule out this class of traps in general remains open.
In the paired P0 experiment, rewriting the carriage exclusions in
next-state form removed one cause of non-well-separation while preserving
the trap verdicts (\resec{pyrobosim-p0-results}).

%% file: CoffeeDiscussion.tex
\subsection{Coffee Demonstration}
\label{sec:coffee-results}

For the Coffee demonstration, we compare the hand-written baseline against
generated capability specifications under the two capability sets used
throughout this section:
\begin{enumerate}
    \item the hand-written baseline,
    \item the base generated capability specification, and
    \item the extended generated capability specification with completion and failure outcomes.
\end{enumerate}

\retab{coffee-1hot-results-summary} and \retab{coffee-enumerated-results-summary}
summarize the one-hot and enumerated results, respectively. The full detail
tables and accompanying discussion appear in
\seedetail{coffee-results-app}.

\refig{coffee-extended-reduction} illustrates what state-merging reduction
(\resec{reduction}) concretely does to one of these strategies before the
tables below quantify it across all configurations.

\begin{table*}[ht]
\centering
\caption{\textbf{[$N=40$ random-seed orderings.]} Coffee synthesis results with one-hot labeling. Reported values are the
mean over 40 random declaration orderings under a CUDD reordering threshold of
250. Integer-valued count means indicate identical samples, while non-identical
count means are shown with one decimal digit, with a companion row giving
$\pm$ one standard deviation wherever the underlying trials disagree. Column~\textbf{1} is the
hand-written baseline. The leading digit
\textbf{2}/\textbf{3} denotes the base/extended capability set, \textbf{S}/\textbf{F}
denotes System-Goal/Fair-Outcome liveness, and \textbf{P} adds pending
propositions.}
\label{tbl:coffee-1hot-results-summary}
\small
\renewcommand{\arraystretch}{1.25}
\setlength{\tabcolsep}{3pt}
\input{tables/coffee_one_hot_summary}
\end{table*}

\begin{table*}[ht]
\centering
\caption{\textbf{[$N=40$ random-seed orderings.]} Coffee synthesis results with the enumerated capability encoding
(same reporting convention and column legend as
\retab{coffee-1hot-results-summary}). The enumerated encoding has no
hand-written baseline, so column~\textbf{1} is omitted.}
\label{tbl:coffee-enumerated-results-summary}
\renewcommand{\arraystretch}{1.4}
\setlength{\tabcolsep}{3pt}
\input{tables/coffee_enumerated-summary}
\end{table*}

\begin{figure*}[t]
    \centering
    \textbf{Symbolic strategy extracted by \texttt{slugs} (7 states)}\par
    \includegraphics[width=\linewidth]{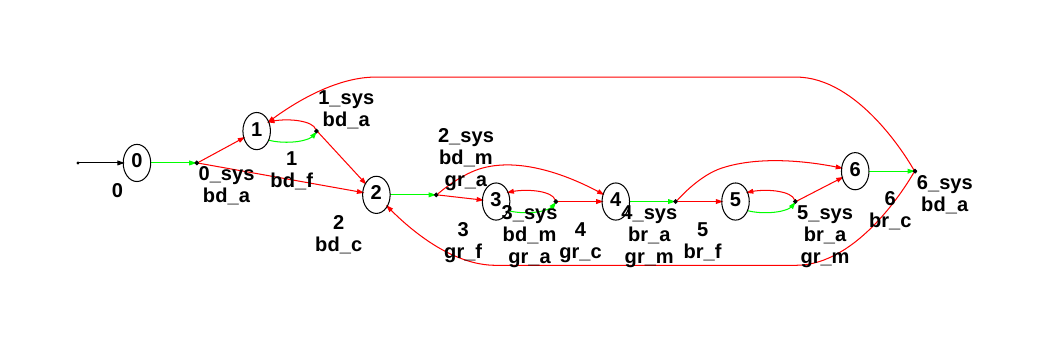}\par
    \vspace{6pt}
    \textbf{After state-merging reduction (3 states)}\par
    \includegraphics[width=\linewidth]{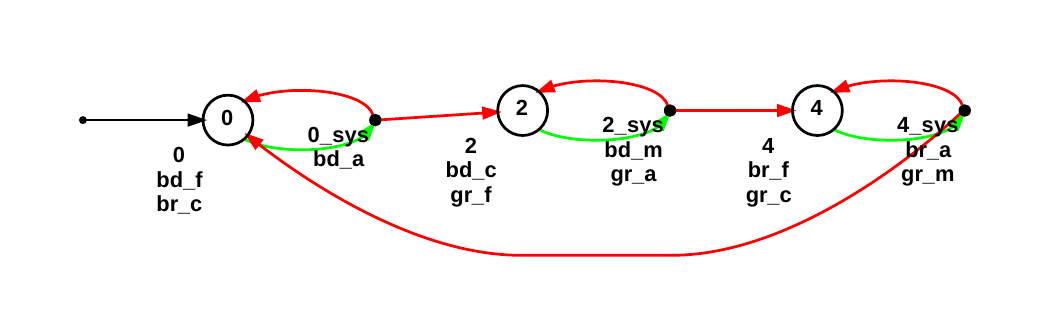}
    \caption{The extended Coffee example's raw synthesized Mealy strategy and
    its state-merged reduction (\resec{reduction}), using the
    \texttt{mealy2dot} visual convention introduced in
    \refig{slugs-example-mealy} (\resec{slugs-example}). Each of the three capabilities'
    \texttt{fail} outcome (\texttt{bd\_f}, \texttt{gr\_f}, \texttt{br\_f})
    loops back to re-activate the same capability rather than advancing, so
    the raw strategy repeats the same activation--fail--reactivate pattern
    at three points in its cycle. Reduction collapses each of those points to
    a single state carrying its own self-loop, leaving exactly the three
    states realized as the executable FlexBE state machine in
    \refig{coffee_sm}. This one-hot extended-capability, System-Goal,
    no-pending configuration reduces from 7 to 3 states identically across
    all 40 declaration orderings (\textbf{3S} in
    \retab{coffee-1hot-results-summary}) -- unlike the smaller non-uniqueness
    case discussed below, this reduction is deterministic.}
    \label{fig:coffee-extended-reduction}
\end{figure*}

Across Tables~\ref{tbl:coffee-1hot-results-summary}
and~\ref{tbl:coffee-enumerated-results-summary},
all Coffee configurations are realizable and produce a
valid FlexBE controller, with the extended Fair-Outcome-pending
configuration (\textbf{3FP}) the sole exception to well-separation (below).
The enumerated encoding uses fewer input propositions but slightly larger
strategies. For extended Coffee under System-Goal without pending memory,
its four reduced states comprise the synthetic \texttt{null} startup state
and the same three action states as one-hot. Realizability time favors
the enumerated encoding in the reported \texttt{slugs} timings, though the
base-capability differences are sub-millisecond and should not be
overinterpreted; the clearest gaps appear in the extended-capability
columns. This comparison, and every symbolic-cost
comparison in this paper, reflects the CUDD-based BDD representation used by
\texttt{slugs} specifically.
\resec{conclusion} discusses extending the pipeline to other GR(1) backends
to test whether these encoding-cost effects generalize beyond it.
Peak node counts are quantized in blocks of 1{,}022 nodes
(\resec{discussion}). This is negligible for the multi-million-node peaks of
the larger domains, but Coffee peaks range only from 1{,}022 to 9{,}198, so a
Coffee peak difference is one or a few blocks and should not be over-read.
\seedetail{coffee-results-app} reports the full
specification-size and BDD-node breakdown behind
these headline numbers, including proposition and clause counts,
live and peak BDD nodes, and the remaining per-stage timing rows.

The well-separation rows isolate a different issue from realizability or
controller extraction. The extended Fair-Outcome-with-pending specification
(\textbf{3FP}) is the sole Coffee configuration that is not well-separated
under either encoding, though all Coffee specifications remain realizable
and all extracted strategies pass the post-synthesis auditor. The
Fair-Outcome pending rule ties the environment's grind-completion
obligation to a system-controlled retry context: the \texttt{P-all/E-just}
case in the Maoz--Ringert classification~\cite{MaozRingert2016}, exactly the
assumption-falsification risk well-separation is designed to expose.
Still, this is a specification diagnostic,
not an unusable controller. The synthesized \textbf{3FP} controller is
itself small, valid, and executable. The check remains inexpensive
throughout, sub-millisecond and well under one percent of \textbf{Overall}
even at its highest observed value. \seedetail{coffee-results-app} gives
the minimized-core assumption breakdown and per-configuration check timing.

Variable declaration order also affects \emph{which} of several equally realizable
strategies \texttt{slugs} extracts, not only the size of the symbolic
representation examined below. The non-integer \textbar{}Slugs SM\textbar{}
and \textbar{}Reduced\textbar{} means reported for the one-hot
base-capability columns (\textbf{2S}/\textbf{2F}/\textbf{2SP}/\textbf{2FP},
4.5 and 3.5 respectively) in
\retab{coffee-1hot-results-summary}
reflect genuine strategy non-uniqueness under the
generated specification, not measurement noise. An unconstrained ordering
choice between two capability activations lets \texttt{slugs} legally
extract either a direct four-state cycle or a five-state variant with one
redundant re-activation, both equally valid strategies for the identical
specification, and the choice correlates exactly with declaration order
across the 40 one-hot \textbf{2S} random-seed trials
(\refig{coffee-nonuniqueness}). 
The enumerated encoding, with its fixed numeric assignment, does not show
the same non-uniqueness in this comparison. \seedetail{coffee-results-app}
gives the specification-level mechanism and the exact per-ordering split.

\begin{figure*}[t]
    \centering
    \textbf{Direct four-state cycle (22/40 seeds; \texttt{gr\_a} not declared first)}\par
    \includegraphics[width=.90\linewidth]{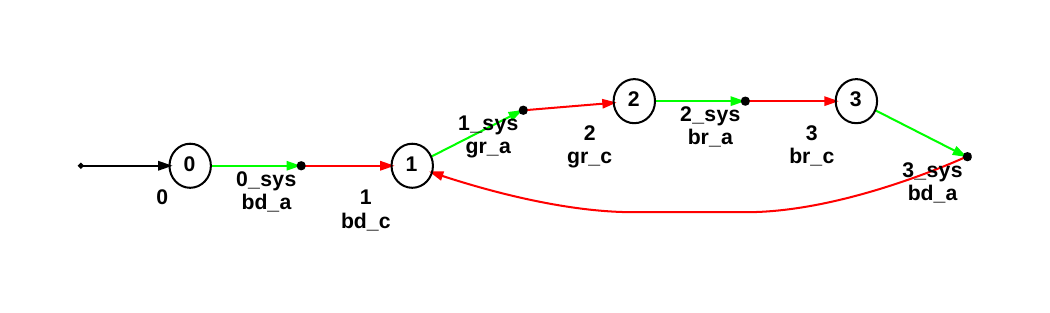}\par
    \vspace{2pt}
    \textbf{Reduced to three states}\par
    \includegraphics[width=.90\linewidth]{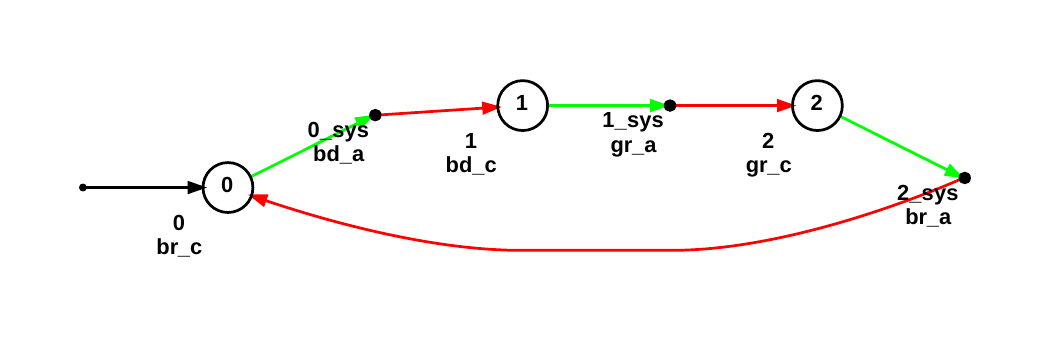}\par
    \vspace{5pt}
    \textbf{Five-state variant with redundant \texttt{bd} activation (18/40 seeds; \texttt{gr\_a} declared first)}\par
    \includegraphics[width=.90\linewidth]{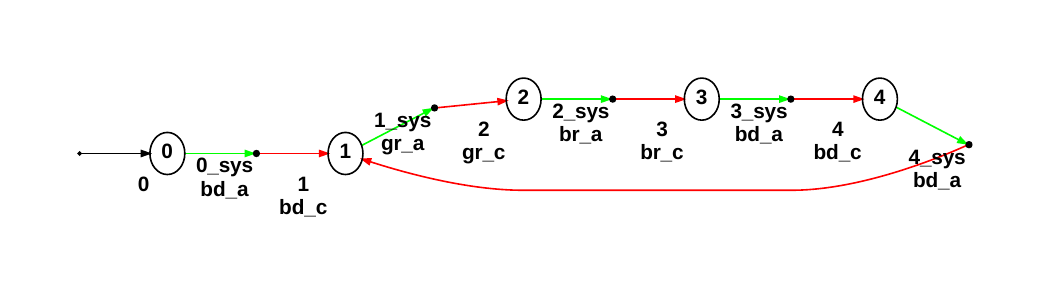}\par
    \vspace{2pt}
    \textbf{Reduced to four states}\par
    \includegraphics[width=.90\linewidth]{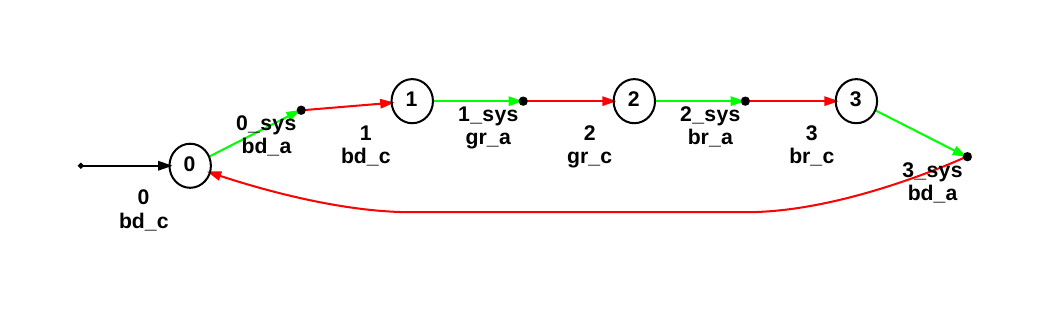}\par
    \caption{Representative symbolic (raw, pre-reduction) and reduced Mealy
    graphs for the two strategy shapes discussed above, both from the
    one-hot \textbf{2S} configuration, using the \texttt{mealy2dot} visual
    convention introduced in \refig{slugs-example-mealy}
    (\resec{slugs-example}). In the top pair, runs in which
    \texttt{gr\_a} is not declared first among the output propositions yield
    the direct four-state cycle, which state-merging reduces to three states.
    In the bottom pair, declaring \texttt{gr\_a} first yields a fifth,
    redundant \texttt{bd} re-activation state before grinding. This extra
    state survives reduction, leaving four states rather than three. Both
    reduced strategies are valid and executable for the identical
    specification.}
    \label{fig:coffee-nonuniqueness}
\end{figure*}

In the Coffee sweep, peak BDD size varied with declaration order only in
extended configurations, especially under one-hot encoding. A CUDD
reordering threshold of 250 live nodes reduced most of that sensitivity.
In the most sensitive case, one-hot Fair-Outcome with pending memory, the
threshold narrowed the gap in mean peak BDD node counts between random and
fixed orderings from about 869 to 128 nodes.
\retab{coffee-ordering-summary-app} gives the per-treatment values, and
\seedetail{coffee-ordering-app} reports the full sweep.

We confirmed this whole picture on the faster machine introduced in
\resec{discussion}. The same 7{,}140-trial Coffee sweep, reproduced there,
reached the same realizability, well-separation, and strategy-size
results throughout, with realizability and overall pipeline times roughly
1.4--3$\times$ faster across configurations.

\FloatBarrier

%% file: tables/coffee_one_hot_summary.tex
\begin{tabular}{|c|r|r|r|r|r|r|r|r|r|}
\hline
 & \textbf{1} & \textbf{2S} & \textbf{3S} & \textbf{2F} & \textbf{3F} & \textbf{2SP} & \textbf{3SP} & \textbf{2FP} & \textbf{3FP} \\
\hline
\rule{0pt}{4ex}\textbf{\shortstack{\rule{0pt}{2.8ex}Well\\Separated\rule[-1ex]{0pt}{1ex}}} & Yes & Yes & Yes & Yes & Yes & Yes & Yes & Yes & No \\
\hline
\textbf{Realizable} & Yes & Yes & Yes & Yes & Yes & Yes & Yes & Yes & Yes \\
\hline
\rule{0pt}{4ex}\textbf{\shortstack{\rule{0pt}{2.8ex}FlexBE HFSM\\Realized\rule[-1ex]{0pt}{1ex}}} & Yes & Yes & Yes & Yes & Yes & Yes & Yes & Yes & Yes \\
\hline
\textbf{$\left|AP_I\right|$} & 4 & 3 & 6 & 3 & 6 & 3 & 6 & 3 & 6 \\
\hline
\textbf{$\left|AP_O\right|$} & 3 & 3 & 5 & 3 & 5 & 3 & 8 & 3 & 8 \\
\hline
\textbf{\textbar{}Slugs SM\textbar} & 7 & 4.5 & 7 & 4.5 & 7 & 4.5 & 7 & 4.5 & 7 \\
\hline
 &  & $\pm$ 0.5 &  & $\pm$ 0.5 &  & $\pm$ 0.5 &  & $\pm$ 0.5 &  \\
\hline
\textbf{\textbar{}Reduced\textbar} & 3 & 3.5 & 3 & 3.5 & 3 & 3.5 & 3 & 3.5 & 3 \\
\hline
 &  & $\pm$ 0.5 &  & $\pm$ 0.5 &  & $\pm$ 0.5 &  & $\pm$ 0.5 &  \\
\hline
\rule{0pt}{4ex}\textbf{\shortstack{\rule{0pt}{2.8ex}Realizability\\(ms)\rule[-1ex]{0pt}{1ex}}} & 0.231 & 0.081 & 2.149 & 0.081 & 3.465 & 0.088 & 2.782 & 0.082 & 6.316 \\
\hline
 & $\pm$ 0.252 & $\pm$ 0.011 & $\pm$ 0.758 & $\pm$ 0.012 & $\pm$ 1.591 & $\pm$ 0.084 & $\pm$ 1.301 & $\pm$ 0.015 & $\pm$ 2.230 \\
\hline
\rule{0pt}{4ex}\textbf{\shortstack{\rule{0pt}{2.8ex}Realize\\HFSM (ms)\rule[-1ex]{0pt}{1ex}}} & 126.492 & 119.206 & 128.289 & 120.229 & 128.522 & 119.985 & 129.209 & 119.671 & 128.766 \\
\hline
 & $\pm$ 15.333 & $\pm$ 7.081 & $\pm$ 10.274 & $\pm$ 8.214 & $\pm$ 6.426 & $\pm$ 7.705 & $\pm$ 7.239 & $\pm$ 7.128 & $\pm$ 6.209 \\
\hline
\rule{0pt}{4ex}\textbf{\shortstack{\rule{0pt}{2.8ex}Overall\\(ms)\rule[-1ex]{0pt}{1ex}}} & 199.273 & 187.268 & 208.771 & 186.761 & 210.066 & 188.341 & 218.585 & 188.579 & 221.405 \\
\hline
 & $\pm$ 23.069 & $\pm$ 8.482 & $\pm$ 11.877 & $\pm$ 8.207 & $\pm$ 7.618 & $\pm$ 8.925 & $\pm$ 9.269 & $\pm$ 7.976 & $\pm$ 7.997 \\
\hline
\end{tabular}

%% file: tables/coffee_enumerated-summary.tex
\begin{tabular}{|c|r|r|r|r|r|r|r|r|}
\hline
 & \textbf{2S} & \textbf{3S} & \textbf{2F} & \textbf{3F} & \textbf{2SP} & \textbf{3SP} & \textbf{2FP} & \textbf{3FP} \\
\hline
\rule{0pt}{4ex}\textbf{\shortstack{\rule{0pt}{2.8ex}Well\\Separated\rule[-1ex]{0pt}{1ex}}} & Yes & Yes & Yes & Yes & Yes & Yes & Yes & No \\
\hline
\textbf{Realizable} & Yes & Yes & Yes & Yes & Yes & Yes & Yes & Yes \\
\hline
\rule{0pt}{4ex}\textbf{\shortstack{\rule{0pt}{2.8ex}FlexBE HFSM\\Realized\rule[-1ex]{0pt}{1ex}}} & Yes & Yes & Yes & Yes & Yes & Yes & Yes & Yes \\
\hline
\textbf{$\left|AP_I\right|$} & 2 & 2 & 2 & 2 & 2 & 2 & 2 & 2 \\
\hline
\textbf{$\left|AP_O\right|$} & 2 & 4 & 2 & 4 & 2 & 7 & 2 & 7 \\
\hline
\textbf{\textbar{}Slugs SM\textbar} & 6 & 8 & 6 & 8 & 6 & 8 & 6 & 8 \\
\hline
\textbf{\textbar{}Reduced\textbar} & 4 & 4 & 4 & 4 & 4 & 4 & 4 & 4 \\
\hline
\rule{0pt}{4ex}\textbf{\shortstack{\rule{0pt}{2.8ex}Realizability\\(ms)\rule[-1ex]{0pt}{1ex}}} & 0.050 & 0.059 & 0.048 & 0.127 & 0.049 & 0.625 & 0.049 & 3.127 \\
\hline
 & $\pm$ 0.029 & $\pm$ 0.004 & $\pm$ 0.004 & $\pm$ 0.068 & $\pm$ 0.010 & $\pm$ 0.703 & $\pm$ 0.008 & $\pm$ 1.015 \\
\hline
\rule{0pt}{4ex}\textbf{\shortstack{\rule{0pt}{2.8ex}Realize\\HFSM (ms)\rule[-1ex]{0pt}{1ex}}} & 118.138 & 126.725 & 117.459 & 127.876 & 118.096 & 127.853 & 117.283 & 128.133 \\
\hline
 & $\pm$ 7.580 & $\pm$ 5.731 & $\pm$ 6.042 & $\pm$ 7.035 & $\pm$ 6.359 & $\pm$ 6.357 & $\pm$ 7.725 & $\pm$ 6.654 \\
\hline
\rule{0pt}{4ex}\textbf{\shortstack{\rule{0pt}{2.8ex}Overall\\(ms)\rule[-1ex]{0pt}{1ex}}} & 188.318 & 211.843 & 188.635 & 212.031 & 189.694 & 220.097 & 188.992 & 218.873 \\
\hline
 & $\pm$ 7.810 & $\pm$ 7.989 & $\pm$ 6.781 & $\pm$ 9.077 & $\pm$ 7.940 & $\pm$ 8.789 & $\pm$ 8.193 & $\pm$ 8.544 \\
\hline
\end{tabular}

%% file: TwoRiversDiscussion.tex
\subsection{Two-Rivers Demonstration}
\label{sec:two-rivers-results}

For the Two-Rivers demonstration, we consider three variants:

\begin{enumerate}
    \item {A full specification with one-hot encoded movements
    (e.g. \texttt{move\_goat\_a}, \texttt{move\_farmer\_a}) and an
    enumerated location value for each item
    (e.g. \texttt{farmer:0...2}, encoded with two bits).
    }
    \item {A shared one-hot movement proposition,
    \texttt{move\_state\_a}, with an enumerated item parameter
    \texttt{move\_item:0...3}.}
    \item {The enumerated capability encoding, otherwise matching
    item 2.}
\end{enumerate}

Item 1 is included only as a System-Goal, no-pending baseline. To save
space, \retab{rivers-results-summary} itself excludes both hand-written baselines,
\texttt{full\_spec} (item 1's per-item one-hot movement encoding) and
\texttt{full\_spec\_parsed} (hand-written using item 2/3's shared
\texttt{move\_state\_a}/\texttt{move\_item} structure). Both are realizable,
well-separated, and produce a valid FlexBE HFSM across all 40 random
declaration orderings. Their symbolic cost is in the same order of
magnitude as the capability-generated columns.
\seedetail{two-rivers-results-app} gives the exact comparison.

\retab{rivers-results-summary} summarizes the Two-Rivers results by realizability,
well-separation, strategy size, and the three headline timing rows. The full
specification-size and BDD-node breakdown, and accompanying discussion, appear in
\seedetail{two-rivers-results-app}.
The table compares the automatically generated one-hot capability encoding
(2) with the enumerated capability encoding (3). Each is evaluated
under System-Goal (S) and Fair-Outcome (F) liveness, with and without
pending (P) bits. In both generated encodings, the specification logs the
finished state and any failure outcomes. Without pending bits, the
Fair-Outcome liveness conditions make both specifications
\emph{unrealizable}; the System-Goal specifications do not require them.
The generated constraints explain this result: a failed move forces
\texttt{log\_failed} to activate, so mutual exclusion makes
\texttt{move\_state\_a} false during the logging step. Thus
$GF(\neg\pi_a \lor X\pi_a^c)$ is satisfied on the intervening logging
steps even if every move fails. The environment can always report move
failure and complete the required logging action, preserving all initial
entity locations while satisfying its safety and liveness assumptions.
The supervisor therefore cannot guarantee $GF\,\texttt{finished}$.
In contrast, extended Coffee permits uninterrupted immediate retries:
keeping the failed capability active until completion prevents the
environment from satisfying this assumption through inactivity alone.

\begin{table*}[htbp]
\centering
\caption{\textbf{[$N=40$ random-seed orderings.]} Two-Rivers synthesis results. Reported values are means over 40 random
declaration orderings under a CUDD reordering threshold of 250 ($N=200$ trials
per column, 40 orderings $\times$ 5 repeats). Integer-valued count means
indicate identical samples, while non-identical count means are shown with one
decimal digit, with a companion row giving $\pm$ one standard
deviation wherever the underlying trials disagree. Columns marked
unrealizable have no extracted strategy.}
\label{tbl:rivers-results-summary}
\renewcommand{\arraystretch}{1.3}
\setlength{\tabcolsep}{3pt}
\input{tables/two_rivers_summary}
\end{table*}

Across the realizable Two-Rivers variants, enumerated encoding reduces
proposition counts and yields smaller raw extracted strategies
(\textbar{}Slugs SM\textbar{} drops from 37 to 26, while
\textbar{}Reduced\textbar{} ties at 23), but it does not
uniformly reduce symbolic synthesis cost. Under System-Goal liveness,
enumerated encoding increases the random-ordering mean peak BDD size
relative to one-hot encoding. This reversal does not hold under fixed
declaration order; \seedetail{two-rivers-ordering-app} traces it to an
ordering-sensitivity asymmetry between the encodings on this domain, not to
a genuine encoding-superiority reversal. Mean realizability time under the
random orderings is likewise higher for enumerated encoding in all three
realizable pairs (53.8 vs.\ 90.7\,ms, 59.5 vs.\ 67.9\,ms, and 76.6 vs.\
82.9\,ms), though each gap lies within one standard deviation of the more
variable column.

The Fair-Outcome formulation without pending memory is unrealizable for both
encodings. Pending memory restores realizability, but within each encoding it
produces the largest peak BDDs among the reported realizable
capability-based variants. Together
with the ordering-sensitivity reversal above, these results highlight the
sensitivity of GR(1) synthesis performance to specification structure
rather than merely problem size.
\seedetail{two-rivers-results-app} reports the exact BDD-node figures behind
these headline numbers.

The well-separation rows isolate the same effect seen in Coffee's extended
Fair-Outcome-pending case (\resec{coffee-results}). Every System-Goal
specification is well-separated, while every Fair-Outcome specification is
not, the same \texttt{P-all/E-just} violation tying the environment's
completion obligation to a system-controlled activation or pending
context. Pending memory repairs realizability for \textbf{2FP}/\textbf{3FP}
without repairing this underlying assumption-structure issue. Both remain
NWS despite being realizable and auditor-valid, a specification-level
hazard rather than a synthesizer failure. The NWS check itself remains cheap
relative to the rest of the pipeline throughout, though less uniformly so
than in Coffee. \seedetail{two-rivers-results-app} gives the representative
minimized well-separation cores and per-configuration check timing.

The same reproduction on the faster machine (\resec{discussion})
confirmed Two-Rivers' realizability, well-separation, and strategy-size
pattern, including the random-seed encoding reversal traced to ordering
sensitivity above, with realizability times roughly 1.7--3$\times$ faster
across configurations.

\FloatBarrier

%% file: tables/two_rivers_summary.tex
\begin{tabular}{|c|r|r|r|r|r|r|r|r|}
\hline
 & \textbf{2S} & \textbf{3S} & \textbf{2F} & \textbf{3F} & \textbf{2SP} & \textbf{3SP} & \textbf{2FP} & \textbf{3FP} \\
\hline
\rule{0pt}{4ex}\textbf{\shortstack{\rule{0pt}{2.8ex}Well\\Separated\rule[-1ex]{0pt}{1ex}}} & Yes & Yes & No & No & Yes & Yes & No & No \\
\hline
\textbf{Realizable} & Yes & Yes & No & No & Yes & Yes & Yes & Yes \\
\hline
\rule{0pt}{4ex}\textbf{\shortstack{\rule{0pt}{2.8ex}FlexBE HFSM\\Realized\rule[-1ex]{0pt}{1ex}}} & Yes & Yes & No & No & Yes & Yes & Yes & Yes \\
\hline
\textbf{$\left|AP_I\right|$} & 13 & 10 & 13 & 10 & 13 & 10 & 13 & 10 \\
\hline
\textbf{$\left|AP_O\right|$} & 10 & 8 & 10 & 8 & 11 & 9 & 11 & 9 \\
\hline
\textbf{\textbar{}Slugs SM\textbar} & 37 & 26 & -- & -- & 37 & 26 & 37 & 26 \\
\hline
\textbf{\textbar{}Reduced\textbar} & 23 & 23 & -- & -- & 23 & 23 & 23 & 23 \\
\hline
\rule{0pt}{4ex}\textbf{\shortstack{\rule{0pt}{2.8ex}Realizability\\(ms)\rule[-1ex]{0pt}{1ex}}} & 53.787 & 90.677 & 8.803 & 16.953 & 59.478 & 67.922 & 76.616 & 82.875 \\
\hline
 & $\pm$ 9.562 & $\pm$ 44.829 & $\pm$ 2.997 & $\pm$ 8.673 & $\pm$ 10.509 & $\pm$ 23.562 & $\pm$ 20.325 & $\pm$ 28.650 \\
\hline
\rule{0pt}{4ex}\textbf{\shortstack{\rule{0pt}{2.8ex}Realize\\HFSM (ms)\rule[-1ex]{0pt}{1ex}}} & 671.875 & 665.201 & -- & -- & 676.240 & 666.855 & 674.775 & 672.307 \\
\hline
 & $\pm$ 13.727 & $\pm$ 8.893 &  &  & $\pm$ 10.789 & $\pm$ 21.182 & $\pm$ 10.395 & $\pm$ 24.959 \\
\hline
\rule{0pt}{4ex}\textbf{\shortstack{\rule{0pt}{2.8ex}Overall\\(ms)\rule[-1ex]{0pt}{1ex}}} & 858.225 & 901.953 & 129.821 & 144.565 & 875.333 & 883.916 & 890.593 & 906.965 \\
\hline
 & $\pm$ 19.771 & $\pm$ 56.821 & $\pm$ 12.336 & $\pm$ 13.946 & $\pm$ 17.532 & $\pm$ 41.879 & $\pm$ 23.688 & $\pm$ 53.419 \\
\hline
\end{tabular}

%% file: PyRoboSimDiscussion.tex
\subsection{PyRoboSim Demonstration}
\label{sec:pyrobosim-results}

This section reports quantitative synthesis results for the two PyRoboSim
capability configurations introduced in \resec{pyrobosim}. P0
abstracts door state away entirely and is our baseline. \resec{pyrobosim-p0-results}
evaluates it across the same encoding/liveness/pending matrix used for the
Coffee and Two-Rivers case studies (\resec{coffee-results},
\resec{two-rivers-results}), applied to the larger PyRoboSim logistics
domain, which tracks six movable objects. P1 then adds the \texttt{open} capability and
door state on top of P0 (\resec{pyrobosim}). \resec{pyrobosim-p1-results}
uses it to test whether P0's conclusions hold as domain complexity grows,
rather than to repeat the full sweep. Per-trial synthesis cost at P1's
scale rules out that full sweep in any case (\resec{pyrobosim-p1-results}),
so the two configurations are reported as separate subsections with
separate tables rather than columns of a single matrix.

\subsubsection{P0: Baseline (No Door Modeling)}
\label{sec:pyrobosim-p0-results}

A single fixed-configuration synthesis of the P0 world was previously
reported for this domain~\cite{eit26-synthesis}. Here, we extend that
single data point into the same encoding/liveness sweep used for the
Coffee and Two-Rivers case studies (\resec{coffee-results},
\resec{two-rivers-results}), applied to the PyRoboSim logistics domain
of \resec{pyrobosim}.

For the PyRoboSim P0 demonstration, we consider the following configurations.
\begin{enumerate}
    \item Two hand-written baselines for the P0 world:
    \begin{enumerate}
        \item[a)] a reduced move-only specification
        (\texttt{p0\_move\_only\_spec}, \textbf{MoveOnly}) that tracks only
        robot location, move outcomes, and the finished state.
        \item[b)] a full hand-written specification (\texttt{p0\_full\_spec},
        \textbf{Full}).
    \end{enumerate}
    \item An automatically generated one-hot capability encoding
    (\texttt{p0\_capabilities}), evaluated under both System-Goal (S) and
    Fair-Outcome (F) liveness.
    \item The enumerated capability encoding
    (\texttt{p0\_parsed}), evaluated under the same two liveness formulations.
\end{enumerate}
The capability-based configurations are further evaluated with and without
pending propositions (P), giving eight capability-generated configurations in
total alongside the two hand-written baselines. In the tables below, the
two baselines appear as the \textbf{MoveOnly} and \textbf{Full} columns, in
that order, and the leading digit \textbf{2}/\textbf{3} denotes the
one-hot/enumerated capability encoding, matching items 2 and 3 above.

At this demonstration's complexity, a 40-ordering sweep is not tractable, so
each capability-based configuration uses two fixed declaration orders
(alphabetic and capability-grouped) plus 20 random-seed orderings under a
CUDD reordering threshold of 250, giving a distribution for characterizing
ordering variance.

\retab{pyrobosim-p0-1hot-summary} and \retab{pyrobosim-p0-enumerated-summary}
summarize the P0 matrix by realizability, well-separation, strategy size,
and the three headline timing rows, following the hand-written
baseline/one-hot and enumerated split used for Coffee
(\resec{coffee-results}). The full specification-size and BDD-node
breakdown, and accompanying discussion, appear in
\seedetail{pyrobosim-p0-results-app}. Reported values are means over the 20
random-seed declaration orderings, matching the Coffee and Two-Rivers
reporting convention. The alphabetic and domain-grouped fixed orderings are
not part of this mean.

\begin{table*}[t]
\centering
\caption{\textbf{[$N=20$ random-seed orderings.]} PyRoboSim P0 synthesis results, full hand-written baseline and one-hot
capability encoding. \textbf{Full} is the hand-written baseline shown here
(item 1b above); the \textbf{MoveOnly} sanity-check baseline (item 1a
above) appears in \retab{pyrobosim-p0-results-app}. The leading digit \textbf{2} denotes the
one-hot capability encoding (item 2). \textbf{S}/\textbf{F} denotes
System-Goal/Fair-Outcome liveness, and \textbf{P} adds pending propositions.
Reported values are means over 20
random-seed declaration orderings under a CUDD reordering threshold of 250,
with a companion row giving $\pm$ one standard deviation wherever the
underlying trials disagree. Columns marked unrealizable have no extracted
strategy. \textbf{2FP} is realizable but shows \textbf{FlexBE HFSM
Realized: No} because the post-synthesis auditor (\resec{auditor}) halts
the pipeline before reduction, discussed below.}
\label{tbl:pyrobosim-p0-1hot-summary}
\renewcommand{\arraystretch}{1.4}
\setlength{\tabcolsep}{3pt}
\input{tables/pyrobosim_p0_one_hot_summary}
\end{table*}

\begin{table*}[t]
\centering
\caption{\textbf{[$N=20$ random-seed orderings.]} PyRoboSim P0 synthesis results, enumerated capability encoding
(item 3 above). The leading digit \textbf{3} denotes this encoding. The
remaining reporting convention and column legend match
\retab{pyrobosim-p0-1hot-summary}. The enumerated encoding has no
hand-written baseline. \textbf{3FP} is realizable but shows
\textbf{FlexBE HFSM Realized: No} for the same reason as \textbf{2FP}.}
\label{tbl:pyrobosim-p0-enumerated-summary}
\renewcommand{\arraystretch}{1.4}
\setlength{\tabcolsep}{3pt}
\input{tables/pyrobosim_p0_enumerated_summary}
\end{table*}

Without pending memory, Fair-Outcome liveness is \emph{unrealizable} for the
P0 capability specification under both encodings (\textbf{2F}/\textbf{3F}),
across all 22 declaration orderings. The generated specifications exhibit
the same mandatory-logging mechanism as Two-Rivers
(\resec{two-rivers-results}): each failed move activates
\texttt{log\_move\_failed}, satisfying the move fairness clause through
inactivity on that step. The environment can fail every move, keeping
\texttt{robot\_loc=18}, where grabbing and placing are prohibited.
Those capabilities remain inactive, satisfying their fairness clauses,
while the delivery goal and hence \texttt{finished} remain unreachable.
Adding pending propositions restores realizability: \textbf{2FP} and
\textbf{3FP} are realizable across all 22 orderings.

Realizability, however, is not the end of the story. The strategy auditor of
\resec{auditor} flags a \texttt{GOAL\_UNREACHABLE\_TRAP} on \emph{every one}
of the 44 \textbf{2FP}/\textbf{3FP} trials (both encodings, all 22
orderings). The interaction of persistent pending memory
with activation-gated environment outcomes admits a realizable strategy
that is not obligated to ever complete the task.
Manual inspection of an early strategy confirmed the same pathology and
motivated this automated check. Fair-Outcome-with-pending in P0 is therefore
realizable, and \texttt{slugs} (correctly) formally admits a strategy given
the specification, but that strategy still fails the task-level progress
property the designer intended. The pipeline halts once the raw-strategy
audit reports a confirmed failure. No executable controller is ever produced
from either configuration.

The example in \refig{pyrobosim-trap} comes from the raw \textbf{2FP}
strategy under alphabetic declaration ordering. Direct SCC analysis found
twelve reachable goal-unreachable traps in that strategy.
\seedetail{pyrobosim-p0-results-app} gives the full listing of trap states.

Targeted trials added $\mathrm{GF}(\neg\pi_a^p \lor \pi_a)$ for each
pending capability to \texttt{SYS\_LIVENESS} in the one-hot
Fair-Outcome-with-pending specifications for extended Coffee (\textbf{3FP})
and PyRoboSim P0 (\textbf{2FP}), leaving \texttt{ENV\_LIVENESS} unchanged. Both retained the
same state valuations and transitions, including P0's trap. Violating
environment justice releases the added system justice obligations; this tested
addition therefore does not prevent abandonment. Safety-level retry rules
with explicit abort or abandonment semantics remain a separate design alternative.

With real seed coverage, cost metrics do not show the sharp warning sign
their eventual defect would seem to warrant. \textbf{2FP}/\textbf{3FP}
remain within 6\% of the System-Goal/no-pending counterparts
\textbf{2S}/\textbf{3S} on peak BDD nodes, while their realizability
times are 1.4--2.0$\times$ longer. Nothing in
proposition count, BDD size, or synthesis time specifically signals that
\textbf{2FP}/\textbf{3FP} contain a reachable goal-unreachable region,
which the auditor detects directly. Realizability and symbolic cost
metrics remain insufficient measures of a synthesized supervisor's
usefulness even once ordering variance is properly accounted for.
\seedetail{pyrobosim-p0-results-app} gives the exact BDD-node and timing
comparisons.

The System-Goal columns give a clean PyRoboSim encoding comparison:
enumerated encoding is smaller and faster than one-hot before reduction,
both under the fixed alphabetic and domain orderings and averaged over the
20 random seeds reported above. The random-seed mean does carry substantial
spread, so any individual random draw remains a poor proxy for the typical
case, fixed or random. This matches the broader lesson of the Coffee
ordering-robustness study
(\seedetail{coffee-ordering-app}): declaration order drives substantial
variation on its own, so encoding comparisons drawn from a single ordering
should be treated cautiously regardless of domain. After state reduction,
the strategies converge to the same size regardless of encoding
(\textbar{}Reduced\textbar{} = 19 for both \textbf{2S} and \textbf{3S}).
\seedetail{pyrobosim-p0-results-app} reports the exact peak-BDD and
realizability multiples behind this comparison.

Comparing the capability-generated specifications to the hand-written
baselines, the full hand-written specification (\textbf{Full}) is not
uniformly faster than the capability-generated \textbf{2S} configuration
despite having fewer output propositions. It is faster under the fixed
orderings but is in fact the slowest of the three System-Goal
configurations once averaged over the 20 random seeds, driven by a small
number of especially slow random draws. The hand-written baseline also shows
the ordering sensitivity that affects the generated encodings.
The move-only baseline, which tracks only robot location and the finished
flag, is trivially cheap by comparison and serves mainly as a sanity check
that most of the symbolic cost in this domain comes from item-manipulation
logic rather than navigation. \seedetail{pyrobosim-p0-results-app} gives
the exact realizability figures and standard deviations behind this
comparison.

The well-separation picture for PyRoboSim differs qualitatively from
Coffee and Two-Rivers (\resec{coffee-results}, \resec{two-rivers-results}).
Only the trivial move-only baseline (\textbf{MoveOnly}) is well-separated;
every other configuration in the reported tables, including the realizable System-Goal
configurations and the hand-written \textbf{Full} baseline, is NWS.
Coffee's NWS cases were limited to Fair-Outcome with pending memory;
in Two-Rivers, every Fair-Outcome specification was NWS.
The P0 diagnostics distinguish two causes. Every NWS configuration has a
\texttt{P-reach/E-safe} finding, including those using System-Goal.
Core minimization on the saved alphabetic specifications for
\textbf{2S}, \textbf{Full}, \textbf{2FP}, and \textbf{3FP} returns the
same single environment safety assumption, $G\,\psi_{\mathrm{carry}}$, where
\[
\psi_{\mathrm{carry}} = \neg\left((\texttt{snacks0}=0)\land
  \bigvee_{x\in\mathcal I\setminus\{\texttt{snacks0}\}}(x=0)\right).
\]
Here, $\mathcal I$ is the set of six modeled items and location 0 denotes
being carried. This is one of the assumptions prohibiting simultaneous
carriage of multiple items.
The well-separation check computes reachability from environment
initialization and safety assumptions and removes system guarantees when
checking whether assumptions can be violated. It therefore omits the
supervisor's prohibition on grabbing while already carrying an item.
Successive successful grabs can reach a two-item carrying state under
this diagnostic reachability. Because the displayed assumption constrains
the current state, no next environment input can satisfy it there.
This witness does not establish reachability under the synthesized
controller, which must obey the system's grab precondition.

The Fair-Outcome configurations additionally have a
\texttt{P-all/E-just} finding. Their environment justice clauses depend
on system-controlled activation or pending memory; for example, with
pending memory the move clause is
$GF(\neg\texttt{move\_robot\_p}\lor X\texttt{move\_robot\_c})$.
The returned safety core does not exclude this additional justice-based
cause: minimization returns a 1-minimal witnessing subset, not a unique
explanation of all reported cases.

\noindent\textbf{Paired check of the carriage assumptions.}
We reran seven P0 configurations with a modified specification that changed
only the six carriage exclusions in \texttt{ENV\_TRANS}. For the exclusion
above, we compared the original $G\,\psi_{\mathrm{carry}}$ (current state)
with the modified $G\,X\,\psi_{\mathrm{carry}}$ (next state).
In every pair, \texttt{ENV\_INIT} already enforces the exclusions, so both
versions admit the same infinite traces from initialization. All 14
specifications were GR(1)-realizable, and the original and modified runs
produced identical raw strategies in every pair under alphabetic declaration
ordering. The main results tables report the original specifications.

The rewrite removed the carriage-related \texttt{P-reach/E-safe} finding in
all seven pairs. For the four System-Goal configurations (\textbf{2S},
\textbf{2SP}, \textbf{3S}, \textbf{3SP}) and \textbf{Full}, this was the only
well-separation finding, so those five became well-separated. All five
passed both the raw and reduced audits with either specification, and their
reduced strategies matched up to bookkeeping metadata. The two Fair-Outcome
configurations with pending memory (\textbf{2FP}/\textbf{3FP}) retained the
separate \texttt{P-all/E-just} finding and therefore remained
non-well-separated. Both still failed the raw audit with the same
goal-unreachable traps.

Thus the same five configurations cleared the audit stage for HFSM
generation, and the same two were blocked before reduction and realization.
The paired experiment compared strategies and audits through reduction for
the passing configurations. It did not rerun HFSM generation or execution.
The per-configuration results and reproduction details appear in
\reappsec{p0-carriage-check}.

\resec{summary-results} draws together the cross-domain findings.

\subsubsection{P1: Adding Door Modeling}
\label{sec:pyrobosim-p1-results}

The P0 capability model of \resec{pyrobosim-p0-results} omits door state
entirely (\resec{pyrobosim}). The fridge and pantry are always accessible,
and the supervisor never reasons about opening them. P1 extends
the same domain with an \texttt{open} capability and initializes both
containers closed, so retrieving an item first requires opening its
container (\resec{pyrobosim}). We evaluate P1 to test whether the encoding
conclusions of \resec{pyrobosim-p0-results} hold as domain complexity grows
on an axis distinct from encoding choice itself. It adds a capability and
its associated propositions to an already-validated domain, rather than
comparing encodings within a fixed domain.

Per-trial cost at this scale makes the 20-random-seed sweep of
\resec{pyrobosim-p0-results} intractable for P1. \retab{pyrobosim-p1-results}
reports a single trial per configuration under a fixed alphabetic
declaration ordering, not a random-seed mean. These values are therefore
not comparable to \retab{pyrobosim-p0-1hot-summary}/\retab{pyrobosim-p0-enumerated-summary}'s
means in the way the P0 configurations are comparable to each other, and
any comparison drawn between the tables should be read with that in mind.
\retab{pyrobosim-p1-results} summarizes this matrix following the same
condensed row selection as the P0 tables above. The full
specification-size and BDD-node breakdown, and accompanying discussion,
appear in \seedetail{pyrobosim-p1-results-app}.
Under the 10-hour per-trial synthesis budget, \textbf{2S} (one-hot,
System-Goal, no pending) is the only one-hot P1 configuration with a
completed run, finishing in 25842.0~s. The remaining one-hot
configurations (\textbf{2F}, \textbf{2SP}, \textbf{2FP}) have no completed
run at this budget, so their columns are omitted.

We consider the hand-written \textbf{FullP1} baseline, the one-hot
capability encoding under System-Goal liveness (\textbf{2S}), and the
enumerated capability encoding under System-Goal (\textbf{4S}),
Fair-Outcome (\textbf{4F}), and both liveness formulations with pending
propositions (\textbf{4SP}, \textbf{4FP}). The column digit \textbf{4}
distinguishes the enumerated columns from P0's (\textbf{3S}, etc.) despite
using the same encoding, since the two are evaluated over different
domains and are not meant to be read as the same configuration. The
one-hot column, by contrast, keeps P0's digit \textbf{2}: it is the same
one-hot/System-Goal/no-pending treatment applied to the larger P1 domain,
with the P1 table and section context distinguishing it from the P0
\textbf{2S} row.

\begin{table*}[t]
\centering
\small
\caption{\textbf{[Exploratory: budget-limited, single fixed ordering
($n=1$). See \resec{pyrobosim-p1-results}.]} PyRoboSim P1 (door-modeling) synthesis results. \textbf{FullP1} is
the hand-written baseline. \textbf{2S} is the one-hot capability encoding
(System-Goal liveness only -- \textbf{2F}/\textbf{2SP}/\textbf{2FP} were
not completed at this budget). The leading digit \textbf{4} denotes the
enumerated capability encoding evaluated over the P1 domain (distinct from
P0's \textbf{3}-family columns despite sharing an encoding).
\textbf{S}/\textbf{F} denotes System-Goal/Fair-Outcome
liveness, and \textbf{P} adds pending propositions. Unlike the P0 tables,
reported values are a single trial under a
fixed alphabetic declaration ordering ($n=1$), not a random-seed mean;
therefore no $\pm$ standard-deviation rows are shown --
see \resec{pyrobosim-p1-results} for why this is not directly comparable
to the P0 tables. These single-ordering rows also use P1's uncapped
(``auto'') CUDD reordering policy; \seedetail{pyrobosim-p1-ordering-app}
reports the corresponding threshold-250 ordering experiment for
\textbf{4S}/\textbf{4F}. \textbf{FullP1}'s well-separation check did not
complete under a dedicated 43200~s (12\,h) cap, so
its \textbf{Well Separated} row reports \textbf{TIME OUT} rather than a
completed verdict. The full specification-size and BDD-node
breakdown appears in \seedetail{pyrobosim-p1-results-app}.}
\label{tbl:pyrobosim-p1-results}
\renewcommand{\arraystretch}{1.4}
\setlength{\tabcolsep}{3pt}
\input{tables/pyrobosim_p1_summary}
\end{table*}

P1 replicates every qualitative P0 Fair-Outcome finding
(\resec{pyrobosim-p0-results}). \textbf{4F} is unrealizable for the same
reason: repeated move failures alternate with logging and keep the robot
at location 18, where opening, grabbing, and placing are prohibited.
\textbf{4FP} is realizable but immediately flagged by the raw-strategy
audit with the identical \texttt{GOAL\_UNREACHABLE\_TRAP} kind reported
for \textbf{2FP}/\textbf{3FP}. This confirms that the
pathology is a property of the specification pattern, not an artifact of
the specific P0 domain.
It also shows that the auditor's automated check catches the defect after a
new capability is added, without fresh hand-tracing.

\textbf{2S} lets us extend the encoding comparison of
\resec{pyrobosim-p0-results} to P1. Under the tested P1 ordering, the gap
between one-hot and enumerated is much larger. In the System-Goal,
no-pending comparison, P0's one-hot \textbf{2S} column has
$1.6\times$ the peak BDD nodes and $3.0\times$ the realizability time of
enumerated \textbf{3S}. After door modeling is added, P1's one-hot
\textbf{2S} column rises to $8.2\times$ the peak BDD nodes and
$109.9\times$ the realizability time of enumerated \textbf{4S}. Both P1
columns use the uncapped (``auto'') reordering policy, so this ratio compares
like with like, but it is not comparable to the P0 ratios above, which use
threshold 250. As in P0, this cost gap closes entirely after reduction.
\textbf{2S} and \textbf{4S} both reduce to 27 states
(\retab{pyrobosim-p1-results}), matching P0's encoding-independent
convergence (19 states for both \textbf{2S} and \textbf{3S}) on a domain
with roughly 50\% more input propositions. In these capability-generated
rows, the reduced strategy size and realized FlexBE state count coincide;
the separate \textbf{FullP1} baseline is the PyRoboSim case where
Controller Realization collapses several terminal-success contexts into one
shared \texttt{finished} outcome (\resec{controller-realization}).
Proposition count remains a poor predictor of reduced strategy size or
realized FlexBE state count.
Because this P1 comparison uses
fixed-ordering trials, and because \textbf{4S} itself shows large
declaration-order sensitivity below, the multipliers should be read as
directional evidence under the tested ordering rather than as an isolated
encoding effect.

Growth from P0 to P1 must be read against the reordering policy. P0 and the
P1 ordering sweep use threshold 250, while the P1 main table uses the
uncapped policy, and at this scale the two differ sharply. For \textbf{4S}
the alphabetic trial takes 235.2\,s and 5.6M peak nodes under the uncapped
policy but 3679.6\,s and 25.7M under threshold 250, whereas \textbf{4F}
moves from 31.3\,s to 24.0\,s (\seedetail{pyrobosim-p1-ordering-app}).
Dividing P1 uncapped values by P0 threshold-250 values therefore mixes
policies, and we do not report such ratios as effect sizes. Where both use
threshold 250, \textbf{4S} costs $7.1\times$ the peak BDD nodes and
$11.0\times$ the realizability time of \textbf{3S} under the alphabetic
ordering (25.7M vs.\ 3.6M peak nodes; 3679.6\,s vs.\ 335\,s), and $8.0\times$
and $61\times$ on the random-ordering means. Door
modeling thus raises System-Goal cost substantially even without pending
memory. The pending columns (\textbf{4SP}, \textbf{4FP}) and \textbf{FullP1}
have only uncapped-policy P1 trials. Against P0 they show $25.1\times$ and
$43.0\times$ the realizability time of \textbf{3SP} and \textbf{3FP}, and
roughly $7\times$ for \textbf{FullP1} against \textbf{Full}, which suggests a
larger pending penalty, but these comparisons mix policies and are
directional only.
The cost of door modeling is therefore not confined to pending-memory
configurations, since System-Goal without pending memory also grows
several-fold under a matched policy, but the pending columns appear to grow
more, most where pending memory meets Fair-Outcome liveness. The P1 pending
columns are single trials under a different policy from P0, so these
multipliers are directional rather than calibrated effect sizes.

P1's single-trial configurations also completed substantially faster on
the faster machine (\resec{discussion}), with the same
realizable/unrealizable and audit outcomes throughout. Given the
single-trial design here, we report this qualitatively rather than as a
calibrated speedup factor. \textbf{FullP1}'s well-separation check is the
exception: on both the canonical-machine run reported in
\retab{pyrobosim-p1-results} and the faster-machine reproduction, it hits
a dedicated 43{,}200\,s (12\,h) cap and returns
\texttt{ANALYSIS\_INCOMPLETE}, even though synthesis completes. This is
consistent with well separation being a distinct winning-region computation
over a transformed game, so its symbolic cost need not track ordinary
synthesis cost~\cite{MaozRingert2016,KindControllers24}.

A follow-up 22-ordering sweep (2 fixed + 20 random, threshold-250
reordering) for \textbf{4S}/\textbf{4F} confirms both configurations'
realizability verdicts exactly, while showing that \textbf{4S}'s cost is
far more ordering-sensitive than the single alphabetic trial above could
reveal. It has a $99\times$ spread between its fastest and slowest observed
ordering, the widest found anywhere in this paper. \textbf{4F}, in
contrast, stays uniformly fast and tightly clustered regardless of
ordering (\seedetail{pyrobosim-p1-ordering-app}) as it detects that the
specifications are unrealizable.

%% file: tables/pyrobosim_p0_one_hot_summary.tex
\begin{tabular}{|c|r|r|r|r|r|}
\hline
 & \textbf{Full} & \textbf{2S} & \textbf{2F} & \textbf{2SP} & \textbf{2FP} \\
\hline
\rule{0pt}{4ex}\textbf{\shortstack{\rule{0pt}{2.8ex}Well\\Separated\rule[-1ex]{0pt}{1ex}}} & No & No & No & No & No \\
\hline
\textbf{Realizable} & Yes & Yes & No & Yes & Yes \\
\hline
\rule{0pt}{4ex}\textbf{\shortstack{\rule{0pt}{2.8ex}FlexBE HFSM\\Realized\rule[-1ex]{0pt}{1ex}}} & Yes & Yes & No & Yes & No \\
\hline
\textbf{$\left|AP_I\right|$} & 33 & 34 & 34 & 34 & 34 \\
\hline
\textbf{$\left|AP_O\right|$} & 16 & 17 & 17 & 20 & 20 \\
\hline
\textbf{\textbar{}Slugs SM\textbar} & 116 & 47 & -- & 47 & 108 \\
\hline
\textbf{\textbar{}Reduced\textbar} & 66 & 19 & -- & 19 & -- \\
\hline
\rule{0pt}{4ex}\textbf{\shortstack{\rule{0pt}{2.8ex}Realizability\\(s)\rule[-1ex]{0pt}{1ex}}} & 570.184 & 449.174 & 58.466 & 570.332 & 620.068 \\
\hline
 & $\pm$ 399.504 & $\pm$ 257.022 & $\pm$ 13.832 & $\pm$ 209.194 & $\pm$ 346.331 \\
\hline
\rule{0pt}{4ex}\textbf{\shortstack{\rule{0pt}{2.8ex}Realize\\HFSM (s)\rule[-1ex]{0pt}{1ex}}} & 3.193 & 0.474 & -- & 0.468 & -- \\
\hline
 & $\pm$ 0.435 & $\pm$ 0.017 &  & $\pm$ 0.013 &  \\
\hline
\rule{0pt}{4ex}\textbf{\shortstack{\rule{0pt}{2.8ex}Overall\\(s)\rule[-1ex]{0pt}{1ex}}} & 580.581 & 455.548 & 61.122 & 577.649 & 627.717 \\
\hline
 & $\pm$ 401.087 & $\pm$ 258.170 & $\pm$ 14.031 & $\pm$ 210.151 & $\pm$ 346.342 \\
\hline
\end{tabular}

%% file: tables/pyrobosim_p0_enumerated_summary.tex
\begin{tabular}{|c|r|r|r|r|}
\hline
 & \textbf{3S} & \textbf{3F} & \textbf{3SP} & \textbf{3FP} \\
\hline
\rule{0pt}{4ex}\textbf{\shortstack{\rule{0pt}{2.8ex}Well\\Separated\rule[-1ex]{0pt}{1ex}}} & No & No & No & No \\
\hline
\textbf{Realizable} & Yes & No & Yes & Yes \\
\hline
\rule{0pt}{4ex}\textbf{\shortstack{\rule{0pt}{2.8ex}FlexBE HFSM\\Realized\rule[-1ex]{0pt}{1ex}}} & Yes & No & Yes & No \\
\hline
\textbf{$\left|AP_I\right|$} & 25 & 25 & 25 & 25 \\
\hline
\textbf{$\left|AP_O\right|$} & 12 & 12 & 15 & 15 \\
\hline
\textbf{\textbar{}Slugs SM\textbar} & 38 & -- & 38 & 84 \\
\hline
\textbf{\textbar{}Reduced\textbar} & 19 & -- & 19 & -- \\
\hline
\rule{0pt}{4ex}\textbf{\shortstack{\rule{0pt}{2.8ex}Realizability\\(s)\rule[-1ex]{0pt}{1ex}}} & 149.430 & 25.051 & 164.243 & 300.756 \\
\hline
 & $\pm$ 72.034 & $\pm$ 8.835 & $\pm$ 33.784 & $\pm$ 266.725 \\
\hline
\rule{0pt}{4ex}\textbf{\shortstack{\rule{0pt}{2.8ex}Realize\\HFSM (s)\rule[-1ex]{0pt}{1ex}}} & 0.466 & -- & 0.470 & -- \\
\hline
 & $\pm$ 0.014 &  & $\pm$ 0.014 &  \\
\hline
\rule{0pt}{4ex}\textbf{\shortstack{\rule{0pt}{2.8ex}Overall\\(s)\rule[-1ex]{0pt}{1ex}}} & 154.511 & 27.657 & 169.211 & 307.086 \\
\hline
 & $\pm$ 72.554 & $\pm$ 9.238 & $\pm$ 34.026 & $\pm$ 268.133 \\
\hline
\end{tabular}

%% file: tables/pyrobosim_p1_summary.tex
\begin{tabular}{|c|r|r|r|r|r|r|}
\hline
 & \textbf{FullP1} & \textbf{2S} & \textbf{4S} & \textbf{4F} & \textbf{4SP} & \textbf{4FP} \\
\hline
\rule{0pt}{4ex}\textbf{\shortstack{\rule{0pt}{2.8ex}Well\\Separated\rule[-1ex]{0pt}{1ex}}} & TIME OUT & No & No & No & No & No \\
\hline
\textbf{Realizable} & Yes & Yes & Yes & No & Yes & Yes \\
\hline
\rule{0pt}{4ex}\textbf{\shortstack{\rule{0pt}{2.8ex}FlexBE HFSM\\Realized\rule[-1ex]{0pt}{1ex}}} & Yes & Yes & Yes & No & Yes & No \\
\hline
\textbf{$\left|AP_I\right|$} & 48 & 52 & 40 & 40 & 40 & 40 \\
\hline
\textbf{$\left|AP_O\right|$} & 14 & 19 & 13 & 13 & 17 & 17 \\
\hline
\textbf{\textbar{}Slugs SM\textbar} & 559 & 59 & 46 & -- & 46 & 118 \\
\hline
\textbf{\textbar{}Reduced\textbar} & 344 & 27 & 27 & -- & 27 & -- \\
\hline
\rule{0pt}{4ex}\textbf{\shortstack{\rule{0pt}{2.8ex}Realizability\\(s)\rule[-1ex]{0pt}{1ex}}} & 3871.901 & 25841.977 & 235.193 & 31.282 & 4129.791 & 12919.358 \\
\hline
\rule{0pt}{4ex}\textbf{\shortstack{\rule{0pt}{2.8ex}Realize\\HFSM (s)\rule[-1ex]{0pt}{1ex}}} & 8.912 & 0.574 & 0.592 & -- & 0.590 & -- \\
\hline
\rule{0pt}{4ex}\textbf{\shortstack{\rule{0pt}{2.8ex}Overall\\(s)\rule[-1ex]{0pt}{1ex}}} & 47111.691 & 28917.231 & 365.211 & 157.286 & 4315.238 & 13123.356 \\
\hline
\end{tabular}

%% file: CrazyFlieDiscussion.tex
\subsection{Quadcopter Hardware-Integration Results}
\label{sec:crazyflie-results}

We consider two capability encodings:

\begin{enumerate}
    \item {A version with 51 one-hot activation propositions, one per
    capability state, including the initial \texttt{begin\_game} action.}
    \item {One 6-bit selector (\texttt{capability:0...50}, 51 values). Value 0
    selects the one-time logging capability \texttt{begin\_game}.}
\end{enumerate}

\retab{crazyflie-results} summarizes the results of the quadcopter
hardware demonstrations described in this paper. Each configuration
is evaluated under a fixed alphabetic ordering, a capability-grouped
domain ordering, and 20 random-seed orderings, under a CUDD reordering
threshold of 250 and a 36{,}000\,s (10\,h) per-trial synthesis budget.
Reported values are means over the 20 random-seed orderings only, excluding
the alphabetic and domain-grouped fixed orderings. All six configurations
(\textbf{1S}, \textbf{2S},
\textbf{1F}, \textbf{2F}, \textbf{1SP}, \textbf{2SP}) complete all 20
trials at this budget. \resec{crazyflie-reliability} below reports how
synthesis reliability for the one-hot Fair-Outcome and pending
configurations compares to the other configurations under smaller budgets.
The full specification-size and BDD-node breakdown appears in
\seedetail{crazyflie-results-app}.

\begin{table*}[htbp]
\centering
\caption{\textbf{[$N=20$ random-seed orderings for the six
configurations shown. \textbf{1FP}/\textbf{2FP} excluded from this
20-seed summary; additional Fair-Outcome-with-pending results appear in
\resec{crazyflie-reliability}.]}
Quadcopter supervisory-synthesis results. The leading digit
\textbf{1}/\textbf{2} denotes the one-hot/enumerated capability encoding
(PyRoboSim uses \textbf{2}/\textbf{3}; \resec{discussion}). Reported values are means
over 20 random-seed declaration orderings under a CUDD reordering threshold
of 250 and a 36{,}000\,s per-trial synthesis budget, with a companion row giving
$\pm$ one standard deviation wherever the underlying trials disagree.
Alphabetic and domain-grouped fixed orderings are excluded from the mean.
All six configurations complete all 20 trials at this budget, including
\textbf{1SP} (\resec{crazyflie-reliability}), and \textbf{1SP} is
well-separated across all 22 declaration orderings.
Integer-valued count means indicate identical samples,
while non-identical count means are shown with decimal digits.}
\label{tbl:crazyflie-results}
\renewcommand{\arraystretch}{1.15}
\setlength{\tabcolsep}{3pt}
\input{tables/crazyflie_summary}
\end{table*}

Comparisons with the earlier conference results require accounting for changes
to the \texttt{flexible\_drones} capability package~\cite{eit26-synthesis}.

As in the other domains, enumerated variants (2*) reduce input
AP count. Unlike Coffee, Two-Rivers, and PyRoboSim,
enumerated encoding does not reduce the extracted strategy here --
\textbar{}Slugs SM\textbar{} and \textbar{}Reduced\textbar{} are consistently
\emph{larger} under enumerated encoding across all three liveness variants,
most dramatically so for the pending case -- while realizability time
favors enumerated encoding across all three liveness variants in this
domain. With all 20 random
orderings completing for both encodings at the 36{,}000\,s budget
(\resec{crazyflie-reliability}), the pending comparison is a true 20-seed
mean on both sides, reinforcing that enumerated encoding dominates on
realizability time in this domain. Within each encoding, Fair-Outcome
liveness still dominates overall difficulty relative to System-Goal,
confirming liveness pressure, not encoding, as the main complexity driver
here. \seedetail{crazyflie-results-app} reports the exact figures behind
these comparisons, along with the \textbf{Extraction}/\textbf{Auditing}
timing rows.
These are strategy-level counts. After Controller Realization, the
realized FlexBE HFSM can be smaller because terminal-output contexts
and bootstrap/null states become shared FlexBE outcomes or root-container
wiring rather than separate executable states
(\resec{controller-realization}). We audited the representative
alphabetic rows directly against the generated artifacts. The large
quadcopter gaps are terminal-context projection rather than omitted
capability execution states. The appendix gives the detailed
reduced-strategy versus realized-HFSM accounting
(\seedetail{crazyflie-realization-app}).

The same 20-seed sweep, reproduced on the faster machine introduced in
\resec{discussion}, confirmed this full picture --
realizability, well-separation, strategy size, and audit outcome all
matched -- with realizability times roughly 2.3--3.2$\times$ faster across
configurations.

\subsubsection{Budget Sensitivity of One-Hot Synthesis}
\label{sec:crazyflie-reliability}

The 20-seed sweep surfaces a reliability difference between encodings that a
single random seed cannot reveal. Enumerated encoding completes all 20
orderings for every configuration in the six-column sweep with zero
timeouts even at the smallest budget tested (900\,s),
while one-hot's completion rate is strongly budget-dependent. The
Fair-Outcome configuration (\textbf{1F}) needs 3600\,s to reach 100\%
completion, and the pending configuration (\textbf{1SP}) needs the full
36{,}000\,s budget to do the same. This is the same
budget-sensitivity phenomenon as \textbf{1F}, with a substantially heavier tail,
not a structurally distinct failure mode or evidence of unrealizability. The
alphabetic and domain-grouped fixed orderings complete well within budget
throughout. The practical implication for a deployment pipeline is that
one-hot encoding here carries a substantial risk that synthesis, even for
a realizable specification, may require an order of magnitude more
wall-clock time to reach a verdict than enumerated encoding needs for the
equivalent configuration, unless the declaration order is chosen carefully
or a correspondingly larger budget is provisioned.
\seedetail{crazyflie-results-app} gives the full per-budget completion
figures behind this comparison.

The theoretical hazard reviewed in \resec{synthesis} applies directly to
the Fair-Outcome-with-pending configuration (\textbf{1FP}/\textbf{2FP}), which appears in
neither \retab{crazyflie-results} nor the sweep above. Because completion
events are only permitted when the corresponding capability is activated, a
Fair-Outcome-with-pending strategy can in principle enter regions where
pending remains asserted while activation has ceased, relieving the system
of its own liveness guarantee under GR(1) semantics. One-hot \textbf{1FP}
needs a large budget to resolve. Under a 24\,h (86{,}400\,s) budget it is
realizable, in 33{,}937\,s (9.4\,h) of \texttt{slugs} realizability time and
36{,}540\,s (10.2\,h) overall, with $|AP_I|=84$, $|AP_O|=124$, and a peak BDD of 52.9M
nodes, well above any one-hot configuration in \retab{crazyflie-results},
consistent with the proposition-count blowup one-hot encoding already
shows throughout this paper. The raw-strategy audit flags it with a
\texttt{GOAL\_UNREACHABLE\_TRAP}, the same failure kind already confirmed
in PyRoboSim (\resec{pyrobosim-p0-results}, \resec{pyrobosim-p1-results}).

Enumerated \textbf{2FP} did not resolve at any budget up to 24\,h on the
original machine. It resolved on the faster machine introduced in
\resec{discussion}, in 16.6\,h. It is realizable, with
the raw-strategy audit confirming the identical
\texttt{GOAL\_UNREACHABLE\_TRAP} kind found for \textbf{1FP} and
throughout PyRoboSim (\resec{pyrobosim-p0-results},
\resec{pyrobosim-p1-results}). These runs used different hardware and were not a matched, budget-normalized
comparison. The follow-up establishes realizability and the audit verdict,
without isolating machine speed as the cause of the timing difference. The pathology is thus confirmed for
both encodings of the Fair-Outcome-with-pending configuration, not just
\textbf{1FP}.
\seedetail{crazyflie-results-app}
gives the detailed specification-size, timing, and BDD-node figures for
both configurations.

\subsubsection{Well-Separation}
\label{sec:crazyflie-well-separation}

Well-separation checks reached a settled status for all eight
quadcopter configurations. \textbf{1S}, \textbf{2S}, \textbf{1F},
\textbf{2F}, \textbf{1SP}, and \textbf{2SP} are well-separated across all
22 declaration orderings with no exceptions; this is a sharp contrast to Coffee,
Two-Rivers, and PyRoboSim, each of which had at least one NWS configuration
(\resec{coffee-results}, \resec{two-rivers-results},
\resec{pyrobosim-results}). \textbf{1SP}'s check simply needs a longer
timeout than the other five to resolve every ordering
(\seedetail{crazyflie-well-separation-app}). \textbf{1FP} and \textbf{2FP} are
both \texttt{NON\_WELL\_SEPARATED}, with the same \texttt{P-all/E-just}
violation signature identified for the Fair-Outcome-with-pending
configurations elsewhere. Both are realizable
(\resec{crazyflie-reliability}), and for both this specification-level
signature is confirmed by a concrete extracted strategy exhibiting the
\texttt{GOAL\_UNREACHABLE\_TRAP} pathology directly. \seedetail{crazyflie-results-app} gives the exact
per-ordering check timings.

The clean no-pending Fair-Outcome results (\textbf{1F}/\textbf{2F})
also reflect a different task contract from the retry-to-success examples.
Here the goal is \texttt{finished} $\lor$ \texttt{failed}: failure handling
can terminate through an abort or recovery sequence
(\resec{crazyflie}), without requiring the failed capability eventually
to succeed. Thus these results concern programmed termination, not a
Fair-Outcome guarantee of mission success.

\subsubsection{Hardware Execution and Failure-Handling Validation}
\label{sec:crazyflie-failure-validation}

We ran the synthesis algorithms on a faster workstation, a Dell Pro Max Tower~T2
with an Intel Core Ultra~9 285K CPU and 64\,GiB memory. This is a different,
and substantially faster, machine than the one used for
\retab{crazyflie-results} and the other domains' reported sweeps
(\resec{discussion}), so the timing here is not directly comparable to
that table.

We ran the basic quick-start command
(\texttt{ros2 launch drone\_demo\_flexbe\_synthesis drones\_gates.launch.py})
from our open source demonstration, which
synthesizes under the \textbf{1S} configuration -- one-hot encoding,
System-Goal liveness, no pending \cite{flexbe-synthesis-demo}.
The system successfully synthesized the FlexBE HFSM, which we executed
with our Flexible Drones system on PiHawk
(\refig{pihawk-static})~\cite{flexible-drones-git}. The hardware validation
covers PiHawk flights under both \textbf{1S} and \textbf{2S}.
The quadcopter audit goal is \texttt{finished} $\lor$ \texttt{failed}
(\retab{auditor-goals}), so it screens for structural obstructions to
programmed termination, including an abort, rather than requiring successful
mission completion. Passing the audit does not establish that every execution
terminates or that the flight mission succeeds.

We additionally synthesized the equivalent enumerated configuration,
\textbf{2S}, on the same deployment workstation.
Both configurations synthesize end-to-end in well under 15\,s
on this faster workstation (\textbf{1S}: 10.6\,s; \textbf{2S}: 10.5\,s).
The execution evidence comprises three observations:
\begin{itemize}[nosep]
\item \textbf{Nominal execution:} Both configurations completed nominally
in simulation and on PiHawk, executing the desired behavior
(\refig{pihawk-flight}).
\item \textbf{Observed hardware failures:} During multiple PiHawk runs,
a temporary external pose dropout caused the onboard software to abort
trajectory execution. In each observed case, \texttt{execute\_trajectory}
returned \texttt{failure}; the FlexBE state machine logged it and supervised
the intended land-disarm behavior.
\item \textbf{Forced-outcome simulation tests:} Before the flights, we ran
each flight-tested synthesized controller in our kinematic Gazebo-based
simulation and used FlexBE to force unexpected capability outcomes. The
realized HFSM followed the audited recovery transitions.
\end{itemize}

These flights are exploratory, qualitative hardware validation, not a
statistically powered reliability campaign. PiHawk testing comprised more
than a dozen flights in total across \textbf{1S} and \textbf{2S}, with
multiple runs of each; the split and outcome counts were not systematically
recorded by configuration. Platform availability and
OptiTrack coverage bounded the number of trials. The observed pose-dropout
recoveries provide qualitative evidence of failure handling on hardware
rather than a formal reliability estimate.

The validation uses five distinct layers of evidence, each
checking a different translation step rather than any one layer
establishing end-to-end deployed correctness.
\begin{itemize}
\item \textbf{Synthesis:} \texttt{slugs} establishes that the raw strategy
satisfies the GR(1) guarantees under the stated environment assumptions.
\item \textbf{Strategy auditing and reduction:} the auditor
(\resec{auditor}) checks the \emph{abstract synthesized strategy} for the
specified defect classes. A clean safety verdict requires full reachable
product exploration. The reduction proof (\resec{reduction}) establishes
trace equivalence between the raw and reduced automata.
\item \textbf{Shared outcome abstraction:} preprocessing generates the
\texttt{class\_decl} and \texttt{state\_outcome\_mapping} artifact consumed by
both specification generation and HFSM realization. This makes the outcome
vocabulary consistent by construction: concrete state outcomes such as
\texttt{done} or \texttt{true} map to the configured success proposition, and
\texttt{failed} or \texttt{timeout} map to the configured failure proposition
when that is the mapping declared for the system.
\item \textbf{HFSM realization tests:} the kinematic Gazebo-based simulation
and FlexBE forced-outcome tests sample the \emph{strategy-to-HFSM}
transition wiring. They check that the realized FlexBE HFSM responds to a
forced concrete capability outcome by taking the reduced-automaton
transition selected by the configured outcome mapping.
\item \textbf{Hardware execution:} the hardware flights sample the
\emph{HFSM-to-robot} link. They check that the deployed HFSM, running on
real hardware, drives the vehicle through the response the state machine
actually computed.
\end{itemize}
The hardware observations support the final HFSM-to-robot link by
showing responses to concrete capability outcomes during physical execution.
The formal checks apply to the abstract strategy and its reduction;
forced-outcome tests provide empirical evidence for the subsequent
translation into FlexBE.

%% file: tables/crazyflie_summary.tex
\begin{tabular}{|c|r|r|r|r|r|r|}
\hline
 & \textbf{1S} & \textbf{2S} & \textbf{1F} & \textbf{2F} & \textbf{1SP} & \textbf{2SP} \\
\hline
\rule{0pt}{4ex}\textbf{\shortstack{\rule{0pt}{2.8ex}Well\\Separated\rule[-1ex]{0pt}{1ex}}} & Yes & Yes & Yes & Yes & Yes & Yes \\
\hline
\textbf{Realizable} & Yes & Yes & Yes & Yes & Yes & Yes \\
\hline
\rule{0pt}{4ex}\textbf{\shortstack{\rule{0pt}{2.8ex}FlexBE HFSM\\Realized\rule[-1ex]{0pt}{1ex}}} & Yes & Yes & Yes & Yes & Yes & Yes \\
\hline
\textbf{$\left|AP_I\right|$} & 84 & 2 & 84 & 2 & 84 & 2 \\
\hline
\textbf{$\left|AP_O\right|$} & 91 & 46 & 91 & 46 & 124 & 79 \\
\hline
\textbf{\textbar{}Slugs SM\textbar} & 119.3 & 131 & 102 & 131 & 187.0 & 347 \\
\hline
 & $\pm$ 10.2 &  &  &  & $\pm$ 9.2 &  \\
\hline
\textbf{\textbar{}Reduced\textbar} & 76.7 & 106 & 68 & 106 & 131.0 & 318 \\
\hline
 & $\pm$ 5.1 &  &  &  & $\pm$ 4.6 &  \\
\hline
\rule{0pt}{4ex}\textbf{\shortstack{\rule{0pt}{2.8ex}Realizability\\(s)\rule[-1ex]{0pt}{1ex}}} & 19.730 & 9.275 & 638.479 & 31.814 & 2764.539 & 166.559 \\
\hline
 & $\pm$ 9.490 & $\pm$ 2.643 & $\pm$ 471.598 & $\pm$ 15.260 & $\pm$ 3321.603 & $\pm$ 49.516 \\
\hline
\rule{0pt}{4ex}\textbf{\shortstack{\rule{0pt}{2.8ex}Realize\\HFSM (s)\rule[-1ex]{0pt}{1ex}}} & 3.532 & 4.193 & 4.217 & 4.193 & 5.997 & 7.492 \\
\hline
 & $\pm$ 0.365 & $\pm$ 0.017 & $\pm$ 0.028 & $\pm$ 0.023 & $\pm$ 1.075 & $\pm$ 0.023 \\
\hline
\rule{0pt}{4ex}\textbf{\shortstack{\rule{0pt}{2.8ex}Overall\\(s)\rule[-1ex]{0pt}{1ex}}} & 40.184 & 23.171 & 666.082 & 45.867 & 4228.671 & 321.273 \\
\hline
 & $\pm$ 10.864 & $\pm$ 2.866 & $\pm$ 474.724 & $\pm$ 15.613 & $\pm$ 4124.537 & $\pm$ 48.982 \\
\hline
\end{tabular}

%% file: tables/summary_sg_nopending.tex
\begin{tabular}{llrrcr}
\toprule
Domain & Encoding & Peak BDD Nodes & Realizability & HFSM Realized & \textbar{}Reduced\textbar{} \\
\midrule
Coffee (base) & One-Hot & 1{,}022 & 0.081\,ms & Yes & \textbf{3.5} \\
Coffee (base) & Enumerated & 1{,}022 & \textbf{0.050\,ms} & Yes & 4 \\
\midrule
Coffee (extended) & One-Hot & 1{,}022 & 2.149\,ms & Yes & \textbf{3} \\
Coffee (extended) & Enumerated & 1{,}022 & \textbf{0.059\,ms} & Yes & 4 \\
\midrule
Two-Rivers & One-Hot & \textbf{13{,}133} & \textbf{53.787\,ms} & Yes & 23 \\
Two-Rivers & Enumerated & 19{,}597 & 90.677\,ms & Yes & 23 \\
\midrule
PyRoboSim P0 & One-Hot & 5{,}102{,}846 & 449.174\,s & Yes & 19 \\
PyRoboSim P0 & Enumerated & \textbf{3{,}278{,}525} & \textbf{149.430\,s} & Yes & 19 \\
\midrule
PyRoboSim P1 & One-Hot & 46{,}278{,}204 & 25{,}841.977\,s & Yes & 27 \\
PyRoboSim P1 & Enumerated & \textbf{5{,}643{,}484} & \textbf{235.193\,s} & Yes & 27 \\
\midrule
Quadcopter & One-Hot & \textbf{854{,}545} & 19.730\,s & Yes & \textbf{76.7} \\
Quadcopter & Enumerated & 1{,}020{,}365 & \textbf{9.275\,s} & Yes & 106 \\
\bottomrule
\end{tabular}

%% file: Conclusion.tex
\section{Conclusion}
\label{sec:conclusion}

This paper provides the complete architecture, implementation, and formal
foundations of an open-source pipeline for generating executable ROS~2
FlexBE supervisors from capability-based GR(1) specifications. It connects
robot software interfaces to specification generation, assumption analysis,
reactive synthesis, strategy auditing, behavior-preserving state reduction, and
controller realization. The formal arguments and comparative evaluation
explain both how these stages work together and what assurance each provides.
The evaluation spans four domains and includes qualitative PiHawk hardware
validation of the \textbf{1S} and \textbf{2S} configurations
(\resec{crazyflie-failure-validation}).

Enumerated capability encoding consistently used fewer atomic propositions
and usually reduced synthesis time with \texttt{slugs}/CUDD. Its effects on
controller size and peak BDD size varied. Fair-Outcome with pending memory
could yield realizable strategies with goal-unreachable traps. Within the
tested budget, System-Goal without pending memory was the only formulation
that produced realized FlexBE HFSMs under both encodings in all six
comparisons (\resec{summary-results}). Pairing it with enumerated encoding
is therefore an empirical starting point for this backend.

Well-separation analysis checks whether the supervisor can prevent the
environment from satisfying its assumptions. Rewriting PyRoboSim P0's
carriage exclusions in next-state form removed the associated well-separation
finding without changing the extracted strategies or audit verdicts in the
seven paired experiments under alphabetic declaration ordering
(\reappsec{p0-carriage-check}). When the initial conditions already rule out
carrying multiple items at once, we recommend expressing that exclusion in
next-state form.

The auditor complements this specification-level analysis by checking each
extracted strategy for protocol violations, deadlocks, and goal-unreachable
traps before realization. The soundness and completeness arguments cover
defect detection under the stated protocol-monitoring and resource
conditions. These checks do not establish general liveness:
System-Goal can accommodate recurring capability noncompletion, and the
auditor does not exclude cycles whose available exits toward the goal
remain unused (Corollary~\ref{cor:unforced-exit}).
State reduction preserves core-output traces under the same admissible input
sequences as the raw strategy (\resec{reduction}). Controller realization is
supported by forced-outcome simulation and hardware tests
(\resec{crazyflie-failure-validation}), while physical execution relies on
the capability-interface contract (\resec{controller-realization}).

Future work should distinguish specification-level effects from those of the
\texttt{slugs}/CUDD backend through ROBDD-oriented encoding refinements
\cite{GR1Patterns15} and comparison with gr1c~\cite{Livingston-gr1c},
Strix~\cite{MeyerSL2018}, SAT-based synthesizers~\cite{FaymonvilleFT2017},
and kind-controller variants~\cite{KindControllers24}. Further analysis
should guide synthesis away from assumption-falsifying strategies and detect
reachable cycles with available but unused exits toward the goal.
Scenario-based specification validation and unnecessary-variable and clone
detection~\cite{Maoz2026} could also complement the pipeline's strategy checks.

Because realizability alone does not establish deployment readiness, the
complete pipeline makes the additional modeling, validation, reduction, and realization
steps explicit and available for reuse~\cite{flexbe-synthesis, flexbe-synthesis-demo}.
Together, its documented interfaces, formal arguments, comparative evidence,
and hardware demonstrations provide a practical foundation for understanding,
reusing, and extending synthesis-based workflows for failure-aware robot
supervision.

%% file: Appendices.tex
\ifIEEEtranloaded
  \appendices
\else
  \begin{appendices}
\fi

\makeatletter
\@ifundefined{theHtable}{}{%
  \renewcommand{\theHtable}{appendix.\thesection.\arabic{table}}%
}
\makeatother

\input{AuditorAppendix}

\input{ReductionAppendix}

\input{CoffeeAppendix}

\input{TwoRiversAppendix}

\input{PyRoboSimAppendix}

\input{CrazyFlieAppendix}

\ifIEEEtranloaded
\else
  \end{appendices}
\fi

%% file: AuditorAppendix.tex
\section{Proofs for the Strategy Auditor (\resec{auditor})}
\label{app:auditor-proofs}

\subsection{Semantic Model}

The auditor works directly on an explicit Mealy strategy of the kind emitted by
\texttt{slugs}. As the main-body JSON example in
\relist{slugs-example-json} illustrates, each JSON node is a strategy state.
The \texttt{variables} item fixes the proposition order for the per-node
\texttt{state} array, and \texttt{trans} lists the node's successors.
Each state therefore records both the environment input observed on arrival
and the system output asserted there. This state-labeled representation gives
the activation--outcome convention used throughout the appendix: a state's
output selects a capability, and the successor state's input reports its
outcome. The activation and outcome labels below are understood after decoding
the chosen encoding (\resec{encodings}). They express the generated
activation--outcome constraints \reqn{env-assumptions} on the explicit
strategy. The raw-strategy audit checks the protocol and failure bounds
during reachable-product exploration. After a clean safety check, it uses
the explored graph to check for goal-unreachable traps.

The optional post-reduction audit has a narrower role. In the reduced
automaton, a state may carry merged input-label metadata from several
unreduced states (\resec{reduction}). The second audit therefore disables
strict protocol monitoring and acts only as a confirmatory structural check
for deadlocks and goal-unreachable traps.

The product automaton, audited safety violations, and goal-unreachable traps
that make Lemma~\ref{lem:safety} and Theorem~\ref{thm:goal-unreachable-trap}
precise are defined below. The auditor can confirm a safety defect as soon
as it finds a witness. Applying the trap check requires a complete reachable
product graph from a clean safety check.

The product includes failure counters so the auditor can check whether an
extracted strategy respects configured limits on immediate retries. The
case-study synthesis runs leave these bounds unset. Separate tests on
hand-constructed strategies exercise the finite-bound check.

\begin{definition}[Product Automaton]
\label{def:product}
Let $S$ denote the finite set of states of the synthesized Mealy strategy,
with transition relation $\delta \subseteq S \times S$ and initial state
$s_0 \in S$. Each state $s \in S$ carries labels
$\mathit{in}(s) \subseteq AP_I$ and $\mathit{out}(s) \subseteq AP_O$. Let
$\mathcal C$ be the set of executable capabilities with declared outcome
tokens (\resec{specifications}). For each capability $\pi \in \mathcal{C}$,
let $O(\pi) \subseteq AP_I$ denote its declared outcome tokens, partitioned
into success tokens $O^s(\pi)$ and failure tokens $O^f(\pi)$, and let
$b_\pi \in \mathbb{N} \cup \{\infty\}$
denote its configured \emph{immediate-retry} failure bound: the number of
failures allowed while $\pi$ is re-activated on consecutive steps, as made
precise below. When unset, the bound defaults to $b_\pi = \infty$.

The auditor augments each strategy state with a monitor state: $p$ records
the capability whose outcome is expected on the next edge, and $\vec k$
records consecutive failures for capabilities with finite bounds.
Capability $\pi$ is \emph{activated} at $s$ if $\pi_a \in \mathit{out}(s)$,
and an outcome of $\pi$ is \emph{observed on arrival} at $s$ if
$\mathit{in}(s) \cap O(\pi) \neq \emptyset$. Let
$\mathcal{C}_\bot = \mathcal{C} \cup \{\bot\}$ denote the pending-capability
domain, where $\bot$ denotes no capability pending. For $\pi \in \mathcal{C}$
with $b_\pi < \infty$, let $K_\pi = \{0, \dots, b_\pi + 1\}$ track $\pi$'s
current consecutive-failure count. Values $0, \dots, b_\pi$ mean that many
consecutive failures observed so far, still within bound, and the extra
value $b_\pi + 1$ is a saturating sentinel meaning the bound has been
exceeded, so $|K_\pi| = b_\pi + 2$. For $\pi$ with $b_\pi = \infty$, fix
$K_\pi = \{0\}$ instead. With no bound to check, this component has nothing
to track. The \emph{product automaton} has state space
\[
    P \;=\; S \times \mathcal{C}_\bot \times \textstyle\prod_{\pi \in \mathcal{C}} K_\pi .
\]
Each strategy edge $(s,s')\in\delta$ whose successor activates at most one
capability deterministically updates the monitor.
The successor component $p'$ is determined by the activation label at $s'$
and names the capability whose outcome is expected on the following edge. If
$s'$ activates exactly one capability $\pi$, then
$p'=\pi$; if $s'$ activates no capability, then $p'=\bot$. If $s'$
activates multiple capabilities, $p'$ is ambiguous, so the edge does not
extend $\delta_P$ and is recorded as a strict protocol violation
(Definition~\ref{def:protocol}(iv)). For every remaining edge, the monitor
successor is defined even when the edge will be reported as another safety
violation. Writing $\mathrm{act}(s')$ for the unique capability activated at
$s'$, or $\bot$ if none is,
$\big((s,p,\vec{k}),\,(s',p',\vec{k}')\big) \in \delta_P$ iff
$(s,s')\in\delta$, $s'$ activates at most one capability,
$p'=\mathrm{act}(s')$, and $\vec k'$ follows the counter rule below, in
particular $k'_\pi = 0$ for every $\pi \neq p'$. If $p'=\bot$, no counter is carried forward. If $p' \neq \bot$, then
$k'_{p'}=0$ unless the same capability is immediately retried after a
finite-bound failure. The only nonzero update is
$k'_{p'}=\min(k_p,b_p){+}1$ when $p'=p$, $b_p<\infty$,
$\mathit{in}(s') \cap O^s(p) = \emptyset$, and
$\mathit{in}(s') \cap O^f(p) \neq \emptyset$.
Only the capability becoming pending next can carry a nonzero counter.
$k_\pi$ counts failures within $\pi$'s current unbroken streak of immediate
re-activation and resets whenever that streak breaks. The initial product
state is $(s_0,p_0,\vec 0)$, where the configured initial state $s_0$
fixes the initial active capability $p_0$, or $p_0=\bot$ if no
capability is initially active. Multiple activations at $s_0$ are a strict
protocol violation and prevent construction of an initial product state.
\end{definition}

The bound is therefore an immediate consecutive-retry bound, not a
lifetime or per-task failure count. For example, in an
\emph{attempt--failure--log--retry} loop, the logging step activates a
different capability (or none), so $p' \neq \pi$ and $k_\pi$ resets to $0$
before the retry. The stored counter remains zero, and the bound is never
reported, however many times the loop repeats. An execution that repeats
this loop forever has infinitely many failures, so it falls outside the
finite-failure premise of the System-Goal guarantee (\resec{liveness}), and
that guarantee does not promise completion for it. The auditor does not
enforce this premise; counting failures across intervening steps would
require different counter semantics. The implementation increments the
observed failure counter, resets counters for capabilities not selected
next, then checks the bounds. It returns on the first violation before
enqueuing that successor. Thus the saturating $b_\pi+1$ value in
Definition~\ref{def:product} represents the first violation; it does not
require the implementation to continue exploring beyond it.

\begin{definition}[Audited Safety Violations]
\label{def:protocol}
A strategy edge $(s,s') \in \delta$ considered from a reachable product node
$(s,p,\vec k)$ has a \emph{strict protocol violation} if any of the following
holds:
\begin{enumerate}[label=(\roman*)]
  \item \emph{missing outcome}: $p\neq\bot$ and
    $\mathit{in}(s')\cap O(p)=\emptyset$;
  \item \emph{multiple outcomes}: $p\neq\bot$ and
    $|\mathit{in}(s')\cap O(p)|>1$;
  \item \emph{unexpected outcome}: $\mathit{in}(s')\cap O(\pi)\neq\emptyset$
    for some capability $\pi\neq p$; or
  \item \emph{multiple activations}:
    $|\{\pi\in\mathcal C:\pi_a\in\mathit{out}(s')\}|>1$.
\end{enumerate}
Condition (iv) is checked before constructing the product successor because
multiple activation makes $p'$ ambiguous. Otherwise, the edge has a
\emph{bounded-failure violation} if its monitor successor satisfies
$p\neq\bot$, $p'=p$, $b_p<\infty$, and $k'_p>b_p$. A state $s$ with no
outgoing edge under $\delta$ is a \emph{deadlock}.
\end{definition}

\begin{definition}[Goal and Goal-Unreachable Region]
\label{def:goal-unreachable-region}
The auditor receives an explicitly configured set of \emph{goal} labels
$G \subseteq AP_I \cup AP_O$. A product node $(s,p,\vec k)$ is a \emph{goal
node} if $\mathit{in}(s) \cap G \neq \emptyset$ or
$\mathit{out}(s) \cap G \neq \emptyset$. A reachable product node $q$ is
\emph{goal-unreachable} if no path in $\delta_P$ from $q$ reaches a
\emph{cyclic} goal-containing SCC of the reachable subgraph of $P$. Paths
may have length zero, so every node of a cyclic goal-containing SCC is
goal-reachable and such a component is never a trap. A
component is cyclic iff it has more than one node or is a single node with a
self-loop. A reachable cyclic SCC all of whose nodes are goal-unreachable is
a \emph{goal-unreachable trap}.
\end{definition}

\begin{remark}[Terminal flags and cyclic goals]
\label{rem:goal-kinds}
Terminal labels $\Theta$ are distinct from the executable capabilities
$\mathcal C$. Only executable capability activations count toward the
multiple-activation check or impose outcome obligations. A one-hot terminal
state with no capability selected gives $p'=\bot$. An enumerated terminal
state can retain a capability selection, which determines $p'$ and remains
subject to outcome monitoring (\resec{enumerated}).

The pipeline encodes terminal persistence as
$G(\theta_a\rightarrow X\theta_a)$ (\resec{terminal-outcomes}).
\texttt{slugs} constrains stored strategy transitions by system safety;
standard explicit extraction intersects these transitions with environment
safety~\cite{Ehlers2016,CNU_slugs}. The persistence rule is a system safety
guarantee. Every extracted edge from a state with $\theta_a$ asserted
therefore leads to a state with $\theta_a$ still asserted. This edge-by-edge
property holds even along paths that violate environment justice.
In a finite, deadlock-free reachable product graph, a terminal goal therefore
has a path to a cyclic goal-containing SCC: continuation is infinite, some
node repeats, and the goal remains true. The first terminal node need not
itself belong to that SCC. FlexBE ends execution at the terminal outcome
(\resec{controller-realization}); the graph represents its continuation.
A node without outgoing edges is instead a deadlock in this model.
For environment-reported recurrence goals without persistence, reaching a
goal once need not permit recurrence; requiring a cyclic goal-containing
SCC excludes such transient visits.
\end{remark}

\begin{remark}[Encoding independence]
The definitions use normalized capability-specific tokens. Under the one-hot
encoding, these tokens appear directly; for example, Coffee capability
\texttt{gr} uses activation \texttt{gr\_a} and outcomes \texttt{gr\_c} and
\texttt{gr\_f} (\relist{coffee-config}, \relist{discrete-abstraction}). Under
the enumerated encoding (\resec{enumerated}), the auditor remaps the shared raw
values back to the same named tokens using the specification's recorded
integer-to-capability mapping (e.g., \texttt{3: gr} in \relist{pending-trans-enumerated-example}).
\end{remark}

\begin{lemma}[Safety-check soundness and completeness]
\label{lem:safety}
Assume strict protocol monitoring is enabled and the safety check reaches a
verdict within its resource budget. The check reports a defect if and only if
the initial state activates multiple capabilities, an edge out of a reachable
product node has a strict protocol or bounded-failure violation, or a
reachable product node is a deadlock. Every reported defect has a reachable
witness. A clean safety verdict certifies that the exploration covered the
entire reachable product graph.
\end{lemma}

\begin{proof}[Proof of Lemma~\ref{lem:safety}]
The auditor first checks $s_0$ for multiple activations. If it finds any,
$s_0$ itself witnesses the defect. Otherwise, the initial product state
$(s_0,p_0,\vec 0)$ is unique. Each strategy prefix with at most one activation
per state lifts to a unique product path because the monitor update in
Definition~\ref{def:product} is deterministic.

The auditor explores from the initial product state in breadth-first order.
Each product node it visits has a path from that initial state. The auditor
returns a failure if the strategy state has no outgoing edge or if a candidate
edge triggers a protocol or bounded-failure violation. An edge whose target
activates multiple capabilities is checked directly because it has no product
successor. Each failure is therefore witnessed by the reached node or by the
path to that node followed by the offending edge.

If the queue empties without a defect, every outgoing edge of every visited
node has been checked and its product successor visited. By induction on path
length, every reachable product node has then been visited, including the
source of any candidate edge with no product successor. Thus any reachable
defect would have caused an earlier failure, and a clean verdict certifies
complete exploration.
\end{proof}

\begin{theorem}[SCC characterization of goal-unreachable traps]
\label{thm:goal-unreachable-trap}
If the reachable subgraph of $P$ contains a goal-unreachable trap -- a
cyclic SCC with no path, in the condensed component graph, to any cyclic
goal-containing component -- then the strategy has a reachable
goal-unreachable node (Definition~\ref{def:goal-unreachable-region}).
Conversely, if the reachable subgraph of $P$ is deadlock-free, then any
reachable goal-unreachable node implies the existence of such a
goal-unreachable trap.
\end{theorem}

\begin{proof}[Proof of Theorem~\ref{thm:goal-unreachable-trap}]
Let $R \subseteq P$ be the set of reachable product nodes with no path in
$\delta_P$ to a cyclic goal-containing component; by Definition~\ref{def:goal-unreachable-region},
$R$ is exactly the set of goal-unreachable nodes, and a goal-unreachable
trap exists iff some reachable cyclic SCC is a subset of $R$.
The deadlock-free assumption is guaranteed in the pipeline because this
check runs only after Lemma~\ref{lem:safety}'s exploration completes without
a safety violation. The proof uses two facts: $R$ is closed under transitions, and a
finite, closed, deadlock-free graph contains a cycle.

\emph{$R$ is closed under $\delta_P$.} If $x \in R$ and $(x,y) \in \delta_P$,
then $y \in R$: otherwise $y$ would have a path to a cyclic goal-containing
component, and prefixing the edge $(x,y)$ would give $x$ such a path,
contradicting $x \in R$.

\emph{Trap implies goal-unreachable node.} Let $Q$ be a reachable cyclic SCC
with no path to a cyclic goal-containing component, and let $q \in Q$. If $q$ had a path to a cyclic
goal-containing component, that path would witness a path from $Q$'s
component to it in the condensed graph, contradicting the hypothesis. So $q$
has no such path, i.e., $q \in R$, so $R \neq \emptyset$ and the strategy has
a reachable goal-unreachable node. This direction uses no assumption about
deadlocks; the assumption is only needed for the converse below.

\emph{Goal-unreachable node implies trap.} Suppose some reachable $q \in R$.
Since the reachable subgraph is deadlock-free, $q$ has an outgoing edge; by closure that
successor lies in $R$, and inductively every node reached by repeatedly
following outgoing edges from $q$ stays in $R$. As $R$ is finite, this
process must revisit a node, producing a cycle entirely within $R$. Let $Q$
be the SCC containing that cycle. From a node on the cycle, every node of
$Q$ is reachable; closure therefore gives $Q\subseteq R$. Hence no node of
$Q$ has a path to a
cyclic goal-containing component, so $Q$ has no path to one in the condensed
graph. $Q$ is reachable (via $q$) and cyclic by construction, so it is a
goal-unreachable trap. If $q$ were instead a deadlock, Lemma~\ref{lem:safety}
would already have reported it before this check runs, which is exactly
why the pipeline sequences the deadlock check first.
\end{proof}

\begin{corollary}
\label{cor:unforced-exit}
An SCC from which a cyclic goal-containing component is reachable is never
flagged, even if an adversarial environment could avoid every such exit
indefinitely.
\end{corollary}

\begin{proof}
Such a path excludes the SCC from the goal-unreachable region by
Definition~\ref{def:goal-unreachable-region}.
\end{proof}

\subsection{Practical Classification and an Example}

The auditor's practical classification follows the constructive direction of
Theorem~\ref{thm:goal-unreachable-trap}. It computes the SCCs of the reachable
subgraph of $P$ using Tarjan's algorithm~\cite{Tarjan1972}, forms the
condensed acyclic component graph, and marks every component that can reach a
\emph{cyclic} goal-containing component by backward reachability seeded from
those cyclic goal components. It then flags as
\texttt{GOAL\_UNREACHABLE\_TRAP} exactly the reachable cyclic components not
marked. By Theorem~\ref{thm:goal-unreachable-trap}, and given the preceding
deadlock-free safety check, this flags a component if and only if it is a
goal-unreachable trap. Thus the check is sound and complete relative to
Definition~\ref{def:goal-unreachable-region} for both termination and
task-completion audits and for both terminal and recurrence-type goals,
without needing to distinguish any of these at classification time.
\refig{pyrobosim-trap} shows a representative example found in the PyRoboSim
case study, namely a closed, reachable cycle with no edge leaving it, flagged exactly
as \texttt{GOAL\_UNREACHABLE\_TRAP} predicts.

\subsection{Complexity}
\label{sec:auditor-complexity}

Let $n=|S|$, $m=|\delta|$, and $k=|\mathcal C|$.
The Cartesian state space in Definition~\ref{def:product} is a conservative
ambient space; its components do not vary independently on reachable nodes.
Initially $p=\mathrm{act}(s_0)$ and all counters are zero. Every admitted
successor sets $p'=\mathrm{act}(s')$ and resets $k'_\pi$ for
$\pi\ne p'$. Inductively, at every reachable node $p=\mathrm{act}(s)$
and only $k_p$ can be nonzero. States with multiple activations are rejected
before insertion, including at initialization.

For a strategy state $s$ with at most one activation, define $h(s)=b_\pi+2$
if $\mathrm{act}(s)=\pi$ has a finite bound, and $h(s)=1$ otherwise.
Let $S_{\rm adm}$ be these states and
$H=\max(\{1\}\cup\{b_\pi+2:b_\pi<\infty\})$.
There are at most $h(s)$ reachable monitor states above $s$, including
the formal bound-exceeded sentinel. Each strategy edge is examined from
at most $H$ source copies. Hence, writing $P_{\rm reach}$ and
$E_{\rm reach}$ for the reachable product nodes and edges,
\begin{equation}
\begin{aligned}
 |P_{\rm reach}|&\le\sum_{s\in S_{\rm adm}}h(s)\le nH, &
 |E_{\rm reach}|&\le mH,\\
 T_{\rm graph}&=O((n+m)H).
\end{aligned}
\label{eqn:auditor-complexity}
\end{equation}
Here $T_{\rm graph}$ counts candidate-edge examinations and the elementary
graph operations of Tarjan's algorithm~\cite{Tarjan1972} and
condensation/backward reachability, assuming constant-cost node and label
operations. The implementation stops before inserting a violating
successor, so including its sentinel only makes the bound conservative.
There is no independent pending-slot factor or product of counter-domain
sizes. For every evaluated configuration, all bounds are infinite,
$H=1$, and graph processing requires $O(n+m)$ such operations.

Total implementation runtime includes additional costs. The Python auditor
constructs input-label sets, scans capability/outcome declarations, and
stores a full $k$-entry counter tuple even when all entries are zero.
Counter updates rebuild and sort these tuples; hashing product keys also
depends on their representation size. Automaton preprocessing sorts
successor lists, and diagnostic selection sorts component indices and
reported state names. These costs are separate from reachable-graph
cardinality and are not covered by the constant-cost assumptions in
\reqn{auditor-complexity}. Consequently, that equation is not a claim
that the entire implementation runs in $O(n+m)$ for the unbounded cases.

The auditor operates entirely on the explicit strategy already extracted
by synthesis. It does not manipulate the synthesis BDDs, whose size depends
on variable interaction and ordering as well as proposition count.
The reported \textbf{Auditing} times also include Python-harness,
interpreter, and subprocess overhead (\retab{timing-rows}); they measure
more than the graph algorithms alone.

\subsection{Pipeline Integration}
\label{sec:pipeline-integration}

The auditor is a pipeline module with the same YAML-declared input/output
interface as other pipeline stages (\resec{architecture}), taking the
synthesized automaton, capability metadata already produced earlier in the
pipeline, and an optional goal-outcome override, and returning a structured
result together with a \texttt{SynthesisErrorCode}. A configured wall-clock
budget bounds worst-case audit cost. A budget overrun reports a distinct
\texttt{SynthesisErrorCode.AUDIT\_INCOMPLETE} status together with an
explicit \texttt{incomplete: true} flag and a warning. An incomplete
audit leaves the result unknown. Lemma~\ref{lem:safety} applies when the
safety check reaches a failure or a clean verdict within the budget. A
confirmed safety failure can be returned before the graph is fully explored.

The pipeline manager treats \texttt{AUDIT\_INCOMPLETE} as non-fatal,
allowing later stages to produce a state machine for inspection or
testing while carrying the status through to the final result. This
continuation is not audit certification. Under the audit policy, a confirmed
pass permits deployment with respect to the defined defect classes, and a
confirmed safety or goal-unreachable-trap failure rejects the strategy. The
current manager has no configuration switch to make this status fatal;
deployment policy should reject the unknown unless a designer explicitly
accepts it.

Every audit that ran in the case studies completed within its budget.
In the reference-machine sweeps (\resec{discussion}), the longest
raw-strategy audit took 75.4\,ms (quadcopter \textbf{2SP}, random seed~18).
The post-reduction audit took 17.5\,ms in the same run. Together the two
audits took 92.9\,ms out of 311.1\,s overall (0.030\%). The per-trial
times are on line~99 of \texttt{release\_data/crazyflie/master.csv} inside
\href{https://raw.githubusercontent.com/CNURobotics/flexbe_synthesis_demo/e7991e65/experiments/experiment_data_15Sep26.zip}{\nolinkurl{experiments/experiment_data_15Sep26.zip}}
at commit \texttt{e7991e65}.

The separate well-separation module (\resec{well-separation-module})
exposes \texttt{fail\_on\_non\_well\_separated} and
\texttt{fail\_on\_incomplete} as opt-in analyzer configuration,
but every configuration evaluated in this paper runs with both disabled.
An NWS verdict is reported alongside the case class and responsible
assumptions, never blocking synthesis, auditing, or reduction. The
pipeline tracks three statuses (\retab{auditor-well-separation}):
\emph{unrealizable} (no winning strategy exists, \resec{synthesis-module});
\emph{NWS} (the specification permits the system to win by forcing an
environment-assumption violation, a property of the game, not of any one
extracted strategy); and \emph{strategy audit failure} (this particular
extracted strategy contains a confirmed safety violation or
goal-unreachable trap). Of these, NWS and strategy-audit failure can
co-occur for the same extracted strategy; an unrealizable specification
produces no strategy to audit. A differently-ordered
synthesis run could in principle produce a strategy that exploits the same
NWS assumption differently than the one tested here, even where the
tested strategy's audit is clean; the auditor's guarantee is scoped to the
concrete strategy it examines, not to every strategy the same NWS
specification could yield.

The optional post-reduction audit follows the raw-audit prerequisite and
masking limitation described in \resec{reduction} and
\reappsec{reduction-remarks}. The reduced
automaton is explicitly marked as reduced in the pipeline; when that marker
is present, the auditor skips strict protocol checks because merged state
labels are not single concrete outcome observations.

The audit's coverage is defined by reachability in the abstract strategy.
A clean verdict requires a complete reachable product graph, and a defect
witness is sufficient for a failure verdict (\reappsec{auditor-complexity}).
Controller size,
reported per domain in \resec{coffee-results}, \resec{two-rivers-results},
\resec{pyrobosim-results}, and \resec{crazyflie-results}, ranges from single
digits for Coffee up to 677 states for the largest quadcopter
configuration (\textbf{2FP}'s raw strategy, \reappsec{crazyflie-reliability-app}).
Manually forcing every declared
capability-failure outcome at least once during execution for Coffee and
Two-Rivers, and selected outcomes for PyRoboSim and the quadcopter case
study, instead confirms that the running HFSM implements the corresponding
audited transitions; the audit is the primary evidence of coverage.

%% file: ReductionAppendix.tex
\section{Proofs for State-Merging Reduction (\resec{reduction})}
\label{app:reduction-proofs}

Throughout this appendix, let $A = (S, \delta, s_0)$ denote the raw
strategy automaton produced by \texttt{slugs}, where $S$ is the finite set
of Mealy strategy states, $\delta \subseteq S \times S$ is the transition
relation, and $s_0 \in S$ is the initial state. Each Mealy strategy state
$s \in S$ carries labels $\mathit{in}(s) \subseteq AP_I$ and
$\mathit{out}(s) \subseteq AP_O$, the environment-input valuation observed
on arrival at $s$ and the system-output valuation asserted at $s$. A valuation
is represented by the propositions assigned true; omission denotes false, so
matching two valuations below requires equality, not set inclusion. Let
$AP_O^p \subseteq AP_O$ be the set of pending-output propositions, and for
any $o \subseteq AP_O$ write $o{\downarrow}=o\setminus AP_O^p$ for the core
output projection that discards those pending propositions.

\begin{definition}[Output-Trace Equivalence]
\label{def:trace-equiv}
An environment input sequence is \emph{admissible} from a state $s$
if replaying it against $\delta$ from $s$ never encounters a step whose
input matches no outgoing edge, i.e., every step matches some
$(s,t) \in \delta$ with $\mathit{in}(t)$ equal to that step's input.

Let $\approx$ be the largest relation $R \subseteq S \times S$ such that
$(s,s') \in R$ implies:
\begin{enumerate}[label=(\roman*)]
    \item $\mathit{out}(s){\downarrow} = \mathit{out}(s'){\downarrow}$; and
    \item for every $(s,t) \in \delta$ there is $(s',t') \in \delta$ with
    $\mathit{in}(t) = \mathit{in}(t')$ and $(t,t') \in R$, and
    symmetrically with $s,s'$ exchanged.
\end{enumerate}
States $s, s' \in S$ are \emph{output-trace equivalent}, written
$s \approx s'$, iff $(s,s') \in {\approx}$. Clause (ii) is the closure
condition of a deterministic bisimulation: related states match successors
under the same input, and those successors must again be related. The identity
relation satisfies both clauses, and taking converses or relational
compositions preserves them. Hence the largest such relation $\approx$ is
reflexive, symmetric, and transitive.

Because the extracted strategy is deterministic
(Definition~\ref{def:outcome-targets}), no two edges from one state share an
input label. Iterating clause (ii) therefore amounts to replaying an
admissible input sequence. Equivalently, $s\approx s'$ iff the two states
admit the same admissible input sequences and, along each one, produce
identical sequences of core output valuations forever.
\end{definition}

\begin{definition}[Reachable Pruning]
\label{def:pruning}
During pruning, a state $s \neq s_0$ with no incoming edge under the
current transition relation cannot be reachable from $s_0$. Such states are
removed, together with the outgoing edges this removal makes dangling. The
process repeats until no such state remains.
\end{definition}

\begin{definition}[Outcome-Indexed Targets]
\label{def:outcome-targets}
After pruning reaches a fixed point, for each surviving state $s$, let
$\tau_0(s)$ be the partial function mapping $\mathit{in}(t) \mapsto t$ for
every remaining edge $(s,t)$. Explicit extraction in \texttt{slugs} selects
one successor per next-input valuation and removes the other candidates
for that valuation~\cite{Ehlers2016,CNU_slugs}. Thus $\tau_0(s)$ is deterministic:
no two edges from $s$ share a label. It is computed once, immediately
after pruning and before any equivalence-merging begins; at that point
every label $\mathit{in}(t)$ is still exactly the raw input that reaches
$t$, since no target has yet had another target's arrival metadata folded
into it. The reduction sequence below redirects the targets of these maps
but never changes their fixed input keys.
\end{definition}

\begin{definition}[Merge Test]
\label{def:merge-test}
Let $G=(S_G,\tau_G,s_G)$ denote the current automaton at some point during
the reduction merge phase, where $\tau_G$ is its input-indexed partial
transition function. Two states $u,w \in S_G$ \emph{pass the merge test},
written $u \sim_G w$, iff
$\mathit{out}(u){\downarrow} = \mathit{out}(w){\downarrow}$ and
$\tau_G(u,\cdot) = \tau_G(w,\cdot)$
(Definition~\ref{def:outcome-targets}). When
$u \sim_G w$, the reducer keeps one representative, redirects every edge
into $w$, and every transition entry valued $w$, to that representative,
and discards $w$ as defined precisely below.
\end{definition}

\begin{remark}[Each merge redirects all incoming edges]
\label{rem:merge-order}
The redirection in Definition~\ref{def:merge-test} updates every edge into
$w$, and every transition entry valued $w$, at the moment of the merge,
regardless of whether the edge's source state has itself already been
tested against another state earlier in the sweep. The sweep order only
determines which pairs of states are tested against each other; it does
not gate which edges, or which transition entries, a given merge
updates. Consequently, a target merge can make two source maps equal and
enable those sources to merge later in the same or a subsequent sweep,
without weakening equality of their exact input keys.
\end{remark}

\begin{definition}[Reduction Sequence and Reduced Automaton]
\label{def:reduced-automaton}
Apply Definition~\ref{def:pruning} to $A$ to a fixed point, and let $S_r$
and $\delta_r$ be the surviving states and edges. Definition~\ref{def:outcome-targets}
then gives $\tau_0$. Set
$G_0=(S_r,\tau_0,s_0)$ and let $q_0:S_r\to S_r$ be the identity map.
At each stage, $q_k$ maps every post-pruning raw state to its current
representative; the map $r_k$ below performs the next single merge.

Suppose the $(k{+}1)$-th merge chooses $u,w\in S_k$ with
$u\sim_{G_k}w$ and keeps $u$. Define the representative map
$r_k:S_k\to S_{k+1}=S_k\setminus\{w\}$ by $r_k(w)=u$ and
$r_k(x)=x$ for $x\neq w$, and set $q_{k+1}=r_k\circ q_k$. The next
transition function is defined for every
$x\in S_k$ and $i\in\operatorname{dom}(\tau_k(x,\cdot))$ by
\[
  \tau_{k+1}(r_k(x),i)=r_k(\tau_k(x,i)).
\]
This is well-defined when $r_k(x)=u$ has the two possible preimages
$x=u,w$, because the merge test requires
$\tau_k(u,\cdot)=\tau_k(w,\cdot)$. The equation redirects every target
globally, including entries belonging to sources already examined in the
current sweep, while preserving every exact input key $i$.
Thus
$G_{k+1}=(S_{k+1},\tau_{k+1},q_{k+1}(s_0))$ is obtained from $G_k$
by precisely this merge and global redirection.

Continue through the reducer's configured left-to-right merge sweeps. If
sweeps are repeated to convergence, repetition stops when a complete sweep
performs no merge. Let $m$ be the total number of merges performed. The
reduced automaton is the final current automaton
\[
  \begin{gathered}
    A'=G_m=(S',\tau',s_0'),\\
    S'=S_m,\quad \tau'=\tau_m,\quad s_0'=q_m(s_0).
  \end{gathered}
\]
If an unlabeled graph relation is needed, it is derived as
$\delta'=\{(u,\tau'(u,i))\mid u\in S',\ i\in
\operatorname{dom}(\tau'(u,\cdot))\}$. An input sequence is admissible
from $u\in S'$ exactly when each successive application of $\tau'$ is
defined.

Because each $u\in S'$ is itself a raw state, $\mathit{out}(u)$ retains its
original, immutable meaning. Its raw arrival label $\mathit{in}(u)$ also
remains unchanged, but execution of $A'$ uses the input key $i$ of $\tau'$,
not that representative's own arrival label. This is distinct from the
incoming-label \emph{metadata} the implementation separately tracks for
each $u\in S'$, the union of the $\mathit{in}(\cdot)$-derived metadata
of every raw state folded into $u$'s class; \reappsec{reduction-remarks}
explains why that metadata is a superset of what any single raw
predecessor context actually observed and is therefore not used for
execution or strict per-edge protocol checking after reduction.
\end{definition}

\begin{lemma}[Pruning preserves behavior]
\label{lem:pruning-sound}
No state removed by Definition~\ref{def:pruning} lies on any $\delta$-path
from $s_0$, so removing it changes neither the set of executions reachable
from $s_0$ nor any label along them.
\end{lemma}

\begin{proof}[Proof of Lemma~\ref{lem:pruning-sound}]
By induction on path length $k$. For $k=0$, $s=s_0$, which is never removed
(it is protected). For the inductive step, suppose every state reachable
from $s_0$ within $k$ steps survives pruning throughout, and let $s$ be
reached by a path $s_0 \to \cdots \to p \to s$ of length $k{+}1$, so $p$ is
itself reached within $k$ steps. By the induction hypothesis $p$ is never
removed, so its outgoing edge to $s$ is never among the edges a removal
retracts (Definition~\ref{def:pruning} only retracts the outgoing edges of a
state that is itself removed). Hence $s$ retains this one incoming edge
throughout pruning and is therefore never a zero-indegree state, so $s$ is
never removed. Every reachable state survives by induction, and every
removed state was unreachable throughout.
\end{proof}

\begin{remark}[Pruning is sound but not complete]
\label{rem:pruning-incomplete}
Definition~\ref{def:pruning} is sound but not necessarily complete: it
removes a state only once iterated indegree-zero deletion reaches it,
which need not include every state unreachable from $s_0$. A cycle
disconnected from $s_0$ (e.g., $u \to v \to u$ with no edge from any
reachable state into $\{u,v\}$) has every member state with a nonzero
indegree supplied from within the cycle itself, so the procedure leaves
it in place even though neither $u$ nor $v$ is reachable. This does not
weaken Lemma~\ref{lem:pruning-sound} or
Theorem~\ref{thm:reduction-behavior}: a state that remains unreachable
from $s_0$ is never visited regardless of whether pruning removes it, so
any left-over unreachable structure is inert with respect to executable
behavior, not a correctness gap. It can, however, leave $A'$ larger than
strictly necessary.
\end{remark}

\begin{lemma}[Quotient transition invariant]
\label{lem:quotient-transition}
For every reduction stage $G_k$ and every post-pruning raw state $s\in S_r$,
\[
  \operatorname{dom}(\tau_k(q_k(s),\cdot))
    =\operatorname{dom}(\tau_0(s,\cdot)),
\]
and, for every input $i$ in this common domain,
\[
  \tau_k(q_k(s),i)=q_k(\tau_0(s,i)).
\]
Thus taking representatives preserves both the available input keys and
their transition targets.
\end{lemma}

\begin{proof}[Proof of Lemma~\ref{lem:quotient-transition}]
For $k=0$, $q_0$ is the identity, so both statements are immediate. Assume
they hold at stage $k$. Definition~\ref{def:reduced-automaton} preserves all
input keys when applying $r_k$; if two source states are identified, the
merge test requires their maps, and hence their domains, to be equal.
Therefore
$\operatorname{dom}(\tau_{k+1}(q_{k+1}(s),\cdot))$
remains $\operatorname{dom}(\tau_0(s,\cdot))$. For every $i$ in that domain,
$q_{k+1}=r_k\circ q_k$ and the transition update give
\[
\begin{aligned}
\tau_{k+1}(r_k(q_k(s)),i)
  &=r_k(\tau_k(q_k(s),i))\\
  &=r_k(q_k(\tau_0(s,i))).
\end{aligned}
\]
The left and right sides are respectively
$\tau_{k+1}(q_{k+1}(s),i)$ and
$q_{k+1}(\tau_0(s,i))$, proving the next stage. Both statements therefore
hold through the final stage $G_m=A'$.
\end{proof}

\begin{lemma}[Merge soundness]
\label{lem:merge-sound}
Every pair $(u,w)$ the reduction procedure merges satisfies $u \approx w$ in
the original raw automaton $A$.
\end{lemma}

\begin{proof}[Proof of Lemma~\ref{lem:merge-sound}]
The idea is to show that each merge joins states with equal core outputs
whose successors, for every possible input, are already known to be
output-trace equivalent. Lemma~\ref{lem:quotient-transition} supplies the
connection between a current transition and the raw successors it
represents.

Proceed by induction on the number $k$ of merges already performed. Maintain
the following two properties of the representative map
$q_k:S_r\to S_k$ from Definition~\ref{def:reduced-automaton}:
\begin{enumerate}[label=(\roman*)]
  \item $q_k(s)\approx s$ for every raw state $s\in S_r$; and
  \item $q_k(x)=x$ for every current representative $x\in S_k$.
\end{enumerate}
At $k=0$, $q_0$ is the identity, so both properties hold by reflexivity of
$\approx$. The preceding pruning step is handled separately by
Lemma~\ref{lem:pruning-sound}; this induction concerns only states in $S_r$.

Assume both properties hold for $G_k$, and suppose the next merge keeps
$u\in S_k$ and discards $w\in S_k$, where $u\sim_{G_k}w$. The merge test
first gives
$\mathit{out}(u){\downarrow}=\mathit{out}(w){\downarrow}$, which is
clause (i) of Definition~\ref{def:trace-equiv}.

It remains to establish clause (ii). The same merge test gives
$\tau_k(u,\cdot)=\tau_k(w,\cdot)$, so these maps have the same fixed input
keys. Choose any such key $i$. Because current representatives are fixed
points of $q_k$ by property (ii), the fixed domains introduced in
Definition~\ref{def:outcome-targets} identify raw successors $t_1$ of $u$
and $t_2$ of $w$ with
$\mathit{in}(t_1)=i=\mathit{in}(t_2)$. Lemma~\ref{lem:quotient-transition}
and equality of the two current maps give
\[
  q_k(t_1)=\tau_k(u,i)=\tau_k(w,i)=q_k(t_2).
\]
Property (i), symmetry, and transitivity of $\approx$ therefore give
\[
  t_1\approx q_k(t_1)=q_k(t_2)\approx t_2.
\]
Thus successors reached under the same raw input $i$ are equivalent. In
particular, equality is tested on the original input keys, so two different
inputs cannot become matched merely because their targets were merged
earlier. Since the maps have equal domains, the same argument covers every
successor of both $u$ and $w$, establishing clause (ii) in both directions.

Now let $R={\approx}\cup\{(u,w),(w,u)\}$. The relation $\approx$ already
satisfies both clauses of Definition~\ref{def:trace-equiv}; the preceding
two paragraphs show that the only newly added pairs do as well, with their
successor pairs lying in $\approx\subseteq R$. Hence $R$ satisfies the same
closure condition. Because $\approx$ is the largest relation satisfying
that condition, $R\subseteq{\approx}$, and therefore $u\approx w$.

Finally, consider the updated map $q_{k+1}=r_k\circ q_k$. If
$q_k(s)\neq w$, then $q_{k+1}(s)=q_k(s)\approx s$. If $q_k(s)=w$, then
$q_{k+1}(s)=u\approx w=q_k(s)\approx s$. Thus property (i) is preserved.
For each $x\in S_{k+1}$, property (ii) at stage $k$ and $x\neq w$ give
$q_{k+1}(x)=r_k(q_k(x))=r_k(x)=x$, so property (ii) is preserved as well.
The induction proves that every performed merge is sound.
\end{proof}

\begin{theorem}[Reduction preserves executable behavior]
\label{thm:reduction-behavior}
The reduced automaton $A'$ (Definition~\ref{def:reduced-automaton}) has the
same core-output traces as $A$ from their respective initial states. An
environment input sequence is admissible from $s_0$ in $A$ if and only if it
is admissible from $s_0'$ in $A'$, and executing any such sequence against
both automata produces identical sequences of core output valuations
$\mathit{out}(\cdot){\downarrow}$, and hence identical capability
activations and parameter choices at every step.
\end{theorem}

\begin{proof}[Proof of Theorem~\ref{thm:reduction-behavior}]
By Lemmas~\ref{lem:pruning-sound} and~\ref{lem:merge-sound}, pruning changes
nothing reachable from $s_0$, and every merge collapses only a pair of
raw states already related by $\approx$, so the
final quotient map $q=q_m$ from $S_r$ onto the states of $A'$ satisfies
$q(s) \approx s$ for every reachable $s$. Lemma~\ref{lem:quotient-transition}
also gives, for every reachable raw state $s$, equality between the input
keys available at $s$ and at $q(s)$, with
$\tau'(q(s),i)=q(\tau_0(s,i))$ for each such key $i$.

Now fix any environment input sequence, and let $s_j$ be the raw state reached
after its first $j$ inputs whenever that prefix is defined. Starting from
$s_0'=q(s_0)$, induction on $j$ shows that the next input is defined in one
automaton if and only if it is defined in the other. Whenever it is defined,
the raw successor $s_{j+1}=\tau_0(s_j,i_{j+1})$ and the reduced successor satisfy
$\tau'(q(s_j),i_{j+1})=q(s_{j+1})$. Thus the sequence is admissible in
either both automata or neither; when admissible, their executions visit
$s_0,s_1,\ldots$ and $q(s_0),q(s_1),\ldots$, respectively. Since
$q(s_j)\approx s_j$, clause (i) of
Definition~\ref{def:trace-equiv} gives
$\mathit{out}(q(s_j)){\downarrow}=\mathit{out}(s_j){\downarrow}$ at every
step. These core output valuations determine which capability is activated
and with which parameters (\resec{syn_spec}), proving the claim. The exact
keys of $\tau'$ preserve the environment valuations themselves; the unioned
arrival metadata of a representative plays no role in this execution.
\end{proof}

\begin{corollary}[Guarantee transfer to the reduced executable automaton]
\label{cor:guarantee-transfer}
If the raw strategy $A$ witnesses realizability of \reqn{varphi} (i.e.,
for every environment-input sequence whose execution with $A$ satisfies
$\varphi_i^a \land \varphi_s^a \land \varphi_l^a$, that execution satisfies
$\varphi_s^g \land \varphi_l^g$), then under every such sequence $A'$ issues
the same capability activations and parameter choices and observes the same
reported outcomes. It therefore satisfies the capability-level consequences
of the guarantees.
\end{corollary}

\begin{proof}[Proof of Corollary~\ref{cor:guarantee-transfer}]
Fix such an input sequence. Because $A$ witnesses realizability and its
execution satisfies the assumptions, that execution satisfies
$\varphi_s^g\land\varphi_l^g$. Theorem~\ref{thm:reduction-behavior} says that
$A'$ reproduces its observed inputs and core outputs. These determine the
capability-activation, parameter-selection, and outcome-observation events
constrained by the capability-level guarantees (\resec{syn_spec}).
Theorem~\ref{thm:reduction-behavior} does not preserve the auxiliary
pending propositions or the system-transition constraints that define their
dynamics. Consequently, $A'$ need not itself be a winning strategy over the
full original GR(1) proposition valuation. Rather, it preserves the
capability-level consequences of the winning raw strategy, including its
activation and parameter-selection traces under the same observed outcome
sequence.
\end{proof}

\begin{remark}[Scope of the guarantee]
Corollary~\ref{cor:guarantee-transfer} concerns the reduced automaton $A'$
itself, not the FlexBE HFSM generated from it. Physical interpretation also
requires the capability-interface contract stated in
\resec{controller-realization}. The pipeline realizes $A'$
as an HFSM through a mechanical translation of strategy states and
transitions into FlexBE states and outcome wiring. The forced-outcome tests
(\resec{crazyflie-failure-validation}) validate the fidelity of that
translation empirically, rather than as part of the proof above.
\end{remark}

\subsection{What Reduction Does Not Preserve}
\label{sec:reduction-remarks}

Theorem~\ref{thm:reduction-behavior} is deliberately a claim about
\emph{forward} behavior (what the automaton does from a state onward) and
says nothing about \emph{how} a merged state was reached.

First, $\mathit{out}(\cdot){\downarrow}$ is preserved exactly by every merge
(the merge test requires it), but the discarded pending bits $\mathit{out}(\cdot)
\cap AP_O^p$ are not. The surviving representative keeps whichever pending
valuation it happened to carry, arbitrarily with respect to the merged
partner. This is harmless for Theorem~\ref{thm:reduction-behavior} because no
downstream consumer branches on $AP_O^p$ (\resec{on-demand-spec}), but it does mean $A'$ is not,
in general, a valid strategy over the full original GR(1) proposition
valuation; only its core projection is meaningful, which is all
Corollary~\ref{cor:guarantee-transfer}
claims.

Second, \resec{auditor} disables strict protocol monitoring after
reduction for the same reason. The implementation represents a merged
state's incoming label as the union of the $\mathit{in}(\cdot)$-derived
metadata from every raw state folded into it. This union may be a
superset of what any single raw predecessor context actually observed.
Definition~\ref{def:trace-equiv} says nothing about this backward
direction (it is not part of $\approx$), so Definition~\ref{def:protocol}'s
per-edge checks, which reason about a concrete single outcome observed
on arrival, cannot be applied unmodified to $A'$; this is exactly the
restriction \resec{auditor} already states.

The post-reduction deadlock check remains valid because it depends only on
the redirected transition relation, which the reducer updates explicitly.
The structural goal-reachability check instead depends on which goal labels
are \emph{present} at a state, and
Definition~\ref{def:goal-unreachable-region} permits a goal label from
either $AP_I$ or $AP_O$. On the output side, nonpending $AP_O$ goal
labels are preserved by the merge test's core-output equality. If such a
goal label appears in any raw state folded into a representative, it also
appears in that representative's core output. Pending output labels are
outside this preservation claim, as discussed above. On the input side,
the implementation's unioned incoming-label metadata covers the
outcome-token propositions that the pipeline's own goal configurations
(\resec{demonstrations}) actually draw from; an arbitrary $AP_I$
proposition outside that category is instead inherited from whichever
raw state serves as the representative, not unioned across the merged
partners, so this argument does not extend to it. For the goal-label
kinds the pipeline actually uses, the set of labels relevant to the
post-reduction structural audit does not shrink. Nonpending output goals
are preserved, and outcome-token input goals can only be added by the
unioned metadata.

Reduction also preserves the paths needed by this check. Repeated application
of Lemma~\ref{lem:quotient-transition} maps every reachable raw strategy path
to a reduced path, and every reachable raw cycle to a reduced closed walk (a
self-loop if the whole cycle collapses). In the pipeline's unbounded-failure
setting, the product monitor adds no nontrivial counters, and its
pending-capability component is determined by the preserved activation
outputs (Definition~\ref{def:product}), so this mapping lifts to the product
graph used by the structural audit. Consequently, a
path to a cyclic goal-containing component cannot disappear. Adding goal
labels can only make more components able to reach a goal. Thus reduction
cannot introduce a goal-unreachable trap into an automaton whose raw form
had none (Theorem~\ref{thm:goal-unreachable-trap}).
Reduction can mask an existing trap by folding it into a goal label it
did not, in the raw automaton, actually earn, but it cannot manufacture a
trap that was not already present in some form in $A$.

This establishes the basis for the confirmatory post-reduction audit and
the completed-raw-audit prerequisite stated in \resec{reduction}; the
incomplete-audit continuation policy is given in
\reappsec{pipeline-integration}.

Finally, even when Definition~\ref{def:merge-test} is applied until the
local sweep reaches a fixed point, the reduced automaton $A'$ need not be
the coarsest (fewest-state) quotient of $A$ under $\approx$. A fixed point
only says that no
remaining pair satisfies this concrete equality test against the current
representatives; it does not decide the largest output-trace equivalence
relation. Failure to achieve the minimum state count affects compactness,
not correctness: Lemma~\ref{lem:merge-sound} only certifies merges that are
already sound. In the experiments
reported here, fixed-point iteration did not reduce any controller beyond the
first productive merge sweep: all additional sweeps were convergence checks.
The coarsest quotient could be computed by partition refinement in the
style of Hopcroft's algorithm~\cite{Hopcroft2001} or Paige--Tarjan
refinement~\cite{PaigeTarjan1987}, adapted to the
output labels and outcome-indexed successor maps of these Mealy strategies.
Such an implementation would improve the theoretical merge-phase bound from
repeated pairwise sweeps, whose conservative worst case is cubic in the number
of states up to comparison cost, to the near-linearithmic bounds typical of
efficient partition refinement on sparse explicit automata.

%% file: CoffeeAppendix.tex
\section{Coffee Demonstration}
\label{sec:coffee-results-app}

\resec{coffee-results} summarizes the Coffee results by realizability,
well-separation, strategy size, and the three headline timing rows.
This appendix gives the full specification-size and BDD-node breakdown
behind those summary tables, the per-configuration well-separation timing,
the specification-level mechanism behind the strategy non-uniqueness noted
there, and the full variable-ordering sweep.

\begin{table*}[t]
\centering
\caption{\textbf{[$N=40$ random-seed orderings.]} Coffee synthesis results with one-hot labeling, full detail
underlying \retab{coffee-1hot-results-summary}. Reported values are the
mean over 40 random declaration orderings under a CUDD reordering threshold
of 250; integer-valued count means indicate identical samples, while
non-identical count means are shown with one decimal digit, with a
companion row giving $\pm$ one standard deviation wherever the underlying
trials disagree. Column~\textbf{1} is the hand-written baseline; the
leading digit \textbf{2}/\textbf{3} denotes the base/extended capability
set, \textbf{S}/\textbf{F} denotes System-Goal/Fair-Outcome liveness, and
\textbf{P} adds pending propositions.}
\label{tbl:coffee-1hot-results-app}
\small
\renewcommand{\arraystretch}{1.0}
\setlength{\tabcolsep}{3pt}
\input{tables/coffee_one_hot}
\end{table*}

\begin{table*}[t]
\centering
\caption{\textbf{[$N=40$ random-seed orderings.]} Coffee synthesis results with the enumerated capability encoding,
full detail underlying \retab{coffee-enumerated-results-summary} (same
reporting convention and column legend as \retab{coffee-1hot-results-app}).
The enumerated encoding has no hand-written baseline, so column~\textbf{1}
is omitted.}
\label{tbl:coffee-enumerated-results-app}
\small
\renewcommand{\arraystretch}{1.0}
\setlength{\tabcolsep}{3pt}
\input{tables/coffee_enumerated}
\end{table*}

Across both tables, moving from the basic (2) to the extended (3) capability
set consistently increases symbolic complexity, and adding pending memory
(P) further grows the output alphabet and transition-formula counts. The
largest BDD footprints and realizability times occur in the extended
Fair-Outcome pending case (one-hot \textbf{3FP}: 3193.8 peak nodes and 6.316\,ms
realizability), indicating that outcome constraints and richer output
alphabets, rather than problem size or encoding alone, drive most of the
observed complexity. The enumerated encoding uses fewer input/output propositions and lower
live BDD-node counts. Peak allocations mostly tie at 1{,}022 nodes;
enumerated is slightly lower for \textbf{3F}/\textbf{3SP}
(1{,}022 versus 1{,}047.6) and substantially lower for \textbf{3FP}
(2{,}044 versus 3{,}193.8). Transition-formula counts are also lower. Its extracted and reduced
strategies are nevertheless larger (\textbar{}Slugs SM\textbar{} and
\textbar{}Reduced\textbar{}), so the symbolic-size savings do not imply
smaller controllers (\resec{coffee-results}).

\textbf{Well Separation} isolates a different issue from realizability or
controller extraction. All Coffee specifications are realizable, and all
extracted strategies passed the post-synthesis auditor, but the extended
Fair-Outcome-with-pending specification (\textbf{3FP}) is not well-separated
under either encoding. Minimized-core well-separation checks over all
\textbf{3FP} orderings consistently identify the grind-completion
environment liveness assumption as the source of the violation; in the
one-hot encoding the corresponding grind-completion safety assumption is
also part of the minimal core. In specification terms, the Fair-Outcome
pending rule says that once grinding is pending, the environment must
eventually report grind completion, while the system controls the
activation and pending structure that makes this obligation relevant --
the environment fairness assumption is entangled with a system-controlled
retry context, the \texttt{P-all/E-just} case in the Maoz--Ringert
classification~\cite{MaozRingert2016}. This is exactly the kind of
assumption-falsification risk well-separation is designed to expose, but it is not
the same as unrealizability or an unusable controller. The synthesized
Coffee controllers here are small, valid, and executable. The check should
be read as a specification diagnostic, complementary to the post-synthesis
auditor, rather than a binary indication that synthesis has failed.
The targeted addition of system retry justice to one-hot extended Coffee
\textbf{3FP} left its extracted state valuations and transitions unchanged;
\reappsec{p0-retry-check} documents this check alongside P0.

Measured as internal check time in \texttt{slugs}
rather than Python wall-clock time around it
(\retab{timing-rows}), the values are sub-millisecond across
every configuration. The extended Fair-Outcome-pending case is the largest
entry in both encodings, but the separation is much clearer for one-hot
(\textbf{3FP}: 0.220\,ms against 0.025--0.096\,ms for the other one-hot
columns) than for enumerated (\textbf{3FP}: 0.086\,ms against
0.041--0.077\,ms elsewhere, within the observed run-to-run variation).
\textbf{3FP} is also the one non-well-separated configuration in each
table, but that status comes from the minimized-core assumption analysis
above, not from this timing magnitude; for the enumerated encoding in
particular, these sub-millisecond timing differences are well within
expected measurement variation and should not be read as corroborating
evidence.

The non-integer \textbar{}Slugs SM\textbar{} and \textbar{}Reduced\textbar{}
means reported for the one-hot base-capability columns
(\textbf{2S}/\textbf{2F}/\textbf{2SP}/\textbf{2FP}, 4.5 and 3.5 respectively,
\resec{coffee-results}) trace to a single unconstrained ordering choice in
the generated specification. That specification constrains capability
activation only in one direction. The clauses \texttt{!bd\_c' -> !gr\_a'} and
\texttt{!gr\_c' -> !br\_a'} prevent grinding or brewing before the preceding
step completes, but no corresponding \texttt{bd\_c' -> gr\_a'} clause forces
grinding to start as soon as button-detection completes. Once
\texttt{bd\_c} holds, \texttt{slugs} may therefore legally re-issue
\texttt{bd\_a} once more before moving on to \texttt{gr\_a} instead of
proceeding directly, and both choices are equally valid strategies for the
identical specification. Across the 40 one-hot \textbf{2S} random-seed
orderings, 22 extract the direct four-state cycle (\texttt{bd} $\to$
\texttt{gr} $\to$ \texttt{br} $\to$ \texttt{bd}) and 18 extract a five-state
variant with one redundant repeated \texttt{bd} activation before
\texttt{gr}; state-merging reduces these to three and four states
respectively, so the extra state survives reduction rather than being an
artifact reduction happens to remove. \refig{coffee-nonuniqueness}
(\resec{coffee-results}) shows a representative symbolic and reduced Mealy
graph pair for each case. The two strategy shapes correlate exactly with
declaration order in this sample. Every five-state trial requests a
declaration order with \texttt{gr\_a} first among the output propositions,
and no four-state trial does. This is a distinct sense in which declaration
order matters, beyond the BDD-size sensitivity examined next. Here it
selects among multiple behaviorally different but equally realizable
strategies for the same specification, rather than among differently sized
symbolic representations of one strategy. The enumerated encoding does not
show this particular non-uniqueness in the same comparison (\textbf{2S}
reports a clean \textbar{}Slugs SM\textbar{} of 6).

\subsection{Variable-Ordering Robustness}
\label{sec:coffee-ordering-app}

The Discussion summary (\resec{coffee-results}) states this sweep's headline
finding: variable-ordering sensitivity in Coffee is concentrated in the
extended configurations, especially the extended one-hot configurations,
and an explicit CUDD reordering threshold largely neutralizes it. This
subsection gives the full sweep behind that summary and the per-treatment
BDD-node means in \retab{coffee-ordering-summary-app}.
We swept all 16
Coffee configurations (base and extended capabilities $\times$ one-hot and
enumerated encodings $\times$ System-Goal and Fair-Outcome liveness $\times$
with/without pending), together with the hand-written baseline, against
42 variable-declaration orderings (alphabetic, a capability-grouped
``domain'' ordering, and 40 random seeds) crossed with two CUDD
dynamic-reordering settings (disabled, and a first-reordering threshold of
250 live nodes), each repeated 5 times, for 7{,}140 total synthesis trials, with
zero synthesis failures across the entire sweep. We omit CUDD's default
automatic reordering. A preliminary sweep confirmed it is effectively
inactive at these BDD sizes (indistinguishable from disabled), so an
explicit low threshold is the only setting that actually triggers
reordering here. CUDD's default first trigger is 4{,}004 live nodes, above even the
allocated peak-node count in about 95\% of the Coffee trials. This observation is specific to
Coffee's small BDDs. In
PyRoboSim P1 the uncapped policy and threshold 250 differ by more than an
order of magnitude on one configuration
(\seedetail{pyrobosim-p1-ordering-app}).

\begin{table*}[t]
\centering
\setlength{\tabcolsep}{3pt}
\caption{Mean CUDD peak BDD nodes across variable-ordering and
CUDD-reordering treatments, for all 16 Coffee configurations plus the
hand-written baseline ($N=420$ trials per row: 5 repeats $\times$
($2$ fixed + $40$ random declaration orderings) $\times$ 2 reordering
settings; 0 failures throughout). ``thr250'' is the first-reordering threshold of 250 live nodes;
integer-valued means indicate identical samples, while non-identical means are
shown with one decimal digit.}
\label{tbl:coffee-ordering-summary-app}
\begin{tabular}{llllrrrrrr}
\toprule
Cap. & Enc. & Live. & Pend. & \shortstack{Alpha\\none} & \shortstack{Alpha\\thr250} & \shortstack{Dom.\\none} & \shortstack{Dom.\\thr250} & \shortstack{Rand.\\none} & \shortstack{Rand.\\thr250} \\
\midrule
Base & full & \textemdash & \textemdash & 1022 & 1022 & 1022 & 1022 & 1022 & 1022 \\
Base & enum. & F & N & 1022 & 1022 & 1022 & 1022 & 1022 & 1022 \\
Base & enum. & F & Y & 1022 & 1022 & 1022 & 1022 & 1022 & 1022 \\
Base & enum. & S & N & 1022 & 1022 & 1022 & 1022 & 1022 & 1022 \\
Base & enum. & S & Y & 1022 & 1022 & 1022 & 1022 & 1022 & 1022 \\
Base & 1-hot & F & N & 1022 & 1022 & 1022 & 1022 & 1022 & 1022 \\
Base & 1-hot & F & Y & 1022 & 1022 & 1022 & 1022 & 1022 & 1022 \\
Base & 1-hot & S & N & 1022 & 1022 & 1022 & 1022 & 1022 & 1022 \\
Base & 1-hot & S & Y & 1022 & 1022 & 1022 & 1022 & 1022 & 1022 \\
Ext. & enum. & F & N & 1022 & 1022 & 1022 & 1022 & 1022 & 1022 \\
Ext. & enum. & F & Y & 3066 & 2044 & 3066 & 2044 & 3014.9 & 2044 \\
Ext. & enum. & S & N & 1022 & 1022 & 1022 & 1022 & 1022 & 1022 \\
Ext. & enum. & S & Y & 1022 & 1022 & 1022 & 1022 & 1022 & 1022 \\
Ext. & 1-hot & F & N & 2044 & 1022 & 2044 & 1022 & 2044 & 1047.6 \\
Ext. & 1-hot & F & Y & 6132 & 3066 & 6132 & 3066 & 7000.7 & 3193.8 \\
Ext. & 1-hot & S & N & 1022 & 1022 & 1022 & 1022 & 1303.1 & 1022 \\
Ext. & 1-hot & S & Y & 3066 & 1022 & 3066 & 1022 & 3883.6 & 1047.6 \\
\bottomrule
\end{tabular}
\end{table*}

\retab{coffee-ordering-summary-app} shows that variable-ordering
sensitivity is concentrated, not general. Every basic-capability
configuration and the hand-written baseline are completely flat (1022 peak
nodes regardless of declaration order or reordering setting), and so are
three of the eight extended configurations. Sensitivity is concentrated in
the extended one-hot cases but not confined to them. Extended enumerated
Fair-Outcome-with-pending (\textbf{3FP}) also moves: mildly with
declaration order when reordering is disabled (3066 alphabetic/domain-grouped,
3014.9 random), and down to 2044 once the reordering threshold is enabled. That
2044--3066 swing is narrower than either one-hot pending case below,
though wider than extended one-hot System-Goal without pending (1022 to
1303.1). Sensitivity is most pronounced where Fair-Outcome liveness
combines with pending memory under one-hot encoding (the same combination
identified above as the most expensive). Without reordering there, a
random declaration order inflates the peak relative to either fixed
ordering (extended one-hot, F, pending: 7000.7 mean peak nodes for random
vs.\ 6132 fixed; System-Goal pending: 3883.6 vs.\ 3066). Enabling the
reordering threshold collapses that gap: random falls to 3193.8 and
1047.6, nearly matching the fixed orderings' own reordered results (3066 and
1022). It also shrinks the fixed orderings themselves, by half for the
Fair-Outcome case and to one third for the System-Goal case. Once an
explicit threshold is enabled, the peak becomes almost independent of the
initial declaration order.

This suggests that, for this capability model, an explicit CUDD reordering
threshold is a more reliable lever than hand-tuned variable declaration
order. Enabling it recovers nearly all of the benefit of the best fixed
ordering regardless of where synthesis started, on the configurations where
ordering matters at all. CUDD's default automatic reordering does not help
here; it never triggers at these BDD sizes, so the explicit threshold, not
automatic reordering, is what does the work. For manual ordering, the
evidence supports a narrower takeaway: in this dataset, effort may be better
spent identifying which specification patterns introduce sensitivity in the
first place (Fair-Outcome liveness with pending) than on tuning declaration
order directly.

\FloatBarrier

%% file: tables/coffee_one_hot.tex
\begin{tabular}{|c|r|r|r|r|r|r|r|r|r|}
\hline
 & \textbf{1} & \textbf{2S} & \textbf{3S} & \textbf{2F} & \textbf{3F} & \textbf{2SP} & \textbf{3SP} & \textbf{2FP} & \textbf{3FP} \\
\hline
\rule{0pt}{4ex}\textbf{\shortstack{\rule{0pt}{2.8ex}Well\\Separated\rule[-1ex]{0pt}{1ex}}} & Yes & Yes & Yes & Yes & Yes & Yes & Yes & Yes & No \\
\hline
\textbf{Realizable} & Yes & Yes & Yes & Yes & Yes & Yes & Yes & Yes & Yes \\
\hline
\rule{0pt}{4ex}\textbf{\shortstack{\rule{0pt}{2.8ex}FlexBE HFSM\\Realized\rule[-1ex]{0pt}{1ex}}} & Yes & Yes & Yes & Yes & Yes & Yes & Yes & Yes & Yes \\
\hline
\textbf{$\left|AP_I\right|$} & 4 & 3 & 6 & 3 & 6 & 3 & 6 & 3 & 6 \\
\hline
\textbf{$\left|AP_O\right|$} & 3 & 3 & 5 & 3 & 5 & 3 & 8 & 3 & 8 \\
\hline
\textbf{$\left|\varphi_i^a\right|$} & 3 & 3 & 6 & 3 & 6 & 3 & 6 & 3 & 6 \\
\hline
\textbf{$\left|\varphi_i^g\right|$} & 3 & 3 & 5 & 3 & 5 & 3 & 8 & 3 & 8 \\
\hline
\textbf{$\left|\varphi_s^a\right|$} & 3 & 6 & 12 & 6 & 12 & 6 & 12 & 6 & 12 \\
\hline
\textbf{$\left|\varphi_s^g\right|$} & 7 & 6 & 19 & 6 & 19 & 6 & 31 & 6 & 31 \\
\hline
\textbf{$\left|\varphi_l^a\right|$} & 1 & 1 & 1 & 1 & 3 & 1 & 1 & 1 & 3 \\
\hline
\textbf{$\left|\varphi_l^g\right|$} & 1 & 1 & 1 & 1 & 1 & 1 & 1 & 1 & 1 \\
\hline
\rule{0pt}{4ex}\textbf{\shortstack{\rule{0pt}{2.8ex}CUDD Live\\BDD Nodes\rule[-1ex]{0pt}{1ex}}} & 136.5 & 89.1 & 250.5 & 89.1 & 258.4 & 89.1 & 414.1 & 89.1 & 621.0 \\
\hline
 & $\pm$ 8.4 & $\pm$ 8.1 & $\pm$ 13.5 & $\pm$ 8.1 & $\pm$ 14.3 & $\pm$ 8.1 & $\pm$ 32.8 & $\pm$ 8.1 & $\pm$ 73.3 \\
\hline
\rule{0pt}{4ex}\textbf{\shortstack{\rule{0pt}{2.8ex}CUDD Peak\\BDD Nodes\rule[-1ex]{0pt}{1ex}}} & 1022 & 1022 & 1022 & 1022 & 1047.6 & 1022 & 1047.6 & 1022 & 3193.8 \\
\hline
 &  &  &  &  & $\pm$ 160.0 &  & $\pm$ 160.0 &  & $\pm$ 468.8 \\
\hline
\textbf{\textbar{}Slugs SM\textbar} & 7 & 4.5 & 7 & 4.5 & 7 & 4.5 & 7 & 4.5 & 7 \\
\hline
 &  & $\pm$ 0.5 &  & $\pm$ 0.5 &  & $\pm$ 0.5 &  & $\pm$ 0.5 &  \\
\hline
\textbf{\textbar{}Reduced\textbar} & 3 & 3.5 & 3 & 3.5 & 3 & 3.5 & 3 & 3.5 & 3 \\
\hline
 &  & $\pm$ 0.5 &  & $\pm$ 0.5 &  & $\pm$ 0.5 &  & $\pm$ 0.5 &  \\
\hline
\rule{0pt}{4ex}\textbf{\shortstack{\rule{0pt}{2.8ex}Generate\\Specs (ms)\rule[-1ex]{0pt}{1ex}}} & 5.166 & 5.098 & 10.290 & 4.816 & 9.588 & 5.247 & 12.947 & 5.206 & 12.710 \\
\hline
 & $\pm$ 1.977 & $\pm$ 2.169 & $\pm$ 3.470 & $\pm$ 1.516 & $\pm$ 2.901 & $\pm$ 2.291 & $\pm$ 3.744 & $\pm$ 2.206 & $\pm$ 3.170 \\
\hline
\rule{0pt}{4ex}\textbf{\shortstack{\rule{0pt}{2.8ex}Well Separation\\(ms)\rule[-1ex]{0pt}{1ex}}} & 0.077 & 0.025 & 0.080 & 0.025 & 0.096 & 0.026 & 0.074 & 0.026 & 0.220 \\
\hline
 & $\pm$ 0.188 & $\pm$ 0.003 & $\pm$ 0.118 & $\pm$ 0.003 & $\pm$ 0.012 & $\pm$ 0.004 & $\pm$ 0.012 & $\pm$ 0.003 & $\pm$ 0.128 \\
\hline
\rule{0pt}{4ex}\textbf{\shortstack{\rule{0pt}{2.8ex}Realizability\\(ms)\rule[-1ex]{0pt}{1ex}}} & 0.231 & 0.081 & 2.149 & 0.081 & 3.465 & 0.088 & 2.782 & 0.082 & 6.316 \\
\hline
 & $\pm$ 0.252 & $\pm$ 0.011 & $\pm$ 0.758 & $\pm$ 0.012 & $\pm$ 1.591 & $\pm$ 0.084 & $\pm$ 1.301 & $\pm$ 0.015 & $\pm$ 2.230 \\
\hline
\rule{0pt}{4ex}\textbf{\shortstack{\rule{0pt}{2.8ex}Extraction\\(ms)\rule[-1ex]{0pt}{1ex}}} & 0.369 & 0.203 & 0.530 & 0.202 & 0.558 & 0.234 & 0.757 & 0.200 & 0.952 \\
\hline
 & $\pm$ 0.221 & $\pm$ 0.044 & $\pm$ 0.178 & $\pm$ 0.041 & $\pm$ 0.228 & $\pm$ 0.467 & $\pm$ 0.158 & $\pm$ 0.053 & $\pm$ 0.594 \\
\hline
\rule{0pt}{4ex}\textbf{\shortstack{\rule{0pt}{2.8ex}Auditing\\(ms)\rule[-1ex]{0pt}{1ex}}} & 0.446 & 0.444 & 0.451 & 0.336 & 0.479 & 0.482 & 0.500 & 0.426 & 0.420 \\
\hline
 & $\pm$ 0.682 & $\pm$ 0.887 & $\pm$ 0.792 & $\pm$ 0.118 & $\pm$ 0.840 & $\pm$ 0.996 & $\pm$ 0.796 & $\pm$ 0.855 & $\pm$ 0.276 \\
\hline
\rule{0pt}{4ex}\textbf{\shortstack{\rule{0pt}{2.8ex}Realize\\HFSM (ms)\rule[-1ex]{0pt}{1ex}}} & 126.492 & 119.206 & 128.289 & 120.229 & 128.522 & 119.985 & 129.209 & 119.671 & 128.766 \\
\hline
 & $\pm$ 15.333 & $\pm$ 7.081 & $\pm$ 10.274 & $\pm$ 8.214 & $\pm$ 6.426 & $\pm$ 7.705 & $\pm$ 7.239 & $\pm$ 7.128 & $\pm$ 6.209 \\
\hline
\rule{0pt}{4ex}\textbf{\shortstack{\rule{0pt}{2.8ex}Overall\\(ms)\rule[-1ex]{0pt}{1ex}}} & 199.273 & 187.268 & 208.771 & 186.761 & 210.066 & 188.341 & 218.585 & 188.579 & 221.405 \\
\hline
 & $\pm$ 23.069 & $\pm$ 8.482 & $\pm$ 11.877 & $\pm$ 8.207 & $\pm$ 7.618 & $\pm$ 8.925 & $\pm$ 9.269 & $\pm$ 7.976 & $\pm$ 7.997 \\
\hline
\end{tabular}

%% file: tables/coffee_enumerated.tex
\begin{tabular}{|c|r|r|r|r|r|r|r|r|}
\hline
 & \textbf{2S} & \textbf{3S} & \textbf{2F} & \textbf{3F} & \textbf{2SP} & \textbf{3SP} & \textbf{2FP} & \textbf{3FP} \\
\hline
\rule{0pt}{4ex}\textbf{\shortstack{\rule{0pt}{2.8ex}Well\\Separated\rule[-1ex]{0pt}{1ex}}} & Yes & Yes & Yes & Yes & Yes & Yes & Yes & No \\
\hline
\textbf{Realizable} & Yes & Yes & Yes & Yes & Yes & Yes & Yes & Yes \\
\hline
\rule{0pt}{4ex}\textbf{\shortstack{\rule{0pt}{2.8ex}FlexBE HFSM\\Realized\rule[-1ex]{0pt}{1ex}}} & Yes & Yes & Yes & Yes & Yes & Yes & Yes & Yes \\
\hline
\textbf{$\left|AP_I\right|$} & 2 & 2 & 2 & 2 & 2 & 2 & 2 & 2 \\
\hline
\textbf{$\left|AP_O\right|$} & 2 & 4 & 2 & 4 & 2 & 7 & 2 & 7 \\
\hline
\textbf{$\left|\varphi_i^a\right|$} & 2 & 2 & 2 & 2 & 2 & 2 & 2 & 2 \\
\hline
\textbf{$\left|\varphi_i^g\right|$} & 1 & 3 & 1 & 3 & 1 & 6 & 1 & 6 \\
\hline
\textbf{$\left|\varphi_s^a\right|$} & 4 & 7 & 4 & 7 & 4 & 7 & 4 & 7 \\
\hline
\textbf{$\left|\varphi_s^g\right|$} & 3 & 16 & 3 & 16 & 3 & 28 & 3 & 28 \\
\hline
\textbf{$\left|\varphi_l^a\right|$} & 1 & 1 & 1 & 3 & 1 & 1 & 1 & 3 \\
\hline
\textbf{$\left|\varphi_l^g\right|$} & 1 & 1 & 1 & 1 & 1 & 1 & 1 & 1 \\
\hline
\rule{0pt}{4ex}\textbf{\shortstack{\rule{0pt}{2.8ex}CUDD Live\\BDD Nodes\rule[-1ex]{0pt}{1ex}}} & 51.1 & 182.1 & 51.1 & 217.1 & 51.1 & 301.3 & 51.1 & 523.7 \\
\hline
 & $\pm$ 1.0 & $\pm$ 1.0 & $\pm$ 1.0 & $\pm$ 1.0 & $\pm$ 1.0 & $\pm$ 20.0 & $\pm$ 1.0 & $\pm$ 47.3 \\
\hline
\rule{0pt}{4ex}\textbf{\shortstack{\rule{0pt}{2.8ex}CUDD Peak\\BDD Nodes\rule[-1ex]{0pt}{1ex}}} & 1022 & 1022 & 1022 & 1022 & 1022 & 1022 & 1022 & 2044 \\
\hline
\textbf{\textbar{}Slugs SM\textbar} & 6 & 8 & 6 & 8 & 6 & 8 & 6 & 8 \\
\hline
\textbf{\textbar{}Reduced\textbar} & 4 & 4 & 4 & 4 & 4 & 4 & 4 & 4 \\
\hline
\rule{0pt}{4ex}\textbf{\shortstack{\rule{0pt}{2.8ex}Generate\\Specs (ms)\rule[-1ex]{0pt}{1ex}}} & 6.413 & 16.093 & 6.654 & 15.767 & 6.846 & 21.816 & 6.453 & 20.856 \\
\hline
 & $\pm$ 2.035 & $\pm$ 3.301 & $\pm$ 2.453 & $\pm$ 3.664 & $\pm$ 2.676 & $\pm$ 4.067 & $\pm$ 1.767 & $\pm$ 3.979 \\
\hline
\rule{0pt}{4ex}\textbf{\shortstack{\rule{0pt}{2.8ex}Well Separation\\(ms)\rule[-1ex]{0pt}{1ex}}} & 0.041 & 0.042 & 0.041 & 0.048 & 0.041 & 0.045 & 0.077 & 0.086 \\
\hline
 & $\pm$ 0.007 & $\pm$ 0.004 & $\pm$ 0.004 & $\pm$ 0.005 & $\pm$ 0.004 & $\pm$ 0.037 & $\pm$ 0.493 & $\pm$ 0.011 \\
\hline
\rule{0pt}{4ex}\textbf{\shortstack{\rule{0pt}{2.8ex}Realizability\\(ms)\rule[-1ex]{0pt}{1ex}}} & 0.050 & 0.059 & 0.048 & 0.127 & 0.049 & 0.625 & 0.049 & 3.127 \\
\hline
 & $\pm$ 0.029 & $\pm$ 0.004 & $\pm$ 0.004 & $\pm$ 0.068 & $\pm$ 0.010 & $\pm$ 0.703 & $\pm$ 0.008 & $\pm$ 1.015 \\
\hline
\rule{0pt}{4ex}\textbf{\shortstack{\rule{0pt}{2.8ex}Extraction\\(ms)\rule[-1ex]{0pt}{1ex}}} & 0.197 & 0.303 & 0.203 & 0.329 & 0.188 & 0.388 & 0.205 & 0.540 \\
\hline
 & $\pm$ 0.061 & $\pm$ 0.086 & $\pm$ 0.149 & $\pm$ 0.127 & $\pm$ 0.035 & $\pm$ 0.071 & $\pm$ 0.088 & $\pm$ 0.207 \\
\hline
\rule{0pt}{4ex}\textbf{\shortstack{\rule{0pt}{2.8ex}Auditing\\(ms)\rule[-1ex]{0pt}{1ex}}} & 0.437 & 0.541 & 0.495 & 0.436 & 0.391 & 0.478 & 0.456 & 0.450 \\
\hline
 & $\pm$ 0.464 & $\pm$ 0.869 & $\pm$ 1.036 & $\pm$ 0.262 & $\pm$ 0.223 & $\pm$ 0.363 & $\pm$ 1.002 & $\pm$ 0.384 \\
\hline
\rule{0pt}{4ex}\textbf{\shortstack{\rule{0pt}{2.8ex}Realize\\HFSM (ms)\rule[-1ex]{0pt}{1ex}}} & 118.138 & 126.725 & 117.459 & 127.876 & 118.096 & 127.853 & 117.283 & 128.133 \\
\hline
 & $\pm$ 7.580 & $\pm$ 5.731 & $\pm$ 6.042 & $\pm$ 7.035 & $\pm$ 6.359 & $\pm$ 6.357 & $\pm$ 7.725 & $\pm$ 6.654 \\
\hline
\rule{0pt}{4ex}\textbf{\shortstack{\rule{0pt}{2.8ex}Overall\\(ms)\rule[-1ex]{0pt}{1ex}}} & 188.318 & 211.843 & 188.635 & 212.031 & 189.694 & 220.097 & 188.992 & 218.873 \\
\hline
 & $\pm$ 7.810 & $\pm$ 7.989 & $\pm$ 6.781 & $\pm$ 9.077 & $\pm$ 7.940 & $\pm$ 8.789 & $\pm$ 8.193 & $\pm$ 8.544 \\
\hline
\end{tabular}

%% file: TwoRiversAppendix.tex
\section{Two-Rivers Demonstration}
\label{sec:two-rivers-results-app}

\resec{two-rivers-results} summarizes the Two-Rivers results by
realizability, well-separation, strategy size, and the three headline
timing rows. This appendix gives the full specification-size and
BDD-node breakdown behind that summary, the exact hand-written baseline
comparison, the representative minimized well-separation cores, and the
per-configuration well-separation timing.

As in \resec{two-rivers-results}, \retab{rivers-results-app} excludes both
hand-written baselines to save space. These are \texttt{full\_spec} (item
1's per-item one-hot movement encoding) and \texttt{full\_spec\_parsed}
(hand-written using item 2/3's shared
\texttt{move\_state\_a}/\texttt{move\_item} structure). Both are realizable, well-separated, and
produce a valid FlexBE HFSM across all 40 random declaration orderings
(200/200 trials each), with symbolic cost in the same order of magnitude as
the capability-generated columns, not simply assumed to match them.
\textbar{}Reduced\textbar{} = 22 for both hand-written variants, against 23
for \textbf{2S}/\textbf{3S}; mean realizability 0.086--0.102\,s and mean
peak BDD 17{,}195--23{,}685 nodes, against 0.054--0.091\,s and
13{,}133--19{,}597 nodes for \textbf{2S}/\textbf{3S}.

\begin{table*}[t]
\centering
\caption{\textbf{[$N=40$ random-seed orderings.]} Two-Rivers synthesis results, full detail underlying
\retab{rivers-results-summary}. Reported values are means over 40 random
declaration orderings under a CUDD reordering threshold of 250 ($N=200$
trials per column: 40 orderings $\times$ 5 repeats). Integer-valued count
means indicate identical samples, while non-identical count means are shown
with one decimal digit, with a companion row giving $\pm$ one standard
deviation wherever the underlying trials disagree. Columns marked
unrealizable have no extracted strategy.}
\label{tbl:rivers-results-app}
\small
\renewcommand{\arraystretch}{1.0}
\setlength{\tabcolsep}{3pt}
\input{tables/two_rivers}
\end{table*}

Across the realizable Two-Rivers variants, enumerated encoding reduces
proposition counts and yields smaller raw extracted strategies
(\textbar{}Slugs SM\textbar{} drops from 37 to 26, while
\textbar{}Reduced\textbar{} ties at 23),
but the symbolic synthesis cost is not uniformly reduced.
Under System-Goal liveness, enumerated encoding increases the random-ordering
mean peak BDD size relative to one-hot encoding (19596.9 vs.\ 13132.7 without
pending, and 15917.7 vs.\ 12723.9 with pending). This reversal is
specific to the random-seed mean and does not hold under fixed declaration
order, as \seedetail{two-rivers-ordering-app} shows. Mean realizability time
is likewise higher for enumerated encoding in all three realizable pairs
(53.8 vs.\ 90.7\,ms, 59.5 vs.\ 67.9\,ms, and 76.6 vs.\ 82.9\,ms), with each
gap inside one standard deviation of the more variable column. The Fair-Outcome formulation
without pending memory is unrealizable for both encodings; adding pending memory
restores realizability, but within each encoding it produces the largest peak
BDDs among the reported realizable capability-based variants
(16224.3 for one-hot and 21308.7 for enumerated). These results
highlight the sensitivity of GR(1) synthesis
performance to specification structure rather than merely problem size.

The well-separation results separate the two effects of the Fair-Outcome
formulation. Every System-Goal specification, with or without pending memory,
is well-separated, while every Fair-Outcome specification is not
well-separated. Thus pending memory repairs realizability for the
\textbf{2FP}/\textbf{3FP} columns but does not repair the underlying
assumption-structure issue. In the generated specifications, the
Fair-Outcome rule requires the environment to eventually report completion of
the move capability once that capability is active, or once its pending
proposition is set. This ties the environment's liveness obligation to a
system-controlled activation or pending context, producing the same
\texttt{P-all/E-just} well-separation violation identified in the Coffee
extended Fair-Outcome-with-pending case. Representative minimized-core checks
for the Two-Rivers Fair-Outcome columns always include the move-completion
environment liveness assumption; some encodings additionally include
capability outcome or variable-bound safety constraints. As with Coffee, this
diagnoses a specification-level assumption-fairness hazard, not simply a
failure of the synthesizer. The \textbf{2F}/\textbf{3F} specifications are
unrealizable, but the \textbf{2FP}/\textbf{3FP} specifications are realizable
and yield auditor-valid controllers despite remaining NWS.

With \textbf{Well Separation} measured as internal check time in \texttt{slugs}
(\retab{timing-rows}), values range 6.22--9.88\,ms across all
eight columns, an order of magnitude above Coffee's sub-millisecond figures
but consistent with Two-Rivers' larger proposition and constraint counts
throughout. Unlike Coffee's extended Fair-Outcome-pending case, well-separation
time here does not cleanly track well-separation status. The highest value
(\textbf{3F}, 9.88\,ms) is non-well-separated, but the lowest (\textbf{3SP},
6.22\,ms) is well-separated, and both classes appear on either side of the
range. The check remains cheap relative to the rest of the pipeline overall,
though less uniformly so than in Coffee. It stays under 1\% of the
corresponding \textbf{Overall} time for five of the eight columns
(\textbf{2S}, \textbf{2SP}, \textbf{3SP}, \textbf{2FP}, \textbf{3FP}), rises
to 1.02\% for \textbf{3S}, and reaches 5.15\%/6.83\% for the unrealizable
\textbf{2F}/\textbf{3F} columns, where \textbf{Overall} itself is much
smaller since no extraction or controller realization occurs.

\subsection{Variable-Ordering Sensitivity}
\label{sec:two-rivers-ordering-app}

The System-Goal, no-pending realizability comparison above uses the same
40-random-seed mean reported throughout this section
(\textbf{2S}/\textbf{3S} in \retab{rivers-results-app}), where enumerated
encoding appears slower and larger than one-hot. Isolating the fixed
alphabetic and domain-grouped declaration orderings from that same sweep
(both under the CUDD reordering threshold of 250 used throughout) shows
the opposite result:

\begin{table*}[t]
\centering
\caption{Two-Rivers System-Goal, no-pending realizability and peak BDD
nodes under fixed declaration order, versus the 40-random-seed mean
already reported for \textbf{2S}/\textbf{3S} in \retab{rivers-results-app}.
Fixed orderings are single deterministic trials.}
\label{tbl:two-rivers-ordering-fixed-app}
\begin{tabular}{llrr}
\toprule
Ordering & Encoding & Realizability (ms) & Peak BDD Nodes \\
\midrule
Alphabetic & One-Hot & 68.1 & 13{,}286 \\
Alphabetic & Enumerated & 51.8 & 10{,}220 \\
Domain-grouped & One-Hot & 70.3 & 13{,}286 \\
Domain-grouped & Enumerated & 57.5 & 10{,}220 \\
Random (mean, $N{=}200$) & One-Hot & 53.8 & 13{,}132.7 \\
Random (mean, $N{=}200$) & Enumerated & 90.7 & 19{,}596.9 \\
\bottomrule
\end{tabular}
\end{table*}

Under both fixed orderings, enumerated encoding is the faster
($\sim$18--24\%) and smaller ($\sim$23\%) choice; the apparent reversal in
\retab{rivers-results-app} is therefore specific to the random-seed mean,
not a general property of the two encodings on this domain.

The distribution behind that mean explains why.
\retab{two-rivers-ordering-dist-app} summarizes the 200 per-trial
realizability times (40 orderings $\times$ 5 repeats) underlying
\textbf{2S}/\textbf{3S}:

\begin{table*}[t]
\centering
\caption{Realizability-time distribution across the 200 random-seed
trials underlying \textbf{2S}/\textbf{3S} in \retab{rivers-results-app}
(System-Goal liveness, no pending memory).}
\label{tbl:two-rivers-ordering-dist-app}
\begin{tabular}{lrrrrr}
\toprule
Encoding & Mean (ms) & Median (ms) & Std.\ Dev.\ (ms) & Min (ms) & Max (ms) \\
\midrule
One-Hot & 53.8 & 53.9 & 9.6 & 29.3 & 86.3 \\
Enumerated & 90.7 & 77.8 & 44.8 & 41.4 & 229.5 \\
\bottomrule
\end{tabular}
\end{table*}

One-hot's mean and median nearly coincide, with a standard deviation
less than a quarter of enumerated's. Realizability time under one-hot is
largely insensitive to declaration order on this domain. Enumerated's
mean exceeds its median by more, and its standard deviation is roughly
$4.7\times$ one-hot's, indicating a right-skewed distribution with a substantial
population of slow orderings, not merely a handful of extreme outliers.
\retab{two-rivers-ordering-trim-app} confirms this directly, by
progressively discarding enumerated's slowest trials and recomputing the
mean over what remains:

\begin{table}[t]
\centering
\caption{Enumerated mean realizability time (System-Goal, no pending)
after discarding the slowest random-seed trials, out of $N{=}200$.}
\label{tbl:two-rivers-ordering-trim-app}
\begin{tabular}{lrr}
\toprule
Trials Discarded & Trials Remaining & Enumerated Mean (ms) \\
\midrule
0 & 200 & 90.7 \\
5 (2.5\%) & 195 & 87.1 \\
10 (5\%) & 190 & 84.6 \\
20 (10\%) & 180 & 79.9 \\
40 (20\%) & 160 & 72.6 \\
\bottomrule
\end{tabular}
\end{table}

Even after discarding the slowest 20\% of enumerated trials, its mean
(72.6\,ms) remains roughly 35\% above one-hot's unfiltered mean
(53.8\,ms). In addition, 85 of the 200 enumerated trials (42.5\%) exceed
one-hot's single \emph{worst} observed trial (86.3\,ms). This is not a
handful of pathological seeds, but over a third of the enumerated
distribution sitting above where one-hot tops out entirely.

The reversal reported in \retab{rivers-results-app} is therefore a genuine
ordering-sensitivity asymmetry, not an artifact of averaging a few slow
outliers. Enumerated's compact encoding is only an advantage under a
favorable variable ordering. Both fixed orderings tested happen to be
favorable here, while one-hot's performance stays roughly flat regardless
of ordering. This is a different kind of finding from the
ordering-\emph{magnitude} sensitivity Coffee exhibits
(\seedetail{coffee-ordering-app}), where reordering affects absolute BDD
size but not which encoding wins. Here, random ordering can flip
\emph{which} encoding is faster, a robustness property distinct from, and
orthogonal to, the mean speed or size advantage a single fixed ordering
reports.

Why enumerated is the more ordering-sensitive of the two encodings is not
settled by the data above, but the pipeline's own random-ordering
implementation offers a plausible mechanism. Unlike one-hot's many
per-capability propositions, which the random-seed sweep shuffles
independently of one another and of every other declared proposition, the
enumerated encoding's shared \texttt{capability} variable is not shuffled
the same way. The pipeline always writes the \texttt{capability}
declaration (e.g., \texttt{capability:0...3},
\relist{pending-trans-enumerated-example}) first in the \texttt{[OUTPUT]}
block, regardless of the random seed. The random ordering therefore affects
the other output declarations (including the remaining bounded-integer ones,
such as \texttt{move\_dest} and \texttt{move\_item}) and the input
declarations separately, while leaving this high-fanout output declaration at
a fixed position. Because
\texttt{capability} is referenced by nearly every transition rule in the
specification, unlike any single one-hot activation or outcome flag, the
relative placement of the remaining declarations may couple more strongly to
BDD size, whereas one-hot's many smaller, more locally-coupled propositions
average out more evenly across a full reshuffle. We have not confirmed
this mechanism against the underlying BDD variable-index assignment
directly, so we report it as a plausible explanation for the asymmetry
above rather than a demonstrated cause.

\FloatBarrier

The mandatory failure-logging mechanism explaining no-pending Fair-Outcome
unrealizability is given in \resec{two-rivers-results}; it applies to both
capability encodings and is distinct from the justice-related NWS diagnostic.

%% file: tables/two_rivers.tex
\begin{tabular}{|c|r|r|r|r|r|r|r|r|}
\hline
 & \textbf{2S} & \textbf{3S} & \textbf{2F} & \textbf{3F} & \textbf{2SP} & \textbf{3SP} & \textbf{2FP} & \textbf{3FP} \\
\hline
\rule{0pt}{4ex}\textbf{\shortstack{\rule{0pt}{2.8ex}Well\\Separated\rule[-1ex]{0pt}{1ex}}} & Yes & Yes & No & No & Yes & Yes & No & No \\
\hline
\textbf{Realizable} & Yes & Yes & No & No & Yes & Yes & Yes & Yes \\
\hline
\rule{0pt}{4ex}\textbf{\shortstack{\rule{0pt}{2.8ex}FlexBE HFSM\\Realized\rule[-1ex]{0pt}{1ex}}} & Yes & Yes & No & No & Yes & Yes & Yes & Yes \\
\hline
\textbf{$\left|AP_I\right|$} & 13 & 10 & 13 & 10 & 13 & 10 & 13 & 10 \\
\hline
\textbf{$\left|AP_O\right|$} & 10 & 8 & 10 & 8 & 11 & 9 & 11 & 9 \\
\hline
\textbf{$\left|\varphi_i^a\right|$} & 9 & 6 & 9 & 6 & 9 & 6 & 9 & 6 \\
\hline
\textbf{$\left|\varphi_i^g\right|$} & 6 & 3 & 6 & 3 & 7 & 4 & 7 & 4 \\
\hline
\textbf{$\left|\varphi_s^a\right|$} & 33 & 29 & 33 & 29 & 33 & 29 & 33 & 29 \\
\hline
\textbf{$\left|\varphi_s^g\right|$} & 30 & 23 & 30 & 23 & 34 & 27 & 34 & 27 \\
\hline
\textbf{$\left|\varphi_l^a\right|$} & 1 & 1 & 1 & 1 & 1 & 1 & 1 & 1 \\
\hline
\textbf{$\left|\varphi_l^g\right|$} & 1 & 1 & 1 & 1 & 1 & 1 & 1 & 1 \\
\hline
\rule{0pt}{4ex}\textbf{\shortstack{\rule{0pt}{2.8ex}CUDD Live\\BDD Nodes\rule[-1ex]{0pt}{1ex}}} & 3044.1 & 5852.9 & 1312.9 & 2254.7 & 3457.0 & 4049.9 & 4453.5 & 4032.8 \\
\hline
 & $\pm$ 793.4 & $\pm$ 3748.4 & $\pm$ 77.0 & $\pm$ 762.1 & $\pm$ 975.6 & $\pm$ 1774.0 & $\pm$ 890.1 & $\pm$ 1251.6 \\
\hline
\rule{0pt}{4ex}\textbf{\shortstack{\rule{0pt}{2.8ex}CUDD Peak\\BDD Nodes\rule[-1ex]{0pt}{1ex}}} & 13132.7 & 19596.9 & 3091.6 & 5723.2 & 12723.9 & 15917.7 & 16224.3 & 21308.7 \\
\hline
 & $\pm$ 2068.9 & $\pm$ 7200.8 & $\pm$ 160.0 & $\pm$ 1829.9 & $\pm$ 1817.7 & $\pm$ 4672.1 & $\pm$ 2349.7 & $\pm$ 6322.3 \\
\hline
\textbf{\textbar{}Slugs SM\textbar} & 37 & 26 & -- & -- & 37 & 26 & 37 & 26 \\
\hline
\textbf{\textbar{}Reduced\textbar} & 23 & 23 & -- & -- & 23 & 23 & 23 & 23 \\
\hline
\rule{0pt}{4ex}\textbf{\shortstack{\rule{0pt}{2.8ex}Generate\\Specs (ms)\rule[-1ex]{0pt}{1ex}}} & 18.036 & 25.897 & 19.334 & 24.967 & 19.292 & 27.964 & 19.287 & 28.854 \\
\hline
 & $\pm$ 3.486 & $\pm$ 3.286 & $\pm$ 5.864 & $\pm$ 3.314 & $\pm$ 3.819 & $\pm$ 3.677 & $\pm$ 3.999 & $\pm$ 6.121 \\
\hline
\rule{0pt}{4ex}\textbf{\shortstack{\rule{0pt}{2.8ex}Well Separation\\(ms)\rule[-1ex]{0pt}{1ex}}} & 6.512 & 9.164 & 6.688 & 9.876 & 7.862 & 6.218 & 7.723 & 7.365 \\
\hline
 & $\pm$ 4.734 & $\pm$ 4.685 & $\pm$ 4.519 & $\pm$ 4.999 & $\pm$ 4.826 & $\pm$ 5.999 & $\pm$ 4.309 & $\pm$ 6.211 \\
\hline
\rule{0pt}{4ex}\textbf{\shortstack{\rule{0pt}{2.8ex}Realizability\\(ms)\rule[-1ex]{0pt}{1ex}}} & 53.787 & 90.677 & 8.803 & 16.953 & 59.478 & 67.922 & 76.616 & 82.875 \\
\hline
 & $\pm$ 9.562 & $\pm$ 44.829 & $\pm$ 2.997 & $\pm$ 8.673 & $\pm$ 10.509 & $\pm$ 23.562 & $\pm$ 20.325 & $\pm$ 28.650 \\
\hline
\rule{0pt}{4ex}\textbf{\shortstack{\rule{0pt}{2.8ex}Extraction\\(ms)\rule[-1ex]{0pt}{1ex}}} & 9.349 & 8.087 & -- & -- & 9.809 & 6.307 & 10.452 & 7.014 \\
\hline
 & $\pm$ 2.287 & $\pm$ 3.472 &  &  & $\pm$ 1.968 & $\pm$ 1.550 & $\pm$ 2.029 & $\pm$ 2.046 \\
\hline
\rule{0pt}{4ex}\textbf{\shortstack{\rule{0pt}{2.8ex}Auditing\\(ms)\rule[-1ex]{0pt}{1ex}}} & 0.992 & 1.755 & -- & -- & 0.955 & 1.205 & 1.007 & 1.400 \\
\hline
 & $\pm$ 1.145 & $\pm$ 2.451 &  &  & $\pm$ 1.024 & $\pm$ 0.926 & $\pm$ 1.660 & $\pm$ 1.901 \\
\hline
\rule{0pt}{4ex}\textbf{\shortstack{\rule{0pt}{2.8ex}Realize\\HFSM (ms)\rule[-1ex]{0pt}{1ex}}} & 671.875 & 665.201 & -- & -- & 676.240 & 666.855 & 674.775 & 672.307 \\
\hline
 & $\pm$ 13.727 & $\pm$ 8.893 &  &  & $\pm$ 10.789 & $\pm$ 21.182 & $\pm$ 10.395 & $\pm$ 24.959 \\
\hline
\rule{0pt}{4ex}\textbf{\shortstack{\rule{0pt}{2.8ex}Overall\\(ms)\rule[-1ex]{0pt}{1ex}}} & 858.225 & 901.953 & 129.821 & 144.565 & 875.333 & 883.916 & 890.593 & 906.965 \\
\hline
 & $\pm$ 19.771 & $\pm$ 56.821 & $\pm$ 12.336 & $\pm$ 13.946 & $\pm$ 17.532 & $\pm$ 41.879 & $\pm$ 23.688 & $\pm$ 53.419 \\
\hline
\end{tabular}

%% file: PyRoboSimAppendix.tex
\section{PyRoboSim Demonstration}
\label{sec:pyrobosim-results-app}

\resec{pyrobosim-results} summarizes the PyRoboSim results by
realizability, well-separation, strategy size, and the headline timing
rows (\resec{pyrobosim-p0-results} for P0, \resec{pyrobosim-p1-results}
for P1). This appendix gives the full specification-size and BDD-node
breakdown behind the P0 summary tables, the exact hand-written baseline and
encoding comparisons, the per-configuration well-separation timing, and the
full well-separation minimized-core breakdown.

\subsection{P0: Baseline (No Door Modeling)}
\label{sec:pyrobosim-p0-results-app}

\retab{pyrobosim-p0-results-app} gives the full P0 matrix underlying
\retab{pyrobosim-p0-1hot-summary} and \retab{pyrobosim-p0-enumerated-summary}
(\resec{pyrobosim-p0-results}); reported values are means over the 20
random-seed declaration orderings, matching the Coffee and Two-Rivers
reporting convention (\resec{coffee-results}, \resec{two-rivers-results}).
The alphabetic and domain-grouped fixed orderings are not part of this
mean.

\begin{table*}[t]
\centering
\small
\caption{\textbf{[$N=20$ random-seed orderings.]} PyRoboSim P0 synthesis results, full detail underlying
\retab{pyrobosim-p0-1hot-summary} and \retab{pyrobosim-p0-enumerated-summary}.
\textbf{MoveOnly}/\textbf{Full} are the hand-written baselines; the leading
digit \textbf{2}/\textbf{3} denotes one-hot/enumerated capability encoding,
\textbf{S}/\textbf{F} denotes System-Goal/Fair-Outcome liveness, and
\textbf{P} adds pending propositions. Reported values are means over 20
random-seed declaration orderings under a CUDD reordering threshold of 250,
with a companion row giving $\pm$ one standard deviation wherever the
underlying trials disagree; the alphabetic and domain-grouped fixed
orderings are not part of this mean. Columns marked unrealizable have no
extracted strategy; \textbf{2FP}/\textbf{3FP} are realizable but show
\textbf{FlexBE HFSM Realized: No} because the post-synthesis auditor
(\resec{auditor}) halts the pipeline before reduction, discussed below.}
\label{tbl:pyrobosim-p0-results-app}
\renewcommand{\arraystretch}{1.4}
\setlength{\tabcolsep}{3pt}
\resizebox{\textwidth}{!}{%
\input{tables/pyrobosim}
}
\end{table*}

Realizability, however, is not the end of the story. The strategy auditor of
\resec{auditor} flags a \texttt{GOAL\_UNREACHABLE\_TRAP} on \emph{every one}
of the 44 \textbf{2FP}/\textbf{3FP} trials, covering both encodings and all 22 orderings.
All audits completed, and none timed out. We confirmed the specific failure
kind directly in the pipeline's own per-trial logs rather than inferring it
from the downstream symptom alone.

Manual inspection of an early P0 Fair-Outcome-with-pending strategy
confirmed the same pathology and motivated the automated
\texttt{GOAL\_UNREACHABLE\_TRAP} check.
\refig{pyrobosim-trap} (\resec{pyrobosim-p0-results}) shows a concrete
instance of one such trap, taken directly from the raw synthesized strategy
for the one-hot, Fair-Outcome-with-pending, alphabetic-ordering trial
(\textbf{2FP}). Direct Tarjan analysis of the raw state graph finds twelve reachable
goal-unreachable-trap SCCs. The one shown in that figure is a closed
three-state cycle among strategy states 72, 73, and 83, entered from state
63.

Because the pipeline halts before state-merging reduction once the
raw-strategy audit reports a confirmed failure (\resec{reduction}),
\retab{pyrobosim-p0-results-app} itself shows \textbf{FlexBE HFSM
Realized: No} for both \textbf{2FP} and \textbf{3FP}. No executable
controller is ever produced from either configuration. \textbf{Reduced} and
the post-reduction \textbf{Auditing} pass both show ``--'' accordingly,
since neither stage ever runs on a strategy the first audit already
rejected.

With real seed coverage, cost metrics do not show the sharp warning sign
their eventual defect would seem to warrant. For one-hot, \textbf{2FP} is
\emph{smaller} than the realizable \textbf{2S} configuration on peak BDD nodes
(4{,}949{,}035.0 vs.\ 5{,}102{,}846.0) despite taking longer to synthesize
(620.1~s vs.\ 449.2~s); for enumerated, \textbf{3FP} is larger and slower
than \textbf{3S} on both axes (3{,}462{,}995.9 vs.\ 3{,}278{,}524.9 peak nodes;
300.8~s vs.\ 149.4~s). The raw extracted strategy is already substantially
larger before any reduction is attempted (\textbar{}Slugs SM\textbar{} = 108
vs.\ 47 for one-hot, 84 vs.\ 38 for enumerated), and it is precisely this
configuration in which the auditor detects a reachable goal-unreachable region --
correctly halting the pipeline before reduction ever runs, which is why
\textbf{Reduced} has no value to report for \textbf{2FP}/\textbf{3FP} at all rather than a
larger one. Nothing in proposition count, BDD size, or a roughly
1.4--2.0$\times$ synthesis-time difference specifically signals that \textbf{2FP}/\textbf{3FP}
contain a reachable goal-unreachable region, which the auditor detects
directly; realizability and symbolic cost metrics remain insufficient
measures of a synthesized supervisor's usefulness even once ordering
variance is properly accounted for.

The System-Goal columns give a clean PyRoboSim encoding comparison:
enumerated encoding is smaller and faster than one-hot before reduction,
both under the fixed alphabetic and domain orderings (peak nodes
3{,}629{,}122 vs.\ 5{,}844{,}818; realizability 335--336~s vs.\
622--651~s) and averaged over the 20 random seeds reported in
\retab{pyrobosim-p0-results-app} (3{,}278{,}524.9 vs.\
5{,}102{,}846.0 peak nodes; 149.4~s vs.\ 449.2~s realizability).
The random-seed mean does carry substantial spread
(realizability standard deviation 72.0~s for enumerated, 257.0~s for
one-hot), so any individual random draw remains a poor proxy for the
typical case, fixed or random. This matches the broader lesson of the Coffee
ordering-robustness study (\seedetail{coffee-ordering-app}): declaration
order drives substantial variation on its own, so encoding comparisons
drawn from a single ordering should be treated cautiously regardless of
domain. After state reduction, the strategies converge to the
same size regardless of encoding (\textbar{}Reduced\textbar{} = 19 for both
\textbf{2S} and \textbf{3S}), as in Two-Rivers. Coffee's reduced sizes
differ by up to one state across encodings and the quadcopter's differ
substantially, so exact convergence is not a general pattern.
For the realized rows in P0, the only reduced-strategy versus realized-HFSM
size difference is the expected root-container bookkeeping in the
\textbf{MoveOnly} and \textbf{Full} baselines; the capability-generated
rows have matching reduced and realized sizes.

Comparing the capability-generated specifications to the hand-written
baselines, the full hand-written specification (\textbf{Full}) realizes in
402--403~s under the fixed alphabetic and domain orderings. This is faster
than the capability-generated \textbf{2S} configuration under the same
fixed orderings (622--651~s), despite \textbf{Full}
having fewer output propositions (16 vs.\ 17) and fewer safety clauses in
several categories. But \textbf{Full}'s mean over the 20 random seeds is 570.2~s,
the slowest of the three System-Goal configurations in
\retab{pyrobosim-p0-results-app} (449.2~s for \textbf{2S}, 149.4~s for
\textbf{3S}), driven by a realizability standard deviation of 399.5~s on
that mean, i.e.\ a small
number of especially slow random draws. The hand-written baseline is
therefore not immune to the ordering sensitivity that affects the generated
encodings; on this evidence it is more exposed to it here, not less. The
move-only baseline, which tracks only robot location and the finished flag,
is trivially cheap by comparison (0.016~s mean realizability, 3{,}832.5
peak BDD nodes over the 20 random seeds) and serves mainly as a sanity check
that most of the symbolic cost in this domain comes from item-manipulation
logic rather than navigation.

The well-separation picture for PyRoboSim differs qualitatively from Coffee
and Two-Rivers (\resec{coffee-results}, \resec{two-rivers-results}). Only the
trivial move-only baseline (\textbf{MoveOnly}) is well-separated. Every
other configuration in \retab{pyrobosim-p0-results-app}, including the
realizable System-Goal configurations
\textbf{2S}/\textbf{3S}/\textbf{2SP}/\textbf{3SP} and the hand-written
\textbf{Full} baseline, is NWS. The unrealizable
Fair-Outcome-without-pending configurations
(\textbf{2F}/\textbf{3F}) are NWS as well. The analyzer runs on the
compiled specification before synthesis, so the check applies regardless
of whether the specification later turns out to be realizable. In Coffee, only Fair-Outcome with pending memory failed the check;
in Two-Rivers, every Fair-Outcome specification was NWS. Here, the violation is already present in configurations with no
Fair-Outcome liveness at all.

The reported case type explains why this broader NWS result should be read
as a specification diagnostic rather than an auditor-failure prediction.
Every NWS PyRoboSim configuration includes a \texttt{P-reach/E-safe} case,
where the system can force an environment safety assumption. The Fair-Outcome configurations
(\textbf{2F}/\textbf{3F}/\textbf{2FP}/\textbf{3FP}) additionally show the
\texttt{P-all/E-just} case already identified in Coffee and Two-Rivers,
appearing here regardless of whether pending memory is present.

A focused core-minimization check on the saved alphabetic specifications
for \textbf{2S}, \textbf{Full}, \textbf{2FP}, and \textbf{3FP} returns
the same single safety assumption: snacks cannot be carried together with
another item. Its formula and the diagnostic reachability that exposes
its violation are given in \resec{pyrobosim-p0-results}.
The assumption has a different declaration index in \textbf{Full}, but
its formula is identical. The minimizer returns a 1-minimal subset, not a
unique or globally minimum core. In particular, selecting this safety
core for \textbf{2FP}/\textbf{3FP} does not remove their separately
reported \texttt{P-all/E-just} cases. These focused checks do not establish
that the selected core is identical across all declaration orderings.

That limited scope is why the final comparison with the auditor matters.
As with Coffee and Two-Rivers, well-separation failure here does not by
itself predict auditor failure. Five of the seven NWS-yet-realizable
configurations (\textbf{Full}, \textbf{2S}, \textbf{3S}, \textbf{2SP},
\textbf{3SP}) pass the post-synthesis auditor cleanly, with no protocol
violations or deadlocks (\resec{auditor}). The exceptions are
\textbf{2FP}/\textbf{3FP}, which additionally fail the auditor with the
goal-unreachable trap discussed above. A specification can be
NWS and still produce a deployable controller. On the other hand, all
instances of the \texttt{GOAL\_UNREACHABLE\_TRAP} found by the auditor
arose only from NWS specifications.
For the P0 matrix, the well-separation check itself remains
inexpensive relative to synthesis, completing in 1.1--2.2~s versus tens to
hundreds of seconds for realizability. Across all three domains,
well-separation and the post-synthesis auditor catch different,
complementary classes of specification risk.

\subsubsection{Targeted Carriage-Restatement Check}
\label{sec:p0-carriage-check}

To isolate the safety finding, we reran well-separation and full explicit
strategy extraction for the seven realizable, originally NWS P0
configurations. Each received one original and one modified run under
alphabetic declaration ordering and CUDD threshold 250. The modified
specification primes every item-location variable in all six carriage
exclusions in \texttt{ENV\_TRANS}, leaving every other section unchanged.
For the example exclusion $\psi_{\mathrm{carry}}$ in
\resec{pyrobosim-p0-results}, this replaces $G\,\psi_{\mathrm{carry}}$
with $G\,X\,\psi_{\mathrm{carry}}$.
The checked \texttt{ENV\_INIT} values are
\texttt{bread0=4}, \texttt{butter0=5}, \texttt{snacks0=4},
\texttt{soda0=5}, \texttt{waste0=7}, and \texttt{waste1=6}.
Thus no item is initially carried. Under these initial conditions, both
forms enforce the same exclusion on infinite traces, although the
diagnostic's transition admissibility differs.
This observation is not a general synthesis-game equivalence claim.

\begin{table*}[t]
\centering
\small
\caption{Targeted P0 carriage-restatement check. Each row comprises one
original and one next-state run. S denotes \texttt{P-reach/E-safe}, J denotes
\texttt{P-all/E-just}, and WS denotes well-separated. State counts and raw
audit verdicts are identical within each pair. Trap denotes
\texttt{GOAL\_UNREACHABLE\_TRAP}; reduction is skipped after that failure.
All fourteen specifications are realizable.}
\label{tbl:p0-carriage-pairs}
\begin{tabular}{lccrrl}
\toprule
Configuration & Original & Next-state & Raw states & Reduced states & Raw audit \\
\midrule
2S   & S   & WS & 47  & 19 & Pass \\
2SP  & S   & WS & 47  & 19 & Pass \\
3S   & S   & WS & 38  & 19 & Pass \\
3SP  & S   & WS & 38  & 19 & Pass \\
2FP  & S,J & J  & 108 & --- & Trap \\
3FP  & S,J & J  & 84  & --- & Trap \\
Full & S   & WS & 116 & 66 & Pass \\
\bottomrule
\end{tabular}
\end{table*}

Compilation used the production compiler's thorough mode. Original compiled
clauses matched the archived experiment inputs, and the freshly extracted
original strategies reproduced the archived parsed strategy JSON. Within
every pair, parsed raw strategy JSON was identical, including ranks.
All raw audits completed with strict protocol checks. For the five passing
pairs, reduction produced identical representations apart from rank/incoming
bookkeeping and the reduced structural audits passed; strict raw-protocol
checks are disabled for reduced automata as explained in
\reappsec{pipeline-integration}. The two Fair-Outcome pairs retained their
justice findings and raw-audit traps (\retab{p0-carriage-pairs}).

These comparisons establish equality of the selected strategies under the
tested ordering, not ordering-independent behavior. MoveOnly, unrealizable
no-pending 2F/3F, P1, and other orderings were not rerun. The original sweep
tables remain unchanged. Reproduction scripts, compiled specifications,
solver logs, and audit reports are in the
\href{https://raw.githubusercontent.com/CNURobotics/flexbe_synthesis_demo/1264e0fd/experiments/p0-supplemental-evidence.zip}{supplemental data (\texttt{1264e0fd})}.

\subsubsection{Targeted Retry-Justice Check}
\label{sec:p0-retry-check}

Separate trials on extended Coffee one-hot \textbf{3FP} and P0 one-hot
\textbf{2FP} added $GF(\neg\pi_a^p\lor\pi_a)$ for each pending capability
to \texttt{SYS\_LIVENESS}, leaving \texttt{ENV\_LIVENESS} unchanged.
The extracted state valuations and transitions were unchanged in both
cases. P0 retained its twelve reachable trap SCCs, including the illustrated
closed component on states 72, 73, and 83. Under strict GR(1) semantics,
violating environment justice releases the added system justice obligations.
The result therefore concerns this tested justice addition, not all possible
retry formulations. Safety-level retry rules with explicit abort or
abandonment semantics remain a separate design alternative.

\subsection{P1: Adding Door Modeling}
\label{sec:pyrobosim-p1-results-app}

P1 extends the P0 domain with an \texttt{open} capability and initializes both
containers closed, so retrieving an item first requires opening its
container (\resec{pyrobosim}). We evaluate P1 to test whether the encoding
conclusions of \reappsec{pyrobosim-p0-results-app}
hold as domain complexity grows by adding a capability and its
associated propositions to an already-validated domain.

Per-trial cost at this scale makes the 20-random-seed sweep of
\reappsec{pyrobosim-p0-results-app} intractable for P1. \retab{pyrobosim-p1-results-app}
reports a single trial per configuration under a fixed alphabetic
declaration ordering, not a random-seed mean. These values are therefore
not comparable to \retab{pyrobosim-p0-results-app}'s means in the way the P0
configurations are comparable to each other, and any comparison drawn
between the two tables should be read with that in mind. Under the
10-hour per-trial synthesis budget, \textbf{2S} (one-hot, System-Goal, no
pending) is the only one-hot P1 configuration with a completed run,
finishing in 25842.0~s. The remaining one-hot configurations
(\textbf{2F}, \textbf{2SP}, \textbf{2FP}) have no completed run at this
budget, so their columns are omitted.

We consider the hand-written \textbf{FullP1} baseline, the one-hot
capability encoding under System-Goal liveness (\textbf{2S}), and the
enumerated capability encoding under System-Goal (\textbf{4S}),
Fair-Outcome (\textbf{4F}), and both liveness formulations with pending
propositions (\textbf{4SP}, \textbf{4FP}). The column digit \textbf{4}
distinguishes the enumerated columns from P0's (\textbf{3S}, etc.) in
\retab{pyrobosim-p0-results-app} despite using the same encoding, since the two
are evaluated over different domains and are not meant to be read as the
same configuration. The one-hot column, by contrast, keeps P0's digit
\textbf{2}: it is the same one-hot/System-Goal/no-pending treatment
applied to the larger P1 domain, with the P1 table and section context
distinguishing it from the P0 \textbf{2S} row.

\begin{table*}[t]
\centering
\small
\caption{\textbf{[Exploratory: budget-limited, single fixed ordering
($n=1$); see \resec{pyrobosim-p1-results}.]} PyRoboSim P1 (door-modeling) synthesis results, full detail
underlying \retab{pyrobosim-p1-results}. \textbf{FullP1} is
the hand-written baseline. \textbf{2S} is the one-hot capability encoding
(System-Goal liveness only; \textbf{2F}/\textbf{2SP}/\textbf{2FP} were
not completed at this budget). The leading digit \textbf{4} denotes the
enumerated capability encoding evaluated over the P1 domain (distinct from
P0's \textbf{3}-family columns in \retab{pyrobosim-p0-results-app} despite sharing
an encoding). \textbf{S}/\textbf{F} denotes System-Goal/Fair-Outcome
liveness, and \textbf{P} adds pending propositions. Unlike
\retab{pyrobosim-p0-results-app}, reported values are a single trial under a
fixed alphabetic declaration ordering ($n=1$), not a random-seed mean;
therefore no $\pm$ standard-deviation rows are shown.
These single-ordering rows use P1's uncapped (``auto'') CUDD reordering
policy; \retab{pyrobosim-p1-ordering-app} reports the corresponding
threshold-250 ordering experiment for \textbf{4S}/\textbf{4F}.
\textbf{FullP1}'s well-separation check did not complete under a 43200~s (12\,h) cap, so
its \textbf{Well Separated} row reports \textbf{TIME OUT} rather than a
completed verdict.}
\label{tbl:pyrobosim-p1-results-app}
\renewcommand{\arraystretch}{1.4}
\setlength{\tabcolsep}{3pt}
\input{tables/pyrobosim_p1}
\end{table*}

Fair-Outcome liveness without pending memory is again unrealizable
(\textbf{4F}), replicating the P0 and Two-Rivers result
(\reappsec{pyrobosim-p0-results-app}, \resec{two-rivers-results}) on a domain that
now also models door state. Under matched threshold-250 random orderings
(\seedetail{pyrobosim-p1-ordering-app}), \textbf{4F} sits modestly above the
corresponding P0 column (\textbf{3F}): its mean realizability time
(36.7\,s) is 1.3 \textbf{3F} standard deviations above the \textbf{3F}
mean (25.1\,s), and its mean peak BDD nodes (813{,}818.6) are 1.0 standard
deviation above (766{,}397.8). This Fair-Outcome/no-pending result therefore
shows little sensitivity to the added door-opening capability.

\textbf{4FP} checks whether pending propositions restore realizability, as
they do for P0 (\textbf{2FP}/\textbf{3FP}) and Two-Rivers, and whether the
same \texttt{GOAL\_UNREACHABLE\_TRAP} pathology persists after adding door
state. Both hold. \textbf{4FP} is realizable (12919.4~s, within the
36{,}000~s budget), and the raw-strategy audit reports the same failure kind
as \textbf{2FP}/\textbf{3FP}. The pipeline therefore halts before
state-merging reduction (\resec{reduction}). \textbf{4FP}'s raw strategy
has 118 states before reduction (versus 108 for \textbf{2FP} and 84 for
\textbf{3FP}), and \retab{pyrobosim-p1-results-app} reports
\textbf{FlexBE HFSM Realized: No} and \textbf{Reduced: --}. This confirms
that the auditor catches the same structural defect after the domain gains
a new capability, output propositions, and an environment-liveness
assumption.

\textbf{FullP1} exposes a separate scaling issue for well separation.
Synthesis completes in 3871.9~s, but the well-separation check still has
no verdict under a dedicated 43200~s cap, versus 1.1~s for the
corresponding P0 \textbf{Full} mean
(\retab{pyrobosim-p0-results-app}). Thus P1 can be synthesis-tractable
while this diagnostic remains unresolved; the table therefore reports
\textbf{TIME OUT} for \textbf{FullP1}'s
\textbf{Well Separated} row. The same non-verdict was reproduced on the
faster machine. This asymmetry is consistent with well separation being a
distinct winning-region computation over a transformed game, so its
symbolic cost need not track ordinary synthesis
cost~\cite{MaozRingert2016,KindControllers24}.

\textbf{2S} lets us extend the encoding comparison of
\reappsec{pyrobosim-p0-results-app} to P1. Under the tested P1 ordering, the
gap between one-hot and enumerated is much larger. In the System-Goal,
no-pending comparison, P0's one-hot \textbf{2S} column has
$1.6\times$ the peak BDD nodes and $3.0\times$ the realizability time of
enumerated \textbf{3S} (\retab{pyrobosim-p0-results-app}). After door
modeling is added, P1's one-hot \textbf{2S} column rises to $8.2\times$
the peak BDD nodes and $109.9\times$ the realizability time of enumerated
\textbf{4S}. Both P1 columns use the uncapped (``auto'') reordering policy,
so this ratio compares like with like, but it is not comparable to the P0
ratios above, which use threshold 250. As in P0, this cost gap closes
entirely after reduction:
\textbf{2S} and \textbf{4S} both reduce to 27 states
(\retab{pyrobosim-p1-results-app}), matching P0's
encoding-independent convergence (19 states for both \textbf{2S} and
\textbf{3S}) on a domain with roughly 50\% more input propositions.
These two capability-generated P1 rows also realize as 27-state FlexBE
HFSMs. The separate \textbf{FullP1} baseline is different: its
344-state reduced strategy consists of 336 executable action states plus
eight success-terminal contexts, which Controller Realization projects to
336 action states, one shared \texttt{finished} outcome, and the root
container (\resec{controller-realization}). Thus its realized HFSM has 338
state instantiations, with no executable action state omitted.
Proposition count remains a poor predictor of reduced strategy size or
realized FlexBE state count.
Because this P1 comparison uses fixed-ordering trials, and because
\textbf{4S} itself shows large declaration-order sensitivity below, the
multipliers should be read as directional evidence under the tested
ordering rather than as an isolated encoding effect.

Growth from P0 to P1 must be read against the reordering policy. P0 and the
P1 ordering sweep use threshold 250, while the P1 main table uses the
uncapped policy, and at this scale the two differ sharply. For \textbf{4S}
the alphabetic trial takes 235.2\,s and 5.6M peak nodes under the uncapped
policy but 3679.6\,s and 25.7M under threshold 250, whereas \textbf{4F}
moves from 31.3\,s to 24.0\,s (\seedetail{pyrobosim-p1-ordering-app}).
Dividing P1 uncapped values by P0 threshold-250 values therefore mixes
policies, and we do not report such ratios as effect sizes. Where both use
threshold 250, \textbf{4S} costs $7.1\times$ the peak BDD nodes and
$11.0\times$ the realizability time of \textbf{3S} under the alphabetic
ordering (25.7M vs.\ 3.6M peak nodes; 3679.6\,s vs.\ 335\,s), and $8.0\times$
and $61\times$ on the random-ordering means. Door
modeling thus raises System-Goal cost substantially even without pending
memory. The pending columns (\textbf{4SP}, \textbf{4FP}) and \textbf{FullP1}
have only uncapped-policy P1 trials. Against P0 they show $25.1\times$ and
$43.0\times$ the realizability time of \textbf{3SP} and \textbf{3FP}, and
roughly $7\times$ for \textbf{FullP1} against \textbf{Full}, which suggests a
larger pending penalty, but these comparisons mix policies and are
directional only.
The cost of door modeling is therefore not confined to pending-memory
configurations, since System-Goal without pending memory also grows
several-fold under a matched policy, but the pending columns appear to grow
more, most where pending memory meets Fair-Outcome liveness. The P1 pending
columns are single trials under a different policy from P0, so these
multipliers are directional rather than calibrated effect sizes.

\subsection{Declaration-Ordering Sensitivity (4S/4F)}
\label{sec:pyrobosim-p1-ordering-app}

\textbf{4SP}, \textbf{4FP}, \textbf{FullP1}, and \textbf{2S} remain
single-trial throughout this appendix, but \textbf{4S}/\textbf{4F} also have
a 22-ordering sweep (2 fixed + 20 random, under
the same 250 CUDD reordering threshold used throughout), the same shape as
PyRoboSim P0's own ordering sweep (\reappsec{pyrobosim-p0-results-app}) and
Two-Rivers' (\seedetail{two-rivers-ordering-app}). This sweep uses the
threshold-250 policy applied everywhere else in this paper, unlike
\retab{pyrobosim-p1-results-app}'s single alphabetic trial for
\textbf{4S}/\textbf{4F}, which uses P1's own uncapped (``auto'') reordering policy.
The two tables' alphabetic-ordering figures for the same column are
therefore not expected to agree. Because both policies were run on the same
alphabetic ordering, \textbf{4S} and \textbf{4F} give a direct policy
comparison, which the table reports as an additional pair of rows.

\begin{table*}[t]
\centering
\caption{PyRoboSim P1's \textbf{4S}/\textbf{4F} realizability and peak BDD
nodes under fixed vs.\ random declaration order, extending
\retab{pyrobosim-p1-results-app}'s single alphabetic trial with a
22-ordering sweep (2 fixed + 20 random, threshold-250 reordering). The
``Alphabetic, auto'' rows repeat the single uncapped-policy trial of
\retab{pyrobosim-p1-results-app} for a direct policy comparison.}
\label{tbl:pyrobosim-p1-ordering-app}
\begin{tabular}{llrr}
\toprule
Ordering & Column & Realizability (s) & Peak BDD Nodes \\
\midrule
Alphabetic, auto & \textbf{4S} & 235.2 & 5{,}643{,}484 \\
Alphabetic & \textbf{4S} & 3679.6 & 25{,}725{,}784 \\
Domain-grouped & \textbf{4S} & 3617.5 & 25{,}725{,}784 \\
Random (mean, $N{=}20$) & \textbf{4S} & 9164.0 & 26{,}155{,}841.6 \\
\midrule
Alphabetic, auto & \textbf{4F} & 31.3 & 713{,}356 \\
Alphabetic & \textbf{4F} & 24.0 & 751{,}170 \\
Domain-grouped & \textbf{4F} & 23.9 & 751{,}170 \\
Random (mean, $N{=}20$) & \textbf{4F} & 36.7 & 813{,}818.6 \\
\bottomrule
\end{tabular}
\end{table*}

This is the direct comparison of the uncapped policy and threshold 250
reported here, and it shows that the choice can matter at scale. On the same
alphabetic ordering, \textbf{4S} is $15.6\times$ slower (235.2\,s to
3679.6\,s) and reaches $4.6\times$ the peak BDD nodes under threshold 250,
whereas \textbf{4F} is $1.3\times$ faster (31.3\,s to 24.0\,s) with about 5\%
more nodes. The effect therefore depends on the configuration, and the
Coffee-scale observation that the uncapped policy is effectively inactive
(\seedetail{coffee-ordering-app}) does not extend to P1. The one-hot
\textbf{2S} column has no threshold-250 run, so the encoding comparison of the
P1 main table cannot be repeated under threshold 250.

\textbf{4S}'s random-seed mean (9164.0\,s) carries an enormous spread
(standard deviation 7629.1\,s, essentially the same order as the mean
itself). Four orderings complete in under 600\,s (269.0--545.1\,s), while
the remaining sixteen span nearly two orders of magnitude on their own,
from 3038.6\,s up to the slowest observed trial of 26{,}685.0\,s, a
$99\times$ range across just 22 orderings, the widest ordering-sensitivity
spread found anywhere in this paper. \textbf{4F}, by contrast, is
uniformly fast and tightly clustered (20.2--76.2\,s across all 22
orderings, standard deviation 14.3\,s), matching the stability already
seen for the other Fair-Outcome/no-pending columns throughout this paper.
In both cases, the qualitative outcome never varies across the sweep.
\textbf{4S} is realizable, audit-valid, and reduces to 27 states in all 22
trials (matching \retab{pyrobosim-p1-results-app}'s single-trial figure
exactly). \textbf{4F} is unrealizable in all 22, with the identical
\texttt{SPEC\_UNSYNTHESIZABLE} failure in every trial. The sweep therefore
confirms \retab{pyrobosim-p1-results-app}'s realizability verdicts exactly
while revealing that \textbf{4S}'s absolute cost is far more
ordering-dependent than its single alphabetic trial could show. This is
consistent in kind with the one-hot ordering-sensitivity finding of
\reappsec{crazyflie-reliability-app}, but here affecting a System-Goal
configuration under enumerated encoding rather than
Fair-Outcome-with-pending under one-hot.

%% file: tables/pyrobosim.tex
\begin{tabular}{|c|r|r|r|r|r|r|r|r|r|r|}
\hline
 & \textbf{MoveOnly} & \textbf{Full} & \textbf{2S} & \textbf{3S} & \textbf{2F} & \textbf{3F} & \textbf{2SP} & \textbf{3SP} & \textbf{2FP} & \textbf{3FP} \\
\hline
\rule{0pt}{4ex}\textbf{\shortstack{\rule{0pt}{2.8ex}Well\\Separated\rule[-1ex]{0pt}{1ex}}} & Yes & No & No & No & No & No & No & No & No & No \\
\hline
\textbf{Realizable} & Yes & Yes & Yes & Yes & No & No & Yes & Yes & Yes & Yes \\
\hline
\rule{0pt}{4ex}\textbf{\shortstack{\rule{0pt}{2.8ex}FlexBE HFSM\\Realized\rule[-1ex]{0pt}{1ex}}} & Yes & Yes & Yes & Yes & No & No & Yes & Yes & No & No \\
\hline
\textbf{$\left|AP_I\right|$} & 8 & 33 & 34 & 25 & 34 & 25 & 34 & 25 & 34 & 25 \\
\hline
\textbf{$\left|AP_O\right|$} & 8 & 16 & 17 & 12 & 17 & 12 & 20 & 15 & 20 & 15 \\
\hline
\textbf{$\left|\varphi_i^a\right|$} & 4 & 17 & 18 & 9 & 18 & 9 & 18 & 9 & 18 & 9 \\
\hline
\textbf{$\left|\varphi_i^g\right|$} & 3 & 6 & 10 & 3 & 10 & 3 & 13 & 6 & 13 & 6 \\
\hline
\textbf{$\left|\varphi_s^a\right|$} & 8 & 42 & 50 & 40 & 50 & 40 & 50 & 40 & 50 & 40 \\
\hline
\textbf{$\left|\varphi_s^g\right|$} & 12 & 34 & 57 & 28 & 57 & 28 & 69 & 40 & 69 & 40 \\
\hline
\textbf{$\left|\varphi_l^a\right|$} & 1 & 1 & 1 & 1 & 3 & 3 & 1 & 1 & 3 & 3 \\
\hline
\textbf{$\left|\varphi_l^g\right|$} & 1 & 1 & 1 & 1 & 1 & 1 & 1 & 1 & 1 & 1 \\
\hline
\rule{0pt}{4ex}\textbf{\shortstack{\rule{0pt}{2.8ex}CUDD Live\\BDD Nodes\rule[-1ex]{0pt}{1ex}}} & 899.0 & 1542036.8 & 1045982.4 & 799104.1 & 334161.4 & 276488.2 & 1277737.5 & 862481.3 & 1360755.6 & 1042385.3 \\
\hline
 & $\pm$ 313.0 & $\pm$ 759323.5 & $\pm$ 285227.3 & $\pm$ 386265.1 & $\pm$ 62476.5 & $\pm$ 32205.1 & $\pm$ 476895.3 & $\pm$ 286343.6 & $\pm$ 308782.2 & $\pm$ 215178.6 \\
\hline
\rule{0pt}{4ex}\textbf{\shortstack{\rule{0pt}{2.8ex}CUDD Peak\\BDD Nodes\rule[-1ex]{0pt}{1ex}}} & 3832.5 & 4303079.9 & 5102846.0 & 3278524.9 & 1182147.4 & 766397.8 & 5562030.6 & 3497130.7 & 4949035.0 & 3462995.9 \\
\hline
 & $\pm$ 1362.1 & $\pm$ 1535395.1 & $\pm$ 1522551.7 & $\pm$ 941354.6 & $\pm$ 218189.4 & $\pm$ 46494.5 & $\pm$ 1324553.2 & $\pm$ 462037.6 & $\pm$ 1131606.1 & $\pm$ 900419.4 \\
\hline
\textbf{\textbar{}Slugs SM\textbar} & 5 & 116 & 47 & 38 & -- & -- & 47 & 38 & 108 & 84 \\
\hline
\textbf{\textbar{}Reduced\textbar} & 3 & 66 & 19 & 19 & -- & -- & 19 & 19 & -- & -- \\
\hline
\rule{0pt}{4ex}\textbf{\shortstack{\rule{0pt}{2.8ex}Generate\\Specs (s)\rule[-1ex]{0pt}{1ex}}} & 0.012 & 0.036 & 0.037 & 0.050 & 0.035 & 0.050 & 0.040 & 0.055 & 0.035 & 0.057 \\
\hline
 & $\pm$ 0.004 & $\pm$ 0.004 & $\pm$ 0.005 & $\pm$ 0.004 & $\pm$ 0.003 & $\pm$ 0.005 & $\pm$ 0.005 & $\pm$ 0.004 & $\pm$ 0.002 & $\pm$ 0.006 \\
\hline
\rule{0pt}{4ex}\textbf{\shortstack{\rule{0pt}{2.8ex}Well Separation\\(s)\rule[-1ex]{0pt}{1ex}}} & 0.000 & 1.121 & 1.541 & 2.201 & 1.539 & 2.143 & 1.275 & 1.829 & 1.276 & 1.833 \\
\hline
 & $\pm$ 0.000 & $\pm$ 0.202 & $\pm$ 0.657 & $\pm$ 1.074 & $\pm$ 0.661 & $\pm$ 1.060 & $\pm$ 0.626 & $\pm$ 0.650 & $\pm$ 0.630 & $\pm$ 0.648 \\
\hline
\rule{0pt}{4ex}\textbf{\shortstack{\rule{0pt}{2.8ex}Realizability\\(s)\rule[-1ex]{0pt}{1ex}}} & 0.016 & 570.184 & 449.174 & 149.430 & 58.466 & 25.051 & 570.332 & 164.243 & 620.068 & 300.756 \\
\hline
 & $\pm$ 0.005 & $\pm$ 399.504 & $\pm$ 257.022 & $\pm$ 72.034 & $\pm$ 13.832 & $\pm$ 8.835 & $\pm$ 209.194 & $\pm$ 33.784 & $\pm$ 346.331 & $\pm$ 266.725 \\
\hline
\rule{0pt}{4ex}\textbf{\shortstack{\rule{0pt}{2.8ex}Extraction\\(s)\rule[-1ex]{0pt}{1ex}}} & 0.001 & 4.369 & 2.709 & 1.554 & -- & -- & 3.624 & 1.711 & 4.343 & 3.324 \\
\hline
 & $\pm$ 0.000 & $\pm$ 1.860 & $\pm$ 1.033 & $\pm$ 0.784 &  &  & $\pm$ 2.429 & $\pm$ 0.548 & $\pm$ 1.074 & $\pm$ 1.588 \\
\hline
\rule{0pt}{4ex}\textbf{\shortstack{\rule{0pt}{2.8ex}Auditing\\(s)\rule[-1ex]{0pt}{1ex}}} & 0.001 & 0.002 & 0.002 & 0.003 & -- & -- & 0.002 & 0.002 & -- & -- \\
\hline
 & $\pm$ 0.001 & $\pm$ 0.000 & $\pm$ 0.002 & $\pm$ 0.003 &  &  & $\pm$ 0.000 & $\pm$ 0.002 &  &  \\
\hline
\rule{0pt}{4ex}\textbf{\shortstack{\rule{0pt}{2.8ex}Realize\\HFSM (s)\rule[-1ex]{0pt}{1ex}}} & 0.183 & 3.193 & 0.474 & 0.466 & -- & -- & 0.468 & 0.470 & -- & -- \\
\hline
 & $\pm$ 0.008 & $\pm$ 0.435 & $\pm$ 0.017 & $\pm$ 0.014 &  &  & $\pm$ 0.013 & $\pm$ 0.014 &  &  \\
\hline
\rule{0pt}{4ex}\textbf{\shortstack{\rule{0pt}{2.8ex}Overall\\(s)\rule[-1ex]{0pt}{1ex}}} & 0.283 & 580.581 & 455.548 & 154.511 & 61.122 & 27.657 & 577.649 & 169.211 & 627.717 & 307.086 \\
\hline
 & $\pm$ 0.014 & $\pm$ 401.087 & $\pm$ 258.170 & $\pm$ 72.554 & $\pm$ 14.031 & $\pm$ 9.238 & $\pm$ 210.151 & $\pm$ 34.026 & $\pm$ 346.342 & $\pm$ 268.133 \\
\hline
\end{tabular}

%% file: tables/pyrobosim_p1.tex
\begin{tabular}{|c|r|r|r|r|r|r|}
\hline
 & \textbf{FullP1} & \textbf{2S} & \textbf{4S} & \textbf{4F} & \textbf{4SP} & \textbf{4FP} \\
\hline
\rule{0pt}{4ex}\textbf{\shortstack{\rule{0pt}{2.8ex}Well\\Separated\rule[-1ex]{0pt}{1ex}}} & TIME OUT & No & No & No & No & No \\
\hline
\textbf{Realizable} & Yes & Yes & Yes & No & Yes & Yes \\
\hline
\rule{0pt}{4ex}\textbf{\shortstack{\rule{0pt}{2.8ex}FlexBE HFSM\\Realized\rule[-1ex]{0pt}{1ex}}} & Yes & Yes & Yes & No & Yes & No \\
\hline
\textbf{$\left|AP_I\right|$} & 48 & 52 & 40 & 40 & 40 & 40 \\
\hline
\textbf{$\left|AP_O\right|$} & 14 & 19 & 13 & 13 & 17 & 17 \\
\hline
\textbf{$\left|\varphi_i^a\right|$} & 32 & 36 & 24 & 24 & 24 & 24 \\
\hline
\textbf{$\left|\varphi_i^g\right|$} & 2 & 12 & 3 & 3 & 7 & 7 \\
\hline
\textbf{$\left|\varphi_s^a\right|$} & 94 & 86 & 73 & 73 & 73 & 73 \\
\hline
\textbf{$\left|\varphi_s^g\right|$} & 42 & 94 & 50 & 50 & 66 & 66 \\
\hline
\textbf{$\left|\varphi_l^a\right|$} & 1 & 1 & 1 & 4 & 1 & 4 \\
\hline
\textbf{$\left|\varphi_l^g\right|$} & 1 & 1 & 1 & 1 & 1 & 1 \\
\hline
\rule{0pt}{4ex}\textbf{\shortstack{\rule{0pt}{2.8ex}CUDD Live\\BDD Nodes\rule[-1ex]{0pt}{1ex}}} & 8986921 & 14483032 & 1819700 & 419861 & 11052977 & 15491652 \\
\hline
\rule{0pt}{4ex}\textbf{\shortstack{\rule{0pt}{2.8ex}CUDD Peak\\BDD Nodes\rule[-1ex]{0pt}{1ex}}} & 30113230 & 46278204 & 5643484 & 713356 & 23878008 & 34117426 \\
\hline
\textbf{\textbar{}Slugs SM\textbar} & 559 & 59 & 46 & -- & 46 & 118 \\
\hline
\textbf{\textbar{}Reduced\textbar} & 344 & 27 & 27 & -- & 27 & -- \\
\hline
\rule{0pt}{4ex}\textbf{\shortstack{\rule{0pt}{2.8ex}Generate\\Specs (s)\rule[-1ex]{0pt}{1ex}}} & 0.062 & 0.058 & 0.108 & 0.110 & 0.108 & 0.101 \\
\hline
\rule{0pt}{4ex}\textbf{\shortstack{\rule{0pt}{2.8ex}Well Separation\\(s)\rule[-1ex]{0pt}{1ex}}} & -- & 3019.282 & 122.507 & 124.618 & 153.111 & 149.566 \\
\hline
\rule{0pt}{4ex}\textbf{\shortstack{\rule{0pt}{2.8ex}Realizability\\(s)\rule[-1ex]{0pt}{1ex}}} & 3871.901 & 25841.977 & 235.193 & 31.282 & 4129.791 & 12919.358 \\
\hline
\rule{0pt}{4ex}\textbf{\shortstack{\rule{0pt}{2.8ex}Extraction\\(s)\rule[-1ex]{0pt}{1ex}}} & 22.228 & 39.578 & 4.495 & -- & 24.485 & 41.270 \\
\hline
\rule{0pt}{4ex}\textbf{\shortstack{\rule{0pt}{2.8ex}Auditing\\(s)\rule[-1ex]{0pt}{1ex}}} & 0.004 & 0.001 & 0.002 & -- & 0.002 & -- \\
\hline
\rule{0pt}{4ex}\textbf{\shortstack{\rule{0pt}{2.8ex}Realize\\HFSM (s)\rule[-1ex]{0pt}{1ex}}} & 8.912 & 0.574 & 0.592 & -- & 0.590 & -- \\
\hline
\rule{0pt}{4ex}\textbf{\shortstack{\rule{0pt}{2.8ex}Overall\\(s)\rule[-1ex]{0pt}{1ex}}} & 47111.691 & 28917.231 & 365.211 & 157.286 & 4315.238 & 13123.356 \\
\hline
\end{tabular}

%% file: CrazyFlieAppendix.tex
\section{Quadcopter Demonstration Results}
\label{sec:crazyflie-results-app}

\resec{crazyflie-results} summarizes the results for the
quadcopter flight supervisor by
realizability, well-separation, strategy size, and the headline timing
rows. This appendix gives the full specification-size and BDD-node
breakdown behind that summary, the exact budget-sensitivity and
well-separation timing figures, and the current status of the
Fair-Outcome-with-pending (\textbf{1FP}/\textbf{2FP}) configurations.

\begin{table*}[htbp]
\centering
\caption{\textbf{[$N=20$ random-seed orderings for the six
configurations shown; \textbf{1FP}/\textbf{2FP} excluded from this
20-seed summary; additional Fair-Outcome-with-pending results appear in
\reappsec{crazyflie-reliability-app}.]}
Quadcopter supervisory-synthesis results, full detail underlying
\retab{crazyflie-results}. The leading digit \textbf{1}/\textbf{2} denotes
the one-hot/enumerated capability encoding
(PyRoboSim uses \textbf{2}/\textbf{3}; \resec{discussion}). Reported values are means
over 20 random-seed declaration orderings under a CUDD reordering threshold
of 250 and a 36{,}000\,s per-trial synthesis budget, matching the Coffee,
Two-Rivers, and PyRoboSim reporting convention, with a companion row giving
$\pm$ one standard deviation wherever the underlying trials disagree;
alphabetic and domain-grouped fixed orderings are excluded from the mean.
All six configurations complete all 20 trials at this budget, including
\textbf{1SP} (\reappsec{crazyflie-reliability-app}), and \textbf{1SP} is
well-separated across all 22 declaration orderings just like the other
five configurations, once its well-separation check is given a longer
timeout (\reappsec{crazyflie-well-separation-app}). Integer-valued count
means indicate identical samples, while non-identical count means are
shown with one decimal digit.}
\label{tbl:crazyflie-results-app}
\small
\renewcommand{\arraystretch}{0.9}
\setlength{\tabcolsep}{3pt}
\input{tables/crazyflie}
\end{table*}

As in the other domains, enumerated variants (2*) reduce input
AP count ($84 \rightarrow 2$). The counts follow from the generated
specification. The quadcopter model has 51 capability states, including the
initial \texttt{begin\_game} action, which is an executable domain
initialization capability rather than a synthetic null slot. One-hot encoding declares 51 activation
outputs, and it observes 51 completion inputs and 33 failure inputs (only the
33 failure-prone capabilities report failure), giving $51+33=84$ input
propositions. Enumerated encoding replaces these with the two shared inputs
\texttt{completed} and \texttt{failure}. On the output side, one-hot
$|AP_O|=91$ is 51 activation outputs, 19 memory propositions, the two
terminal flags \texttt{failed} and \texttt{finished}, and 19 bits for the
five bounded-integer parameters (\texttt{gate1} to \texttt{gate4} at 4 bits
each and \texttt{route\_segment} at 3 bits). Enumerated encoding replaces the
51 activation outputs with a 6-bit \texttt{capability} value, giving 46. The
pending variants add one pending proposition per failure-prone capability
(33), giving 124 and 79. The System-Goal input files have empty \texttt{ENV\_LIVENESS} sections,
consistent with $\varphi_l^a=\True$ in \resec{liveness}.
Before synthesis, \texttt{slugs} inserts an internal \texttt{True} liveness
assumption when that list is empty; the reported count
$|\varphi_l^a|=1$ includes this normalization~\cite{CNU_slugs}.
Fair-Outcome instead supplies 33 nontrivial clauses, one per failure-prone
capability, so its reported count is 33, not 34. Apart from these liveness
clauses, the environment specifications match between the corresponding
System-Goal and Fair-Outcome columns (\textbf{1S}/\textbf{1F} and
\textbf{2S}/\textbf{2F}).
Unlike the Coffee, Two-Rivers, and PyRoboSim results, enumerated encoding
does not reduce the extracted strategy here: \textbar{}Slugs SM\textbar{}
and \textbar{}Reduced\textbar{} are consistently \emph{larger} under
enumerated encoding across all three liveness variants ($119.3 \rightarrow
131$, $76.7 \rightarrow 106$ under System-Goal; $102 \rightarrow
131$, $68 \rightarrow 106$ under Fair-Outcome; $187.0 \rightarrow 347$,
$131.0 \rightarrow 318$ under pending), most dramatically so for the
pending case. Realizability time, by contrast, favors enumerated encoding
across all three reported comparable pairs here ($19.7 \rightarrow
9.3$\,s under System-Goal, $638.5 \rightarrow 31.8$\,s under
Fair-Outcome, and $2764.5 \rightarrow 166.6$\,s with pending under
System-Goal). With all 20 random orderings completing for both encodings
at the 36{,}000\,s budget (\reappsec{crazyflie-reliability-app}), the
pending comparison is a true 20-seed mean on both sides and reinforces
that enumerated encoding dominates on realizability time in this domain.
Within each encoding, Fair-Outcome liveness still dominates overall
difficulty relative to System-Goal (one-hot: $19.7 \rightarrow 638.5$\,s;
enumerated: $9.3 \rightarrow 31.8$\,s), confirming liveness pressure, not
encoding, as the main complexity driver in this domain.

\subsection{Reduced Strategy Size versus Realized HFSM Size}
\label{sec:crazyflie-realization-app}

The \textbar{}Reduced\textbar{} row in \retab{crazyflie-results-app}
counts states in the reduced Mealy strategy, not FlexBE state
instantiations. For the quadcopter this distinction is unusually visible because
the reduced strategies retain many terminal-output contexts. Those
contexts matter at the strategy level, where different valuations of
terminal outputs can remain distinct after reduction, but they project to
shared terminal outcomes in the realized FlexBE HFSM. Controller
Realization therefore preserves executable action states while mapping
bootstrap/null states into root-container wiring and mapping terminal
\texttt{finished}/\texttt{failed} contexts to shared FlexBE outcomes
(\resec{controller-realization}).

\begin{table*}[t]
\centering
\small
\caption{Quadcopter reduced-strategy states versus realized FlexBE HFSM
state instantiations for representative alphabetic, threshold-250 rows.
Terminal/bootstrap contexts are reduced-strategy states that project to
root-container wiring or shared FlexBE terminal outcomes rather than
separate executable capability states.}
\label{tbl:crazyflie-realization-accounting-app}
\renewcommand{\arraystretch}{1.15}
\resizebox{\textwidth}{!}{%
\begin{tabular}{lrrp{0.36\textwidth}r}
\toprule
\textbf{Configuration} & \textbf{Reduced} & \textbf{Executable actions} & \textbf{Terminal/bootstrap contexts} & \textbf{Realized HFSM} \\
\midrule
\textbf{1F}  & 68  & 51  & 15 \texttt{failed} + 1 \texttt{finished} + 1 \texttt{begin\_game} & 54 \\
\textbf{1S}  & 83  & 66  & 15 \texttt{failed} + 1 \texttt{finished} + 1 \texttt{begin\_game} & 69 \\
\textbf{1SP} & 141 & 105 & 31 \texttt{failed} + 4 \texttt{finished} + 1 \texttt{begin\_game} & 108 \\
\textbf{2F}  & 106 & 51  & 53 \texttt{failed} + 1 \texttt{finished} + 1 null/bootstrap & 54 \\
\textbf{2S}  & 106 & 51  & 53 \texttt{failed} + 1 \texttt{finished} + 1 null/bootstrap & 54 \\
\textbf{2SP} & 318 & 86  & 227 \texttt{failed} + 4 \texttt{finished} + 1 null/bootstrap & 89 \\
\bottomrule
\end{tabular}
}
\end{table*}

We verified the accounting in \retab{crazyflie-realization-accounting-app}
against the generated run artifacts. The original- and reduced-strategy
auditors pass for these rows; the generation logs contain no missing
FlexBE-outcome declarations, skipped generated states, or skipped
transitions; and every non-realized reduced state is one of the terminal or
bootstrap/null contexts listed in the table. In the enumerated rows, many
terminal \texttt{failed} contexts also carry a capability value. That value
is retained symbolic context at the terminal state, not an additional
action that should be instantiated after failure. Thus the sharpest gaps,
especially \textbf{2S}/\textbf{2F}'s 106 reduced states versus 54 realized
HFSM states and \textbf{2SP}'s 318 versus 89, do not indicate lost robot
behavior.

\retab{crazyflie-results-app} reports \textbf{Extraction} and
\textbf{Auditing} alongside the other domains (\resec{coffee-results},
\resec{two-rivers-results}, \resec{pyrobosim-results}; timing-row definitions in
\resec{discussion}, \retab{timing-rows}). Auditing remains
sub-20\,ms across every column (6--18\,ms) despite $|\mathrm{Slugs\ SM}|$
reaching 347 for \textbf{2SP}, comparable to the largest strategies audited
elsewhere in this paper (559 states for PyRoboSim's hand-written
\textbf{FullP1} baseline, \resec{pyrobosim-p1-results}; 677 states for this
domain's own enumerated \textbf{2FP} raw strategy,
\reappsec{crazyflie-reliability-app}). These are measured implementation times;
\reappsec{auditor-complexity} distinguishes them from the linear count of
graph operations for the evaluated unbounded configurations.

\subsection{Budget Sensitivity of One-Hot Synthesis}
\label{sec:crazyflie-reliability-app}

The 20-seed sweep surfaces a reliability difference between encodings that a
single random seed cannot reveal. Across per-trial budgets from 900\,s to
36{,}000\,s, enumerated encoding completed all 20 orderings for every
configuration in the six-column sweep with zero timeouts even at the
smallest budget, while one-hot
encoding's completion rate is strongly budget-dependent and differs sharply
between its two affected configurations. The Fair-Outcome configuration
(\textbf{1F}) reaches 100\% completion once the budget reaches
3600\,s; the pending configuration (\textbf{1SP}) needs the full 36{,}000\,s budget
(40$\times$ the smallest budget tested) to do the same, including one
ordering whose realizability time alone reached 12{,}558\,s (3.5\,h) and
whose total pipeline time reached 14{,}531\,s. The two configurations
differ in degree, not in kind: \textbf{1SP} is not a structurally distinct failure
mode from \textbf{1F}, but the same budget-sensitivity phenomenon with a
substantially heavier tail, requiring roughly an order of magnitude more
budget than \textbf{1F} to fully resolve. This is not a difference in whether a
solution exists: the alphabetic and domain-grouped fixed orderings for both
configurations complete well within budget at every setting tested,
comparable to the successful random trials. It is a difference in how much
wall-clock budget the space of declaration orders demands under one-hot
encoding specifically. The practical implication for a deployment pipeline
is that one-hot encoding here carries a substantial risk that synthesis,
even for a realizable specification, may require an order of magnitude more
wall-clock time to reach a verdict than enumerated encoding needs for the
equivalent configuration (2764.5\,s mean realizability for
\textbf{1SP} vs.\ 166.6\,s for \textbf{2SP}), unless the declaration order
is chosen carefully or a correspondingly larger budget is provisioned.

The theoretical hazard reviewed in \resec{synthesis} applies directly to
the Fair-Outcome-with-pending configuration (\textbf{1FP}/\textbf{2FP}), which appears in
neither \retab{crazyflie-results-app} nor the budget sweep above. Because
completion events are only permitted when the corresponding capability is
activated, a Fair-Outcome-with-pending strategy can in principle enter a
region where pending remains asserted while activation has ceased. That
region makes the environment's liveness obligation unsatisfiable and, under
GR(1) semantics, relieves the system of its own liveness guarantee. The
strategy would then win, in part, by steering the environment out of its own
assumptions, a known hazard in GR(1) synthesis.

Under alphabetic ordering and a 24\,h (86{,}400\,s) per-trial budget,
one-hot \textbf{1FP} reaches a realizability verdict in roughly ten hours
overall, while enumerated \textbf{2FP} does not resolve at any budget tested,
up to and including 24\,h, on the original machine.

\resec{pyrobosim-p0-results} reports where this pathology was first found,
by manual inspection of an early PyRoboSim strategy. That inspection
motivated the auditor's (\resec{auditor}) automated
\texttt{GOAL\_UNREACHABLE\_TRAP} check, which now confirms the same
pathology automatically across all 44 PyRoboSim trials. The quadcopter
domain exhibits it as well. For both encodings, once the budget is extended
far enough for realizability to resolve, the extracted
Fair-Outcome-with-pending strategy fails the same trap check.

\textbf{1FP} is realizable, in 33{,}937\,s (9.4\,h) of
\texttt{slugs} realizability time and 36{,}540\,s (10.2\,h) overall
pipeline time, comfortably inside budget. Its specification size is
$|AP_I|=84$, $|AP_O|=124$, with a peak BDD of 52.9M nodes -- well above any
one-hot configuration in \retab{crazyflie-results-app}, consistent with
the proposition-count blowup one-hot encoding already shows throughout
this paper. The raw-strategy audit flags the extracted strategy with a
\texttt{GOAL\_UNREACHABLE\_TRAP}, so \textbf{1FP} is not realized as a
FlexBE HFSM. It is realizable but not deployable, exactly the pattern
already established for PyRoboSim's \textbf{2FP}/\textbf{3FP}
(\resec{pyrobosim-p0-results}) and \textbf{4FP}
(\resec{pyrobosim-p1-results}).
This is the first quadcopter-scale confirmation that the theoretical
pending/activation-gated hazard produces the same concrete cyclic-retry
pathology already found in PyRoboSim's smaller simulated domains, for
\textbf{1FP}.

Enumerated \textbf{2FP} resolved on the faster
hardware described in \reappsec{crazyflie-failure-validation-app}. It is
realizable, in 59{,}674\,s (16.6\,h) of \texttt{slugs} realizability time and
59{,}769\,s (16.6\,h) overall pipeline time. Its specification size is
$|AP_I|=2$, $|AP_O|=79$, with a peak BDD of 52.9M nodes
(52{,}870{,}104, distinct from \textbf{1FP}'s 52{,}914{,}050) and 15.8M live
BDD nodes. The extracted raw strategy has 677 states and fails the raw
\texttt{GOAL\_UNREACHABLE\_TRAP} audit, the identical failure kind found
for \textbf{1FP}, so no reduced or realized FlexBE HFSM is reported.
These runs used different hardware and were not a matched, budget-normalized
comparison. The faster-machine result settles realizability and the audit
verdict but does not isolate machine speed as the cause of the timing difference.
The pathology itself is
confirmed for both encodings of the Fair-Outcome-with-pending
configuration.
\reappsec{crazyflie-well-separation-app} below reports the
well-separation status for both configurations, which is settled for both.

\subsection{Well-Separation}
\label{sec:crazyflie-well-separation-app}

Well-separation checks, of the kind reported for Coffee, Two-Rivers, and
PyRoboSim (\resec{coffee-results}, \resec{two-rivers-results},
\resec{pyrobosim-results}), are settled for all six quadcopter
configurations of the 20-seed sweep. \textbf{1S}, \textbf{2S}, \textbf{1F},
\textbf{2F}, \textbf{1SP}, and \textbf{2SP} have each been checked across
all 22 declaration orderings (alphabetic, domain-grouped, and 20 random
seeds) and are well-separated in every ordering, with no exceptions. This
is a sharp contrast to Coffee, Two-Rivers, and PyRoboSim, each of which had
at least one NWS configuration; on current evidence, the quadcopter domain
does not exhibit the \texttt{P-all/E-just} or \texttt{P-reach/E-safe}
violations found elsewhere among these six. \textbf{1SP} is well-separated
across all 22 orderings like the other five, not merely a majority of
them, but needs a longer, 21{,}600\,s (6\,h) well-separation timeout than
the other five configurations to resolve every ordering. Check time among
all 22 orderings
varies substantially (0.05--883.6\,s: near-instant for the alphabetic and
domain-grouped orderings, up to nearly 15\,min for the slowest random
seed), unlike the near-instant checks for the enumerated configurations
(\textless\,0.001\,s) or the sub-1\,s checks for the other one-hot
configurations (\textbf{1S}, \textbf{1F}); \textbf{1SP}'s well-separation
check is simply substantially more expensive than the other one-hot
configurations', not less reliable.

The more consequential result in this section is \textbf{1FP}/\textbf{2FP}.
Both configurations' well-separation checks resolve under a 21{,}600\,s
(6\,h) well-separation timeout and report \texttt{NON\_WELL\_SEPARATED},
with a \texttt{P-all/E-just} case implicating an environment-liveness
assumption -- the same violation signature already identified for the
Fair-Outcome-with-pending configurations in Coffee, Two-Rivers, and
PyRoboSim (\resec{coffee-results}, \resec{two-rivers-results},
\resec{pyrobosim-results}). \textbf{1FP}'s check takes 249.3\,s of
analyzer time reported by \texttt{slugs} (2184.6\,s of wrapper wall-clock
time); \textbf{2FP}'s takes 0.036\,s analyzer (72.8\,s wrapper), consistent
with one-hot's well-separation checks being substantially more expensive
than enumerated's throughout this domain (\textbf{1SP}: 0.05--883.6\,s
vs.\ enumerated's \textless\,0.001\,s, above).

The \textbf{2FP} well-separation verdict is especially informative because, on
this machine, \textbf{2FP}'s realizability itself remains unresolved under
the 24\,h synthesis budget reported in
\reappsec{crazyflie-reliability-app}; well-separation is a pre-synthesis
check on the specification alone. With \textbf{1FP} realizable on this
machine and \textbf{2FP} realizable on the faster machine, both are
audit-confirmed as \texttt{GOAL\_UNREACHABLE\_TRAP}
(\reappsec{crazyflie-reliability-app}), this
specification-level signature is no longer merely a prediction about what
an eventual extracted strategy might do for either configuration: it is
confirmed by a concrete extracted strategy exhibiting exactly that
pathology in both cases.

\subsection{Hardware Failure-Handling Validation}
\label{sec:crazyflie-failure-validation-app}

We additionally synthesized both \textbf{1S} and \textbf{2S} directly on
the deployment hardware used for the flight demonstrations described in
\resec{crazyflie-failure-validation}: a Dell Pro Max Tower T2 with an Intel
Core Ultra 9 285K CPU and 64\,GiB memory, a different and substantially
faster machine than the Dell Precision
Tower 5810 used for \retab{crazyflie-results-app} and the other domains'
reported sweeps (\resec{discussion}), so the timing figures here are not
directly comparable to that table's. The same six-column 20-seed sweep was
also rerun on this hardware. Realizability, well-separation, strategy
size, and audit outcome all matched \retab{crazyflie-results-app}; the
random-seed mean realizability times were lower by roughly 2.3--3.2$\times$
(\textbf{1S}: 7.2\,s; \textbf{2S}: 3.4\,s; \textbf{1F}: 260.3\,s;
\textbf{2F}: 11.1\,s; \textbf{1SP}: 1021.5\,s; \textbf{2SP}: 55.0\,s).
Both flight-tested configurations synthesized
end-to-end (preprocessing through FlexBE HFSM realization) in well under
15\,s in the alphabetic-ordering faster-machine rows: 10.6\,s for
\textbf{1S} and 10.5\,s for \textbf{2S}. Those same rows report
\texttt{WELL\_SEPARATED} with no violated assumption, consistent with
\textbf{1S}/\textbf{2S}'s status in
\reappsec{crazyflie-well-separation-app} above. The corresponding reduced
strategies have 83 states for \textbf{1S} and 106 states for \textbf{2S},
while Controller Realization emits 69- and 54-state FlexBE HFSMs,
respectively, for the reasons detailed in
\reappsec{crazyflie-realization-app}. Both
\textbf{1S} and \textbf{2S} were flight-tested on PiHawk
(\resec{crazyflie-failure-validation}); these tests provide qualitative
integration evidence rather than a measured comparison of success rates.

\FloatBarrier

%% file: tables/crazyflie.tex
\begin{tabular}{|c|r|r|r|r|r|r|}
\hline
 & \textbf{1S} & \textbf{2S} & \textbf{1F} & \textbf{2F} & \textbf{1SP} & \textbf{2SP} \\
\hline
\rule{0pt}{4ex}\textbf{\shortstack{\rule{0pt}{2.8ex}Well\\Separated\rule[-1ex]{0pt}{1ex}}} & Yes & Yes & Yes & Yes & Yes & Yes \\
\hline
\textbf{Realizable} & Yes & Yes & Yes & Yes & Yes & Yes \\
\hline
\rule{0pt}{4ex}\textbf{\shortstack{\rule{0pt}{2.8ex}FlexBE HFSM\\Realized\rule[-1ex]{0pt}{1ex}}} & Yes & Yes & Yes & Yes & Yes & Yes \\
\hline
\textbf{$\left|AP_I\right|$} & 84 & 2 & 84 & 2 & 84 & 2 \\
\hline
\textbf{$\left|AP_O\right|$} & 91 & 46 & 91 & 46 & 124 & 79 \\
\hline
\textbf{$\left|\varphi_i^a\right|$} & 84 & 2 & 84 & 2 & 84 & 2 \\
\hline
\textbf{$\left|\varphi_i^g\right|$} & 77 & 27 & 77 & 27 & 110 & 60 \\
\hline
\textbf{$\left|\varphi_s^a\right|$} & 168 & 85 & 168 & 85 & 168 & 85 \\
\hline
\textbf{$\left|\varphi_s^g\right|$} & 1508 & 183 & 1508 & 183 & 1640 & 315 \\
\hline
\textbf{$\left|\varphi_l^a\right|$} & 1 & 1 & 33 & 33 & 1 & 1 \\
\hline
\textbf{$\left|\varphi_l^g\right|$} & 1 & 1 & 1 & 1 & 1 & 1 \\
\hline
\rule{0pt}{4ex}\textbf{\shortstack{\rule{0pt}{2.8ex}CUDD Live\\BDD Nodes\rule[-1ex]{0pt}{1ex}}} & 135724.9 & 117589.4 & 1392251.7 & 180870.9 & 3294766.7 & 873350.8 \\
\hline
 & $\pm$ 54516.6 & $\pm$ 18337.6 & $\pm$ 626130.8 & $\pm$ 46089.9 & $\pm$ 2513696.0 & $\pm$ 154130.7 \\
\hline
\rule{0pt}{4ex}\textbf{\shortstack{\rule{0pt}{2.8ex}CUDD Peak\\BDD Nodes\rule[-1ex]{0pt}{1ex}}} & 854545.3 & 1020364.8 & 7947480.8 & 1035950.3 & 10409427.7 & 3980127.9 \\
\hline
 & $\pm$ 246360.4 & $\pm$ 297257.6 & $\pm$ 3037765.9 & $\pm$ 293300.1 & $\pm$ 7417297.7 & $\pm$ 655188.2 \\
\hline
\textbf{\textbar{}Slugs SM\textbar} & 119.3 & 131 & 102 & 131 & 187.0 & 347 \\
\hline
 & $\pm$ 10.2 &  &  &  & $\pm$ 9.2 &  \\
\hline
\textbf{\textbar{}Reduced\textbar} & 76.7 & 106 & 68 & 106 & 131.0 & 318 \\
\hline
 & $\pm$ 5.1 &  &  &  & $\pm$ 4.6 &  \\
\hline
\rule{0pt}{4ex}\textbf{\shortstack{\rule{0pt}{2.8ex}Generate\\Specs (s)\rule[-1ex]{0pt}{1ex}}} & 0.342 & 1.308 & 0.340 & 1.317 & 0.375 & 1.452 \\
\hline
 & $\pm$ 0.005 & $\pm$ 0.013 & $\pm$ 0.006 & $\pm$ 0.010 & $\pm$ 0.006 & $\pm$ 0.014 \\
\hline
\rule{0pt}{4ex}\textbf{\shortstack{\rule{0pt}{2.8ex}Well Separation\\(s)\rule[-1ex]{0pt}{1ex}}} & 0.686 & 0.000 & 0.711 & 0.000 & 120.779 & 0.000 \\
\hline
 & $\pm$ 0.184 & $\pm$ 0.000 & $\pm$ 0.222 & $\pm$ 0.000 & $\pm$ 228.948 & $\pm$ 0.000 \\
\hline
\rule{0pt}{4ex}\textbf{\shortstack{\rule{0pt}{2.8ex}Realizability\\(s)\rule[-1ex]{0pt}{1ex}}} & 19.730 & 9.275 & 638.479 & 31.814 & 2764.539 & 166.559 \\
\hline
 & $\pm$ 9.490 & $\pm$ 2.643 & $\pm$ 471.598 & $\pm$ 15.260 & $\pm$ 3321.603 & $\pm$ 49.516 \\
\hline
\rule{0pt}{4ex}\textbf{\shortstack{\rule{0pt}{2.8ex}Extraction\\(s)\rule[-1ex]{0pt}{1ex}}} & 1.442 & 0.187 & 6.493 & 0.302 & 14.196 & 1.229 \\
\hline
 & $\pm$ 0.302 & $\pm$ 0.039 & $\pm$ 2.375 & $\pm$ 0.103 & $\pm$ 10.091 & $\pm$ 0.294 \\
\hline
\rule{0pt}{4ex}\textbf{\shortstack{\rule{0pt}{2.8ex}Auditing\\(s)\rule[-1ex]{0pt}{1ex}}} & 0.007 & 0.009 & 0.009 & 0.008 & 0.011 & 0.018 \\
\hline
 & $\pm$ 0.001 & $\pm$ 0.003 & $\pm$ 0.000 & $\pm$ 0.000 & $\pm$ 0.002 & $\pm$ 0.000 \\
\hline
\rule{0pt}{4ex}\textbf{\shortstack{\rule{0pt}{2.8ex}Realize\\HFSM (s)\rule[-1ex]{0pt}{1ex}}} & 3.532 & 4.193 & 4.217 & 4.193 & 5.997 & 7.492 \\
\hline
 & $\pm$ 0.365 & $\pm$ 0.017 & $\pm$ 0.028 & $\pm$ 0.023 & $\pm$ 1.075 & $\pm$ 0.023 \\
\hline
\rule{0pt}{4ex}\textbf{\shortstack{\rule{0pt}{2.8ex}Overall\\(s)\rule[-1ex]{0pt}{1ex}}} & 40.184 & 23.171 & 666.082 & 45.867 & 4228.671 & 321.273 \\
\hline
 & $\pm$ 10.864 & $\pm$ 2.866 & $\pm$ 474.724 & $\pm$ 15.613 & $\pm$ 4124.537 & $\pm$ 48.982 \\
\hline
\end{tabular}

%% file: Acknowledgment.tex

The authors are responsible for the overall design, concept development,
and testing of all code in the project, and are
responsible for all text, images, and references.
During development, generative AI tools (ChatGPT, Codex, and Claude) were
used to assist with prototyping scripts, code refactoring, debugging, and
drafting, revising, and copyediting text.
No AI tool determined the original scientific content, analysis, or
conclusions, is credited as an author, or bears responsibility for the
work. The authors reviewed and take full responsibility for the accuracy
and integrity of the manuscript and associated code.